%% file: main.tex
\documentclass[11pt,a4paper]{article}
\input{settings.tex}
\allowdisplaybreaks[4]
\appto{\appendix}{%
  \clubpenalty=500
  \widowpenalty=500
  \displaywidowpenalty=500
}

\makeatletter
\renewcommand\paragraph{\@startsection{paragraph}{4}{\z@}%
  {1.6ex plus .2ex minus .1ex}%
  {-1em}%
  {\normalfont\normalsize\bfseries}}
\makeatother

\theoremstyle{definition}
\newtheorem{definition}{Definition}[section]
\theoremstyle{plain}
\newcommand{\PDTA}{\mathtt{PDTA}}
\newcommand{\SPDTA}{\mathtt{SPDTA}}
\newcommand{\RDTA}{\mathtt{RDTA}}
\newcommand{\SGDTA}{\mathtt{SGDTA}}
\newcommand{\ADTA}{\mathtt{ADTA}}
\begin{document}
\MakeMainTitle

\begin{abstract}
The statistical accuracy of neural networks depends on both their
approximation power and the complexity of the class fitted from data.
We establish sharp width--radius and width--bits tradeoffs at fixed depth
using one elementary bounded $1$-Lipschitz Dyadic--Triangular Activation.
For the unit $\beta$-H\"older ball on $[0,1]^d$ ($\beta\in(0,1]$), the optimal $L^p$ ($p\in(0,\infty)$)
approximation error is $\Theta([N^2\log(eNT)]^{-\beta/d})$ for
$N\ge2d+3$ and $T\ge1$,
where $T$ bounds every weight and bias.
Matching lower bounds hold for every fixed globally H\"older activation.
With $B$ binary digits per parameter, the optimal rate is
$\Theta([N^2B]^{-\beta/d})$ at the same depth, for every independent choice of
$N\ge2d+3$ and $B\ge2$. Under bounded design densities and independent centered sub-Gaussian
noise, approximate least squares over the full clipped classes at depth
$19$ attains the classical H\"older minimax risk
$\mathcal{O}(M^{-2\beta/(2\beta+d)})$ without logarithmic loss when
$N^2\log(eNT)=\Theta(M^{d/(2\beta+d)})$ or
$N^2B=\Theta(M^{d/(2\beta+d)})$, respectively.
These balances give a continuum of statistically optimal choices of
width and parameter radius or bit budget, extending previously known
isolated optimal choices to entire optimal curves. We also extend this optimal width--bits rate and this optimal statistic tradeoff to a broad class of elementary activations.
At fixed size, four hidden layers with at most $8d+7$ nonzero parameters
give a near-optimal radius, while six layers with at most $8d+27$
attain the optimal worst-case order $\log T=\Theta(\varepsilon^{-d/\beta})$
at approximation error $\varepsilon$.
The same decoding method also yields fixed-size Transformer approximation.
\end{abstract}


\pdfbookmark[1]{Contents}{contents}
\tableofcontents

\section{Introduction}\label{sec:introduction}
\begingroup
\clubpenalty=500
\widowpenalty=500
\displaywidowpenalty=500

Neural networks provide flexible function classes for nonparametric regression,
where an unknown regression function is estimated from noisy observations.
Their statistical accuracy depends on a balance between approximation and
estimation: the network class must be rich enough to approximate the target,
yet sufficiently controlled to permit accurate estimation from finite samples.
A natural way to improve approximation is to increase network width or depth
as the sample size grows. However, network size is only one resource governing
expressivity. Even within a fixed architecture, the allowed magnitudes and
precision of weights and biases affect the richness of the function class,
with the resulting approximation power depending on the activation function.
This raises a basic question: can larger or more precise parameters compensate
for fewer neurons while preserving optimal statistical accuracy?

Answering this question requires tracking parameter magnitude and precision
in both the approximation error and the statistical complexity of the network
class. We establish sharp tradeoffs between network width and parameter radius,
and between network width and parameter precision, at fixed depth, where the
radius bounds every weight and bias in absolute value and the precision is
measured by the number of binary digits per parameter.
For one explicit elementary, bounded, globally Lipschitz activation,
we characterize the optimal approximation error over H\"older classes as
a joint function of width and radius, and of width and bits.
Applied to nonparametric regression, these approximation laws yield a
continuum of minimax-optimal models, ranging from networks with unit parameter
radius or two bits per parameter to a fixed architecture with increasing
parameter radius or precision. Thus, optimal statistical accuracy can be
attained through different allocations of network width, parameter magnitude,
and parameter precision. The starting point is the approximation--estimation
balance.

Let $\mathcal F$ be any class of measurable functions from $[0,1]^d$ to
$[-1,1]$. We observe $Y_i=f(\bm X_i)+\varepsilon_i$, $i=1,\ldots,M$,
where $f\in\mathcal F$, the covariates are i.i.d.\ with law $\mu$, and
the errors are i.i.d.\ centered $\sigma$-sub-Gaussian variables independent
of the covariates. Let $\mathcal G$ be a nonempty, supremum-norm separable
class of measurable functions with the same domain and range, and let
$\widehat f_M$ be a measurable $M^{-1}$-approximate minimizer of empirical
squared loss over $\mathcal G$. Classical least-squares theory separates its worst-case prediction risk
into approximation and estimation costs
\citep{GyorfiEtAl2002}; see also \citet[Lemma~4]{SchmidtHieber2020}.
Proposition~\ref{prop:oracle} gives the following form in terms of empirical
covering numbers:
\begin{equation}\label{eq:intro-oracle}
\sup_{f\in\mathcal F}\mathbb E\norm{\widehat f_M-f}_{L^2(\mu)}^2
\lesssim
\IntroBrace{\sup_{f\in\mathcal F}\inf_{g\in\mathcal G}\norm{g-f}_{L^2(\mu)}^2}{approximation error}
+\IntroBrace{\frac{\log[e\mathcal N((48M)^{-1},\mathcal G,2M)]}{M}}{estimation error}.
\end{equation}
Here $\mathcal N(\delta,\mathcal G,n)$ is the proper covering number of
the evaluation vectors in the maximum norm, maximized over the $n$ input
locations; its centers are evaluations of functions in $\mathcal G$.
The implicit constant depends only on the noise level, the factor $e$
absorbs the optimization tolerance, and the expectation is under the target $f$.
One can estimate this covering number directly or use the
pseudo-dimension, the VC dimension of the subgraph class.
Discretization and the Sauer bound give \citep{AnthonyBartlett1999}
\begin{equation*}
\frac{\log[e\mathcal N((48M)^{-1},\mathcal G,2M)]}{M}
\lesssim\frac{1+\operatorname{Pdim}(\mathcal G)\log(eM)}{M}.
\end{equation*}
A direct cover retains information about parameter magnitudes.
Pseudo-dimension can sometimes control a class even when its parameters
are unrestricted. These are two ways to bound the same estimation cost.

Take $\mathcal F=\mathcal H^\beta([0,1]^d)$, the unit H\"older ball,
$\beta\in(0,1]$, and suppose that $\mu$ has density at most $\kappa$.
This transfers Lebesgue $L^2$ approximation bounds to prediction error.
The squared-risk benchmark is $M^{-2\beta/(2\beta+d)}$, with a matching
minimax lower bound under uniform design and nondegenerate Gaussian
noise \citep{Stone1982}. For an activation $\varphi$, let
$\mathcal F_\varphi(N,L,T)$ contain all fully connected networks with
at most $L$ hidden layers, width at most $N$, and every weight and bias
bounded in absolute value by $T$, followed by clipping
$\clip(t)=\max\{-1,\min\{1,t\}\}$. Clipping keeps predictions in the target
range without increasing their pointwise error. We fit this entire class.
Write $\mathcal E_\varphi(N,L,T)=\sup_{f\in\mathcal H^\beta}
\inf_{g\in\mathcal F_\varphi(N,L,T)}\norm{g-f}_{L^2([0,1]^d)}$ for its
worst-case approximation error.
Fix $L=L_*$ independently of sample size. For a globally
$\alpha$-H\"older activation, $\alpha\in(0,1]$, parameter discretization
gives $\log\mathcal N(\delta,\mathcal F_\varphi(N,L_*,T),n)
\lesssim N^2\log(eNT/\delta)$ for $T\ge1$ and $\delta\in(0,1]$
(Proposition~\ref{prop:holder-class-cover}). Hence
\begin{equation}\label{eq:intro-radius-risk}
\sup_{f\in\mathcal H^\beta}\mathbb E\norm{\widehat f_M-f}_{L^2(\mu)}^2
\lesssim\mathcal E_\varphi(N,L_*,T)^2+\frac{N^2\log(eMNT)}{M}.
\end{equation}
In \eqref{eq:intro-radius-risk}, the estimation bound contains a factor
$N^2$, while radius enters only logarithmically. Reducing width
is therefore useful if an increase in radius can recover the required
approximation accuracy. This requires a quantitative bound for
$\mathcal E_\varphi(N,L,T)$ with both $N$ and $T$ specified. At fixed
$N,L$, the error must tend to zero as $T\to\infty$ fast enough to offset
the increase in the estimation bound.

ReLU, logistic sigmoid, and tanh belong to the piecewise Pfaffian class
considered in \citet[Theorem~5]{Yarotsky2021}. At fixed architecture, their
oscillations along a line are bounded uniformly over the parameters, giving
a positive worst-case error for our finite-$L^p$ targets
($\mathcal E_\varphi(N,L,\infty)>0$). A periodic branch permits more
oscillations by rescaling its input, without adding neurons. The remaining
task is to recover the target values through a fixed number of layers and
to bound the required parameter magnitudes. Following
fixed-size universal-activation constructions \citep{ZhangShenYang2022},
we combine a triangular wave with a shifted exponential branch to obtain
the Dyadic--Triangular Activation, $\DTA$: one explicit bounded
$1$-Lipschitz function, fixed for all targets, dimensions, and accuracies.
For DTA, $\operatorname{Pdim}(\mathcal F_\DTA(N,L,\infty))=\infty$
for every $N,L\ge1$; infinite VC dimension at fixed architecture
also occurs for superexpressive activations \citep[Section~3]{Yarotsky2021}.
We therefore retain $T$ in the covering-number bound
\eqref{eq:intro-radius-risk} to control the estimation error.

\subsection{Main contributions}

Our main approximation results determine not only the sharp tradeoff
between width and parameter radius, but also the sharp tradeoff between
width and parameter precision. For $\beta\in(0,1]$ and $p\in(0,\infty)$,
take $L_p=13+\lceil3p\beta\rceil$.
Theorems~\ref{thm:joint-approximation}--\ref{thm:near-optimal-parameter}
give, uniformly for $N,B\in\mathbb N$ with $N\ge2d+3$ and $B\ge2$,
and for $T\ge1$,
\begin{equation}\label{eq:intro-joint-law}
\left\{
\begin{aligned}
\sup_{f\in\mathcal H^\beta}
\inf_{g\in\mathcal F_\DTA(N,L_p,T)}
\norm{f-g}_{L^p([0,1]^d)}
&\asymp[N^2\log(eNT)]^{-\beta/d},\\[0.5em]
\sup_{f\in\mathcal H^\beta}
\inf_{g\in\mathcal F_\DTA^{\mathrm{bit}}(N,L_p,B)}
\norm{f-g}_{L^p([0,1]^d)}
&\asymp[N^2B]^{-\beta/d}.
\end{aligned}
\right.
\end{equation}
Both classes have width at most $N$ and hidden depth at most $L_p$,
with outputs clipped to $[-1,1]$. The first bounds every weight and bias
in magnitude by $T$; the second stores each weight and bias in $B$
binary digits. The bounds hold for every independent choice of width
and radius, or width and bit budget, in the stated ranges.
The radius lower bound holds for \emph{every} fixed globally H\"older activation,
and the bits lower bound holds for \emph{every} Borel measurable activation.
DTA, a very elementary bounded $1$-Lipschitz activation, attains both optimal orders.
To provide greater flexibility in the choice of activation for training
and a unified theoretical treatment, Section~\ref{sec:beyond-dta} extends
the optimal width--bits rate (\ref{eq:intro-joint-law}) and the optimal statistic tradeoff (\ref{eq:intro-regression-curve}) below to a broad class of activations through a
finite-precision replacement principle.
The examples include parametric and softsign variants, ReLU-like and
sigmoidal functions, and an explicit real analytic family, all of which
are elementary Lipschitz (see Figure~\ref{fig:diverse-activations} for illustration).
\begin{figure}[htbp]
\centering
\includegraphics[width=\textwidth]{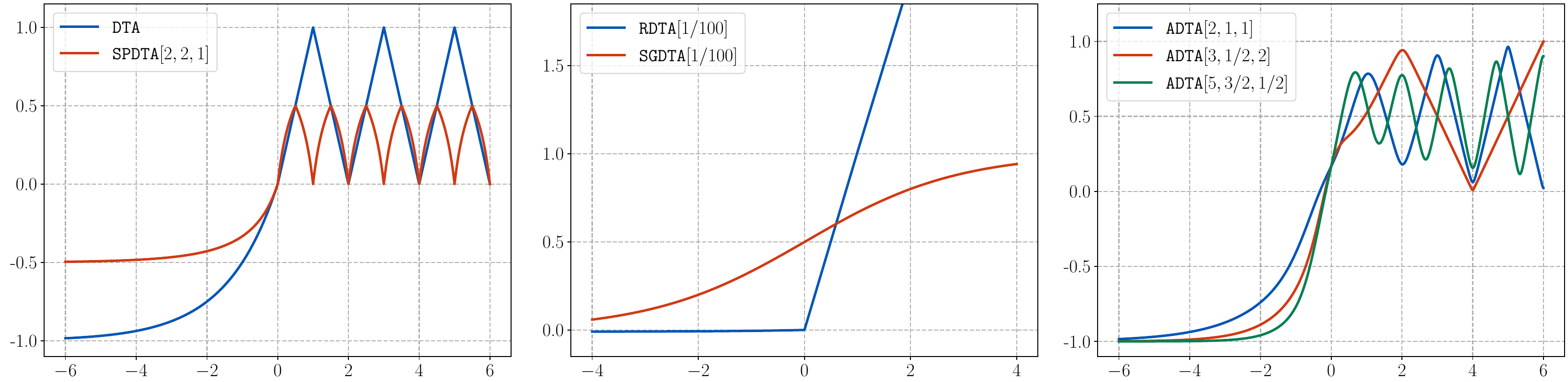}
\caption{Original $\DTA$ and its parametric softsign variant
$\SPDTA[2,2,1]$ (left), ReLU-like and sigmoidal activations
$\RDTA[1/100]$ and $\SGDTA[1/100]$ (middle), and three real analytic
activations $\ADTA$ with different parameter choices (right).
The specific parameters are chosen for illustration; each family allows
the full range of parameters.}
\label{fig:diverse-activations}
\end{figure}

The two laws in (\ref{eq:intro-joint-law}) fully characterize the phenomenon of superexpressivity: at fixed
width and depth, increasing radius or precision drives the worst-case
error to zero. For DTA networks, the two joint laws hold uniformly over all admissible
widths and parameter budgets: regardless of how large $T$ or $B$ becomes or how small the width $N$ can be,
the optimal approximation errors continue to decrease at the respective
rates $\Theta([N^2\log(eNT)]^{-\beta/d})$ and
$\Theta([N^2B]^{-\beta/d})$.
Such a dependence on parameter budgets is \emph{impossible} for classical
activations such as ReLU, sigmoid, and tanh, whose oscillatory
complexity is bounded at fixed architecture, regardless of the
parameter values \citep[Theorem~5]{Yarotsky2021}.
In particular, even with unrestricted real parameters, fixed-depth
ReLU networks admit the lower bound
$\Omega([N^2\log(eN)]^{-\beta/d})$
\citep[Theorem~22 and (5.1)]{Siegel2023}.
The quantization-regime theory of
\citet[Theorem~5.1 and Proposition~2.11]{OuSchenkelBolcskei2024}
further shows that, for one-dimensional Lipschitz targets, the optimal ReLU approximation error saturates at
$\Theta([N^2\log(eN)]^{-1})$ once $B/\log(eN)$ is sufficiently large at fixed sufficiently
large depth.
The fundamental distinction is that DTA exhibits no such regime
dependence at all: network size and parameter choice are independently adjustable approximation resources, governed
by a single optimal law throughout the entire budget range.
From an information-theoretic perspective, an $\mathcal{O}(1)$-depth
dense network of width $N$ stores $\Theta(N^2B)$ total parameter bits.
The width--bits law is thus an optimal
multiplicative resource separation theorem, with the rate
$N^{-2\beta/d}\times B^{-\beta/d}$ for all $N$ and $B$. At the other extreme, the width--bits law tells a different
story in the low-precision regime: even with $B=2$, where
every weight and bias is restricted to $\{-1,0,1\}$, DTA networks
still attain the approximation rate $\mathcal{O}(N^{-2\beta/d})$.

\begin{figure}[htbp]
\centering
\input{figures/phase_pair.tex}
\caption{Schematic width--radius and width--bits tradeoffs at fixed $M$.
The green bands satisfy the two balances in
\eqref{eq:intro-regression-curve}, equating the orders of the
approximation and estimation terms. The other colors mark budgets above
or below these balancing ranges and widths below the common threshold.}
\label{fig:phase-strip}
\end{figure}

Both approximation laws yield optimal regression tradeoffs. Since
$13+\lceil6\beta\rceil\le19$, the common hidden depth $19$ suffices
for $p=2$. Let $\widehat f_M^{\,T}$ and $\widehat f_M^{\,B}$ be
measurable $M^{-1}$-approximate empirical-risk minimizers over
$\mathcal F_\DTA(N,19,T)$ and
$\mathcal F_\DTA^{\mathrm{bit}}(N,19,B)$, respectively.
Combining \eqref{eq:intro-joint-law} with parameter covering bound (\ref{eq:intro-radius-risk}) for the
first class and cardinality bounds for the second gives
\begin{equation*}
\left\{
\begin{aligned}
\sup_{f\in\mathcal H^\beta}
\mathbb E\norm{\widehat f_M^{\,T}-f}_{L^2(\mu)}^2
&\lesssim[N^2\log(eNT)]^{-2\beta/d}
 +\frac{N^2\log(eMNT)}{M},\\[0.35em]
\sup_{f\in\mathcal H^\beta}
\mathbb E\norm{\widehat f_M^{\,B}-f}_{L^2(\mu)}^2
&\lesssim[N^2B]^{-2\beta/d}+\frac{N^2B}{M}.
\end{aligned}
\right.
\end{equation*}
The implicit constants depend only on $d,\beta,\kappa,\sigma$.
The budgets $N,T,B$ are chosen deterministically and may depend on $M$.
For $N,B\in\mathbb N$ with $N\ge2d+3$, $B\ge2$, and $T\ge1$,
Theorem~\ref{thm:generalization} and Remark~\ref{rem:bits-regression} give
(as can also be seen by substituting the respective conditions into the two inequalities above)
\begin{equation}\label{eq:intro-regression-curve}
\left\{
\begin{alignedat}{2}
&N^2\log(eNT)\asymp M^{\frac{d}{2\beta+d}}
&\qquad\Longrightarrow\quad
&\sup_{f\in\mathcal H^\beta}
\mathbb E\norm{\widehat f_M^{\,T}-f}_{L^2(\mu)}^2
\lesssim M^{-\frac{2\beta}{2\beta+d}},\\[0.35em]
&N^2B\asymp M^{\frac{d}{2\beta+d}}
&\qquad\Longrightarrow\quad
&\sup_{f\in\mathcal H^\beta}
\mathbb E\norm{\widehat f_M^{\,B}-f}_{L^2(\mu)}^2
\lesssim M^{-\frac{2\beta}{2\beta+d}}.
\end{alignedat}
\right.
\end{equation}
These laws extend an optimal choice of network size to a full tradeoff
between two adjustable resources. Each curve in Figure~\ref{fig:phase-strip}
gives a family of allocations with the same minimax risk, from
increasing width alone at the right endpoint to a fixed architecture at
the left endpoint.
\begin{itemize}
\item \textbf{Width--radius.}
At the right endpoint, $T=1$ and
$N\asymp M^{d/(4\beta+2d)}/\sqrt{\log M}$.
At the left endpoint, $N=2d+3$ and
$\log T\asymp M^{d/(2\beta+d)}$.
\item \textbf{Width--bits.}
At the right endpoint, $B=2$ and $N\asymp M^{d/(4\beta+2d)}$.
At the left endpoint, $N=2d+3$ and $B\asymp M^{d/(2\beta+d)}$.
\end{itemize}

At fixed width, \eqref{eq:intro-joint-law} requires
$\log T\asymp\varepsilon^{-d/\beta}$ for worst-case error of order $\varepsilon$.
Theorems~\ref{thm:approximation} and~\ref{thm:optimal-approximation}
give small explicit architectures: four hidden layers with at most
$8d+7$ nonzero parameters attain
$\log T\lesssim\varepsilon^{-d/\beta}\log(\varepsilon^{-1})$, and six hidden layers
with at most $8d+27$ nonzero parameters attain
$\log T\lesssim\varepsilon^{-d/\beta}$, matching the fixed-size lower bound
for globally H\"older activations. The improvement uses a coupled decoder
that stores each block through its initial value and small increments.
Distributing the codes among more parameters gives the width dependence,
and a separate construction handles unit radius for every bit budget.
It follows the same sequence of grid encoding and increment decoding:
one binary table records the row starts, while short codes record the
increments between them.
The same method yields fixed-size Transformer approximation for
matrix-valued H\"older maps of fixed sequence length $n$: one attention
head achieves $L^p$ error $\varepsilon$ with $\log T\lesssim\varepsilon^{-dn/\beta}$
(Theorem~\ref{thm:transformer}).

\subsection{Related work and comparison}\label{sec:related-work}

Table~\ref{tab:intro-comparison} compares the closest approximation and
regression bounds. We discuss their assumptions and their relation to
fixed-size universality below.
\paragraph{Approximation theory with classical activations.}
Classical theory measures accuracy by width $N$ and hidden depth $L$, with
unrestricted parameters ($T=\infty$) unless specified. Rates concern unit
balls on $[0,1]^d$ and sufficiently large widths or depths.
Growing width gives density \citep{Cybenko1989,Hornik1991,LeshnoEtAl1993};
for functions with bounded first Fourier moment, \citet{Barron1993}
obtains $L^2$ error $\mathcal O(N^{-1/2})$ with one
sigmoidal hidden layer. For ReLU, work on smooth
\citep{Yarotsky2017} and piecewise smooth functions
\citep{PetersenVoigtlaender2018} revealed how depth improves accuracy. On H\"older balls with $\beta\in(0,1]$, width $2d+10$
suffices for uniform error $\mathcal O(L^{-2\beta/d})$
\citep{Yarotsky2018,shijun:Characterized:by:Numer:Neurons}; varying width as well yields
$\mathcal O([N^2L^2\log(eN)]^{-\beta/d})$
\citep{ShenYangZhang2022ReLU}. Depth and target-to-parameter continuity organize these rates into a phase
diagram \citep{YarotskyZhevnerchuk2020}.
Greater smoothness improves the exponent: \citet{LuShenYangZhang2021}
use $s$ continuous derivatives to obtain uniform error
$\mathcal O([NL/(\log(eN)\log(eL))]^{-2s/d})$ on $C^s$ balls.
Sobolev $W^{s,q}$ and Besov $B^s_{q,r}$ spaces allow less uniform
regularity. On their unit balls, the $L^p$ rate
$\mathcal O(L^{-2s/d})$ at width $25d+31$ \citep{Siegel2023}
extends to $\mathcal O((NL)^{-2s/d})$ when both resources vary
\citep{YunfeiYang2025Sobolev}, provided $s>0$,
$p,q,r\in[1,\infty]$, and $s/d>1/q-1/p$.
Derivative approximation follows a similar pattern: mixed
ReLU--$\mathrm{ReLU}^2$ networks attain $W^{2,p}$ error
$\mathcal O([NL/(\log(eN)\log(eL))]^{-2(s-2)/d})$ on $W^{s,p}$ balls
for $s\in\{3,4,\ldots\}$ and $p\in[1,\infty]$
\citep{YangWuYangXiang2023}.

Other activations inherit the function-value rates:
``Beyond ReLU'' \citep{ZhangLuZhao2024} transfers ReLU upper
bounds to sigmoid and tanh with width--depth pair $(3N,2L)$, and to
Softplus, GELU, and SiLU with $(N,L)$, without controlling parameter size.
For ReLU itself, the fixed-depth $L^2$ H\"older benchmark is sharp
\citep{ShenYangZhang2022ReLU,Siegel2023}:
\begin{equation}\label{relu benchmark}
\mathcal E_{\mathrm{ReLU}}(N,29,\infty)
\asymp\left[N^2\log(eN)\right]^{-\beta/d},\qquad N\ge 48d.
\end{equation}
Thus arbitrarily large parameters cannot replace the need for growing width.

\paragraph{Near-minimax and minimax regression.}
Sparse ReLU estimators achieve risk
$\mathcal O(M^{-2\beta/(2\beta+d)}\log^3M)$ with
$N\asymp M^{d/(2\beta+d)}$, $L\asymp\log M$, and $T=1$
\citep{SchmidtHieber2020,SchmidtHieberVu2024}.
Related results cover smooth hierarchical sigmoidal models
\citep{BauerKohler2019}, sparse ReLU on Besov classes \citep{Suzuki2019}, and smooth compositional ReLU models without sparsity constraints
\citep{KohlerLanger2021}; all retain logarithmic losses.
For fixed hidden depth $L_*$, the pseudo-dimension bound
$\operatorname{Pdim}(\mathcal F_{\mathrm{ReLU}}(N,L_*,\infty))
\lesssim N^2\log(eN)$ \citep{BartlettEtAl2019}, combined with
\eqref{eq:intro-oracle}--\eqref{eq:intro-radius-risk}, yields
\begin{equation*}
\sup_{f\in\mathcal H^\beta}\mathbb E\norm{\widehat f_M-f}_{L^2(\mu)}^2
\lesssim\mathcal E_{\mathrm{ReLU}}(N,L_*,T)^2
+\frac{N^2}{M}\min\{\log(eMNT),\log(eN)\log(eM)\}.
\end{equation*}
Using the depth-$29$ benchmark (\ref{relu benchmark}) and balancing the two
terms gives the near-minimax order $\mathcal O((\log M/M)^{2\beta/(2\beta+d)})$.
Alternatively, \citet{FanGu2024} give fixed-depth ReLU error
$\mathcal O(N^{-2\beta/d})$ in $L^\infty$, and
\citet[Theorem~4.1]{LiuWangWuZhang2026} give the same error order in
$L^2$ for sigmoid, SiLU, and GELU. Both constructions use radii
polynomial in $N$.
Equation~\eqref{eq:intro-radius-risk} then gives the same near-minimax order
for the full clipped classes. These sufficient radii leave open the best
approximation at a prescribed $(N,T)$
\citep[Remark~4.3]{LiuWangWuZhang2026}.

\begin{table}[htbp]
\centering
\caption{Width--radius approximation and regression for $\beta\in(0,1]$.
Errors are unsquared; risks are squared. Resource choices are sufficient,
with suitable constants and integer rounding; $\operatorname{poly}(N)$ denotes a
fixed-degree polynomial, and $C_0,L_0$ are absolute. The first four statistical rows
combine the cited approximation results with the oracle bounds; the fixed-depth
Ou--B\"olcskei choice uses their Lemma~3.4.
}
\label{tab:intro-comparison}
\begingroup
\fontsize{8}{10}\selectfont
\resizebox{\linewidth}{!}{%
\begin{tblr}{
 width=572pt,
 colspec={Q[wd=64.3pt]Q[wd=48.6pt]Q[wd=50.3pt]Q[wd=65pt]Q[wd=84pt]Q[wd=175.5pt]X},
 cells={halign=c,valign=m},
 colsep=1.6pt, rowsep=2.5pt,
 rows={ht=31pt}, row{1}={ht=25pt,font=\bfseries,bg=black!6},
 row{6}={bg=black!5}, row{7,8}={bg=black!2},
 hlines={.35pt}, vlines={.35pt}, hline{1,9}={.6pt}, hline{6}={.7pt}
}
\textbf{Reference}&\textbf{Activation}&{{\textbf{Hidden} \\\textbf{layers}}}&{{\textbf{Resources}\\$\left(N,T\right)$}}&{{\textbf{Approximation} \textbf{error}}}&{{\textbf{Statistical choice}\\$\left(N,T\right)$}}&{{\textbf{Risk}\\$\mathcal O(\cdot)$}}\\
{{\citet{FanGu2024}}}&ReLU&$12+2d$
&{{$\left(N,\operatorname{poly}\left(N\right)\right)$\\[3pt]$N\ge 34d\:3^d$}}
&$\mathcal O\left(N^{-2\beta/d}\right)$
&$\left(\,\Theta\left(\left(\tfrac{M}{\log M}\right)^{\frac{d}{2(2\beta+d)}}\right),\;\Theta\left(M^{\frac{2d}{2\beta+d}}\right)\,\right)$
&$\left(\tfrac{\log M}{M}\right)^{\frac{2\beta}{2\beta+d}}$\\
{{\citet{LiuWangWuZhang2026}}}&{{sigmoid\\SiLU\\GELU}}&$5$
&{{$\left(N,\operatorname{poly}\left(N\right)\right)$\\[3pt]$N\ge C$}}
&$\mathcal O\left(N^{-2\beta/d}\right)$
&$\left(\,\Theta\left(\left(\tfrac{M}{\log M}\right)^{\frac{d}{2(2\beta+d)}}\right),\;\Theta\left(M^{\frac{\max\{d,5\}}{2}}\right)\,\right)$
&$\left(\tfrac{\log M}{M}\right)^{\frac{2\beta}{2\beta+d}}$\\
{{\citet{ShenYangZhang2022ReLU}}}&ReLU&$29$
&{{$\left(N,\infty\right)$\\[3pt]$N\ge 48d$}}
&$\Theta\left([N^2\log(eN)]^{-\frac{\beta}{d}}\right)$
&$\left(\,\Theta\left(\frac{M^{\frac{d}{2(2\beta+d)}}}{\left(\log M\right)^{\frac{\beta+d}{2\beta+d}}}\right),\;\infty\,\right)$
&$\left(\tfrac{\log M}{M}\right)^{\frac{2\beta}{2\beta+d}}$\\
{{\citet{OuBolcskei2024}\\$d=\beta=1$}}&ReLU&$L_0$
&{{$\left(N,1\right)$\\[3pt]$N\ge C_0$}}
&$\Theta\left([N^2\log(eN)]^{-1}\right)$
&$\left(\,\Theta\left(\frac{M^{1/6}}{\sqrt{\log M}}\right),\;1\,\right)$
&$M^{-2/3}$\\
{{\textbf{This paper}\\joint law}}&$\DTA$&$19$
&{{$\left(N,T\right)$\\[3pt]$N\ge2d+3$\\$T\ge1$}}
&$\Theta\left([N^2\log(eNT)]^{-\frac{\beta}{d}}\right)$
&$N^2\log(eNT)\asymp M^{\frac{d}{2\beta+d}}$
&$M^{-\frac{2\beta}{2\beta+d}}$\\
{{This paper\\unit radius}}&$\DTA$&$19$
&{{$\left(N,1\right)$\\[3pt]$N\ge2d+3$}}
&$\Theta\left([N^2\log(eN)]^{-\frac{\beta}{d}}\right)$
&$\left(\,\Theta\left(\frac{M^{\frac{d}{2(2\beta+d)}}}{\sqrt{\log M}}\right),\;1\,\right)$
&$M^{-\frac{2\beta}{2\beta+d}}$\\
{{This paper\\fixed width}}&$\DTA$&$6$
&{{$\left(2d+3,T\right)$\\[3pt]$T\ge1$}}
&$\Theta\left([\log(eT)]^{-\frac{\beta}{d}}\right)$
&$\left(\,2d+3,\;\exp\left(\Theta\left(M^{\frac{d}{2\beta+d}}\right)\right)\,\right)$
&$M^{-\frac{2\beta}{2\beta+d}}$\\
\end{tblr}%
}
\endgroup
\end{table}

For rates without logarithmic loss, \citet{LiuBoukaiShang2022} obtain
$\mathcal O_P(M^{-2\beta/(2\beta+d)})$ for the conditional risk of a
spline-based ReLU estimator, with $N\asymp M^{d/(2\beta+d)}$ and
$L\asymp\log M$.
For $d=\beta=1$, \citet{OuBolcskei2024} obtain the minimax expected-risk
rate $\mathcal O(M^{-2/3})$ under Gaussian noise with $T=1$; their bounds
also yield the fixed-depth choice in Table~\ref{tab:intro-comparison}.
Our result attains expected risk $\mathcal O(M^{-2\beta/(2\beta+d)})$
for least squares over every full clipped class satisfying
$N^2\log(eNT)\asymp M^{d/(2\beta+d)}$.
This includes growing width at unit radius and growing radius at fixed
architecture, as shown in Table~\ref{tab:intro-comparison}.
For the finite-bit classes, the same risk order holds along
$N^2B\asymp M^{d/(2\beta+d)}$
(Remark~\ref{rem:bits-regression}).

\paragraph{Fixed-size approximation theory.}
The classical size obstruction motivates a different route: choose an
activation that makes a fixed architecture universal. The Kolmogorov--Arnold theorem
\citep{Arnold1957,Kolmogorov1957} provides a starting point: continuous
multivariate functions admit representations by a fixed number of
univariate functions, with fixed inner functions and target-dependent
outer functions. \citet{MaiorovPinkus1999} obtained fixed-size universality
with one specially constructed sigmoid, followed by algorithmic
constructions \citep{GuliyevIsmailov2018} and elementary superexpressive
activations \citep{Yarotsky2021}. A change of activation also accelerates convergence: floor--exponential--step networks
achieve uniform H\"older error $\mathcal O(2^{-\beta N})$ at depth three
and width $N\ge d$ \citep{ShenYangZhang2021FLES,doi:10.1137/21M144431X}.
The EUAF construction \citep{ZhangShenYang2022} achieves uniform density
at fixed size with one bounded $1$-Lipschitz activation.
Subsequent work considers parameter reuse and derivative approximation:
subnetwork reuse reduces the number of unique parameters to
$\mathcal O(d)$ \citep{MaitiMichelleYang2024}, while smooth
$\mathrm{DUAF}_\infty$ gives density in $W^{s-1,\infty}$ for
$W^{s,\infty}$ targets with $s\in\mathbb N$
\citep{LiYangZhang2026Sobolev}. PEUAF introduces a trainable
triangular-wave frequency and performs competitively on industrial
fault-diagnosis benchmarks \citep{WangEtAl2025PEUAF}.
For statistical applications, fixed-size approximation also requires a
bound on the parameter magnitudes. For uniform H\"older error $\varepsilon$, \citet{Beknazaryan2022} obtains
$\log T=\mathcal O(\varepsilon^{-d/\beta}\log(\varepsilon^{-1}))$ using floor and an
abstract discontinuous selector; \citet{FanLiWangWang2026} obtain
$\log T=\mathcal O(\varepsilon^{-2d/\beta}\log(\varepsilon^{-1}))$ using floor,
ReLU, $\mathrm{ReLU}^2$, and a modified reciprocal.
For finite $L^p$, four DTA layers match the former radius order with a
single bounded $1$-Lipschitz activation; six remove the logarithm and attain
$\log T=\mathcal O(\varepsilon^{-d/\beta})$.
Our radius lower bound holds for every fixed globally H\"older activation
at fixed architecture and every $p\in(0,\infty)$.
The joint theorem gives the optimal dependence on both $N$ and $T$,
including unit radius and fixed width.

\FloatBarrier
\par\endgroup

\section{Setting and main results}\label{sec:setting}

We first define the activation, network classes, parameter radius, and
bit budget. The approximation results quantify the parameter magnitude
needed by a fixed architecture and the error attainable when width and
radius or bit budget are prescribed independently. We then give a
fixed-size Transformer construction and apply both joint approximation
laws to nonparametric regression. The activation is fixed throughout,
and depth is independent of target accuracy and sample size.

We use $\mathbb N=\{1,2,\ldots\}$ and
$\mathbb N_0=\{0,1,2,\ldots\}$, and $\log$ denotes the natural logarithm.
Sequence, grid, vector, and table indices are integers.
The symbols $\bmone_m,\bzero_m\in\mathbb R^m$ and
$\bmone_{m\times n},\bzero_{m\times n}$ denote the all-ones and zero vectors
and matrices, $\bm I_m$ is the identity matrix, and $\one_{\{E\}}$ is an
indicator. We write
$\norm{\bm A}_{\max}=\max_{i,j}|A_{ij}|$ for the largest absolute matrix
entry. When coordinates are indexed by $0,\ldots,m-1$, their standard
basis is denoted by $\bm e_0,\ldots,\bm e_{m-1}$.

\subsection{The Dyadic--Triangular Activation}
Motivated by the EUAF construction of \citet{ZhangShenYang2022}, we combine
two elementary functions in one activation.  Related EUAF-type activations
have shown empirical utility in signal-based applications and as complements
to standard activations \citep{WangEtAl2025PEUAF}, providing practical
motivation for studying such nonstandard activations beyond their theoretical properties.

\begin{figure}[htbp]
\centering
\includegraphics[width=\FigureActivationWidth]{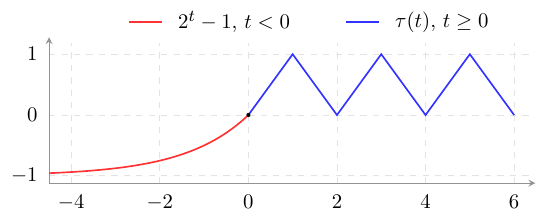}
\caption{The Dyadic--Triangular Activation: a shifted exponential on the
negative half-line and the triangular wave on the nonnegative half-line.}
\label{fig:dta-function}
\end{figure}

Define the triangular wave $\tau$ and the Dyadic--Triangular Activation
$\DTA$ by
\begin{equation}\label{eq:dta-definition}
\tau(t)=\left|t-2\left\lfloor\frac{t+1}{2}\right\rfloor\right|,
\quad t\in\R,
\qquad
\DTA(t)=
\begin{cases}
2^t-1,&t<0,\\
\tau(t),&t\ge0.
\end{cases}
\end{equation}
The function $\tau$ is $2$-periodic, takes values in $[0,1]$, and is
$1$-Lipschitz. It can also be written as
$\tau(t)=\frac{2}{\pi}|\arcsin(\sin(\pi t/2))|$, so the floor in
\eqref{eq:dta-definition} does not introduce a discontinuity.
The branches of $\DTA$ agree at the origin, with $\DTA(0)=0$.
On the negative half-line its derivative is $(\log2)2^t<1$, and on the
nonnegative half-line it is $1$-Lipschitz. For $u<0\le v$, the inequalities
$\tau(v)\le v$ and $1-2^u\le-u$ give
\[
0\le\DTA(v)-\DTA(u)=\tau(v)+1-2^u\le v-u.
\]
Hence $\DTA$ is globally $1$-Lipschitz and satisfies $-1<\DTA\le1$.
The constructions repeatedly use the identities
\begin{equation}\label{eq:dta-identities}
\DTA(-u)=2^{-u}-1,
\qquad
\DTA(u)=\tau(u),
\qquad u\ge0,
\end{equation}
together with $\DTA(z)=z$ for $z\in[0,1]$. The latter identity allows
coordinates in $[0,1]$ to pass through an activated layer unchanged.
Figure~\ref{fig:dta-function} shows the two branches.

\subsection{Network convention}

Let $\sigma:\R\to\R$ act coordinatewise. For input dimension $d$,
output dimension $n$, width bound $N$, hidden-depth bound $L$, and
parameter-radius bound $T\in(0,\infty]$, denote the class of fully
connected networks by $\NN_\sigma(d,n;N,L,T)$. Width is the largest
hidden-layer dimension, and depth counts hidden layers only; the final
affine map is included in the parameter bounds but not in the depth.
More precisely, a map $\phi:\R^d\to\R^n$ belongs to this class if it
admits a realization with an integer $\ell$ satisfying $1\le\ell\le L$,
dimensions $n_0=d$, $n_{\ell+1}=n$, positive integers
$n_1,\ldots,n_\ell\le N$, and affine maps $\mathcal L_j(\bm z)=\bm W_j\bm z+\bm b_j$ such that
\begin{equation}\label{eq:network-realization}
\phi=\mathcal L_{\ell+1}\circ\sigma\circ\mathcal L_\ell
\circ\cdots\circ\sigma\circ\mathcal L_1,
\end{equation}
where $\bm W_j\in\R^{n_j\times n_{j-1}}$ and
$\bm b_j\in\R^{n_j}$ for $j=1,\ldots,\ell+1$, and the parameter radius
\[
\max_{1\le j\le\ell+1}\parnorm{\mathcal L_j},
\qquad
\parnorm{\mathcal L_j}:=\max\{\norm{\bm W_j}_{\max},\norm{\bm b_j}_\infty\}
\]
is at most $T$. The choice $T=\infty$ leaves the weights and biases
unrestricted. When an exact architecture is relevant, we list its hidden
widths as $\langle n_1,\ldots,n_\ell\rangle$. The number of nonzero parameters counts
all nonzero scalar entries of the weight matrices and bias vectors,
including the output map. Radius and parameter counts refer to the chosen
realization; membership in the class requires the existence of one
realization satisfying the stated bounds.

For a bit constraint, we use the convention of fixed-point alphabet from
\citet[Section~1.1]{OuSchenkelBolcskei2024}. For nonnegative integer
indices $a$ and $b$, define
\begin{equation*}
\mathbb Q_b^a
:=(-2^{a+1},2^{a+1})\cap2^{-b}\mathbb Z
=\left\{\pm\sum_{j=-b}^{a}\theta_j2^j:
\theta_j\in\{0,1\}\right\},
\qquad a,b\in\mathbb N_0.
\end{equation*}
This format has one sign bit, $a+1$ bits before the binary point, and
$b$ bits after it, for a total of $a+b+2$ bits per parameter. Its spacing
is $2^{-b}$, its values have magnitude strictly less than $2^{a+1}$, and
$\#\mathbb Q_b^a=2^{a+b+2}-1$. Denote by
$\NN_\sigma(d,n;N,L,\mathbb Q_b^a)$ the networks with a realization
\eqref{eq:network-realization} in which every weight and bias, including
the output parameters, belongs to $\mathbb Q_b^a$. More generally,
$\NN_\sigma(d,n;N,L,\Omega)$ imposes membership in a specified parameter
set $\Omega\subseteq\R$ on every weight and bias. In particular,
$\Omega=\mathbb Q_{B-2}^{0}\cap[-1,1]
=2^{-(B-2)}\mathbb Z\cap[-1,1]$ imposes both unit radius and
$B$ bits per parameter. For a bit
budget $B\in\mathbb N$ with $B\ge2$, define
\begin{equation}\label{eq:bit-network-class}
\NN_\sigma^{\mathrm{bit}}(d,n;N,L,B)
:=\bigcup_{a+b+2=B}\NN_\sigma(d,n;N,L,\mathbb Q_b^a)
=\bigcup_{a=0}^{B-2}\NN_\sigma(d,n;N,L,\mathbb Q_{B-a-2}^a).
\end{equation}
Both unions range over nonnegative integer precision indices. Every
parameter of a realization uses the same pair $(a,b)$ with $a+b+2=B$;
the allocation may vary between realizations. Thus each weight and bias is stored in $B$ binary digits, allowing
leading or trailing zeros. At $B=2$ the only format is
$\mathbb Q_0^0=\{-1,0,1\}$. This budget specifies parameter storage;
affine operations and activation evaluations retain their real-valued
meaning in \eqref{eq:network-realization}.

\subsection{Target, hypothesis classes, and approximation theorems}

For $\beta\in(0,1]$, let
\[
\mathcal H^\beta([0,1]^d)
:=\left\{f:[0,1]^d\to\R:
\norm{f}_\infty\le1,\
|f(\bm x)-f(\bm y)|\le\norm{\bm x-\bm y}_\infty^\beta
\right\}
\]
be the unit $\beta$-H\"older ball, with the inequality imposed for every
$\bm x,\bm y\in[0,1]^d$. We abbreviate this class as $\mathcal H^\beta$
when the domain is clear. Approximation is measured by
$\norm{g}_{L^p}=(\int_{[0,1]^d}|g(\bm x)|^p\,\dd\bm x)^{1/p}$ for
$p\in(0,\infty)$; this is a quasi-norm when $p<1$. Unless another domain or
measure is specified, all function norms below are taken on $[0,1]^d$
with respect to Lebesgue measure.

The targets take values in $[-1,1]$, so we use
$\clip(t)=\max\{-1,\min\{t,1\}\}$ to restrict network outputs to this
interval. For $N,L\in\mathbb N$ and $T\in[1,\infty]$, define
\begin{equation*}
\mathcal F_\sigma(N,L,T)
:=\left\{\left.(\clip\circ\phi)\right|_{[0,1]^d}:
\phi\in\NN_\sigma(d,1;N,L,T)\right\}.
\end{equation*}
For a parameter set $A\subseteq\R$, the notation
$\mathcal F_\sigma(N,L,A)$ has the same meaning, with
$\phi\in\NN_\sigma(d,1;N,L,A)$. In particular, for a fixed-point
alphabet and a bit budget, define
\begin{equation*}
\mathcal F_\sigma(N,L,\mathbb Q_b^a)
:=\left\{\left.(\clip\circ\phi)\right|_{[0,1]^d}:
\phi\in\NN_\sigma(d,1;N,L,\mathbb Q_b^a)\right\},
\end{equation*}
\begin{equation*}
\mathcal F_\sigma^{\mathrm{bit}}(N,L,B)
:=\bigcup_{a+b+2=B}\mathcal F_\sigma(N,L,\mathbb Q_b^a)
=\bigcup_{a=0}^{B-2}\mathcal F_\sigma(N,L,\mathbb Q_{B-a-2}^a).
\end{equation*}
The input dimension $d$ is fixed and omitted from this notation. In
particular, $\mathcal F_\DTA$ denotes these classes with activation
$\sigma=\DTA$. The width, depth, radius, and bit budget constrain the base
network $\phi$; clipping is applied externally. Since
$|\clip(t)-y|\le|t-y|$ for $y\in[-1,1]$, this operation cannot increase
the approximation error for our targets. It can also be realized exactly
by a fixed $\DTA$ network: the identity
$\DTA(1+z)=1-|z|$ for $|z|\le1$ gives
\begin{equation*}
\begin{aligned}
&3\DTA\left(
 3^{-1}\left(\DTA(\tfrac{t+1}{2})+\DTA(\tfrac{t-1}{2})+2\right)
 \right)
-\DTA\left(1+\DTA(\tfrac{t+1}{2})\right)
-\DTA\left(1+\DTA(\tfrac{t-1}{2})\right)-1\\
&\qquad=
\DTA(\tfrac{t+1}{2})+\DTA(\tfrac{t-1}{2})+1
-\DTA\left(1+\DTA(\tfrac{t+1}{2})\right)
-\DTA\left(1+\DTA(\tfrac{t-1}{2})\right)\\
&\qquad=
\DTA(\tfrac{t+1}{2})+\DTA(\tfrac{t-1}{2})
+\left|\DTA(\tfrac{t+1}{2})\right|
+\left|\DTA(\tfrac{t-1}{2})\right|-1\\
&\qquad=
2\max\left\{\DTA(\tfrac{t+1}{2}),0\right\}
+2\max\left\{\DTA(\tfrac{t-1}{2}),0\right\}-1\\
&\qquad=\clip(t).
\end{aligned}
\end{equation*}
The first equality uses
$\DTA((t+1)/2)+\DTA((t-1)/2)+2\in(0,3]$, and the final identity follows
by considering $t\in(-\infty,-1]$, $t\in[-1,1]$, and $t\in[1,\infty)$.
Thus clipping is realizable with hidden widths $\langle 2,3\rangle$ and parameter
radius at most three. In the classes above, clipping is external, so the
stated architecture and parameter budgets apply to the approximating
network.

The first two theorems hold the architecture fixed and quantify the
parameter radius needed for a prescribed error. The subsequent results
allow width to vary and determine the optimal joint rates under either a
radius bound or a bit budget. The lower bounds concern the corresponding
full classes of networks.

\begin{theorem}[Four-hidden-layer approximation]\label{thm:approximation}
Let $d\in\mathbb N$, $\beta\in(0,1]$, and $p\in(0,\infty)$.  For every
$f\in\mathcal H^\beta([0,1]^d)$ and $\varepsilon\in(0,1/2]$, there exists a
$\DTA$ network $\phi$ with hidden-layer widths $\langle 2d+1,d+1,1,1\rangle$ and at
most $8d+7$ nonzero parameters such that
\[
\norm{f-\phi}_{L^p}\le\varepsilon,
\qquad
\norm{\phi}_{L^\infty}\le1,
\qquad
\log T_\varepsilon\le C_0\varepsilon^{-d/\beta}\log(\varepsilon^{-1}),
\]
where $T_\varepsilon$ denotes the parameter radius of $\phi$, and $C_0$ depends only
on $d$, $\beta$, and $p$.
\end{theorem}

The hidden widths and the bound on the number of nonzero parameters are
independent of $f$ and $\varepsilon$. The construction therefore improves
accuracy by increasing parameter magnitudes within one fixed
architecture. Section~\ref{sec:proof-first} gives the construction and
its radius estimate. The next theorem removes the factor
$\log(\varepsilon^{-1})$ in that estimate by using two additional hidden layers.

\begin{theorem}[Optimal parameter-radius approximation]
\label{thm:optimal-approximation}
Let $d\in\mathbb N$, $\beta\in(0,1]$, and $p\in(0,\infty)$.  For every
$f\in\mathcal H^\beta([0,1]^d)$ and $\varepsilon\in(0,1/2]$, there exists a
$\DTA$ network $\phi^\star$ of depth $6$, with hidden-layer widths
$\langle 2d+3,d+2,2,2,2,2\rangle$ and at most $8d+27$ nonzero parameters such that
\[
\norm{f-\phi^\star}_{L^p}\le\varepsilon,
\qquad
\norm{\phi^\star}_{L^\infty}\le 5/4,
\qquad
\log T_\varepsilon^\star\le C_1\varepsilon^{-d/\beta},
\]
where $T_\varepsilon^\star$ denotes the parameter radius of $\phi^\star$, and $C_1$ depends only
on $d$, $\beta$, and $p$.
\end{theorem}

Theorem~\ref{thm:near-optimal-parameter} will show that a fixed
architecture with any fixed globally H\"older activation requires
$\log T\gtrsim\varepsilon^{-d/\beta}$ to approximate the whole unit H\"older
ball. Hence Theorem~\ref{thm:optimal-approximation} has the optimal
radius order at fixed size. Its construction uses two more hidden layers
than Theorem~\ref{thm:approximation}, and the bound on nonzero parameters
increases from $8d+7$ to $8d+27$; the proof is in
Section~\ref{sec:proof-optimal}. We next allow width to vary and quantify
how it reduces the required parameter radius, with a bound that is
uniform in both $N$ and $T$.

\begin{theorem}[Joint width--radius approximation]\label{thm:joint-approximation}
Let $d\in\mathbb N$, $\beta\in(0,1]$, and $p\in(0,\infty)$.  There is a constant
$C_2=C_2(d,\beta,p)>0$ such that, for every integer $N\ge2d+3$ and every $T\ge1$,
\[
\sup_{f\in\mathcal H^\beta}
\inf_{g\in\mathcal F_\DTA(N,13+\lceil3p\beta\rceil,T)}
\norm{f-g}_{L^p}
\le C_2\bigl[N^2\log(eNT)\bigr]^{-\beta/d}.
\]
\end{theorem}

The constant $C_2$ is independent of both $N$ and $T$, so the bound applies
to arbitrary simultaneous choices of width and radius. At $T=1$ it gives
$[N^2\log(eN)]^{-\beta/d}$, while at fixed $N$ the error decreases with
order $[\log(eT)]^{-\beta/d}$. Section~\ref{sec:proof-joint} proves this
bound by combining the unit-radius and larger-radius constructions.
The next theorem measures precision by the $B$-bit parameter constraint
in \eqref{eq:bit-network-class}.  

\begin{theorem}[Joint width--bits approximation]
\label{thm:joint-bits-approximation}
Let $d\in\mathbb N$, $\beta\in(0,1]$, and $p\in(0,\infty)$.
There is a constant $C_2^\star=C_2^\star(d,\beta,p)>0$ such that, for
every $N,B\in\mathbb N$ with $N\ge2d+3$ and $B\ge2$,
\[
\sup_{f\in\mathcal H^\beta}
\inf_{g\in\mathcal F_\DTA^{\mathrm{bit}}(N,13+\lceil3p\beta\rceil,B)}
\norm{f-g}_{L^p}
\le C_2^\star[N^2B]^{-\beta/d}.
\]
\end{theorem}

Width and bit budget can be prescribed independently, with the same
depth as in Theorem~\ref{thm:joint-approximation}. Every affine
parameter, including the output weights and bias, uses a common
binary format with $a+b+2=B$. At $B=2$, the alphabet is
$\{-1,0,1\}$ and the error is of order $N^{-2\beta/d}$.
Section~\ref{sec:proof-joint-bits} proves the bound for all $N$ and $B$
by combining two constructions: parameter rounding of the
width--radius approximant when $B$ is large relative to $\log(eN)$,
and Theorem~\ref{thm:unit-radius} for the remaining budgets.

\begin{remark}[Constants and extensions]\label{rem:approximation-constants}
The explicit constructions give computable constants, which we do not
optimize. Optimality concerns the dependence on width, parameter radius,
and bit budget up to multiplicative constants. The approximation results
extend to general H\"older balls, with constants also depending on the
ball and the clipping interval containing the target range. The preceding bit-complexity estimates also extend to our fixed-size network. For every $f\in\mathcal H^\beta([0,1]^d)$ and $\varepsilon\in(0,1/2]$,
the fixed architectures in Theorems~\ref{thm:approximation}
and~\ref{thm:optimal-approximation}, followed by external clipping (indeed, clipping is unnecessary),
attain $L^p$ error at most $\varepsilon$ with, respectively,
$B=\mathcal O(\varepsilon^{-d/\beta}\log(\varepsilon^{-1}))$ and
$B=\mathcal O(\varepsilon^{-d/\beta})$ bits per parameter. The parameter Lipschitz estimate controls the
rounding error, and quantization preserves zero entries. Thus the
original hidden widths, depths, and respective nonzero-parameter bounds
$8d+7$ and $8d+27$ remain unchanged.
\end{remark}

We next give lower bounds for the full network classes under the two
parameter constraints. The radius bound requires a globally H\"older
activation, whereas the bit bound follows from the finite number of
parameter choices and only requires Borel measurability.

\begin{theorem}[Lower bounds]\label{thm:near-optimal-parameter}
Let $d\in\mathbb N$, $\beta\in(0,1]$, and $p\in(0,\infty)$.
If $\varphi:\R\to\R$ is globally $\alpha$-H\"older for some
$\alpha\in(0,1]$, then, for every $L\in\mathbb N$, there is
$c=c(d,L,\alpha,\beta,p,\varphi)>0$ such that, for all
$N\in\mathbb N$ and $T\ge1$,
\[
\sup_{f\in\mathcal H^\beta}
\inf_{g\in\mathcal F_\varphi(N,L,T)}
\norm{f-g}_{L^p([0,1]^d)}
\ge c\bigl[N^2\log(eNT)\bigr]^{-\beta/d}.
\]
There is also $c^\star=c^\star(d,\beta,p)>0$ such that, for every
Borel measurable activation $\varphi:\R\to\R$ and every
$N,L,B\in\mathbb N$ with $B\ge2$,
\[
\sup_{f\in\mathcal H^\beta}
\inf_{g\in\mathcal F_\varphi^{\mathrm{bit}}(N,L,B)}
\norm{f-g}_{L^p([0,1]^d)}
\ge c^\star[N^2LB]^{-\beta/d}.
\]
\end{theorem}
Applying the respective lower bounds with $\varphi=\DTA$ and
$L=13+\lceil3p\beta\rceil$ matches the upper bounds in
Theorems~\ref{thm:joint-approximation} and~\ref{thm:joint-bits-approximation}.
Thus both rates are optimal uniformly over their independent width and
parameter budgets. At $N=2d+3$ and $L=6$, the radius lower bound gives
$\log T\gtrsim\varepsilon^{-d/\beta}$ as $\varepsilon\downarrow0$, which proves the
fixed-size optimality claimed in Theorem~\ref{thm:optimal-approximation}.
The proof is in Section~\ref{sec:lower-proof}.

\subsection{Fixed-size Transformer approximation}\label{sec:transformer}

We next apply the encoding and decoding method to matrix-valued maps on
sequences of fixed length. The Transformer uses one self-attention layer
between two feedforward blocks. Their widths and attention dimensions
are independent of the target accuracy, and the theorem bounds the
parameter radius needed to achieve that accuracy. We use columnwise
attention as in \citet{JiaoLaiWangYan2026}: a sequence of $n$ tokens,
each with $q$ coordinates, is represented by $\bm X\in\R^{q\times n}$,
with one token per column. A feedforward affine layer acts as
\[
\bm X\longmapsto\bm W\bm X+\bm B,
\qquad
\bm W\in\R^{q_{\mathrm{out}}\times q},
\quad
\bm B\in\R^{q_{\mathrm{out}}\times n}.
\]
The weight matrix $\bm W$ is shared across token positions, while the
bias matrix may have different columns; every entry of $\bm B$ is
counted as a parameter. For $\bm A=(A_{rs})\in\R^{n\times n}$, define
columnwise softmax by
$[\bm\sigma_{\mathrm S}(\bm A)]_{rs}
=e^{A_{rs}}/\sum_{t=1}^n e^{A_{ts}}$.
The self-attention layer, including its residual connection, is
\[
\bm{\mathcal F}_{\mathrm{SA}}(\bm X)
=\bm X+\sum_{i=1}^h
\bm W_O^{(i)}\bm W_V^{(i)}\bm X\,
\bm\sigma_{\mathrm S}\!\bigl(
(\bm W_K^{(i)}\bm X)^\top(\bm W_Q^{(i)}\bm X)
\bigr).
\]
The query and key matrices determine the attention weights, which are
applied to the transformed token values before the output map returns
them to the token space. The leading $\bm X$ is the residual connection,
and the feedforward activations act coordinatewise. The parameter radius
bounds every scalar entry of the feedforward weight and bias matrices
and the attention matrices.

Using $\norm{\bm X}_{\max}=\max_{r,s}|x_{rs}|$, let
$\mathcal H_{d,n}^\beta$ denote the class of maps
$\bm f=(f_{rs}):[0,1]^{d\times n}\to\R^{d\times n}$ such that
$\max_{r,s}\norm{f_{rs}}_\infty\le1$ and
\(
|f_{rs}(\bm X)-f_{rs}(\bm Y)|
\le
\norm{\bm X-\bm Y}_{\max}^\beta
\)
for every $\bm X,\bm Y$ in the input cube and every output coordinate
$r,s$. The input and output each consist of $n$ tokens in $\R^d$.
For $p\in(0,\infty)$, measure approximation error by
\[
d_p(\bm f,\bm g)
=
\left(
\int_{[0,1]^{d\times n}}
\sum_{r=1}^d\sum_{s=1}^n
|f_{rs}(\bm X)-g_{rs}(\bm X)|^p
\,\dd\bm X
\right)^{1/p};
\]
for $p<1$, this is the usual $L^p$ quasi-distance.

\begin{theorem}[Fixed-size Transformer approximation]
\label{thm:transformer}
Let $d,n\in\mathbb N$, $p\in(0,\infty)$, $\beta\in(0,1]$, and
$\bm f\in\mathcal H_{d,n}^\beta$.  For every $\varepsilon\in(0,1/2]$, there is a
Transformer
\[
\bm{\mathcal T}_\varepsilon
=
\bm{\mathcal F}_{\mathrm{FF}}^{\mathrm{out}}
\circ
\bm{\mathcal F}_{\mathrm{SA}}
\circ
\bm{\mathcal F}_{\mathrm{FF}}^{\mathrm{in}}.
\]
The input feedforward block has hidden widths $\langle 2d+3,d+2,6\rangle$ and output
dimension six, the self-attention block has one head of size two, and the
output feedforward block has hidden widths $\langle 2d,2d,2d,2d\rangle$ and output
dimension $d$.  Its parameter radius $T_{\mathrm{Tr},\varepsilon}$ and realization
satisfy
\[
d_p(\bm f,\bm{\mathcal T}_\varepsilon)\le\varepsilon,
\qquad
\sup_{\bm X\in[0,1]^{d\times n}}
\norm{\bm{\mathcal T}_\varepsilon(\bm X)}_{\max}\le 5/4,
\qquad
\log T_{\mathrm{Tr},\varepsilon}
\le C_{\mathrm{Tr}}\varepsilon^{-dn/\beta},
\]
where $C_{\mathrm{Tr}}>0$ depends only on $d,n,p$, and $\beta$. Both feedforward blocks are activated by \DTA.
\end{theorem}

For fixed $d$ and $n$, the same attention head and feedforward
architectures work for all targets and accuracies in the theorem.
Only their parameter values depend on the target and on $\varepsilon$.
The exponent $dn/\beta$ comes from the $dn$ scalar input coordinates.
Appendix~\ExternalNumber{D}{app:transformer} gives the construction,
including the composition of the encoder, attention layer, and decoder
and the parameter estimates for each block.

\subsection{Sampling model and the regression rate}
We use the approximation bounds to control prediction risk for empirical
least squares over the full clipped classes. The fitted class has either
a radius constraint or a bit constraint, and the empirical squared loss
may exceed its infimum by $M^{-1}$. We work under the following sampling
assumptions.

Fix $d\in\mathbb N$, $\beta\in(0,1]$, and
$f\in\mathcal H^\beta([0,1]^d)$.  Let $\mathbf{m}$ denote Lebesgue measure on
$[0,1]^d$, and suppose that $\mu$ has density
$w=\mathrm d\mu/\mathrm d\mathbf{m}\le\kappa$ $\mathbf{m}$-almost everywhere for some
$\kappa\ge1$.  For each $M\in\mathbb N$, let
\[
\bm X_1,\ldots,\bm X_M\stackrel{\mathrm{i.i.d.}}{\sim}\mu,
\qquad
\varepsilon_1,\ldots,\varepsilon_M\stackrel{\mathrm{i.i.d.}}{\sim}\nu,
\qquad
Y_i=f(\bm X_i)+\varepsilon_i,
\]
where the design and noise variables are independent, and $\nu$ is centered
and $\sigma$-sub-Gaussian for some $\sigma\ge0$:
$\mathbb E_\nu e^{t\varepsilon}\le e^{\sigma^2t^2/2}$ for every $t\in\mathbb R$.
Unless conditioning is specified, expectations are over the observations
generated by $f$. For a measurable predictor $g:[0,1]^d\to[-1,1]$, define
its empirical squared loss by
\[
\widehat{\mathcal R}_M(g) =\frac1M\sum_{i=1}^M\bigl(g(\bm X_i)-Y_i\bigr)^2.
\]
For an independent observation $(\bm X,\varepsilon,Y)$, independence and
centering of the noise give the population excess-risk identity
\begin{equation*}
\begin{aligned}
\mathbb E\{(g(\bm X)-Y)^2-(f(\bm X)-Y)^2\}
&=\mathbb E(g(\bm X)-f(\bm X))^2
 -2\mathbb E\{(g(\bm X)-f(\bm X))\varepsilon\}
=\norm{g-f}_{L^2(\mu)}^2.
\end{aligned}
\end{equation*}
The density bound transfers the Lebesgue approximation error to
prediction error under the design law:
\begin{equation}\label{eq:density-transfer}
\norm{g-f}_{L^2(\mu)}^2
=\int_{[0,1]^d}|g-f|^2w\,\mathrm d\mathbf{m}
\le\kappa\norm{g-f}_{L^2([0,1]^d)}^2.
\end{equation}
More generally, if $w\in L^q(\mathbf{m})$ for some $q>1$, H\"older's inequality
gives 
\[
\norm{g-f}_{L^2(\mu)}^2 \le\norm{w}_{L^q(\mathbf{m})} \norm{g-f}_{L^{2q/(q-1)}([0,1]^d)}^2,
\]
so Theorem~\ref{thm:joint-approximation} also supplies the approximation
bound at the finite exponent $2q/(q-1)$.

An $M^{-1}$-approximate empirical-risk minimizer over $\mathcal G$ is a
measurable estimator $\widehat f_M$ taking values in $\mathcal G$ and
satisfying
\[
\widehat{\mathcal R}_M(\widehat f_M) \le\inf_{g\in\mathcal G}\widehat{\mathcal R}_M(g)+M^{-1} \qquad\text{almost surely}.
\]

\begin{theorem}[Optimal regression along the width--radius curve]
\label{thm:generalization}
Fix $d,M\in\mathbb N$ and $\beta\in(0,1]$, and assume the
sampling model above. Let $N\ge2d+3$ be an integer and $T\ge1$, and suppose
that, for some fixed constants $0<c_-\le c_+<\infty$,
\begin{equation}\label{eq:width-radius-curve}
N^2\log(eNT)\in\left[c_-M^{\frac{d}{2\beta+d}}, c_+M^{\frac{d}{2\beta+d}}\right].
\end{equation}
If $\widehat f_M$ is a measurable $M^{-1}$-approximate empirical-risk
minimizer over $\mathcal F_\DTA(N,19,T)$, then
\[
\sup_{f\in\mathcal H^\beta}
\mathbb E\norm{\widehat f_M-f}_{L^2(\mu)}^2
\le C_3M^{-\frac{2\beta}{2\beta+d}},
\]
where the constant $C_3$ depends only on $d,\beta,\kappa,\sigma, c_-, c_+$.
\par
\end{theorem}

The condition \eqref{eq:width-radius-curve} permits different allocations
of width and parameter radius with the same risk guarantee. It is
nonempty for every
$M\ge\left\lceil\{(2d+3)^2\log(e(2d+3))/c_+\}^{(2\beta+d)/d}\right\rceil$:
one may fix $N=2d+3$ and choose $T\ge1$ accordingly. At fixed width,
$\log T\asymp M^{d/(2\beta+d)}$; at unit radius,
$N^2\log(eN)\asymp M^{d/(2\beta+d)}$.
The theorem applies to every deterministic choice satisfying
\eqref{eq:width-radius-curve}, with constants depending on $c_-$ and
$c_+$ but not on that choice or on $M$. Under uniform design and fixed
nondegenerate Gaussian noise, the bound matches the classical H\"older
minimax lower bound \citep{Stone1982}, so there is no logarithmic loss.

The estimator is chosen from the full class
$\mathcal F_\DTA(N,19,T)$, and the guarantee applies to every measurable
approximate minimizer satisfying the stated tolerance. Such estimators
exist by the separability argument in
Appendix~\ExternalNumber{E}{app:gen}. The proof in
Section~\ref{sec:regression-proof} combines the approximation bound with
an empirical-risk inequality.

\begin{remark}[Joint width--bits regression]\label{rem:bits-regression}
Under the same sampling model, let $N,B\in\mathbb N$ with
$N\ge2d+3$ and $B\ge2$, and let $\widehat f_M$ be a measurable
$M^{-1}$-approximate empirical-risk minimizer over
$\mathcal F_\DTA^{\mathrm{bit}}(N,19,B)$.
The proof of Theorem~\ref{thm:near-optimal-parameter} gives
$\log\#\mathcal F_\DTA^{\mathrm{bit}}(N,19,B)\le19(d+5)N^2B$.
Since $13+\lceil6\beta\rceil\le19$, Theorem~\ref{thm:joint-bits-approximation},
\eqref{eq:density-transfer}, and Proposition~\ref{prop:oracle} therefore give
\[
\sup_{f\in\mathcal H^\beta}
\mathbb E\norm{\widehat f_M-f}_{L^2(\mu)}^2
\lesssim[N^2B]^{-2\beta/d}+M^{-1}N^2B,
\]
Thus deterministic choices with $N^2B\asymp M^{d/(2\beta+d)}$
and fixed comparison constants attain risk
$\mathcal O(M^{-2\beta/(2\beta+d)})$.
The finite class makes the oracle argument immediate and admits a
measurable exact minimizer by selecting the first minimizer in a fixed
ordering.
\end{remark}

\section{Beyond DTA to diverse activation}\label{sec:beyond-dta}

The choice of activation affects numerical evaluation and gradient-based
training. Extending our approximation results to alternative activations
enlarges the collection of network models available for these tasks.
At fixed width $N$ and hidden depth $L$, we demonstrated that increasing parameter precision
can improve approximation for $\DTA$-network and also provide the joint rate after varying the width $N$. More generally, we want to ask what type of elementary Lipschitz activations (something easy to train) retain the optimal
joint dependence on width and precision. In view of the lower bound
in Theorem~\ref{thm:near-optimal-parameter}, the correct question is to find $\sigma:\R\to\R$ such that
\begin{equation}\label{eq:beyond-bits-rate}
\sup_{f\in\mathcal H^\beta}
\inf_{g\in\mathcal F_\sigma^{\mathrm{bit}}(N,L_*,B)}
\norm{f-g}_{L^p}
\asymp [N^2B]^{-\beta/d}.
\end{equation}
What lower bound on $B$ suffices for this equivalence to hold uniformly
over all sufficiently large widths? The following definition distinguishes
two natural answers.

\begin{definition}[Bit admissibility]\label{def:bit-admissibility}
Let $\sigma:\R\to\R$ be continuous.
\begin{enumerate}[label=\textup{(\roman*)}]
\item We call $\sigma$ \emph{bit-admissible} if, for every
$d\in\mathbb N$, $\beta\in(0,1]$, and $p\in(0,\infty)$, there are
$N_*,L_*,B_*\in\mathbb N$, with $B_*\ge2$, such that
\eqref{eq:beyond-bits-rate} holds for all $N\ge N_*$ and $B\ge B_*$.
We call $(N_*,L_*,B_*)$ an \emph{admissible triple} for $\sigma$.
\item We call $\sigma$ \emph{almost bit-admissible} if the same statement
holds for all $N\ge N_*$ and $B\ge B_*\log_2(N+1)$, and call
$(N_*,L_*,B_*)$ an \emph{almost-admissible triple} for $\sigma$.
\end{enumerate}
The triples and comparison constants may depend on $\sigma,d,\beta,p$,
but are independent of $N,B$. When specifying a triple, we also say that
$\sigma$ is \emph{$(N_*,L_*,B_*)$-admissible} or
\emph{$(N_*,L_*,B_*)$-almost admissible}, respectively.
\end{definition}

The distinction is whether the rate holds from a constant bit budget or
from a logarithmic one. Increasing any component of an admissible or
almost-admissible triple preserves that property, since width and depth
are upper bounds and the counting lower bound still applies. The
implications and their approximation meaning are
\begingroup
\setlength{\arraycolsep}{2pt}
\setlength{\fboxsep}{4pt}
\[
\begin{array}{@{}ccccc@{}}
\boxed{\text{bit-admissible}}
&\xRightarrow{\hspace{1.5em}}&
\boxed{\text{almost bit-admissible}}
&\xRightarrow{\hspace{1.5em}}&
\boxed{\text{superexpressive in }L^p}\\[3mm]
\bigg\Downarrow&&\bigg\Downarrow&&\bigg\Downarrow\\[3mm]
\boxed{
\begin{gathered}
\text{with rate}\\
\Theta([N^2B]^{-\beta/d})
\end{gathered}
}
&\xRightarrow{\hspace{1.5em}}&
\boxed{
\begin{gathered}
\text{with rate }\Theta([N^2B]^{-\beta/d})\\
\text{in high precision regime}
\end{gathered}
}
&\xRightarrow{\hspace{1.5em}}&
\boxed{
\begin{gathered}
\text{arbitrary accuracy}\\
\text{as }B\to\infty
\end{gathered}
}
\end{array}
\]
\endgroup
For the last implication, fix $N=N_*$ and let $B\to\infty$; approximation
of continuous targets follows by Lipschitz approximation and output
rescaling. Writing $L_p=13+\lceil3p\beta\rceil$,
Theorems~\ref{thm:joint-bits-approximation}
and~\ref{thm:near-optimal-parameter} give the admissible triple
$(2d+3,L_p,2)$ for $\DTA$.
The fixed-architecture obstruction discussed in
Section~\ref{sec:introduction} shows that classic activations, including ReLU, logistic sigmoid, and
tanh, have no almost-admissible triple due to the fact that piecewise Pfaffian function is not superexpressive .
\begin{remark}
By \eqref{eq:intro-oracle} and \eqref{eq:beyond-bits-rate}, the optimal width--bits regression tradeoff
extends to any bit-admissible or almost bit-admissible activation with
$p=2$. In both cases, the balance
$N^2B\asymp M^{d/(2\beta+d)}$ yields the minimax risk
$\mathcal O(M^{-2\beta/(2\beta+d)})$, with the additional restriction
$B\ge B_*\log_2(N+1)$ in the latter case (i.e., minimax-optimal in the high-precision regime).
\end{remark}
To approach such activations, a natural simplification is to build on
the approximation result already established for $\DTA$, namely $\DTA$ is $(2d+3,L_p,2)$-admissible. The following replacement
principle incorporates this result into a layerwise replacement strategy
inspired by \citet[Proposition~10]{ZhangLuZhao2024}, with explicit control
of the binary formats.

\begin{theorem}[Replacement principle]\label{thm:activation-replacement}
Let $\sigma:\R\to\R$ be continuous, and let $N',L',B'\in\mathbb N$
with $B'\ge2$. Each of the following conditions gives the stated triple.
\begin{enumerate}[label=\textup{($\bm{C1}$)},leftmargin=*,itemsep=5pt]

\item[\textup{($\bm{C1}$)}] If $\DTA\in\NN_\sigma^{\mathrm{bit}}(1,1;N',L',B')$ and has a
realization with zero final affine bias in this specific class, then $\sigma$ is
$\bigl(N'(2d+3),L'L_p,4B'-4\bigr)$-admissible.
\item[\textup{($\bm{C2}$)}] If, for every $s\in\mathbb N$, there exists
\begin{equation}\label{eq:beyond-scalar-replacement}
\phi_s\in\NN_\sigma^{\mathrm{bit}}(1,1;N',L',B's),
\qquad
\norm{\phi_s-\DTA}_{L^\infty([-2^s,2^s])}\le2^{-s},
\end{equation}
then $\sigma$ is
$\bigl(2N'(2d+3),L'L_p,32B'(L_p+2)\bigr)$-almost admissible.
\end{enumerate}
\end{theorem}
The proof (which is not difficult) consists of layerwise replacement, keeping track of the errors
and binary formats. We omit it and use the theorem only in this section.
The same argument applies with $\DTA$ replaced by any globally Lipschitz
bit-admissible activation, starting from any of its admissible triples.
Thus a small scalar network that represents or approximates $\DTA$ well in the sense of this Theorem is
enough to transfer the optimal rate. In applications of
\textup{($\bm{C2}$)}, we write $B_\sigma=32B'(L_p+2)$ for the resulting
precision lower bound.

\subsection{Parametric DTA and softsign DTA}\label{sec:parametric-dta}
We now apply the two criteria to concrete activation families. We omit
the scalar network constructions and their direct verification. All exact
realizations below have zero final affine bias, as required by
\textup{($\bm{C1}$)}. Parameters in square brackets are fixed when an
activation is chosen.

We start by varying the exponential tail and the triangular wave of $\DTA$ independently.
For $\lambda_1\in(1,\infty)$, $\lambda_2\in(0,\infty)$, and
$\lambda_3\in[0,\infty)$, define the following families
(see Figure~\ref{fig:beyond-parametric} for illustration):
\begin{equation}\label{eq:beyond-pdta}
\PDTA[\lambda_1,\lambda_2](t)=
\begin{cases}
\lambda_1^t-1,&t\in(-\infty,0),\\[6pt]
\tau(\lambda_2t),&t\in[0,\infty),
\end{cases}
\qquad
\SPDTA[\lambda_1,\lambda_2,\lambda_3]
=\frac{\PDTA[\lambda_1,\lambda_2]}
{1+\lambda_3|\PDTA[\lambda_1,\lambda_2]|}.
\end{equation}

\begin{figure}[htbp]
\centering
\includegraphics[page=1,width=.323\linewidth]{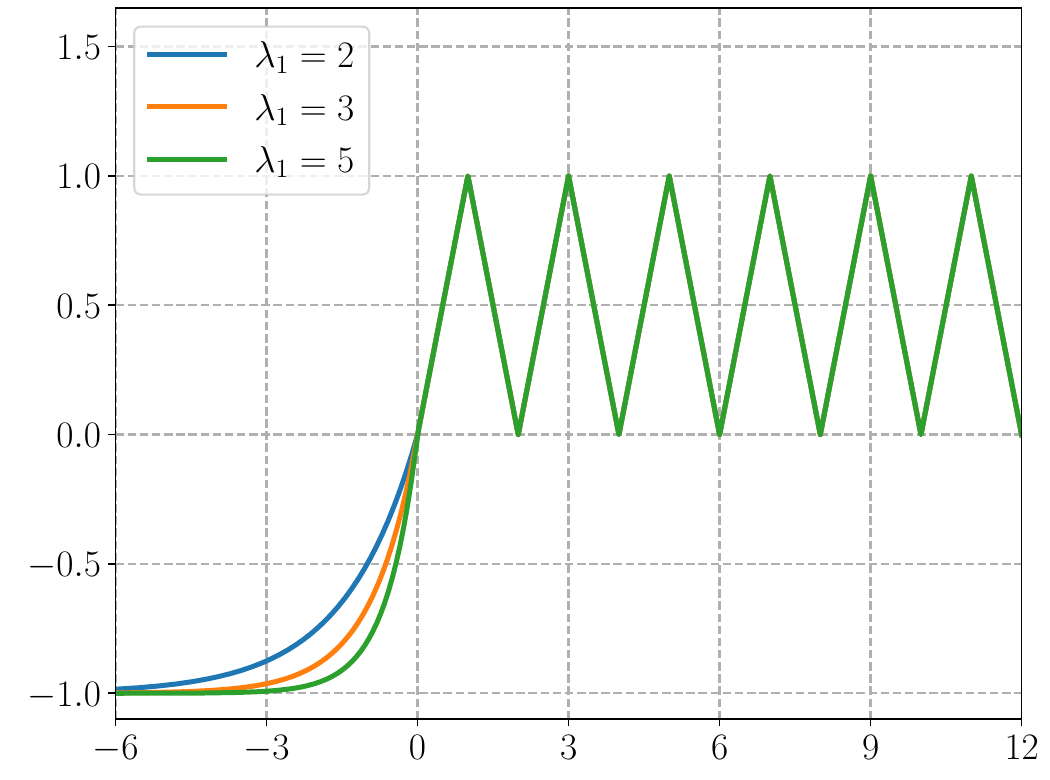}\hfill
\includegraphics[page=2,width=.323\linewidth]{beyond_dta_panels.pdf}\hfill
\includegraphics[page=3,width=.323\linewidth]{beyond_dta_panels.pdf}
\caption{The families $\PDTA[\lambda_1,1]$ (left),
$\PDTA[2,\lambda_2]$ (middle), and $\SPDTA[2,1,\lambda_3]$ (right).
The parameters vary the exponential tail, positive frequency, and output
compression, respectively.}
\label{fig:beyond-parametric}
\end{figure}
Here $\lambda_1$ changes the negative exponential tail, $\lambda_2$
sets the positive period $2/\lambda_2$, and $\lambda_3$ controls softsign
compression. Both families are bounded and globally Lipschitz.
The identities $\PDTA[2,1]=\DTA$ and
$\SPDTA[\lambda_1,\lambda_2,0]=\PDTA[\lambda_1,\lambda_2]$
identify the original activation within these families.

\paragraph{Almost admissibility.}
For every activation $\sigma$ in \eqref{eq:beyond-pdta}, there is a
constant $B'$ depending only on its bracket parameters such that
\[
\exists\phi_s\in\NN_\sigma^{\mathrm{bit}}(1,1;2,4,B's)\text{ with }\eqref{eq:beyond-scalar-replacement}\xRightarrow{(\bm{C2})}\sigma\text{ is }\bigl(4(2d+3),\,4L_p,\,B_\sigma\bigr)\text{-almost admissible}.
\]
Therefore, every Parametric DTA $\PDTA$ and softsign DTA $\SPDTA$ is almost bit-admissible activation, regardless of how we parametrize it. Next, we aim to provide some families of special cases where they are bit-admissible, by using $(\bm{C1})$ of Theorem~\ref{thm:activation-replacement} .

\paragraph{Admissibility.}
The following explicit parameter subfamilies satisfy
\textup{($\bm{C1}$)} and hence attain the same rate with a constant
lower bound on $B$.
\begin{enumerate}[label=\textup{(\roman*)},leftmargin=*,itemsep=5pt]
\item For $k\in\mathbb N_0$ and $m,n\in\mathbb N$, the activation
$\sigma=\PDTA[2^{2^k/m},2^k/n]$ satisfies
\textup{($\bm{C1}$)} with
$B'=k+2+\lceil\log_2(1+\max\{m,n\})\rceil$. More explicitly,
\[
\DTA\in\NN_\sigma^{\mathrm{bit}}(1,1;5,2,B')
\xRightarrow{(\bm{C1})}
\sigma\text{ is }\bigl(5(2d+3),\,2L_p,\,4B'-4\bigr)\text{-admissible}.
\]
For example, both $\PDTA[4,1/2]$ and $\PDTA[\sqrt2,2]$ have the
admissible triple $\bigl(5(2d+3),2L_p,20\bigr)$.

\item For $k,j,r\in\mathbb N_0$ and $m,n\in\mathbb N$, the activation
$\sigma=\SPDTA[2^{2^k/m},2^k/n,j2^{-r}]$ satisfies
\textup{($\bm{C1}$)} with
$B'=k+r+2+\lceil\log_2(1+\max\{m,2n(j+1)\})\rceil$.
The corresponding exact realization gives
\[
\DTA\in\NN_\sigma^{\mathrm{bit}}(1,1;4,4,B')
\xRightarrow{(\bm{C1})}
\sigma\text{ is }\bigl(4(2d+3),\,4L_p,\,4B'-4\bigr)\text{-admissible}.
\]
For example, both $\SPDTA[4,1,1]$ and
$\SPDTA[\sqrt2,2,1/2]$ have the admissible triple
$\bigl(4(2d+3),4L_p,24\bigr)$ by taking $(k,j,r,m,n)=(1,1,0,1,2)$ and
$(1,1,1,4,1)$ respectively.
\end{enumerate}
Note that these parameter choices are dense in the respective parameter domains.
For any prescribed exponential base, positive frequency, and softsign
compression, one can therefore choose an admissible activation with
arbitrarily close parameters.
One important case deserves a remark is the ordinary softsign transformation
$\sigma=\SPDTA[2,1,1]=\DTA/(1+|\DTA|)$ satisfies
\textup{($\bm{C1}$)} with a smaller scalar realization:
\[
\DTA\in\NN_\sigma^{\mathrm{bit}}(1,1;2,4,4)
\xRightarrow{(\bm{C1})}
\sigma\text{ is }\bigl(2(2d+3),\,4L_p,\,12\bigr)\text{-admissible}.
\]
Next, we study activations that more closely resemble ordinary activation functions.

\subsection{ReLU-like and sigmoidal activations}\label{sec:relu-sigmoidal}
We next retain the oscillatory component in two increasing activations
with familiar shapes (see Figure~\ref{fig:beyond-shapes} for illustration).
For $\lambda\in(0,1)$ in the first family and $\lambda\in(0,1/2]$
in the second, define
\begin{equation}\label{eq:beyond-rdta-sgdta}
\RDTA[\lambda](t)=\operatorname{ReLU}(t)+\lambda\DTA(t),
\qquad
\SGDTA[\lambda](t)=\frac{2^t}{1+2^t}
-\frac{\lambda2^t\tau(2^t)}{(1+2^t)^3}.
\end{equation}

\begin{figure}[htbp]
\centering
\includegraphics[page=4,width=.323\linewidth]{beyond_dta_panels.pdf}\hfill
\includegraphics[page=5,width=.323\linewidth]{beyond_dta_panels.pdf}\hfill
\includegraphics[page=6,width=.323\linewidth]{beyond_dta_panels.pdf}
\caption{$\RDTA[\lambda]$ on domain $[-6,12]$ (left) and $[0,4]$ (middle),
and $\SGDTA[\lambda]$ on $[-6,12]$ (right).
The curves marked $\lambda=0$ are the unperturbed ReLU and logistic
functions.}
\label{fig:beyond-shapes}
\end{figure}
The function $\RDTA[\lambda]$ is strictly increasing and satisfies
$\norm{\RDTA[\lambda]-\operatorname{ReLU}}_{L^\infty(\R)}\le\lambda$.
It approaches ReLU uniformly as the bracket parameter tends to zero.
Its positive piecewise linear slopes belong to $\{1-\lambda,1+\lambda\}$,
while its negative branch is the exponential tail $\lambda(2^t-1)$.
The function $\SGDTA[\lambda]$ is continuous and strictly increasing,
takes values in $(0,1)$, and has limits $0$ and $1$ at $-\infty$ and
$+\infty$, respectively. It is therefore sigmoidal in the usual sense. Moreover, its distance from the logistic function with base
$2$ is bounded by $4\lambda/27$ uniformly on $\R$.

\paragraph{Almost admissibility.}
Both families in \eqref{eq:beyond-rdta-sgdta} satisfy the approximation
criterion throughout their stated parameter ranges. For each fixed
activation $\sigma$, there is a constant $B'$ such that
\[
\exists\phi_s\in\NN_\sigma^{\mathrm{bit}}(1,1;4,2,B's)\text{ with }\eqref{eq:beyond-scalar-replacement}\xRightarrow{(\bm{C2})}\sigma\text{ is }\bigl(8(2d+3),\,2L_p,\,B_\sigma\bigr)\text{-almost admissible}.
\]
Therefore, every $\RDTA$ and $\SGDTA$ is almost bit-admissible activation. We now give a parameter family such that $\RDTA$ is bit-admissible.

\paragraph{Admissibility of RDTA.}
The activation $\sigma=\RDTA[1/2]$ admits an exact realization
using nine bits per parameter:
\[
\DTA\in\NN_\sigma^{\mathrm{bit}}(1,1;3,2,9)
\xRightarrow{(\bm{C1})}
\sigma\text{ is }\bigl(3(2d+3),\,2L_p,\,32\bigr)\text{-admissible}.
\]
More generally, $\sigma=\RDTA[2^k/m]$ satisfies
\textup{($\bm{C1}$)} whenever $k\in\mathbb N_0$,
$m\in\mathbb N$, and $2^k/m\in(0,2/3]$.
With $K=\lceil\log_2(4(1+m2^{-k})^2)\rceil$, we have
\[
\DTA\in\NN_\sigma^{\mathrm{bit}}(1,1;5,2,2k+2K+3)
\xRightarrow{(\bm{C1})}
\sigma\text{ is }\bigl(5(2d+3),\,2L_p,\,8k+8K+8\bigr)\text{-admissible}.
\]
In particular, for every integer $m\ge2$,
\[
\RDTA[1/m]\quad\text{is}\quad
\bigl(5(2d+3),\,2L_p,\,24+8\lceil2\log_2(m+1)\rceil\bigr)\text{-admissible},
\]
providing admissible activations arbitrarily close to ReLU in the
uniform norm.

\subsection{An explicit analytic family}\label{sec:analytic-activations}
An analytic deformation of parametric DTA rounds its corners
(see Figure~\ref{fig:beyond-analytic} for illustration).
For $\lambda_1\in(1,\infty)$ and $\lambda_2,\lambda_3\in(0,\infty)$,
define
\begin{equation}\label{eq:beyond-analytic}
\ADTA[\lambda_1,\lambda_2,\lambda_3](t)
=\frac{\displaystyle\lambda_1^t-1+\lambda_1^{2t}e^{\lambda_3t}
\left(\frac12-\frac1\pi\arcsin
\frac{\cos(\pi\lambda_2t)}{1+e^{-\lambda_3t}}\right)}
{1+\lambda_1^{2t}e^{\lambda_3t}}.
\end{equation}

\begin{figure}[htbp]
\centering
\includegraphics[page=7,width=.323\linewidth]{beyond_dta_panels.pdf}\hfill
\includegraphics[page=8,width=.323\linewidth]{beyond_dta_panels.pdf}\hfill
\includegraphics[page=9,width=.323\linewidth]{beyond_dta_panels.pdf}
\caption{$\ADTA[2,1,\lambda_3]$ with varying rounding (left), and
$\ADTA[\lambda_1,\lambda_2,1/2]$ with varying base and frequency
(middle). The right panel enlarges a peak of $\ADTA[2,1,\lambda_3]$
on $[1/2,3/2]$; the dashed curve is $\PDTA[2,1]$.}
\label{fig:beyond-analytic}
\end{figure}
Every member is bounded, globally Lipschitz, and real analytic on $\R$.
The first two parameters retain the exponential base and frequency from
PDTA; the third controls the rounding of the triangular peaks and the
transition between the negative and positive branches. As
$\lambda_3\to\infty$, $\ADTA[\lambda_1,\lambda_2,\lambda_3](t)$ converges
to $\PDTA[\lambda_1,\lambda_2](t)$ for every $t\ne0$.

\paragraph{Almost admissibility.}
For each fixed $\sigma=\ADTA[\lambda_1,\lambda_2,\lambda_3]$, there is a
constant $B'$ depending only on its bracket parameters such that
\[
\exists\phi_s\in\NN_\sigma^{\mathrm{bit}}(1,1;2,2,B's)\text{ with }\eqref{eq:beyond-scalar-replacement}\xRightarrow{(\bm{C2})}\sigma\text{ is }\bigl(4(2d+3),\,2L_p,\,B_\sigma\bigr)\text{-almost admissible}.
\]
Hence, every $\ADTA$ is almost bit-admissible activation. For example, $\ADTA[2,1,1]$ satisfies \textup{($\bm{C2}$)} with
$B'=4$, giving the explicit conclusion
\[
\ADTA[2,1,1]\quad\text{is}\quad
\bigl(4(2d+3),\,2L_p,\,128(L_p+2)\bigr)\text{-almost admissible}.
\]

The family $\ADTA$ provides explicit elementary activations that are
real analytic on $\R$ and almost bit-admissible (hence superexpressive) in the uniform norm (we demonstrated the conclusion in $L^p$, but it's also true uniformly).
\citet[Theorem~4]{MaiorovPinkus1999} obtained analytic superexpressivity
through an abstract construction. \citet[Section~4]{Yarotsky2021} asked
whether a single elementary activation could have this property while
being real analytic on all of $\R$. Each fixed member of $\ADTA$ gives
such an example. We also believe that a large subfamily of $\ADTA$ is in fact bit-admissible. 

\FloatBarrier

\section{Proofs of Theorems~\ref{thm:approximation}--\ref{thm:joint-bits-approximation}}\label{sec:approximation}

The constructions approximate $f$ on a uniform grid. An encoder determines
the grid cell containing the input, and a decoder recovers the value
assigned to that cell. We use narrow transition strips at cell boundaries to preserve continuity;
their measure controls the remaining $L^p$ error.

For Theorem~\ref{thm:approximation}, the stored values are arbitrary.
The dyadic table decoder encodes the entire table through a single
integer. The proof, which is a mix of algebraic and analytic flavors, is given in the Appendix.  For Theorem~\ref{thm:optimal-approximation}, we
instead use the H\"older condition (indeed, continuity is enough): after rescaling and rounding, neighboring
values differ by at most one.  Each block can therefore be stored by its
initial value and increments from $\{-1,0,1\}$. From this, we develop the concepts of a block-grid encoder and a coupled block decoder, which together enable highly efficient retrieval of these coefficients and thereby achieve the optimal order of parameter magnitude.
For Theorem~\ref{thm:joint-approximation}, we extend the coupled block
decoder by distributing its codes among a square array of parameters.
The remaining radius range is handled by Theorem~\ref{thm:unit-radius},
which constructs networks with quantized parameters in $[-1,1]$.
This theorem also covers the low-precision case of
Theorem~\ref{thm:joint-bits-approximation}, while the high-precision
case follows by quantizing the networks from
Theorem~\ref{thm:joint-approximation}.

\subsection{Dyadic table decoder}
The following decoder approximates prescribed values $c_1,\ldots,c_N$
at the inputs $2^{-i/(N+1)}-1$ using one hidden unit.
Its output bound also controls approximation error on the trifling region.

\begin{proposition}[Dyadic table decoder]\label{prop:decoder}
Let $N\in\mathbb N$, $\eps\in(0,1)$, and $c_1,\ldots,c_N\in[-1,1]$.
There is a $\DTA$ network $\mathcal R:\R\to\R$ with one hidden unit such
that
\begin{equation*}
\max_{1\le i\le N}
\left|\mathcal R\!\left(2^{-i/(N+1)}-1\right)-c_i\right|\le\eps,
\qquad
\sup_{v\in(-1,1]}|\mathcal R(v)|\le1.
\end{equation*}
The network has at most four nonzero parameters.  Its parameter radius is at
most $6\bigl(40N^{3/2}\eps^{-1}\bigr)^N$.
\end{proposition}
The proof is given in Appendix~\ExternalNumber{A}{app:orbit}.
For later compositions we write
$\mathcal R=\mathcal B_2\circ\DTA\circ \mathcal B_1$, where
$\mathcal B_1(v)=2mv+4|m|+2$ and $\mathcal B_2(w)=2w-1$.  Thus, once an encoder returns
$-i/(N+1)$, one more activation produces $2^{-i/(N+1)}-1$ and
$\mathcal R$ approximates the $i$th table value to error at most $\eps$.

\subsection{Trifling region and grid encoder}\label{sec:grid-encoder}
A continuous network cannot make an exact jump at every grid boundary.  We
therefore let the staircase interpolate on short intervals of width
$\delta$ immediately before those boundaries.  The encoder will be exact
outside their union, while the volume of the union controls the remaining
$L^p$ error.  For an integer $K\ge2$ and $\delta\in(0,1/K)$, define
\[
\Omega_{K,\delta}([0,1]^d)
=\bigcup_{j=1}^d
\left\{\bm x\in[0,1]^d:
 x_j\in\bigcup_{k=1}^{K}
 \left(\frac{k}{K}-\delta,\frac{k}{K}\right]
\right\}.
\]

\begin{figure}[htbp]
\centering
\begin{minipage}[c]{.43\textwidth}
\centering
\begin{tikzpicture}[x=.82cm,y=1cm]
  \fill[goodgreenstrong] (0,-.10) rectangle (6,.10);
  \foreach \k in {1,...,6}{
    \fill[stripred] ({\k-.18},-.11) rectangle (\k,.11);
  }
  \foreach \k in {1,...,5}{
    \draw[gridred,densely dashed,line width=.65pt]
      (\k,-.30)--(\k,.48);
  }
  \draw[softgray,line width=.58pt] (0,0)--(6,0);
  \foreach \k in {0,...,6}{
    \draw[softgray,line width=.45pt]
      (\k,-.055)--(\k,.055);
  }
  \foreach \k in {0,...,5}{
    \fill[black] ({\k+.5},0) circle (.035);
  }
  \draw[<->,stripred!85!black,line width=.62pt]
    (1.82,.55)--(2,.55);
  \node[
    font=\DiagramLabelFont,
    text=stripred!70!black,
    above=2pt
  ] at (1.91,.55) {$\delta$};
  \node[font=\DiagramLabelFont,below=3pt] at (0,-.10) {$0$};
  \node[font=\DiagramLabelFont,below=3pt] at (6,-.10) {$1$};
  \node[font=\DiagramPanelFont] at (3,-.70)
    {(a) $d=1$, $K=6$};
\end{tikzpicture}
\end{minipage}
\hfill
\begin{minipage}[c]{.53\textwidth}
\centering
\begin{tikzpicture}[x=4.05cm,y=4.05cm]
  \fill[goodgreenstrong] (0,0) rectangle (1,1);
  \foreach \k in {1,2,3,4}{
    \fill[stripred]
      ({\k/4-.035},0) rectangle ({\k/4},1);
    \fill[stripred]
      (0,{\k/4-.035}) rectangle (1,{\k/4});
  }
  \foreach \k in {1,2,3}{
    \draw[gridred,densely dashed,line width=.62pt]
      ({\k/4},0)--({\k/4},1);
    \draw[gridred,densely dashed,line width=.62pt]
      (0,{\k/4})--(1,{\k/4});
  }
  \draw[softgray,line width=.58pt]
    (0,0) rectangle (1,1);
  \foreach \i in {0,...,3}{
    \foreach \j in {0,...,3}{
      \fill[black]
        ({(\i+.5)/4},{(\j+.5)/4}) circle (.010);
    }
  }
  \node[font=\DiagramLabelFont,below=3pt] at (0,0) {$0$};
  \node[font=\DiagramLabelFont,below=3pt] at (1,0) {$1$};
  \node[font=\DiagramLabelFont,left=3pt] at (0,1) {$1$};
  \node[font=\DiagramPanelFont] at (.5,-.19)
    {(b) $d=2$, $K=4$};
\end{tikzpicture}
\end{minipage}

\GridLegendGap{}

\begin{tikzpicture}[x=1cm,y=1cm]
  \fill[goodgreenstrong] (0,0) rectangle (.34,.14);
  \node[font=\DiagramLabelFont,anchor=west] at (.44,.07)
    {good part of a cell};
  \fill[stripred] (4.05,0) rectangle (4.39,.14);
  \node[font=\DiagramLabelFont,anchor=west] at (4.49,.07)
    {$\Omega_{K,\delta}([0,1]^d)$};
  \fill[black] (7.65,.07) circle (.027);
  \node[font=\DiagramLabelFont,anchor=west] at (7.77,.07)
    {cell representative};
\end{tikzpicture}

\caption{Good cells and the trifling strips $\Omega_{K,\delta}$.}
\label{fig:trifling-region}
\end{figure}
These are the strips of width $\delta$ immediately to the left of the grid
boundaries, including the boundary at $1$.  In particular, a point with
$x_j=1$ lies in the trifling region, so the floor formula below is used only
when every grid index belongs to $\{0,\ldots,K-1\}$.  The one-dimensional
intervals are disjoint.  Taking the product of their complements gives
\begin{equation*}
\begin{aligned}
\mathbf{m}\left(\Omega_{K,\delta}([0,1]^d)\right)
&=1-(1-K\delta)^d
 =K\delta\sum_{j=0}^{d-1}(1-K\delta)^j\le dK\delta.
\end{aligned}
\end{equation*}
Figure~\ref{fig:trifling-region} shows these sets in one and two
dimensions.  The next network combines the $d$ cell indices
$\lfloor Kx_j\rfloor$ into one base-$K$ integer, then rescales it to the
input needed by Proposition~\ref{prop:decoder}.

\begin{proposition}[Grid encoder]\label{prop:grid-encoder}
Let $d,K\in\mathbb N$, $K\ge2$, and $\delta\in(0,1/K)$.  There is a $\DTA$
network $\Lambda_{K,\delta}:\R^d\to\R$ with hidden-layer widths
$\langle 2d+1,d+1\rangle$.  It is negative on $[0,1]^d$, and outside the trifling region,
\[
\Lambda_{K,\delta}(\bm x)
=-\frac1{K^d+1}
  \left(1+\sum_{j=1}^dK^{j-1}\lfloor Kx_j\rfloor\right).
\]
The network has at most $8d+3$ nonzero parameters.  Its parameter radius is at
most $4\max\{K,(K\delta)^{-1}\}$.
\end{proposition}

\begin{proof}
We first construct a scalar staircase, apply it to each input coordinate, and combine the
resulting linear terms in one shared coordinate.

\par\medskip\noindent
\textbf{Step 1: the one-dimensional staircase.}
We first demonstrate how the staircase function can be exactly realized as a very small $\DTA$-network, which is important for the overall construction of both grid encoders. Let
\begingroup
\predisplaypenalty=0
\[
\bm W_1=
\begin{bmatrix}1\\K\\K\end{bmatrix},
\quad
\bm b_1=
\begin{bmatrix}0\\0\\K\delta\end{bmatrix},
\quad
\bm W_2=
\begin{bmatrix}
0&1-\dfrac1{2K\delta}&\dfrac1{2K\delta}\\[1mm]
\dfrac K{K+1}&0&0
\end{bmatrix},
\quad
\bm b_2=
\begin{bmatrix}\dfrac12\\[1mm]0\end{bmatrix}.
\]
\endgroup
Define the affine maps $\mathcal A_1:\R\to\R^3$,
$\mathcal A_2:\R^3\to\R^2$, and
$\mathcal A_3^{\mathrm{stair}}:\R^2\to\R$
by $\mathcal A_1(x)=\bm W_1x+\bm b_1$,
$\mathcal A_2(\bm h)=\bm W_2\bm h+\bm b_2$, and
$\mathcal A_3^{\mathrm{stair}}(\bm u)=u_1+(K+1)u_2-1$,
respectively, and define the staircase
\[
S_{K,\delta}
=\mathcal A_3^{\mathrm{stair}}\circ\DTA\circ\mathcal A_2
\circ\DTA\circ\mathcal A_1.
\]
We claim that $S_{K,\delta}(x)\in[0,K]$ for every $x\in[0,1]$
and that $S_{K,\delta}(x)=\lfloor Kx\rfloor$ whenever
$x\notin\bigcup_{k=1}^{K}
\bigl(\frac{k}{K}-\delta,\frac{k}{K}\bigr]$
(see Figure~\ref{fig:staircase-function}).

\begin{figure}[htbp]
\centering
\includegraphics[width=\FigureStaircaseWidth]{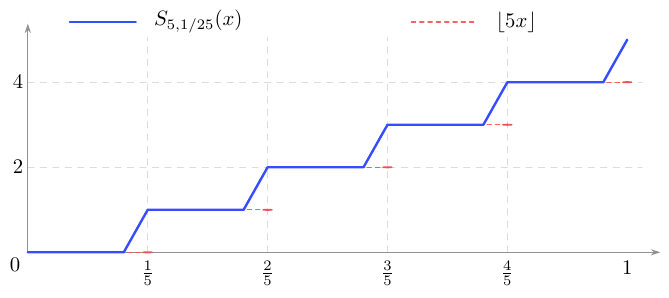}
\caption{The exact staircase away from its trifling intervals.}
\label{fig:staircase-function}
\end{figure}

Since $\tau$ is $1$-Lipschitz and takes values in $[0,1]$,
\[
\frac12+
\frac{\tau(Kx+K\delta)-\tau(Kx)}{2K\delta}
\in[0,1],\qquad x\in[0,1].
\]
Using the activation identities in
\eqref{eq:dta-identities}, we explicitly compute
\begin{align*}
S_{K,\delta}(x)
&=\mathcal A_3^{\mathrm{stair}}\Bigl(
\DTA\bigl(\mathcal A_2(
[\DTA(x),\DTA(Kx),\DTA(Kx+K\delta)]^\top
)\bigr)\Bigr)\\
&=\mathcal A_3^{\mathrm{stair}}\Bigl(
\DTA\bigl(
\bm W_2[x,\tau(Kx),\tau(Kx+K\delta)]^\top
+\bm b_2\bigr)\Bigr)\\
&=\mathcal A_3^{\mathrm{stair}}\left(
\Bigl[
\DTA\Bigl(
\bigl(1-\tfrac1{2K\delta}\bigr)\tau(Kx)
+\tfrac{\tau(Kx+K\delta)}{2K\delta}+\tfrac12
\Bigr),
\DTA\bigl(\tfrac{Kx}{K+1}\bigr)
\Bigr]^\top\right)\\
&=\DTA\Bigl(
\underbrace{\tau(Kx)+\frac12+
\frac{\tau(Kx+K\delta)-\tau(Kx)}{2K\delta}}
_{\in[0,2]}
\Bigr)
+(K+1)\DTA\Bigl(
\underbrace{\frac{Kx}{K+1}}_{\in[0,1]}
\Bigr)-1\\
&=Kx+\tau\left(
\tau(Kx)+\frac12+
\frac{\tau(Kx+K\delta)-\tau(Kx)}{2K\delta}
\right)-1.
\end{align*}
For $x\notin\bigcup_{k=1}^{K}
\bigl(\frac{k}{K}-\delta,\frac{k}{K}\bigr]$, denote
$r=Kx-\lfloor Kx\rfloor\in[0,1-K\delta]$.
By the $2$-periodicity of $\tau$ and its explicit form on
$[0,2]$, we can further compute
\begin{align*}
S_{K,\delta}(x)
&=Kx+
\begin{cases}
\displaystyle
\tau\Bigl(r+\frac12+
\frac{(r+K\delta)-r}{2K\delta}\Bigr)-1,
&\lfloor Kx\rfloor\text{ even},\\[3mm]
\displaystyle
\tau\Bigl(1-r+\frac12+
\frac{(1-r-K\delta)-(1-r)}{2K\delta}\Bigr)-1,
&\lfloor Kx\rfloor\text{ odd}
\end{cases}\\
&=Kx+
\begin{cases}
\tau\bigl(\underbrace{1+r}_{\in[1,2]}\bigr)-1,
&\lfloor Kx\rfloor\text{ even},\\[2mm]
\tau\bigl(\underbrace{1-r}_{\in[0,1]}\bigr)-1,
&\lfloor Kx\rfloor\text{ odd}
\end{cases}\\
&=Kx-r=\lfloor Kx\rfloor.
\end{align*}
Therefore
\begin{equation}\label{eq:staircase-exact}
S_{K,\delta}(x)=\lfloor Kx\rfloor,\qquad
x\in[0,1]\setminus\bigcup_{k=1}^{K}
\Bigl(\frac{k}{K}-\delta,\frac{k}{K}\Bigr].
\end{equation}

To show the first assertion, notice that for each
$j=1,\ldots,K$ and $s\in[0,1]$, the
$2$-periodicity of $\tau$ similarly gives
\begin{align*}
S_{K,\delta}\Bigl(\frac jK-\delta s\Bigr)
&=j-K\delta s+
\begin{cases}
\displaystyle
\tau\left(K\delta s+\frac12+
\frac{K\delta(1-s)-K\delta s}{2K\delta}\right)-1,
&j\text{ even},\\[3mm]
\displaystyle
\tau\left(1-K\delta s+\frac12+
\frac{(1-K\delta(1-s))-(1-K\delta s)}
{2K\delta}\right)-1,
&j\text{ odd}
\end{cases}\\
&=j-K\delta s+
\begin{cases}
\tau\bigl(\underbrace{1-(1-K\delta)s}_{\in[0,1]}\bigr)-1,
&j\text{ even},\\[2mm]
\tau\bigl(\underbrace{1+(1-K\delta)s}_{\in[1,2]}\bigr)-1,
&j\text{ odd}
\end{cases}\\
&=j-K\delta s-(1-K\delta)s\\
&=j-s\in[j-1,j].
\end{align*}
Since $\lfloor Kx\rfloor\in[0,K]$ for every $x\in[0,1]$,
\eqref{eq:staircase-exact} gives the desired bound outside
the trifling intervals. On the other hand, on each trifling interval
$(j/K-\delta,j/K]$, the above computation shows that
$S_{K,\delta}(x)\in[j-1,j]\subset[0,K]$.
Thus $S_{K,\delta}(x)\in[0,K]$ for every $x\in[0,1]$,
as claimed.

\par\medskip\noindent
\textbf{Step 2: construction of the grid encoder.}
Recall that we aim to construct a $\DTA$-network
$\Lambda_{K,\delta}:\R^d\to\R$ with hidden-layer widths
$\langle 2d+1,d+1\rangle$ satisfying $\Lambda_{K,\delta}(\bm x)<0$
for every $\bm x\in[0,1]^d$ and
\[
\Lambda_{K,\delta}(\bm x)
=-\frac{1+\sum_{j=1}^dK^{j-1}\lfloor Kx_j\rfloor}{K^d+1},
\qquad
\bm x\in[0,1]^d\setminus\Omega_{K,\delta}([0,1]^d).
\]
To make the expression cleaner, we denote
$D=\sum_{j=0}^dK^j$
and $\bm\gamma^\top=D^{-1}[K,K^2,\ldots,K^d]$.
Define the affine maps
$\mathcal A_1^{(d)}:\R^d\to\R^{2d+1}$,
$\mathcal A_2^{(d)}:\R^{2d+1}\to\R^{d+1}$, and
$\mathcal A_3^{(d)}:\R^{d+1}\to\R$ by
\begingroup
\EncoderMatrixLayout{}
\begin{align}
&\mathcal A_1^{(d)}(\bm x)
=\begin{bmatrix}
K\bm I_d\\ K\bm I_d\\ \bm\gamma^\top
\end{bmatrix}\bm x
+\begin{bmatrix}
\bzero_d\\K\delta\bmone_d\\0
\end{bmatrix},
\quad
\mathcal A_2^{(d)}(\bm h)
=\begin{bmatrix}
\bigl(1-(2K\delta)^{-1}\bigr)\bm I_d&
(2K\delta)^{-1}\bm I_d&\bzero_d\\
\bzero_d^\top&\bzero_d^\top&1
\end{bmatrix}\bm h
+\begin{bmatrix}\frac{1}{2}\bmone_d\\0\end{bmatrix},
\notag\\[1mm]
&\mathcal A_3^{(d)}(\bm u)
=-(K^d+1)^{-1}
\Bigl(
\begin{bmatrix}
1&K&\cdots&K^{d-1}&D
\end{bmatrix}\bm u
-\sum_{j=1}^dK^{j-1}+1
\Bigr).
\label{eq:compact-encoder-matrices}
\end{align}
\endgroup
In this explicit way, we can easily track the bound of their individual parameters later. We then define
\[
\Lambda_{K,\delta}
=\mathcal A_3^{(d)}\circ\DTA\circ\mathcal A_2^{(d)}
\circ\DTA\circ\mathcal A_1^{(d)}
\]
(see Figure~\ref{fig:parallel-encoder} for the case $d=2$).

\begin{figure}[!t]
\centering
\includegraphics[width=\FigureEncoderWidth]{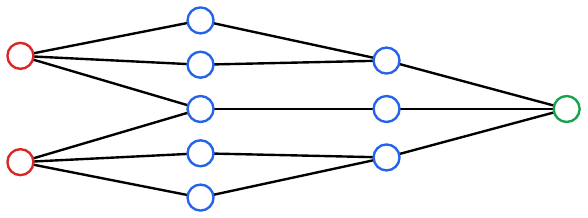}
\caption{The grid encoder $\Lambda_{K,\delta}$ for $d=2$.}
\label{fig:parallel-encoder}
\end{figure}

Note that $\bm\gamma^\top\bm x\in[0, 1-D^{-1}]\subset[0,1)$
for every $\bm x\in[0,1]^d$. Using the activation identities
in \eqref{eq:dta-identities}, we have
\begin{align*}
\DTA(\mathcal A_1^{(d)}(\bm x))
&=\bigl[
(\DTA(Kx_j))_{j=1}^d,\,
(\DTA(Kx_j+K\delta))_{j=1}^d,\,
\DTA(\bm\gamma^\top\bm x)
\bigr]^\top\\
&=\bigl[
(\tau(Kx_j))_{j=1}^d,\,
(\tau(Kx_j+K\delta))_{j=1}^d,\,
\bm\gamma^\top\bm x
\bigr]^\top.
\end{align*}
Applying $\mathcal A_2^{(d)}$ followed by $\DTA$ and using
the expression for $S_{K,\delta}$ obtained in Step~1, we obtain
\begin{align*}
&\DTA(\mathcal A_2^{(d)}(\DTA(\mathcal A_1^{(d)}(\bm x))))
\notag\\
&\quad=\Bigl[
\Bigl(
\DTA\bigl(
\underbrace{
\bigl(1-(2K\delta)^{-1}\bigr)\tau(Kx_j)
+(2K\delta)^{-1}\tau(Kx_j+K\delta)+\tfrac12
}_{\in[0,2]}
\bigr)
\Bigr)_{j=1}^d,\,
\DTA(\bm\gamma^\top\bm x)
\Bigr]^\top
\notag\\
&\quad=\Bigl[
\Bigl(
\tau\bigl(
\tau(Kx_j)+\tfrac12+
(2K\delta)^{-1}\bigl(\tau(Kx_j+K\delta)-\tau(Kx_j)\bigr)
\bigr)
\Bigr)_{j=1}^d,\,
\bm\gamma^\top\bm x
\Bigr]^\top
\notag\\
&\quad=\Bigl[
(S_{K,\delta}(x_j)-Kx_j+1)_{j=1}^d,\,
D^{-1}\sum_{j=1}^dK^jx_j
\Bigr]^\top.
\end{align*}
Consequently, for every $\bm x\in[0,1]^d$,
\begin{align*}
\Lambda_{K,\delta}(\bm x)
&=\mathcal A_3^{(d)}(
\DTA(\mathcal A_2^{(d)}(\DTA(\mathcal A_1^{(d)}(\bm x)))))
=\mathcal A_3^{(d)}\Bigl(
\bigl[
(S_{K,\delta}(x_j)-Kx_j+1)_{j=1}^d,\,
D^{-1}\sum_{j=1}^dK^jx_j
\bigr]^\top
\Bigr)
\notag\\
&=-\frac{1}{K^d+1}\Bigl(
\sum_{j=1}^dK^jx_j
+\sum_{j=1}^dK^{j-1}
\bigl(S_{K,\delta}(x_j)-Kx_j+1\bigr)
-\sum_{j=1}^dK^{j-1}+1
\Bigr)
\notag\\
&=-\frac{1}{K^d+1}\Bigl(
1+\sum_{j=1}^dK^{j-1}S_{K,\delta}(x_j)
+\sum_{j=1}^dK^jx_j-\sum_{j=1}^dK^jx_j
\Bigr)
\notag\\
&=-\frac{1}{K^d+1}\Big(1+\sum_{j=1}^dK^{j-1}S_{K,\delta}(x_j)\Big).
\end{align*}
Since $S_{K,\delta}(x_j)\in[0,K]$ by Step~1,
we have $\Lambda_{K,\delta}(\bm x)
\le-(K^d+1)^{-1}<0$ on $[0,1]^d$.
Moreover, by \eqref{eq:staircase-exact},
\[
\Lambda_{K,\delta}(\bm x)
=-\frac{1+\sum_{j=1}^dK^{j-1}S_{K,\delta}(x_j)}{K^d+1}
=-\frac{1+\sum_{j=1}^dK^{j-1}\lfloor Kx_j\rfloor}{K^d+1},
\quad
\bm x\in[0,1]^d\setminus\Omega_{K,\delta}([0,1]^d).
\]

Lastly, the affine maps in \eqref{eq:compact-encoder-matrices}
contain at most $4d$, $3d+1$, and $d+2$ nonzero parameters,
respectively, giving a total of at most $8d+3$.
Since $K\ge2$ and $K\delta<1$, we have
$\parnorm{\mathcal A_1^{(d)}}\le K$ and
$\parnorm{\mathcal A_2^{(d)}}\le\max\{1,(2K\delta)^{-1}\}$.
Moreover, the absolute value of the bias in $\mathcal A_3^{(d)}$
is $(K^d+1)^{-1}(\sum_{j=1}^dK^{j-1}-1)<1$,
so that $\parnorm{\mathcal A_3^{(d)}}=D(K^d+1)^{-1}<2$.
Hence
\[
\parnorm{\Lambda_{K,\delta}}
\le\max\{K,(2K\delta)^{-1},2\}
=\max\{K,(2K\delta)^{-1}\}
\le4\max\{K,(K\delta)^{-1}\}.
\]
This completes the proof.
\end{proof}

\subsection{Proof of Theorem~\ref{thm:approximation}}\label{sec:proof-first}
The encoder and decoder approximate the prescribed cell value outside the
trifling region. We choose the grid size and table accuracy to bound this
error, and the strip width to control the remaining $L^p$ contribution.
The affine maps then give the parameter count and radius of the composition.

\begin{proof}[Proof of Theorem~\ref{thm:approximation}]
Choose $K=\left\lceil\left(2^{1+1/p}/\varepsilon\right)^{1/\beta}\right\rceil$ and $\delta=\varepsilon^p/(2^{p+1}dK)$.
Then $K\ge2$, $\delta\in(0,1/K)$, and
\begin{equation}\label{eq:approximation-error-scales}
K^{-\beta}\le2^{-1-1/p}\varepsilon,
\qquad
\mathbf{m}\bigl(\Omega_{K,\delta}([0,1]^d)\bigr)
\le dK\delta=\frac{\varepsilon^p}{2^{p+1}}.
\end{equation}
The first inequality controls the oscillation in one grid cell, and the
second makes the total trifling region small.  Let
$\mathcal I_K=\{0,\ldots,K-1\}^d$ and set \[
J_r=
\begin{cases}
[r/K,(r+1)/K),&r=0,\ldots,K-2,\\
[(K-1)/K,1],&r=K-1.
\end{cases}
\] For $\bm\ell=(\ell_1,\ldots,\ell_d)^\top\in\mathcal I_K$, define
\begin{equation*}
Q_{\bm\ell}=\prod_{j=1}^dJ_{\ell_j},
\qquad
\bm x_{\bm\ell}=\frac1K\left(\bm\ell+\frac12\bmone_d\right),
\qquad
I(\bm\ell)=1+\sum_{j=1}^dK^{j-1}\ell_j.
\end{equation*}
The cells $Q_{\bm\ell}$ form a disjoint partition of $[0,1]^d$, and
$I$ assigns them distinct indices in $\{1,\ldots,K^d\}$. To store the
value at each cell center, apply Proposition~\ref{prop:decoder} with $N=K^d$,
accuracy $2^{-1-1/p}\varepsilon$, and
$c_{I(\bm\ell)}=f(\bm x_{\bm\ell})$.  It gives a decoder
$\mathcal R=\mathcal B_2\circ\DTA\circ \mathcal B_1$ such that
\begin{equation}\label{eq:decoder-table}
\max_{\bm\ell}
\left|\mathcal R\!\left(2^{-I(\bm\ell)/(K^d+1)}-1\right)
-f(\bm x_{\bm\ell})\right|\le2^{-1-1/p}\varepsilon,
\qquad
\sup_{v\in(-1,1]}|\mathcal R(v)|\le1.
\end{equation}
Let $\Lambda_{K,\delta}=\mathcal A_3^{(d)}\circ\DTA\circ\mathcal A_2^{(d)}\circ\DTA\circ\mathcal A_1^{(d)}$
be the encoder from Proposition~\ref{prop:grid-encoder} and define
\begin{equation}\label{eq:compact-network-composition}
\Phi:=\mathcal R\circ\DTA\circ\Lambda_{K,\delta}
=\mathcal B_2\circ\DTA\circ \mathcal B_1\circ\DTA\circ\mathcal A_3^{(d)}
 \circ\DTA\circ\mathcal A_2^{(d)}\circ\DTA\circ\mathcal A_1^{(d)}.
\end{equation}
The network $\Phi$ has hidden-layer widths $\langle 2d+1,d+1,1,1\rangle$.
Figure~\ref{fig:full-network}
shows the composition for $d=3$: the first scalar node after the encoder
forms the dyadic code, and the last hidden node decodes the table value.

\begin{figure}[htbp]
\centering
\includegraphics[width=\FigureNetworkWidth]{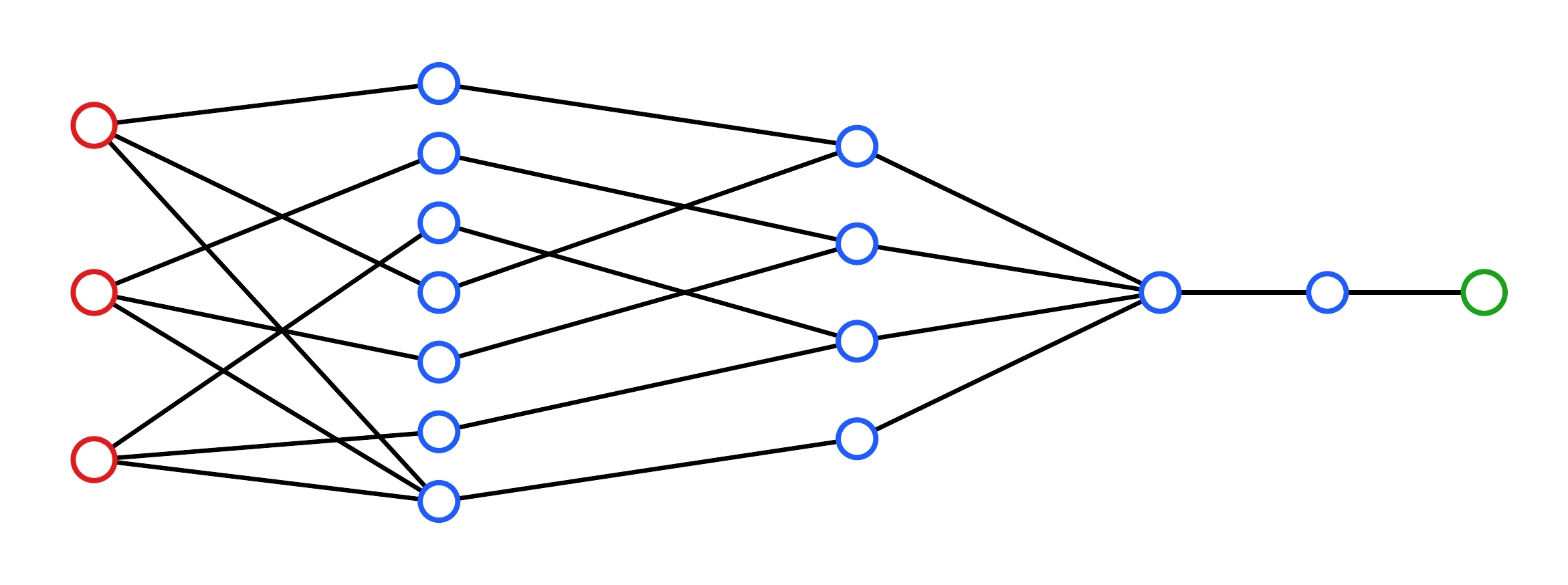}
\caption{The four-hidden-layer network $\Phi$ for $d=3$.}
\label{fig:full-network}
\end{figure}

Fix $\bm x\in Q_{\bm\ell}\setminus\Omega_{K,\delta}([0,1]^d)$.  Then
$\lfloor Kx_j\rfloor=\ell_j$,
$\norm{\bm x-\bm x_{\bm\ell}}_\infty\le1/(2K)$, and
\[
\Lambda_{K,\delta}(\bm x)=-\frac{I(\bm\ell)}{K^d+1}<0,
\qquad
\Phi(\bm x)=\mathcal R\!\left(2^{-I(\bm\ell)/(K^d+1)}-1\right).
\]
The H\"older condition, \eqref{eq:approximation-error-scales}, and
\eqref{eq:decoder-table} therefore give
\[
|f(\bm x)-\Phi(\bm x)|
\le|f(\bm x)-f(\bm x_{\bm\ell})|+|f(\bm x_{\bm\ell})-\Phi(\bm x)|
\le K^{-\beta}+2^{-1-1/p}\varepsilon
\le2^{-1/p}\varepsilon.
\]
On the whole cube, $\Lambda_{K,\delta}<0$, so
$-1<\DTA(\Lambda_{K,\delta})<0$ and $|\Phi|\le1$ by
\eqref{eq:decoder-table}. Writing
$\Omega=\Omega_{K,\delta}([0,1]^d)$ and splitting the integral over
$\Omega$ and its complement gives
\begin{align*}
\|f-\Phi\|_{L^p([0,1]^d)}^p
&=\int_{[0,1]^d\setminus\Omega}|f(\bm x)-\Phi(\bm x)|^p\,\dd\bm x
 +\int_\Omega|f(\bm x)-\Phi(\bm x)|^p\,\dd\bm x\\
&\le\frac{\varepsilon^p}{2}+2^p\mathbf{m}(\Omega)
\le\frac{\varepsilon^p}{2}+\frac{\varepsilon^p}{2}=\varepsilon^p.
\end{align*}
This argument uses only the $p$th power of the quasi-norm and is therefore
valid for every $p\in(0,\infty)$.

It remains to bound the parameters in \eqref{eq:compact-network-composition}.
Since $(K\delta)^{-1}=2^{p+1}d\varepsilon^{-p}$, Propositions~\ref{prop:decoder}
and~\ref{prop:grid-encoder} give
\begin{align*}
\log T_\varepsilon
&\le\max\Bigg\{
\log\left(4\max\{K,(K\delta)^{-1}\}\right),
\log6+K^d\log\left(40K^{3d/2}2^{1+1/p}\varepsilon^{-1}\right)
\Bigg\}\\
&\le CK^d\{1+\log K+\log(\varepsilon^{-1})\}
\le C_0\varepsilon^{-d/\beta}\log(\varepsilon^{-1}),
\end{align*}
where $K\le C\varepsilon^{-1/\beta}$ and
$\log K\le C\log(\varepsilon^{-1})$; here $C>0$ depends only on
$d$, $\beta$, and $p$.
The encoder and decoder use at most $8d+3$ and $4$ nonzero affine parameters,
respectively, so the network has at most $8d+7$ nonzero parameters.  Taking $\phi=\Phi$ completes the proof.
\end{proof}

\subsection{Block-grid encoding and coupled decoding}\label{sec:block-grid-encoder}

To exploit neighboring values, divide the ordered grid into blocks of
$\Gamma$ cells.  The quotient and remainder of a grid index on division by
$\Gamma$ give its block index and position within the block.  Staircases at
frequencies $H$ and $H/\Gamma$ compute these two quantities; the next
proposition gives their normalized form.

\begin{proposition}[Block-grid encoder]\label{prop:block-grid-encoder}
Let $d,H,\Gamma\in\mathbb N$, $\Gamma\mid H$, $1\le \Gamma\le H/2$, and
$\delta\in(0,1/H)$.  There is a $\DTA$ network
$\mathcal E=(\mathcal E_1,\mathcal E_2):\R^d\to\R^2$ with hidden-layer
widths $\langle 2d+3,d+2\rangle$ such that, for $\bm x\in[0,1]^d$,
\[
\mathcal E_1(\bm x)\ge0,\qquad
\mathcal E_2(\bm x)\in\left[(2\Gamma)^{-1}, 1-(2\Gamma)^{-1}\right],
\]
and, whenever $\bm x\notin\Omega_{H,\delta}([0,1]^d)$ and
$\ell_j=\lfloor Hx_j\rfloor$ for the integers $j=1,...,d$,
\[
\mathcal E(\bm x)=
\left(
\left\lfloor\frac{\ell_1+H\ell_2+\cdots+H^{d-1}\ell_d}{\Gamma}\right\rfloor,
\frac{\ell_1-\Gamma\lfloor\ell_1/\Gamma\rfloor+1/2}{\Gamma}
\right).
\]
The network has at most $8d+12$ nonzero parameters, and its parameter radius is at most $\max\{4H^d,\Gamma/(2H\delta)\}$.
\end{proposition}

\begin{proof}
Let $H,\Gamma\in\mathbb N$ with $\Gamma\mid H$ and $1\le \Gamma\le H/2$,
and let $\delta\in(0,1/H)$ as in the hypothesis of this proposition.
We begin by studying the relation between
$S_{H,\delta}$ and $S_{H/\Gamma,\delta}$.
\par\medskip\noindent
\textbf{Step 1: the relation between two staircases.}
We first show that
$S_{H,\delta}(x)-\Gamma S_{H/\Gamma,\delta}(x)\in[0,\Gamma-1]$
for every $x\in[0,1]$.
Recall from the proof of Proposition~\ref{prop:grid-encoder} that $S_{K,\delta}(x)\in[0,K]$
for every $x\in[0,1]$ and that
$S_{K,\delta}(j/K-\delta s)=j-s$ for $j=1,\ldots,K$ and $s\in[0,1]$.
Since $\Gamma\mid H$, every trifling interval of $S_{H/\Gamma,\delta}$ is also
a trifling interval of $S_{H,\delta}$. Hence \eqref{eq:staircase-exact}
gives
\begin{equation}\label{eq:fine-coarse-exact}
S_{H,\delta}(x)=\lfloor Hx\rfloor,\qquad
S_{H/\Gamma,\delta}(x)=\left\lfloor\frac{Hx}{\Gamma}\right\rfloor,
\qquad x\in[0,1]\setminus\Omega_{H,\delta}([0,1]).
\end{equation}
Since $Hx-\Gamma\lfloor Hx/\Gamma\rfloor\in[0,\Gamma)$, we have
\begin{align*}
S_{H,\delta}(x)-\Gamma S_{H/\Gamma,\delta}(x)
=\lfloor Hx\rfloor-\Gamma\left\lfloor\frac{Hx}{\Gamma}\right\rfloor
=\left\lfloor Hx-\Gamma\left\lfloor\frac{Hx}{\Gamma}\right\rfloor\right\rfloor
\in\{0,1,\ldots,\Gamma-1\}
\end{align*}
outside $\Omega_{H,\delta}([0,1])$.

Now consider a trifling interval $(k/H-\delta,k/H]$,
where $k=1,\ldots,H$, and write $k=\Gamma q+r$ with $q,r\in\Z$
and $r=0,\ldots,\Gamma-1$. If $r=0$, this is also a trifling interval
of $S_{H/\Gamma,\delta}$; if $r=1,\ldots,\Gamma-1$, then $H\delta<1$ gives
$[k/H-\delta,k/H]\subset(\Gamma q/H,\Gamma(q+1)/H-\delta]$,
on which $S_{H/\Gamma,\delta}$ is identically equal to $q$.
Thus, for every $s\in[0,1]$, we obtain
\begin{align*}
S_{H,\delta}\Bigl(\frac{k}{H}-\delta s\Bigr)
-\Gamma S_{H/\Gamma,\delta}\Bigl(\frac{k}{H}-\delta s\Bigr)
&=
\begin{cases}
(\Gamma q-s)-\Gamma(q-s),&r=0,\\
(\Gamma q+r-s)-\Gamma q,&r=1,\ldots,\Gamma-1
\end{cases}\\
&=
\begin{cases}
(\Gamma-1)s,&r=0,\\
r-s,&r=1,\ldots,\Gamma-1
\end{cases}
\in[0,\Gamma-1],
\end{align*}
where the last inclusion follows from $(\Gamma-1)s\in[0,\Gamma-1]$
and $r-s\in[r-1,r]\subset[0,\Gamma-1]$ in the respective cases.
Since the above computations cover both the trifling intervals
and their complement in $[0,1]$, we conclude that
\begin{equation}\label{eq:fine-coarse-range}
S_{H,\delta}(x)-\Gamma S_{H/\Gamma,\delta}(x)\in[0,\Gamma-1],
\qquad x\in[0,1].
\end{equation}

\par\medskip\noindent
\textbf{Step 2: construction of the block-grid encoder.}
Recall that we aim to construct a $\DTA$-network
$\mathcal E=(\mathcal E_1,\mathcal E_2):\R^d\to\R^2$
with hidden-layer widths $\langle 2d+3,d+2\rangle$ satisfying
\[
\mathcal E_1(\bm x)\ge0,\qquad
\mathcal E_2(\bm x)\in\left[(2\Gamma)^{-1}, 1-(2\Gamma)^{-1}\right],
\qquad \bm x\in[0,1]^d,
\]
and the asserted expressions for its two coordinates outside
$\Omega_{H,\delta}([0,1]^d)$.

Let $\bm e_1=(1,0,\ldots,0)^\top\in\R^d$ and
$\bm\omega=(1,H,\ldots,H^{d-1})^\top\in\R^d$.
Define the affine maps
$\mathcal A_1^{\mathrm{bl}}:\R^d\to\R^{2d+3}$,
$\mathcal A_2^{\mathrm{bl}}:\R^{2d+3}\to\R^{d+2}$, and
$\mathcal A_3^{\mathrm{bl}}:\R^{d+2}\to\R^2$ by
\begingroup
\setlength{\arraycolsep}{6pt}
\renewcommand{\arraystretch}{1.2}
\begin{equation}\label{eq:block-encoder-A1-matrix}
\mathcal A_1^{\mathrm{bl}}(\bm x)
=\begin{bmatrix}
H\bm I_d\\
H\bm I_d\\
(H/\Gamma)\,\bm e_1^\top\\
(H/\Gamma)\,\bm e_1^\top\\
(\bmone_d^\top\bm\omega)^{-1}\,\bm\omega^\top
\end{bmatrix}\bm x
+\begin{bmatrix}
\bzero_d\\
H\delta\,\bmone_d\\
0\\
H\delta/\Gamma\\
0
\end{bmatrix},
\end{equation}
\begin{equation*}
\mathcal A_2^{\mathrm{bl}}(\bm z_1)
=\begin{bmatrix}
\bigl(1-(2H\delta)^{-1}\bigr)\,\bm I_d&
(2H\delta)^{-1}\,\bm I_d&
\bzero_d&\bzero_d&\bzero_d\\
\bzero_d^\top&\bzero_d^\top&
1-\dfrac{\Gamma}{2H\delta}&\dfrac{\Gamma}{2H\delta}&0\\
\bzero_d^\top&\bzero_d^\top&0&0&1
\end{bmatrix}\bm z_1
+\begin{bmatrix}
\frac12\bmone_d\\
\frac12\\
0
\end{bmatrix},
\end{equation*}
and
\begin{equation}\label{eq:block-encoder-A3-matrix}
\mathcal A_3^{\mathrm{bl}}(\bm z_2)
=\begin{bmatrix}
\Gamma^{-1}(\bm\omega-\bm e_1)^\top&
1&H\Gamma^{-1}\,\bmone_d^\top\bm\omega\\
\Gamma^{-1}\bm e_1^\top&-1&0
\end{bmatrix}\bm z_2
+\begin{bmatrix}
-1-\Gamma^{-1}(\bmone_d^\top\bm\omega-1)\\
1-(2\Gamma)^{-1}
\end{bmatrix}.
\end{equation}
\endgroup
We then define
\[
\mathcal E
=\mathcal A_3^{\mathrm{bl}}\circ\DTA
\circ\mathcal A_2^{\mathrm{bl}}\circ\DTA
\circ\mathcal A_1^{\mathrm{bl}},
\]
which has the desired hidden-layer widths $\langle 2d+3,d+2\rangle$.

We are now ready to explicitly compute the output of this network. Note that $(\bmone_d^\top\bm\omega)^{-1}\bm\omega^\top\bm x\in[0,1]$
for every $\bm x\in[0,1]^d$. Using the activation identities in
\eqref{eq:dta-identities}, we obtain
\begin{align*}
&\DTA(\mathcal A_1^{\mathrm{bl}}(\bm x))
=\\
&\Bigl[
(\DTA(Hx_j))_{j=1}^d,\,
(\DTA(Hx_j+H\delta))_{j=1}^d,\,
\DTA(Hx_1/\Gamma),\,
\DTA(H(x_1+\delta)/\Gamma),
\DTA\Bigl(
\underbrace{
(\bmone_d^\top\bm\omega)^{-1}\bm\omega^\top\bm x
}_{\in[0,1]}
\Bigr)
\Bigr]^\top\\
&\quad=\Bigl[
(\tau(Hx_j))_{j=1}^d,\,
(\tau(Hx_j+H\delta))_{j=1}^d,\,
\tau(Hx_1/\Gamma),\,
\tau(H(x_1+\delta)/\Gamma),\,
(\bmone_d^\top\bm\omega)^{-1}\bm\omega^\top\bm x
\Bigr]^\top.
\end{align*}
Applying $\mathcal A_2^{\mathrm{bl}}$ followed by $\DTA$,
and using the $1$-Lipschitz continuity of $\tau$ as in the
construction of $S_{K,\delta}$, we further compute
\begin{align*}
&\DTA(\mathcal A_2^{\mathrm{bl}}
(\DTA(\mathcal A_1^{\mathrm{bl}}(\bm x))))\\
&\quad=\Bigl[
\Bigl(
\DTA\Bigl(
\underbrace{
\bigl(1-(2H\delta)^{-1}\bigr)\tau(Hx_j)
+(2H\delta)^{-1}\tau(Hx_j+H\delta)+\tfrac12
}_{\in[0,2]}
\Bigr)
\Bigr)_{j=1}^d,\\
&\qquad\quad
\DTA\Bigl(
\underbrace{
\bigl(1-\tfrac{\Gamma}{2H\delta}\bigr)\tau(Hx_1/\Gamma)
+\tfrac{\Gamma}{2H\delta}\tau(H(x_1+\delta)/\Gamma)+\tfrac12
}_{\in[0,2]}
\Bigr),
\DTA\bigl((\bmone_d^\top\bm\omega)^{-1}\bm\omega^\top\bm x\bigr)
\Bigr]^\top\\
&\quad=\Bigl[
\Bigl(
\tau\Bigl(
\tau(Hx_j)+\tfrac12+
(2H\delta)^{-1}\bigl(\tau(Hx_j+H\delta)-\tau(Hx_j)\bigr)
\Bigr)
\Bigr)_{j=1}^d,\\
&\qquad\quad
\tau\Bigl(
\tau(Hx_1/\Gamma)+\tfrac12+
\Gamma(2H\delta)^{-1}
\bigl(\tau(H(x_1+\delta)/\Gamma)-\tau(Hx_1/\Gamma)\bigr)
\Bigr),
(\bmone_d^\top\bm\omega)^{-1}\bm\omega^\top\bm x
\Bigr]^\top\\
&\quad=\Bigl[
(S_{H,\delta}(x_j)-Hx_j+1)_{j=1}^d,\,
S_{H/\Gamma,\delta}(x_1)-H\Gamma^{-1}x_1+1,\,
(\bmone_d^\top\bm\omega)^{-1}\bm\omega^\top\bm x
\Bigr]^\top.
\end{align*}
where the last equality follows from the expression for
$S_{K,\delta}$ obtained in the proof of Proposition~\ref{prop:grid-encoder}, applied
with $K=H$ and $K=H/\Gamma$, respectively.
Consequently, applying $\mathcal A_3^{\mathrm{bl}}$ to the
above vector and using
$\bm\omega^\top\bm x=\sum_{j=1}^dH^{j-1}x_j$, we obtain
\begin{align*}
\mathcal E_1(\bm x)
&=\Gamma^{-1}\sum_{j=2}^dH^{j-1}
\bigl(S_{H,\delta}(x_j)-Hx_j+1\bigr)
+S_{H/\Gamma,\delta}(x_1)-H\Gamma^{-1}x_1+1\\
&\quad+H\Gamma^{-1}(\bmone_d^\top\bm\omega)
(\bmone_d^\top\bm\omega)^{-1}\bm\omega^\top\bm x
-1-\Gamma^{-1}(\bmone_d^\top\bm\omega-1)\\
&=\Gamma^{-1}\sum_{j=2}^dH^{j-1}
\bigl(S_{H,\delta}(x_j)-Hx_j\bigr)
+S_{H/\Gamma,\delta}(x_1)-H\Gamma^{-1}x_1
+\Gamma^{-1}\sum_{j=1}^dH^jx_j\\
&=S_{H/\Gamma,\delta}(x_1)
+\Gamma^{-1}\Bigl(
\sum_{j=2}^dH^{j-1}S_{H,\delta}(x_j)
+\sum_{j=1}^dH^jx_j-\sum_{j=1}^dH^jx_j
\Bigr)\\
&=S_{H/\Gamma,\delta}(x_1)
+\Gamma^{-1}\sum_{j=2}^dH^{j-1}S_{H,\delta}(x_j),
\end{align*}
together with
\begin{align*}
\mathcal E_2(\bm x)
&=\Gamma^{-1}\bigl(S_{H,\delta}(x_1)-Hx_1+1\bigr)
-\bigl(S_{H/\Gamma,\delta}(x_1)-H\Gamma^{-1}x_1+1\bigr)
+1-\frac1{2\Gamma}\\
&=\Gamma^{-1}\Bigl(
S_{H,\delta}(x_1)-\Gamma S_{H/\Gamma,\delta}(x_1)
-Hx_1+Hx_1+\frac12
\Bigr)\\
&=\frac{S_{H,\delta}(x_1)-\Gamma S_{H/\Gamma,\delta}(x_1)+1/2}{\Gamma}.
\end{align*}
Here and below, the sum from $j=2$ to $d$ is understood
to be zero when $d=1$. Since $S_{H/\Gamma,\delta}(x_1)\in[0,H/\Gamma]$
and $S_{H,\delta}(x_j)\in[0,H]$, the expression for
$\mathcal E_1$ gives $\mathcal E_1(\bm x)\ge0$.
Moreover, by \eqref{eq:fine-coarse-range} in Step 1,
\[
\mathcal E_2(\bm x)
=\Gamma^{-1}\Bigl(
\underbrace{
S_{H,\delta}(x_1)-\Gamma S_{H/\Gamma,\delta}(x_1)
}_{\in[0,\Gamma-1]}
+\frac12
\Bigr)
\in\left[(2\Gamma)^{-1}, 1-(2\Gamma)^{-1}\right],
\qquad \bm x\in[0,1]^d.
\]

For $\bm x$ outside $\Omega_{H,\delta}([0,1]^d)$,
let $\ell_j=\lfloor Hx_j\rfloor$ for $j=1,\ldots,d$.
Since $\Gamma$ is a positive integer,
$\lfloor Hx_1/\Gamma\rfloor=\lfloor\ell_1/\Gamma\rfloor$, and
since $\Gamma\mid H$, we have
$\Gamma^{-1}\sum_{j=2}^dH^{j-1}\ell_j\in\Z$.
It follows from \eqref{eq:fine-coarse-exact} that
\begin{align*}
\mathcal E_1(\bm x)
&=S_{H/\Gamma,\delta}(x_1)
+\Gamma^{-1}\sum_{j=2}^dH^{j-1}S_{H,\delta}(x_j)\\
&=\left\lfloor\frac{Hx_1}{\Gamma}\right\rfloor
+\Gamma^{-1}\sum_{j=2}^dH^{j-1}\lfloor Hx_j\rfloor\\
&=\left\lfloor\frac{\ell_1}{\Gamma}\right\rfloor
+\Gamma^{-1}\sum_{j=2}^dH^{j-1}\ell_j\\
&=\left\lfloor
\frac{\ell_1+H\ell_2+\cdots+H^{d-1}\ell_d}{\Gamma}
\right\rfloor,
\end{align*}
together with
\begin{align*}
\mathcal E_2(\bm x)
&=\frac{S_{H,\delta}(x_1)-\Gamma S_{H/\Gamma,\delta}(x_1)+1/2}{\Gamma}
=\frac{\lfloor Hx_1\rfloor-\Gamma\lfloor Hx_1/\Gamma\rfloor+1/2}{\Gamma}
=\frac{\ell_1-\Gamma\lfloor\ell_1/\Gamma\rfloor+1/2}{\Gamma},
\end{align*}
as desired.

Lastly, the three affine maps in
\eqref{eq:block-encoder-A1-matrix}--\eqref{eq:block-encoder-A3-matrix}
contain at most $4d+3$, $3d+4$, and $d+5$ nonzero parameters,
respectively, giving a total of at most $8d+12$.
Since $1\le \Gamma\le H/2$ and $0<H\delta<1$, we have
$\parnorm{\mathcal A_1^{\mathrm{bl}}}\le H$ and
$\parnorm{\mathcal A_2^{\mathrm{bl}}}\le\max\{1,\Gamma/(2H\delta)\}$.
Moreover, $\bmone_d^\top\bm\omega=\sum_{j=1}^dH^{j-1}<2H^{d-1}$,
so the weights of $\mathcal A_3^{\mathrm{bl}}$ are bounded
in absolute value by $H\Gamma^{-1}\bmone_d^\top\bm\omega<2H^d$,
and its biases by $\bmone_d^\top\bm\omega<2H^{d-1}$.
In particular, $\parnorm{\mathcal A_3^{\mathrm{bl}}}\le4H^d$.
Hence
\[
\parnorm{\mathcal E}
\le\max\left\{H,1,\frac{\Gamma}{2H\delta},4H^d\right\}
=\max\left\{4H^d,\frac{\Gamma}{2H\delta}\right\}.
\]
This completes the proof.
\end{proof}

The two encoder coordinates give a block index and a normalized position
within that block. For integers $a_t$ with adjacent differences in $\{-1,0,1\}$ within each
row, the value at position $r$ in a block is its starting value plus the
first $r$ increments. We encode these increments as ternary digits. Two
code evaluations share the same error; an affine subtraction cancels it
and gives the exact decoder input for the required partial sum. This cancellation
defines the \emph{coupled} decoder. Part~(a) uses two neurons per hidden
layer, while part~(b) distributes the codes across a square array to reduce
their length.

\begin{proposition}[Coupled block decoder]\label{prop:block-decoder}
Let $H,Q\in\mathbb N$, let $4\le \Upsilon\le H$, assume that $H$ is a power of
two and $H\mid Q$, and suppose that
$a_0,\ldots,a_{Q-1}\in[-\Upsilon,\Upsilon]\cap\mathbb Z$ satisfy
$|a_{t+1}-a_t|\le1$ for $t=0,\ldots,Q-2$ and
$t\not\equiv H-1\pmod H$.  There is a power of two $\Gamma\mid H$ such that
$\Gamma\le\min\{\Upsilon/4,\Upsilon^{1/5}\}$, with the following properties.
\begin{enumerate}
\item[(a)] There is a four-hidden-layer $\DTA$ network
$\mathcal D_2:\R^2\to\R$ whose hidden layers all have width two and such that
\[
\max_{\substack{b,r\in\mathbb Z\\0\le b<Q/\Gamma,\ 0\le r<\Gamma}}
\left|\mathcal D_2\!\left(b,\frac{r+1/2}{\Gamma}\right)
-\frac{a_{b\Gamma+r}}\Upsilon\right|<\frac1{25\Upsilon},\qquad
\sup_{\substack{b\ge-1\\\nu\in[0,1]}}|\mathcal D_2(b,\nu)|
\le1+\frac \Gamma \Upsilon.
\]
Moreover, the network has at most $18$ nonzero parameters and parameter radius
$T_{\mathcal D_2}$ with $\log T_{\mathcal D_2}\le12Q$.

\item[(b)] For every integer $N\ge2$, there is a four-hidden-layer
$\DTA$ network $\mathcal D_N:\R^2\to\R$ whose hidden layers all have width
$N$.  Its table error satisfies the first inequality in part~(a), with $b,r$ integers,
and 
\[
\sup_{\substack{b\ge-1\\\nu\in[0,1]}}|\mathcal D_N(b,\nu)| \le1+\frac{3\Gamma}{\Upsilon}.
\]
 When $Q=H^d$, write
$\mathcal D_N=\mathcal A_{N,5}\circ\DTA\circ \mathcal A_{N,4}\circ\DTA\circ \mathcal A_{N,3}\circ\DTA\circ
\mathcal A_{N,2}\circ\DTA\circ \mathcal A_{N,1}$ and take $\mathcal A_3^{\mathrm{bl}}$ from
Proposition~\ref{prop:block-grid-encoder}.  With
$C_{\mathrm{dec}}=1500$, the maps may be chosen so that
\[
\parnorm{\mathcal A_{N,1}\circ \mathcal A_3^{\mathrm{bl}}}\le C_{\mathrm{dec}}H^d,\qquad
\max_{2\le j\le5}\log\parnorm{\mathcal A_{N,j}}
\le C_{\mathrm{dec}}\frac{Q}{N^2}+\log Q+\tfrac95\log \Upsilon+C_{\mathrm{dec}}.
\]
\end{enumerate}
\end{proposition}
Parts~(a) and~(b) are proved in online
Appendices~\ExternalNumber{B.2}{app:coupled-a-proof}
and~\ExternalNumber{B.4}{app:coupled-b-proof}, respectively.
The two radius bounds in part~(b) distinguish the map merged with the
grid encoder from the remaining affine maps.  This distinction is needed
to control the composed network.  In the proof of
Theorem~\ref{thm:optimal-approximation}, we will verify the adjacent-difference
condition directly from H\"older continuity and rounding.

\subsection{Proof of Theorem~\ref{thm:optimal-approximation}}\label{sec:proof-optimal}
The block-grid encoder gives the block number and position; the coupled
decoder approximates the corresponding rounded value.  We construct the
network for an arbitrary decoder width $n\ge2$.  Taking $n=2$ proves this
theorem, while keeping $n$ variable gives the radius estimate for
Theorem~\ref{thm:joint-approximation}.

\begin{proof}[Proof of Theorem~\ref{thm:optimal-approximation}]
Fix $f\in\mathcal H^\beta([0,1]^d)$, $\varepsilon\in(0,1/2]$, and an integer
$n\ge2$.  Choose~$\Upsilon=\lceil2^{1+1/p}/\varepsilon\rceil$ and
$\log_2H=\lceil\log_2((\Upsilon+1)^{1/\beta})\rceil$.  Then $\Upsilon\ge4$, and $H$ is a power of two satisfying $H\ge \Upsilon+1$; in
particular, $H\ge8$.  Moreover,
\begin{equation*}
H\in\left[(\Upsilon+1)^{1/\beta},2(\Upsilon+1)^{1/\beta}\right],\qquad \Upsilon+1\le C_p\varepsilon^{-1},
\end{equation*}
where one may take $C_p=2^{1+1/p}+1$.
For $\bm\ell\in\{0,\ldots,H-1\}^d$, choose the cell center
$\bm x_{\bm\ell}=H^{-1}(\bm\ell+\bmone_d/2)$ and round its target value to
$a_{\bm\ell}=\lfloor \Upsilon f(\bm x_{\bm\ell})+1/2\rfloor$.  Then
$a_{\bm\ell}\in[-\Upsilon,\Upsilon]\cap\mathbb Z$ and
$|a_{\bm\ell}-\Upsilon f(\bm x_{\bm\ell})|\le1/2$.  If $\ell_1\le H-2$, then 
\[
|a_{\bm\ell+\bm e_1}-a_{\bm\ell}| \le1+\Upsilon|f(\bm x_{\bm\ell+\bm e_1})-f(\bm x_{\bm\ell})| \le1+\Upsilon H^{-\beta}<2,
\]
 so this integer is at most one.  Flatten the table by $a_t=a_{\bm\ell}$
when $t=\ell_1+H\ell_2+\cdots+H^{d-1}\ell_d$.  Whenever
$t\not\equiv H-1\pmod H$, increasing $t$ changes only $\ell_1$; hence
Proposition~\ref{prop:block-decoder} applies with $Q=H^d$.

Use part~(a) when $n=2$ and part~(b) otherwise, and denote the resulting
decoder by $\mathcal D_n$.  Let $T_n^\star$ be the parameter radius after
the encoder and decoder are merged, and set
\begin{equation*}
\delta=\frac{\varepsilon^p}{2dH(2+3\Gamma/\Upsilon)^p}.
\end{equation*}
Since $0<H\delta<1$, Proposition~\ref{prop:block-grid-encoder} supplies
$\mathcal E=\mathcal A_3^{\mathrm{bl}}\circ\DTA\circ \mathcal A_2^{\mathrm{bl}}\circ\DTA\circ
\mathcal A_1^{\mathrm{bl}}$.  Define $\phi_n^\star=\mathcal D_n\circ\mathcal E$ and merge
the last encoder map with the first decoder map.  The resulting $\DTA$
network has hidden-layer widths $\langle 2d+3,d+2,n,n,n,n\rangle$.
For $n=2$, the first decoder map is
$\mathcal A_{2,1}(b,\nu)=(-g(b+1),\nu)^\top$.
For the encoder in Proposition~\ref{prop:block-grid-encoder}, write
$\bm\omega=(1,H,\ldots,H^{d-1})^\top$ and let $\bm e_1$ be the first
coordinate vector.  The encoder output map in
\eqref{eq:block-encoder-A3-matrix} then gives
\begin{align*}
(\mathcal A_{2,1}\circ\mathcal A_3^{\mathrm{bl}})(\bm u,\bar u,c)
&=\begin{bmatrix}
-g\Gamma^{-1}(\bm\omega-\bm e_1)^\top&-g&-gH\bmone_d^\top\bm\omega/\Gamma\\
\Gamma^{-1}\bm e_1^\top&-1&0
\end{bmatrix}
\begin{bmatrix}\bm u\\\bar u\\c\end{bmatrix}
+\begin{bmatrix}g(\bmone_d^\top\bm\omega-1)/\Gamma\\1-1/(2\Gamma)\end{bmatrix}.
\end{align*}
Its first row has at most $d+2$ nonzero parameters including its bias,
and its second row has three.  The first two encoder maps and the four
remaining decoder maps therefore give the total
\[
(4d+3)+(3d+4)+(d+5)+5+4+3+3=8d+27.
\]  Figure~\ref{fig:optimal-full-network}
shows this decoder-width-two network when $d=3$.

For $\bm x\notin\Omega_{H,\delta}([0,1]^d)$, its grid indices are
$\ell_j=\lfloor Hx_j\rfloor$. Write
$t=\ell_1+H\ell_2+\cdots+H^{d-1}\ell_d=\Gamma b+r$.  The encoder gives
$\mathcal E(\bm x)=(b,(r+1/2)/\Gamma)^\top$, and therefore
\begin{align*}
|f(\bm x)-\phi_n^\star(\bm x)|
&\le |f(\bm x)-f(\bm x_{\bm\ell})|
 +\frac{|\Upsilon f(\bm x_{\bm\ell})-a_{\bm\ell}|}{\Upsilon}
 +\left|\mathcal D_n\!\left(b,\frac{r+1/2}{\Gamma}\right)
 -\frac{a_{\bm\ell}}\Upsilon\right|\\
&\le H^{-\beta}+\frac1{2\Upsilon}+\frac1{25\Upsilon}
 <\frac2\Upsilon\le2^{-1/p}\varepsilon.
\end{align*}
On the whole cube, the encoder output lies in the decoder domain and
$\Gamma/\Upsilon\le1/4$, so $\|\phi_n^\star\|_{L^\infty}\le7/4$; for $n=2$,
part~(a) improves this to $5/4$.  With
$\Omega=\Omega_{H,\delta}([0,1]^d)$ and $\mathbf{m}(\Omega)\le dH\delta$,
\begin{align*}
\|f-\phi_n^\star\|_{L^p([0,1]^d)}^p
&=\int_{[0,1]^d\setminus\Omega}|f(\bm x)-\phi_n^\star(\bm x)|^p\,\dd\bm x
 +\int_\Omega|f(\bm x)-\phi_n^\star(\bm x)|^p\,\dd\bm x\\
&\le\frac{\varepsilon^p}{2}+(2+3\Gamma/\Upsilon)^pdH\delta=\varepsilon^p.
\end{align*}
Only the $p$th power is used, so this calculation is valid for every
$p\in(0,\infty)$.

\begin{figure}[!tbp]
\centering
\includegraphics[width=\FigureOptimalNetworkWidth]{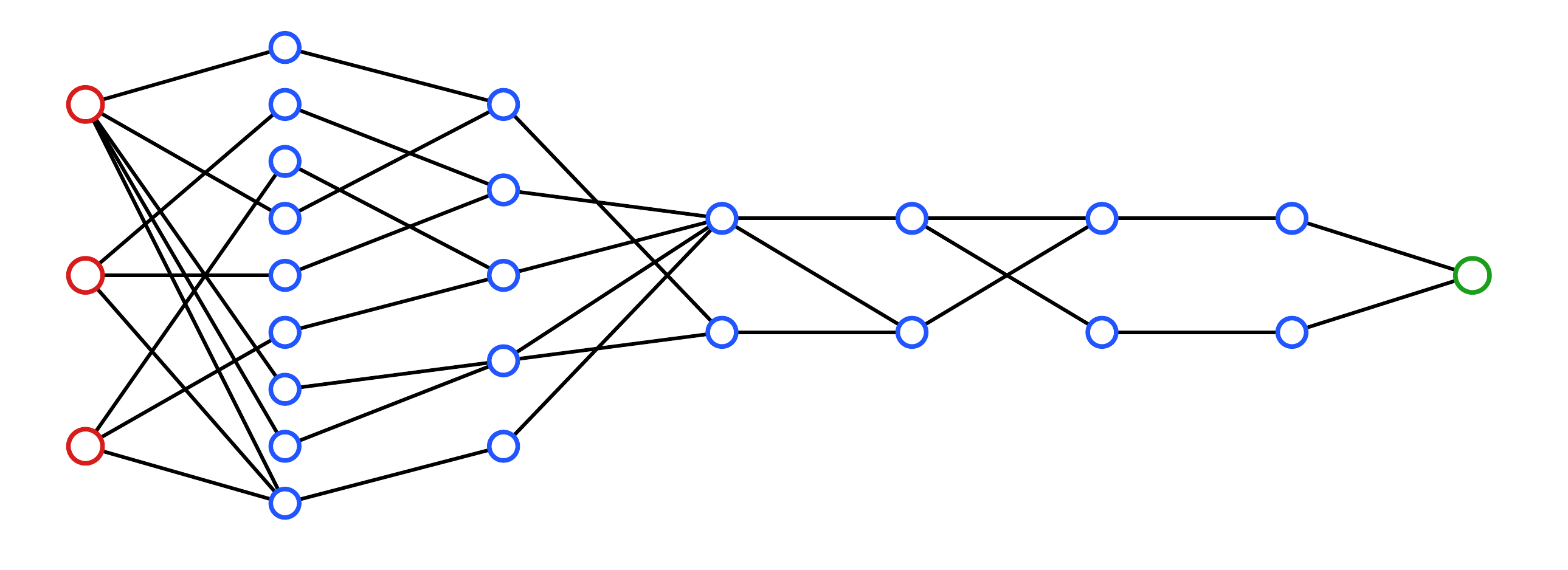}
\caption{The six-hidden-layer approximator for $d=3$ and decoder width two.
The first two hidden layers encode the block index and position within the
block, and the four width-two layers carry out the coupled decoder evaluations.}
\label{fig:optimal-full-network}
\end{figure}

We bound the radius after merging the affine interface.  The decoder
bound covers the scalar choice $n=2$, since $12\le C_{\mathrm{dec}}/4$.
Moreover, $\log(4H^d),\log(C_{\mathrm{dec}}H^d)
\le d\log H+C_{\mathrm{dec}}$.
Since $\Gamma/(2H\delta)=d\Gamma(2+3\Gamma/\Upsilon)^p\varepsilon^{-p}$,
$\Gamma\le \Upsilon^{1/5}$, and $\Gamma/\Upsilon\le1/4$, the two possible contributions give
\begin{align*}
\log T_n^\star
&\le\max\left\{
 C_{\mathrm{dec}}\frac{H^d}{n^2}+d\log H+\tfrac95\log \Upsilon+C_{\mathrm{dec}},\,
 \log\frac{\Gamma}{2H\delta}\right\}\\
&\le C_{\mathrm{dec}}n^{-2}H^d
 +\max\left\{d\log H+\tfrac95\log \Upsilon,\,
 p\log(\varepsilon^{-1})+\tfrac15\log \Upsilon\right\}+C\\
&\le Cn^{-2}\varepsilon^{-d/\beta}
 +\max\left\{\tfrac d\beta+\tfrac95,\,p+\tfrac15\right\}\log(C/\varepsilon).
\end{align*}
Here we used $\Upsilon+1\le C_p\varepsilon^{-1}$ and
$H<2(\Upsilon+1)^{1/\beta}$; $C\ge1$ depends only on $d,\beta,p$.
Set $\gamma=3+3(p+2)\beta/d$.  Since
$2d\gamma/(5\beta)$ dominates both coefficients in the maximum, we may
choose $C_1=C_1(d,\beta,p)\ge(\gamma+2)^2$ large enough so that
\begin{equation}\label{eq:coupled-width-radius-bound}
\log T_n^\star
\le C_1n^{-2}\varepsilon^{-d/\beta}
 +\frac{2d\gamma}{5\beta}
  \log\left(C_1^{\beta/(2d)}\varepsilon^{-1}\right).
\end{equation}
For $n=2$, the logarithmic term is bounded by a constant multiple of
$\varepsilon^{-d/\beta}$.  Thus
$\log T_\varepsilon^\star:=\log T_2^\star\le C_1\varepsilon^{-d/\beta}$ after fixing
$C_1$ once and for all.  The network has at most $8d+27$ nonzero parameters and
$\norm{\phi_2^\star}_{L^\infty}\le5/4$, so
Theorem~\ref{thm:optimal-approximation} follows.
\end{proof}

\subsection{A uniform endpoint result and proof of Theorem~\ref{thm:joint-approximation}}\label{sec:proof-joint}
The unit-radius boundary $T=1$ is the difficult endpoint of
Theorem~\ref{thm:joint-approximation}. The next theorem treats it together
with the low-precision boundary of Theorem~\ref{thm:joint-bits-approximation}:
all parameters lie in $[-1,1]$, and $B$ can be as small as two.
At $B=2$, every affine parameter belongs to $\{-1,0,1\}$.
The corresponding fixed-depth ternary ReLU rate in finite $L^p$ remains
open \citep[Section~5]{NakaharaKobayashiMatsushima2026Ternary}:
existing ternary approximation bounds use growing depth and an
additional output scale \citep[Theorem~2.2]{Beknazaryan2026Ternary}.

\begin{theorem}[Unit-radius low-precision approximation]\label{thm:unit-radius}
Let $d\in\mathbb N$, $\beta\in(0,1]$, and $p\in(0,\infty)$.
There exists a constant $C_{\mathrm{ur}}=C_{\mathrm{ur}}(d,\beta,p)>0$
such that, for every $N,B\in\mathbb N$ with $B\ge2$,
\begin{equation*}
\sup_{f\in\mathcal H^\beta}
\inf_{g\in\mathcal F_\DTA
 (N,13+\lceil3p\beta\rceil,\mathbb Q_{B-2}^{0}\cap[-1,1])}
\norm{f-g}_{L^p}
\le C_{\mathrm{ur}}
\bigl[N^2\min\{B,\log(eN)\}\bigr]^{-\beta/d}.
\end{equation*}
\end{theorem}
The alphabet is
$\mathbb Q_{B-2}^{0}\cap[-1,1]=2^{-(B-2)}\mathbb Z\cap[-1,1]$.
The bound holds for every $N\in\mathbb N$ and every $B\ge2$;
the zero network covers the finitely many small widths excluded by
the construction. For $B\ge\log(eN)$ it gives the unit-radius
order $[N^2\log(eN)]^{-\beta/d}$ required by
Theorem~\ref{thm:joint-approximation}; for $B\le\log(eN)$ it gives
the width--bits order $[N^2B]^{-\beta/d}$ required by
Theorem~\ref{thm:joint-bits-approximation}.
The proof is given in Appendix~\ExternalNumber{C}{app:unit-radius}.

The two constructions cover all radii.  When $T$ is bounded by a fixed
power of $N$, $\log(eNT)$ is comparable to $\log(eN)$, so
Theorem~\ref{thm:unit-radius} suffices.  For larger $T$, we use
\eqref{eq:coupled-width-radius-bound}.  We now make this division and the
constants explicit.

\begin{proof}[Proof of Theorem~\ref{thm:joint-approximation}]
Keep the constant $C_1$ in \eqref{eq:coupled-width-radius-bound}, set
$\gamma=3+3(p+2)\beta/d$, and choose
\begin{equation*}
C_2:=\max\left\{
2(16C_1)^{\beta/d},\,
C_{\mathrm{ur}}(1+\gamma)^{\beta/d}
\right\}.
\end{equation*}
Write $L_\star=13+\lceil3p\beta\rceil$ for this proof, and fix
an integer $N\ge2d+3$, a radius $T\ge1$, and
$f\in\mathcal H^\beta([0,1]^d)$.

\par\medskip\noindent
\textbf{Case 1: $T\ge N^\gamma$.}
Choose $\varepsilon=(16C_1)^{\beta/d}[N^2\log(eNT)]^{-\beta/d}$.
If $\varepsilon>1/2$, the zero network gives error at most
$1<2\varepsilon\le C_2[N^2\log(eNT)]^{-\beta/d}$.
For $\varepsilon\in(0,1/2]$, use the construction in the proof of
Theorem~\ref{thm:optimal-approximation} with $n=N$ and error
$\varepsilon$. It gives $\phi_N^\star$ with
$\|f-\phi_N^\star\|_{L^p}\le\varepsilon$, hidden widths
$\langle2d+3,d+2,N,N,N,N\rangle$, and the parameter-radius bound
\eqref{eq:coupled-width-radius-bound}. Because $N\ge5$ and
$T\ge N^\gamma$, we have $\log N\le\gamma^{-1}\log T$ and
$\log(eNT)\le(\gamma+2)\gamma^{-1}\log T$.
Using $\log x\le x/e$ for $x>0$, followed by
$\gamma\ge3$ and $\sqrt{C_1}\ge\gamma+2$, we obtain
\begin{align*}
\log T_N^\star
&\le
C_1N^{-2}\varepsilon^{-d/\beta}
+\frac{2d\gamma}{5\beta}
 \log\left(C_1^{\beta/(2d)}\varepsilon^{-1}\right)
=
\frac1{16}\log(eNT)
+\frac{2\gamma}{5}
 \log\left(\frac{N^2\log(eNT)}{16\sqrt{C_1}}\right)\\
&=
\frac1{16}\log(eNT)
+\frac{4\gamma}{5}\log N
+\frac{2\gamma}{5}
 \log\left(\frac{\log(eNT)}{16\sqrt{C_1}}\right)\\
&\le
\frac1{16}\log(eNT)
+\frac{4\gamma}{5}\log N
+\frac{\gamma}{40e\sqrt{C_1}}\log(eNT)\\
&=
\left(\frac1{16}+\frac{\gamma}{40e\sqrt{C_1}}\right)\log(eNT)
+\frac{4\gamma}{5}\log N
\le
\left(
\frac{\gamma+2}{16\gamma}
+\frac{\gamma+2}{40e\sqrt{C_1}}
+\frac45
\right)\log T\\
&\le
\left(\frac5{48}+\frac1{40e}+\frac45\right)\log T
<
\left(\frac5{48}+\frac1{96}+\frac45\right)\log T
=
\frac{439}{480}\log T
<
\log T.
\end{align*}
Thus $\phi_N^\star$ has width at most $N$, depth six, and radius at most
$T$. Since clipping cannot increase its error and $6\le L_\star$,
$\clip\circ\phi_N^\star\in\mathcal F_\DTA(N,6,T)
\subseteq\mathcal F_\DTA(N,L_\star,T)$. Together with the zero-network
case, this proves
\[
\inf_{g\in\mathcal F_\DTA(N,L_\star,T)}
\norm{f-g}_{L^p}
\le C_2[N^2\log(eNT)]^{-\beta/d}.
\]

\par\medskip\noindent
\textbf{Case 2: $1\le T<N^\gamma$.}
Apply Theorem~\ref{thm:unit-radius} with the integer
$B=\lceil\log(eN)\rceil$. Since $N\ge5$, this choice has $B\ge2$ and
$\min\{B,\log(eN)\}=\log(eN)$. Its parameter alphabet is contained in
$[-1,1]$, so the definitions of the clipped classes give
\[
\mathcal F_\DTA(N,L_\star,\mathbb Q_{B-2}^{0}\cap[-1,1])
\subseteq\mathcal F_\DTA(N,L_\star,1)
\subseteq\mathcal F_\DTA(N,L_\star,T).
\]
The width and depth are unchanged in both inclusions. Using the
unit-radius theorem and $\log(eNT)\le(1+\gamma)\log(eN)$, we obtain
\begin{align*}
\inf_{g\in\mathcal F_\DTA(N,L_\star,T)}\norm{f-g}_{L^p}
&\le
\inf_{g\in\mathcal F_\DTA(N,L_\star,\mathbb Q_{B-2}^{0}\cap[-1,1])}
\norm{f-g}_{L^p}
\le C_{\mathrm{ur}}[N^2\log(eN)]^{-\beta/d}\\
&\le C_{\mathrm{ur}}(1+\gamma)^{\beta/d}
[N^2\log(eNT)]^{-\beta/d}
\le C_2[N^2\log(eNT)]^{-\beta/d}.
\end{align*}
The two cases cover every $T\ge1$. Taking the supremum over
$f\in\mathcal H^\beta$ proves the theorem.
\end{proof}

\subsection{Proof of Theorem~\ref{thm:joint-bits-approximation}}
\label{sec:proof-joint-bits}
To prove this result, we will make full use of Theorems~\ref{thm:joint-approximation}
and~\ref{thm:unit-radius}. The remaining difficulty therefore lies in establishing
the existence of the coupled block decoder (Proposition~\ref{prop:block-decoder})
and in proving Theorem~\ref{thm:unit-radius} itself. These constitute the two
most technical parts of the paper, and we devote a substantial portion of
the Appendix to their proofs.
\begin{proof}[Proof of Theorem~\ref{thm:joint-bits-approximation}]
Fix $N,B\in\mathbb N$ with $N\ge2d+3$ and $B\ge2$, and write
$L_\star=13+\lceil3p\beta\rceil$. We use the radius theorem when $B$ is
large relative to $\log(eN)$ and the unit-radius theorem for the remaining
budgets.

\par\medskip\noindent
\textbf{Case 1: high precision.}
Fix $f\in\mathcal H^\beta([0,1]^d)$.
\begin{itemize}
\item \textbf{Choice of parameters.}
Assume we are in the high-precision regime, namely
\[
B\ge4(L_\star+1)\log_2(N+1),
\]
and choose $a=\lfloor B/[4(L_\star+1)]\rfloor$, $b=B-a-2$, and
$T=2^a$. Since $B/[4(L_\star+1)]\ge\log_2(N+1)>2$, these choices
satisfy $a,b\in\mathbb N_0$ and $a+b+2=B$.
Theorem~\ref{thm:joint-approximation} supplies a network
$\phi_{\bm\theta}\in\NN_\DTA(d,1;N,L_\star,T)$ satisfying
\[
\|f-\clip\circ\phi_{\bm\theta}\|_{L^p}
\le2C_2[N^2\log(eNT)]^{-\beta/d}
=2C_2\left(\frac{B}{N^2B(1+\log N+a\log2)}\right)^{\beta/d},
\]
where $\bm\theta$ contains every affine parameter, including those of
the output map.

\item \textbf{Quantization and Lipschitz regularity.}
For $T=2^a$, define
\[
\pi:[-T,T]\longrightarrow\mathbb Q_b^a
=(-2T,2T)\cap2^{-b}\mathbb Z,\qquad
\pi(\theta):=2^{-b}\left\lfloor2^b\theta+\frac12\right\rfloor,
\]
and apply $\pi$ to every coordinate of $\bm\theta$. The endpoints
$\pm T$ belong to $2^{-b}\mathbb Z$, so rounding keeps all parameters
in $[-T,T]$ and gives
$\|\bm\theta-\pi(\bm\theta)\|_\infty\le2^{-b}\le1$.
For two realizations of the same architecture with parameters in
$[-T,T]$ and parameter distance at most one,
(\ExternalNumber{E.3}{eq:holder-parameter-continuity}) with
$\alpha=H_\DTA=1$ and $K_{\DTA,d}=2(d+1)$ gives
\[
\|\phi_{\bm\theta}-\phi_{\bm\theta'}\|_{L^\infty}
\le[8(d+1)^2(N+1)T]^{L_\star}
\|\bm\theta-\bm\theta'\|_\infty.
\]
This estimate includes the output affine map.
\end{itemize}

The quantized realization has the same architecture. Since
$a+b+2=B$, its clipping belongs to
$\mathcal F_\DTA(N,L_\star,\mathbb Q_b^a)
\subseteq\mathcal F_\DTA^{\mathrm{bit}}(N,L_\star,B)$.
The parameter estimate, the chosen values of $a,b$, and
$x e^{-x}\le e^{-1}$ for $x>0$ give
\begin{align*}
&\|\phi_{\bm\theta}-\phi_{\pi(\bm\theta)}\|_{L^\infty}
\le[8(d+1)^2(N+1)2^a]^{L_\star}2^{-b}\\
&\quad=4[8(d+1)^2]^{L_\star}
 2^{L_\star\log_2(N+1)+(L_\star+1)a-B}
 \le4[8(d+1)^2]^{L_\star}
 2^{\left(\frac{L_\star}{4(L_\star+1)}+\frac14-1\right)B}\\
&\quad=4[8(d+1)^2]^{L_\star}
 2^{-\frac{2L_\star+3}{4(L_\star+1)}B}
 \le4[8(d+1)^2]^{L_\star}2^{-B/2}\\
&\quad=\frac{16[8(d+1)^2]^{L_\star}}{\log2}
 \left(\frac{B\log2}{4}\right)
 \exp\!\left(-\frac{B\log2}{4}\right)\frac{2^{-B/4}}B\\
&\quad\le\frac{16[8(d+1)^2]^{L_\star}}{e\log2}
 \frac{(N+1)^{-(L_\star+1)}}B
 \le\frac{16[8(d+1)^2]^{L_\star}}{e\log2}[N^2B]^{-1}
 \le\frac{16[8(d+1)^2]^{L_\star}}{e\log2}[N^2B]^{-\beta/d}.
\end{align*}
Clipping is $1$-Lipschitz, and
$|u+v|^p\le2^{\max\{p-1,0\}}(|u|^p+|v|^p)$ for every $p>0$.
Since $a\ge B/[4(L_\star+1)]-1\ge B/[8(L_\star+1)]$, we therefore obtain
\begin{align*}
&\|f-\clip\circ\phi_{\pi(\bm\theta)}\|_{L^p}^p\\
&\quad\le2^{\max\{p-1,0\}}
\left\{\|f-\clip\circ\phi_{\bm\theta}\|_{L^p}^p
 +\|\phi_{\bm\theta}-\phi_{\pi(\bm\theta)}\|_{L^\infty}^p\right\}\\
&\quad\le2^{\max\{p-1,0\}}
\left\{(2C_2)^p
 \left(\frac{B}{N^2B(1+\log N+a\log2)}\right)^{p\beta/d}
 +\left(\frac{16[8(d+1)^2]^{L_\star}}{e\log2}\right)^p
 [N^2B]^{-p\beta/d}\right\}\\
&\quad\le2^{\max\{p-1,0\}}
\left\{(2C_2)^p
 \left(\frac{B}{(B/[4(L_\star+1)]-1)\log2}\right)^{p\beta/d}
 +\left(\frac{16[8(d+1)^2]^{L_\star}}{e\log2}\right)^p\right\}
 [N^2B]^{-p\beta/d}\\
&\quad\le
\underbrace{2^{\max\{p-1,0\}}
\left\{(2C_2)^p
 \left(\frac{8(L_\star+1)}{\log2}\right)^{p\beta/d}
 +\left(\frac{16[8(d+1)^2]^{L_\star}}{e\log2}\right)^p\right\}}
 _{=:(C_2^\star)^p}
 [N^2B]^{-p\beta/d}.
\end{align*}
Taking $p$th roots and then the infimum over the bit-constrained class
gives, for every $f\in\mathcal H^\beta$,
\[
\inf_{g\in\mathcal F_\DTA^{\mathrm{bit}}(N,L_\star,B)}
\norm{f-g}_{L^p}\le C_2^\star[N^2B]^{-\beta/d}.
\]

\par\medskip\noindent
\textbf{Case 2: $2\le B<4(L_\star+1)\log_2(N+1)$.}
Choose $a=0$ and $b=B-2$, so $a+b+2=B$. The unit-radius alphabet
in Theorem~\ref{thm:unit-radius} then gives the class inclusions
\[
\mathcal F_\DTA(N,L_\star,\mathbb Q_{B-2}^{0}\cap[-1,1])
\subseteq\mathcal F_\DTA(N,L_\star,\mathbb Q_{B-2}^{0})
\subseteq\mathcal F_\DTA^{\mathrm{bit}}(N,L_\star,B).
\]
These inclusions keep the same width, depth, and clipping convention.
The case assumption implies
$B<4(L_\star+1)\log(N+1)/\log2
\le4(L_\star+1)\log(eN)/\log2$, and hence
$\min\{B,\log(eN)\}\ge B\log2/[4(L_\star+1)]$.
Consequently, Theorem~\ref{thm:unit-radius} yields
\begin{align*}
\sup_{f\in\mathcal H^\beta}
\inf_{g\in\mathcal F_\DTA^{\mathrm{bit}}(N,L_\star,B)}\norm{f-g}_{L^p}
&\le
\sup_{f\in\mathcal H^\beta}
\inf_{g\in\mathcal F_\DTA(N,L_\star,\mathbb Q_{B-2}^{0}\cap[-1,1])}
\norm{f-g}_{L^p}\\
&\le C_{\mathrm{ur}}[N^2\min\{B,\log(eN)\}]^{-\beta/d}\\
&\le C_{\mathrm{ur}}\left[\frac{4(L_\star+1)}{\log2}\right]^{\beta/d}
[N^2B]^{-\beta/d}.
\end{align*}
Take the larger of the constant in Case~1 and
$C_{\mathrm{ur}}[4(L_\star+1)/\log2]^{\beta/d}$ as $C_2^\star$.
It depends only on $d,\beta,p$, and the two cases prove, for every
$N,B\in\mathbb N$ with $N\ge2d+3$ and $B\ge2$,
\[
\sup_{f\in\mathcal H^\beta}
\inf_{g\in\mathcal F_\DTA^{\mathrm{bit}}(N,13+\lceil3p\beta\rceil,B)}
\norm{f-g}_{L^p}
\le C_2^\star[N^2B]^{-\beta/d}.
\]
\end{proof}

\section{Proofs of Theorems~\ref{thm:near-optimal-parameter} and~\ref{thm:generalization}}\label{sec:generalization}

Covering estimates yield both results: comparison with the H\"older
packing number gives the approximation lower bounds, and an
empirical-process inequality gives the regression risk bound.
For a class $\mathcal G$ of real-valued functions on $\Omega$ and
$M\in\mathbb N$, define the uniform
proper empirical covering number by
\begin{equation*}
\mathcal N(\eps,\mathcal G,M)
=\sup_{\bm x_1,\ldots,\bm x_M\in\Omega}
\mathcal N_{\mathrm{prop}}\!\left(
\eps,
\left\{\bigl(g(\bm x_1),\ldots,g(\bm x_M)\bigr):g\in\mathcal G\right\},
\norm{\cdot}_\infty
\right).
\end{equation*}
Every center is the evaluation vector of a function in the same class,
so the loss and multiplier estimates also apply to the centers. The
supremum makes the bound uniform over all input locations.

\subsection{Proof of Theorem~\ref{thm:near-optimal-parameter}}\label{sec:lower-proof}
We compare a packing of the H\"older ball with a covering of the network
class. The parameter-radius lower bound follows from parameter
discretization, and the bit lower bound from counting the admissible
parameter choices.

The entropy order below is classical
\citep{KolmogorovTikhomirov1961,KerkyacharianPicard2003}.  We include a
direct packing proof because it treats the endpoint $\beta=1$ and the
quasi-norm range $p\in(0,1)$ in exactly the form used here.

\begin{lemma}[H\"older packing]\label{lem:holder-packing}
Let $d\in\mathbb N$, $\beta\in(0,1]$, and $p\in(0,\infty)$.  There are
$c_{\mathrm{H}},\delta_{\mathrm{H}}>0$, depending only on $d,\beta,p$, such that
\begin{equation*}
\log\mathcal N_{\mathrm{prop}}\!\left(
\delta,\mathcal H^\beta([0,1]^d),\norm{\cdot}_{L^p([0,1]^d)}
\right)
\ge c_{\mathrm{H}}\delta^{-d/\beta},
\qquad\delta\in(0,\delta_{\mathrm{H}}].
\end{equation*}
\end{lemma}

\begin{proof}
Use the tent function $\psi(\bm x)=(1-2\norm{\bm x}_\infty)_+$ at scale
$h=(4m)^{-1}$, where $m\ge1$ is an integer. For
$\bm k\in\{0,\ldots,m-1\}^d$, let
$\bm z_{\bm k}=((4k_j+2)h)_{j=1}^d$.  The supports of
$\psi((\cdot-\bm z_{\bm k})/h)$ lie in $[0,1]^d$ and are mutually
separated by at least $3h$.  For
$\bm\omega=(\omega_{\bm k})\in\{0,1\}^{m^d}$, define
\[
f_{\bm\omega}(\bm x)
=\frac14 h^\beta
\sum_{\bm k\in\{0,\ldots,m-1\}^d}
\omega_{\bm k}\,
\psi\!\left(\frac{\bm x-\bm z_{\bm k}}h\right).
\]
Since the supports are disjoint and $0\le\psi\le1$,
$\norm{f_{\bm\omega}}_\infty\le h^\beta/4\le1$. To check the H\"older
condition for points at distance $r=\norm{\bm x-\bm y}_\infty$, first suppose
that $\bm x,\bm y$ do not belong to two distinct supports.
The $2$-Lipschitz property of $\psi$ gives
$|f_{\bm\omega}(\bm x)-f_{\bm\omega}(\bm y)|
\le h^\beta\min\{2r/h,1\}/4\le r^\beta/2$,
using $\min\{2t,1\}\le2t^\beta$.  If they belong to distinct supports,
then $r\ge3h$, while the difference is at most $h^\beta/4\le
r^\beta/(4\cdot3^\beta)$.  Hence
$f_{\bm\omega}\in\mathcal H^\beta([0,1]^d)$ for every $\bm\omega$.

For $m^d$ sufficiently large, the Varshamov--Gilbert bound
\citep[Lemma~2.9]{Tsybakov2009} gives
$\mathcal V\subset\{0,1\}^{m^d}$ with
$\log|\mathcal V|\ge c m^d$ such that any two distinct elements of
$\mathcal V$ differ in at least $m^d/8$ coordinates.  Since the
corresponding bump supports are disjoint, for
$\bm\omega\ne\bm\omega'$ in $\mathcal V$,
\[
\begin{aligned}
\norm{f_{\bm\omega}-f_{\bm\omega'}}_{L^p}^p
&=
\frac{h^{\beta p}}{4^p}
\sum_{\bm k:\,\omega_{\bm k}\ne\omega'_{\bm k}}
\int_{\mathbb R^d}
\psi\!\left(\frac{\bm x-\bm z_{\bm k}}h\right)^p
\,\mathrm d\bm x 
=
\frac{h^{\beta p+d}}{4^p}
\#\{\bm k:\omega_{\bm k}\ne\omega'_{\bm k}\}
\int_{\mathbb R^d}\psi(\bm u)^p\,\mathrm d\bm u \\
&\ge
\frac{m^d}{8\cdot4^p}h^{\beta p+d}
\int_{\mathbb R^d}\psi(\bm u)^p\,\mathrm d\bm u
=
\frac{1}{8\cdot4^{d+p}}
\left(\int_{\mathbb R^d}\psi(\bm u)^p\,\mathrm d\bm u\right)
h^{\beta p}
:=
c_{d,p}h^{\beta p}.
\end{aligned}
\]
Thus the pairwise $L^p$ distance is at
least $c_*h^\beta$ for some $c_*=c_*(d,p)>0$.

For sufficiently small $\delta$, choose
$m=\lfloor c_0\delta^{-1/\beta}\rfloor$, where
$c_0=c_0(d,\beta,p)>0$ is fixed sufficiently small.  Then
$m\ge(c_0/2)\delta^{-1/\beta}$ and
$c_*h^\beta\ge c_*(4c_0)^{-\beta}\delta$, which by the choice of
$c_0$ is larger than $2\delta$ when $p\ge1$ and larger than
$2^{1/p}\delta$ when $p\in(0,1)$.  Hence a radius-$\delta$ $L^p$ ball
contains at most one $f_{\bm\omega}$ with $\bm\omega\in\mathcal V$:
for $p\ge1$ this is the triangle inequality, while for $p<1$ two
such functions in the same ball would satisfy
$\norm{f_{\bm\omega}-f_{\bm\omega'}}_{L^p}^p\le2\delta^p$.
Therefore every proper radius-$\delta$ cover of
$\mathcal H^\beta([0,1]^d)$ contains at least $|\mathcal V|$ centers,
and hence
\[
\log\mathcal N_{\mathrm{prop}}\!\left(
\delta,\mathcal H^\beta([0,1]^d),\norm{\cdot}_{L^p([0,1]^d)}
\right)
\ge \log|\mathcal V|
\ge c m^d
\ge c_{\mathrm H}\delta^{-d/\beta}.
\]
After decreasing $\delta_{\mathrm H}>0$ so that the preceding choices
are valid, the proof is complete.
\end{proof}

The next proposition converts this entropy into a lower bound on the worst-case
$L^p$ approximation error of any class satisfying a logarithmic empirical-cover
estimate.

\begin{proposition}[Entropy-to-approximation lower bound]
\label{prop:necessary-parameter-growth}
Let $d\in\mathbb N$, $\beta\in(0,1]$, and $p\in(0,\infty)$, and let
$\mathcal U$ be a nonempty class of measurable functions on $[0,1]^d$.
There is a constant $c=c(d,\beta,p)>0$ with the following properties.
If $A_0,A_1\ge1$ satisfy
$\log\mathcal N(\eps,\mathcal U,M)\le A_0\log(A_1\eps^{-1})$
for every $M\in\mathbb N$ and $\eps\in(0,1]$, then
\[
\sup_{f\in\mathcal H^\beta}\inf_{g\in\mathcal U}
\norm{f-g}_{L^p([0,1]^d)}
\ge c\left[A_0\log(eA_0A_1)\right]^{-\beta/d}.
\]
If $A_0\ge1$ satisfies
$\log\mathcal N(\eps,\mathcal U,M)\le A_0$ for every
$M\in\mathbb N$ and $\eps\in(0,1]$, the bound improves to
\begin{equation}\label{eq:finite-class-approximation-lower}
\sup_{f\in\mathcal H^\beta}\inf_{g\in\mathcal U}
\norm{f-g}_{L^p([0,1]^d)}\ge c A_0^{-\beta/d}.
\end{equation}
\end{proposition}

\begin{proof}
Let $c_{\mathrm{H}}$ and $\delta_{\mathrm{H}}$ be the constants from
Lemma~\ref{lem:holder-packing}, write $a=d/\beta$ for the entropy exponent,
and set
\[
a_p=\begin{cases}8,&p\ge1,\\8^{1/p},&0<p<1.\end{cases}
\]
Choose $c_0>0$, depending only on $d,\beta,p$, so small that
\[
c_0\le\min\{1/2,\delta_{\mathrm{H}}/a_p\},\qquad
c_{\mathrm{H}}a_p^{-a}c_0^{-a}>1+\log(c_0^{-1})+\frac2a.
\]
With $\Gamma=\log(eA_0A_1)$ and
$\eta=c_0(A_0\Gamma)^{-1/a}$, suppose for contradiction that every target
in the H\"older ball has distance at most $\eta$ from $\mathcal U$.

Since $\mathcal H^\beta([0,1]^d)$ is compact in $L^p([0,1]^d)$, it
contains a finite maximal $a_p\eta$-separated set
$f_1,\ldots,f_J$.  Maximality makes these same functions a proper
$a_p\eta$-cover of the H\"older ball, and Lemma~\ref{lem:holder-packing}
therefore gives
\begin{equation}\label{eq:holder-packing-lower}
\log J\ge c_{\mathrm{H}}a_p^{-a}\eta^{-a}.
\end{equation}
For each $i$, choose $g_i\in\mathcal U$ with
$\norm{f_i-g_i}_{L^p}<3\eta/2$.  If $p\ge1$, the triangle inequality gives
$\norm{g_i-g_j}_{L^p}>5\eta$ for $i\ne j$.  If $p\in(0,1)$ and the latter
inequality failed, then
\[
\norm{f_i-f_j}_{L^p}^p
\le\norm{f_i-g_i}_{L^p}^p
 +\norm{g_i-g_j}_{L^p}^p
 +\norm{g_j-f_j}_{L^p}^p
<\{2(3/2)^p+5^p\}\eta^p<(a_p\eta)^p,
\]
again contradicting the separation.  Thus
\begin{equation}\label{eq:approximant-lp-separation}
\norm{g_i-g_j}_{L^p}>5\eta,
\qquad i\ne j.
\end{equation}

For every $i<j$, choose $\bm x_{ij}\in[0,1]^d$ with
$|g_i(\bm x_{ij})-g_j(\bm x_{ij})|>5\eta$; otherwise
\eqref{eq:approximant-lp-separation} would fail because the cube has unit
measure.  List these points as $\bm x_1,\ldots,\bm x_M$ and let
\[
\bm v_i=(g_i(\bm x_1),\ldots,g_i(\bm x_M)),\qquad i=1,\ldots,J.
\]
Then $\norm{\bm v_i-\bm v_j}_\infty>5\eta$ whenever $i\ne j$.  Hence every
proper empirical $2\eta$-cover of $\mathcal U$ on these sample points has at
least $J$ elements.  Combining this with \eqref{eq:holder-packing-lower},
$\eta^{-a}=c_0^{-a}A_0\Gamma$, and $\Gamma=\log(eA_0A_1)\ge1$, we obtain
\begin{align*}
c_{\mathrm{H}}a_p^{-a}c_0^{-a}A_0\Gamma
&\le\log J\le\log\mathcal N(2\eta,\mathcal U,M)
 \le A_0\log\left(\frac{A_1}{2\eta}\right)
\le A_0\left\{\log A_1+\log(\eta^{-1})\right\}\\
&=A_0\left\{\log A_1+\log(c_0^{-1})+\frac1a\log A_0+\frac1a\log \Gamma\right\}
\le A_0\Gamma\left\{1+\log(c_0^{-1})+\frac2a\right\},
\end{align*}
contradicting the choice of $c_0$.  Thus the asserted lower bound holds with
$c=c_0$.

For \eqref{eq:finite-class-approximation-lower}, retain $a=d/\beta$ and
$a_p$ and take $\eta=c_0A_0^{-1/a}$.  The same choice of $c_0$ ensures
$2\eta\le1$, $a_p\eta\le\delta_{\mathrm H}$, and
$c_{\mathrm H}a_p^{-a}c_0^{-a}>1$.
If every target had distance at most $\eta$ from $\mathcal U$, the same
separated functions and evaluation points used above would give
\[
c_{\mathrm H}a_p^{-a}c_0^{-a}A_0
=c_{\mathrm H}a_p^{-a}\eta^{-a}
\le\log J
\le\log\mathcal N(2\eta,\mathcal U,M)
\le A_0,
\]
a contradiction.  This proves the second assertion with the same
constant $c=c_0$.
\end{proof}

To apply this comparison to $\mathcal F_\varphi(N,L,T)$, cover each
parameter cube and then take the finite union over the possible architectures.
For ReLU networks, related parameter-covering estimates appear in
Schmidt-Hieber~\citeyearpar[Lemma~5]{SchmidtHieber2020}; the parameter-grid correction for that
lemma is recorded in Schmidt-Hieber and Vu~\citeyearpar{SchmidtHieberVu2024}.  See also
Ou and B\"olcskei~\citeyearpar[Theorem~2.1]{OuBolcskei2024}.

\begin{proposition}[H\"older covering for bounded width and depth]
\label{prop:holder-class-cover}
Let $d,N,L,M\in\mathbb N$, $T\ge1$, and $\eps\in(0,1]$.  Suppose that
$\alpha\in(0,1]$ and $\varphi:\mathbb R\to\mathbb R$ is globally
$\alpha$-H\"older continuous with constant $H_\varphi$.  Then
\begin{equation*}
\log\mathcal N\!\left(\eps,\mathcal F_\varphi(N,L,T),M\right)
\le C_{\varphi,d}\alpha^{-L}(dN+LN^2)
\log\left(\frac{e(N+1)^LT^{L+1}}{\eps}\right).
\end{equation*}
One may take
$C_{\varphi,d}=4+4\log_2\bigl((d+1)
(1+|\varphi(0)|+H_\varphi)\bigr)$.
\end{proposition}
The proof is given in Appendix~\ExternalNumber{E}{app:gen}.

\begin{corollary}[Covering the general $\DTA$ class]
\label{cor:dta-cover}
For $d,N,L,M\in\mathbb N$, $T\in[1,\infty)$, and $\eps\in(0,1]$,
\begin{equation*}
\log\mathcal N\!\left(\eps,\mathcal F_\DTA(N,L,T),M\right)
\le C_d(dN+LN^2)
\log\left(\frac{e(N+1)^LT^{L+1}}{\eps}\right).
\end{equation*}
One may take $C_d=8+4\log_2(d+1)$.
\end{corollary}

\begin{proof}
The activation $\DTA$ is globally $1$-Lipschitz and satisfies $\DTA(0)=0$.
Apply Proposition~\ref{prop:holder-class-cover} with
$\alpha=H_\varphi=1$.
\end{proof}

The parameter cover and finite-class count give the two joint lower
bounds.
\begin{proof}[Proof of Theorem~\ref{thm:near-optimal-parameter}]
Let $c_0=c_0(d,\beta,p)>0$ be the constant in
Proposition~\ref{prop:necessary-parameter-growth}.
For the radius bound, write
$C_*=C_{\varphi,d}\alpha^{-L}(d+L)\ge1$.
Since $dN+LN^2\le(d+L)N^2$,
Proposition~\ref{prop:holder-class-cover} verifies the logarithmic
covering hypothesis of Proposition~\ref{prop:necessary-parameter-growth} for
$\mathcal U=\mathcal F_\varphi(N,L,T)$ with
$A_0=C_*N^2$ and $A_1=e(N+1)^LT^{L+1}$.
Applying Proposition~\ref{prop:necessary-parameter-growth} and using
$N+1\le2N$ and $N,T\ge1$, we obtain
\begin{align*}
&\sup_{f\in\mathcal H^\beta}
  \inf_{g\in\mathcal F_\varphi(N,L,T)}\norm{f-g}_{L^p}\\
&\quad\ge c_0\bigl[A_0\log(eA_0A_1)\bigr]^{-\beta/d}\\
&\quad=c_0\Bigl[C_*N^2\bigl\{
 2+\log C_*+2\log N+L\log(N+1)+(L+1)\log T
 \bigr\}\Bigr]^{-\beta/d}\\
&\quad\ge c_0\Bigl[C_*N^2\bigl\{
 2+\log C_*+L\log2+(L+2)\log N+(L+1)\log T
 \bigr\}\Bigr]^{-\beta/d}\\
&\quad\ge c_0
 \bigl[C_*\bigl(L+4+\log C_*+L\log2\bigr)\bigr]^{-\beta/d}
 \bigl[N^2\log(eNT)\bigr]^{-\beta/d}.
\end{align*}
Thus the first assertion holds with
$c=c_0[C_*(L+4+\log C_*+L\log2)]^{-\beta/d}$,
which has the stated dependence.

For the bit bound, fix a Borel measurable activation $\varphi$.
The definition in \eqref{eq:bit-network-class} allows $B-1$ common
formats, each with $2^B-1$ values for every parameter.
Counting the hidden-width sequences and all affine parameters, including
the final output map, and noting that clipping cannot increase the number
of realizations, gives, for every
$M\in\mathbb N$ and $\eps\in(0,1]$,
\begin{align*}
&\log\mathcal N\!\left(
 \eps,\mathcal F_\varphi^{\mathrm{bit}}(N,L,B),M\right)
 \le\log\#\mathcal F_\varphi^{\mathrm{bit}}(N,L,B)\\
&\quad\le\log(B-1)+\log\left(\sum_{\ell=1}^L N^\ell\right)
 +B\bigl[N(d+1)+(L-1)N(N+1)+N+1\bigr]\log2\\
&\quad\le\log(B-1)+L\log(N+1)
 +B\bigl[(L-1)N^2+(d+L+1)N+1\bigr]\log2\\
&\quad\le B+LN+B(d+2L+1)N^2\\
&\quad\le LN^2B+LN^2B+(d+3)LN^2B
 =(d+5)LN^2B.
\end{align*}
Applying the finite-cover conclusion
\eqref{eq:finite-class-approximation-lower} with
$A_0=(d+5)LN^2B$ now gives
\[
\sup_{f\in\mathcal H^\beta}
\inf_{g\in\mathcal F_\varphi^{\mathrm{bit}}(N,L,B)}
\norm{f-g}_{L^p}
\ge c_0[(d+5)LN^2B]^{-\beta/d}
=c^\star[N^2LB]^{-\beta/d},
\]
where $c^\star=c_0(d+5)^{-\beta/d}$ is independent of
$N,L,B$ and $\varphi$.
\end{proof}

For $\varphi=\DTA$ and $L=13+\lceil3p\beta\rceil$, Theorems~\ref{thm:joint-approximation} and~\ref{thm:near-optimal-parameter} give
\[
\sup_{f\in\mathcal H^\beta}
\inf_{g\in\mathcal F_\DTA(N,13+\lceil3p\beta\rceil,T)}
\norm{f-g}_{L^p}
\asymp_{d,\beta,p}[N^2\log(eNT)]^{-\beta/d}
\]
for every integer $N\ge2d+3$ and every $T\ge1$.  This is joint optimality in width and
parameter radius, including the unit-radius endpoint $T=1$.
Likewise, Theorems~\ref{thm:joint-bits-approximation}
and~\ref{thm:near-optimal-parameter} give
\[
\sup_{f\in\mathcal H^\beta}
\inf_{g\in\mathcal F_\DTA^{\mathrm{bit}}(N,13+\lceil3p\beta\rceil,B)}
\norm{f-g}_{L^p}
\asymp_{d,\beta,p}[N^2B]^{-\beta/d},
\qquad N\ge2d+3,\quad B\ge2.
\]
The architecture in Theorem~\ref{thm:optimal-approximation} has width
$N_0=2d+3$ and hidden depth $L_0=6$. If radius $T_\varepsilon$ suffices uniformly
for this architecture, clipping and the lower bound give
\[
\begin{gathered}
\varepsilon
\ge\sup_{f\in\mathcal H^\beta}
 \inf_{g\in\mathcal F_\DTA(N_0,L_0,T_\varepsilon)}\norm{f-g}_{L^p}
\ge c\{N_0^2\log(eN_0T_\varepsilon)\}^{-\beta/d},
\\[0.4em]
\log T_\varepsilon
\ge\frac{c^{d/\beta}}{N_0^2}\varepsilon^{-d/\beta}-\log(eN_0)
\ge\frac{c^{d/\beta}}{N_0^2}\varepsilon^{-d/\beta}
-\frac{c^{d/\beta}}{2N_0^2}\varepsilon^{-d/\beta}
=\frac{c^{d/\beta}}{2N_0^2}\varepsilon^{-d/\beta}
\end{gathered}
\]
for sufficiently small $\varepsilon$.  Thus the fixed-size logarithmic radius in
Theorem~\ref{thm:optimal-approximation} is optimal in the worst case over
the H\"older ball.

\subsection{Proof of Theorem~\ref{thm:generalization}}\label{sec:regression-proof}
We combine the approximation bound with the following oracle inequality,
which separates approximation, optimization, and estimation errors. The
width--radius condition controls the additional $\log M$ in the covering
term, including at unit radius. Related oracle inequalities appear in
\citet[Theorem~2]{SchmidtHieber2020} and
\citet[Proposition~4]{Suzuki2019}.

\Needspace{3\baselineskip}
\begin{proposition}[Sub-Gaussian empirical-risk bound]\label{prop:oracle}
Let $d,M\in\mathbb N$, $\xi\ge0$, and let $\mathcal G$ be a nonempty class
of measurable functions on
$[0,1]^d$ that is separable under $\norm{\cdot}_\infty$.  Assume that
$\norm{g}_\infty\le1$ for every $g\in\mathcal G$, and let
$f:[0,1]^d\to\R$ be measurable with $\norm{f}_\infty\le1$.  Under the
sampling model of Section~\ref{sec:setting}, suppose that a measurable
$\widehat g\in\mathcal G$ satisfies, almost surely, 
\[
\widehat{\mathcal R}_M(\widehat g) \le\inf_{g\in\mathcal G}\widehat{\mathcal R}_M(g)+\xi.
\]
 Then
\begin{align*}
\mathbb E\norm{\widehat g-f}_{L^2(\mu)}^2
&\le
4\inf_{g\in\mathcal G}\norm{g-f}_{L^2(\mu)}^2+3\xi
+\frac{83+37\sigma^2}{M}
\left\{
\log\left[2\mathcal N\!\left(\frac1{48M},\mathcal G,2M\right)\right]+1
\right\}.
\end{align*}
\end{proposition}

The proof of Proposition~\ref{prop:oracle} is given in
Appendix~\ExternalNumber{E}{app:gen}.

\begin{proof}[Proof of Theorem~\ref{thm:generalization}]
Let $0<c_-\le c_+<\infty$ be the fixed lower and upper comparison constants
in \eqref{eq:width-radius-curve}.  Fix $M\ge1$, let
$f\in\mathcal H^\beta([0,1]^d)$, and denote $\mathcal G=\mathcal F_\DTA(N,19,T)$.
Since $13+\lceil6\beta\rceil\le19$, Theorem~\ref{thm:joint-approximation}
and the density bound \eqref{eq:density-transfer} yield
\begin{equation}\label{eq:regression-bias-bound}
\begin{aligned}
\inf_{g\in\mathcal G}\norm{g-f}_{L^2(\mu)}^2
&\le\kappa\inf_{g\in\mathcal G}\norm{g-f}_{L^2([0,1]^d)}^2
\le\kappa C_2^2\bigl[N^2\log(eNT)\bigr]^{-2\beta/d}\\
&\le\kappa C_2^2c_-^{-2\beta/d}M^{-\frac{2\beta}{2\beta+d}}.
\end{aligned}
\end{equation}

For the estimation term, the only additional factor is $\log M$.
The inequality $\log x\le x/e$, applied to
$x=M^{\frac{d}{2\beta+d}}/N^2$, and the upper comparison in
\eqref{eq:width-radius-curve} give
\begin{equation*}
\begin{aligned}
N^2\log M
&=\frac{2\beta+d}{d}\left\{
 2N^2\log N+N^2\log\left(\frac{M^{\frac{d}{2\beta+d}}}{N^2}\right)\right\}\\
&\le\frac{2\beta+d}{d}\left\{
 2N^2\log(eNT)+e^{-1}M^{\frac{d}{2\beta+d}}\right\}\\
&\le\frac{2\beta+d}{d}(2c_++e^{-1})M^{\frac{d}{2\beta+d}}.
\end{aligned}
\end{equation*}
Since $\log(N+1)\le\log(eN)$ and $\log(eNT)\ge1$, it follows that
\begin{equation}\label{eq:curve-full-cover-bound}
\begin{aligned}
N^2\log\bigl(48eM(N+1)^{19}T^{20}\bigr)
&=N^2\log M+N^2\{\log(48e)+19\log(N+1)+20\log T\}\\
&\le N^2\log M+\{\log(48e)+20\}N^2\log(eNT)\\
&\le\left\{\frac{2\beta+d}{d}(2c_++e^{-1})
 +c_+\bigl(\log(48e)+20\bigr)\right\}M^{\frac{d}{2\beta+d}}\\
&=:K_+M^{\frac{d}{2\beta+d}}.
\end{aligned}
\end{equation}
Using Corollary~\ref{cor:dta-cover} with $L=19$ and
$\eps=(48M)^{-1}$, we obtain
\begin{equation}\label{eq:regression-cover-bound}
\begin{aligned}
\log\left[2\mathcal N\left(\frac1{48M},\mathcal G,2M\right)\right]+1
&\le
\log 2+1
+C_d(dN+19N^2)
 \log\bigl(48eM(N+1)^{19}T^{20}\bigr)\\
&\le
\log 2+1
+20C_dN^2
 \log\bigl(48eM(N+1)^{19}T^{20}\bigr)\\
&\le
\log 2+1
+20C_dK_+M^{\frac{d}{2\beta+d}}\\
&\le
\bigl(\log 2+1+20C_dK_+\bigr)
M^{\frac{d}{2\beta+d}}.
\end{aligned}
\end{equation}
Here $C_d=8+4\log_2(d+1)$ is the constant from
Corollary~\ref{cor:dta-cover}. Since $N\ge2d+3$, we have
$dN+19N^2\le20N^2$; the remaining bounds use
\eqref{eq:curve-full-cover-bound} and $M^{\frac{d}{2\beta+d}}\ge1$.

Continuous realization maps send the finitely many parameter cubes to
compact subsets of the supremum-norm function space, so $\mathcal G$ is
separable.  Applying Proposition~\ref{prop:oracle} with $\xi=M^{-1}$,
then using \eqref{eq:regression-bias-bound},
\eqref{eq:regression-cover-bound}, and
$M^{-1}\le M^{-\frac{2\beta}{2\beta+d}}$, gives
\begingroup
\begin{align*}
&\mathbb E\norm{\widehat f_M-f}_{L^2(\mu)}^2\\[-2pt]
&\le
4\inf_{g\in\mathcal G}\norm{g-f}_{L^2(\mu)}^2+\frac3M
+\frac{83+37\sigma^2}{M}
 \left\{\log\left[2\mathcal N\!\left(\frac1{48M},\mathcal G,2M\right)\right]+1\right\}\\
&\le
4\kappa C_2^2c_-^{-2\beta/d}M^{-\frac{2\beta}{2\beta+d}}
+\frac3M
+(83+37\sigma^2)
\left\{
\frac{\log2+1}{M}
+20C_dK_+M^{-\frac{2\beta}{2\beta+d}}
\right\}\\
&\le
\left\{
4\kappa C_2^2c_-^{-2\beta/d}
+(83+37\sigma^2)(\log2+1+20C_dK_+)
\right\}
M^{-\frac{2\beta}{2\beta+d}}
+\frac3M\\
&\le
\left\{
4\kappa C_2^2c_-^{-2\beta/d}+3
+(83+37\sigma^2)(\log2+1+20C_dK_+)
\right\}
M^{-\frac{2\beta}{2\beta+d}}.
\end{align*}
\endgroup
The expression in braces defines $C_3$ and depends only on
$d,\beta,\kappa,\sigma,c_-,c_+$.  Taking the supremum over
$f\in\mathcal H^\beta([0,1]^d)$ completes the proof.
\end{proof}

\section{Conclusion}\label{sec:conclusion}

We have established the sharp width--radius approximation law
$\Theta([N^2\log(eNT)]^{-\beta/d})$ for the unit $\beta$-H\"older ball,
$\beta\in(0,1]$, in every finite $L^p$, at fixed depth. One explicit
bounded $1$-Lipschitz activation, $\DTA$, attains this order for all
$N\ge2d+3$ and $T\ge1$, matching the lower bound for globally
H\"older activations at fixed depth. With $B$ binary digits per weight
or bias, the same depth also attains the optimal rate
$\Theta([N^2B]^{-\beta/d})$ for all independent widths $N\ge2d+3$
and bit budgets $B\ge2$, matching the lower bound for arbitrary
Borel measurable activations at fixed depth.
At fixed size, four hidden layers with at most $8d+7$ nonzero
parameters give
$\log T=\mathcal{O}(\varepsilon^{-d/\beta}\log(\varepsilon^{-1}))$,
while six hidden layers with at most $8d+27$ nonzero parameters
remove the logarithmic factor and attain the optimal worst-case order
$\log T=\Theta(\varepsilon^{-d/\beta})$.
The decoding method also yields fixed-size Transformer approximation with
$\log T=\mathcal{O}(\varepsilon^{-dn/\beta})$.
Under our sampling assumptions, approximate least squares over the
corresponding full clipped classes at depth $19$ attains the minimax
squared risk rate $\mathcal{O}(M^{-2\beta/(2\beta+d)})$ without
logarithmic loss along either
$N^2\log(eNT)=\Theta(M^{d/(2\beta+d)})$
or $N^2B=\Theta(M^{d/(2\beta+d)})$.
These two laws permit different allocations of width and parameter
radius or precision with the same minimax risk order, including fixed
width, unit radius, and two bits per parameter. We further extend the optimal joint width--bits approximation rate and
the resulting minimax regression tradeoff to a significantly broader
class of activation functions beyond $\DTA$, including many with desirable numerical
and analytical properties.

Given the present constructions, it would not be surprising if similar
methods yielded the optimal joint width--depth--radius and
width--depth--bits rates $\Theta([N^2L^2\log(eNT)]^{-\beta/d})$ and
$\Theta([N^2LB]^{-\beta/d})$, respectively, naturally extending the
statistically optimal curves to the surfaces
$N^2L^2\log(eNT)=\Theta(M^{d/(2\beta+d)})$ and
$N^2LB=\Theta(M^{d/(2\beta+d)})$.
A more fundamental problem is to fully characterize the class of elementary Lipschitz
activations beyond $\DTA$ that attain these optimal approximation and
statistical rates. Developing efficient and numerically stable training
algorithms for such networks is another important direction.

\clearpage
\pdfbookmark[1]{References}{references}
\bibliographystyle{asa-compatible}
\bibliography{references}
\clearpage

\appendix

\section{Proof of Proposition~\ExternalNumber{3.1}{prop:decoder}}\label{app:orbit}

The radius estimate in Proposition~\ExternalNumber{3.1}{prop:decoder}
leads to a quantitative Kronecker problem: given a target vector, find one
integer $m$ that approximates all its dyadic phases, with a bound on $|m|$
that is uniform over the targets.  Qualitative orbit density does not give
the parameter bound needed here.  We resolve this problem for the present
frequencies by combining an algebraic separation estimate with Gaussian
Fourier analysis.  The separation estimate excludes short dual-lattice
vectors; comparing periodized Gaussian sums at two scales then forces a
primal-lattice point near every target.  A suitable scaling converts that
point into the required integer $m$.  Effective Kronecker results in more
general settings are developed by
\citet{GonekMontgomery2016,FukshanskyMoshchevitin2018}; the argument below
uses the specific dyadic frequencies to obtain the explicit bound required
by our decoder.

\begin{lemma}[Dyadic orbit density]\label{lem:orbit}
Let $N\in\mathbb N$ and let $\bm\alpha=(\alpha_1,\ldots,\alpha_N)^\top$
have coordinates $\alpha_i=2^{-i/(N+1)}$ for $i=1,\ldots,N$.
For every $\delta\in(0,1)$ and
$\bm t\in\R^N$, there exist $m\in\Z$ and $\bm p\in\Z^N$ such that
\[
\abs m\le\bigl(10\delta^{-1}N^{3/2}\bigr)^N,
\quad
\norm{m\bm\alpha-\bm p-\bm t}_\infty\le\delta.
\]
\end{lemma}

\begin{proof}
Set $\balpha=(\alpha_1,\ldots,\alpha_N)^\top$ and write the prescribed target
as $\bt=(t_1,\ldots,t_N)^\top$.

\par\medskip\noindent
\textbf{Step 1: algebraic separation of the frequencies.}
We first show that no nonzero integer combination of the dyadic frequencies
can lie too close to an integer.  This is the only place where the special
choice $\alpha_i=2^{-i/(N+1)}$ is used.  For every nonzero $\bh\in\Z^N$,
\[
\norm{\bh\cdot\balpha}_{\Torus}\ge 5^{-N}\norm{\bh}_1^{-N},
\qquad
\norm{x}_{\Torus}:=\min_{m\in\Z}|x-m|.
\]
Let $A=\norm{\bh}_1\ge1$ and $\theta=2^{1/(N+1)}$, and choose a nearest
integer $k\in\Z$ such that
\[
\left|k+\sum_{i=1}^Nh_i\theta^{-i}\right|=\left\|\sum_{i=1}^Nh_i\theta^{-i}\right\|_{\Torus}.
\]
Denote $\zeta=e^{2\pi i/(N+1)}$ and let $r_j=\theta\zeta^j$, $j=0,\ldots,N$.  These are precisely the distinct roots of $x^{N+1}-2$.  Define
\[
L_j:=k+\sum_{i=1}^Nh_ir_j^{-i}\quad (j=0,\ldots,N),
\qquad
P(x):=kx^N+h_1x^{N-1}+h_2x^{N-2}+\cdots+h_N.
\]
Then $P(r_j)=r_j^NL_j$.  By Eisenstein's criterion at the prime $2$,
$x^{N+1}-2$ is irreducible over $\Q$, so every $r_j$ has degree $N+1$.
The polynomial $P$ is nonzero because some $h_i$ is nonzero, and
$\deg P\le N$; hence $P(r_j)\ne0$ for every $j$. Because \(r_j^N L_j=P(r_j)\neq0\), we have
\begin{equation}\label{eq:algebraic-norm}
\left(\prod_{j=0}^N r_j^N\right)
\left(\prod_{j=0}^N L_j\right)
=
\prod_{j=0}^N P(r_j)
=
N_{\mathbb Q(\theta)/\mathbb Q}\bigl(P(\theta)\bigr)
\in\mathbb Z\setminus\{0\}.
\end{equation}
Indeed, the embeddings of $\mathbb Q(\theta)$ send $\theta$ to
$r_0,\ldots,r_N$, which gives the product in \eqref{eq:algebraic-norm}.
Since $P(\theta)$ is a nonzero algebraic integer, its norm is a nonzero
integer; see Neukirch~\citeyearpar[Chapter~I, \S~2]{Neukirch1999}.
Taking absolute values and using
\(
\left|\prod_{j=0}^N r_j\right|=2
\)
gives
\[
1
\le
\left|\left(\prod_{j=0}^N r_j^N\right)
\left(\prod_{j=0}^N L_j\right)\right|
=
\left|\prod_{j=0}^N r_j\right|^{N}
\prod_{j=0}^N|L_j|
=
2^N\prod_{j=0}^N|L_j|.
\]
Equivalently,
$
\prod_{j=0}^N|L_j|\ge2^{-N}
$.
It remains to upper bound the $N$ conjugate factors $L_j$, $j\ne0$.  By the choice of $k$,
\[
|k|\le \left|\sum_{i=1}^Nh_i\theta^{-i}\right|+\frac12
\le \sum_{i=1}^N|h_i|+\frac12
\le A+\frac A2=\frac32A,
\]
where we used $A\ge1$.  Also $|r_j^{-i}|=\theta^{-i}\le1$.  Hence, for $j\ne0$,
\[
|L_j|\le |k|+\sum_{i=1}^N|h_i||r_j|^{-i}\le \frac32A+A=\frac52A.
\]
Combining this estimate with $\prod_{j=0}^{N}|L_j|\ge2^{-N}$ gives
\[
\norm{\bh\cdot\balpha}_{\Torus}=\left|k+\sum_{i=1}^Nh_i\theta^{-i}\right|=|L_0|\ge 2^{-N}\left(\frac52A\right)^{-N}=5^{-N}\norm{\bh}_1^{-N}.
\]
This proves the no-resonance estimate.  We next turn such a lower bound on
the dual lattice into a covering bound for the primal lattice.

\par\medskip\noindent
\textbf{Step 2: Gaussian Fourier analysis and lattice covering.}
Let $r\ge2$, let $\bm B$ be an invertible $r\times r$ real matrix, and set
$G=\bm B\Z^r$ and $G^*=\bm B^{-T}\Z^r$.  Suppose every nonzero
$\bw\in G^*$ satisfies $\norm{\bw}_1\ge\Lambda$.  We claim that, for every
$\by\in\R^r$, there is $\bv\in G$ such that
\[
\norm{\bv-\by}_2\le\frac{r^{3/2}}{\Lambda}.
\]
We compare two periodized Gaussian sums. Their Fourier expansions give
bounds in terms of the dual lattice. If every lattice point were farther
than the claimed distance from the target, the ratio of the sums would
contradict those bounds.
From $\norm{\bw}_1\le\sqrt r\norm{\bw}_2$, every nonzero point of $G^*$ has Euclidean norm at least $\lambda:=\Lambda/\sqrt r$.
Use the reference scale $s_0:=\sqrt r/\lambda=r/\Lambda$.
For $a>0$ and $\by,\bu\in\R^r$, write $\bv=\bm B\bn$ with
$\bn\in\Z^r$, and define
\[
S_a(\by):=\sum_{\bv\in G}
\exp\!\left(-\pi\frac{\norm{\bv-\by}_2^2}{a^2}\right),
\qquad
F_{\by}(\bu):=\sum_{\bn\in\Z^r}
\exp\!\left(-\pi\frac{\norm{\bm B(\bu+\bn)-\by}_2^2}{a^2}\right).
\]
Gaussian decay makes $F_{\by}$ smooth and $\Z^r$-periodic, with
$S_a(\by)=F_{\by}(\bzero)$, and justifies termwise integration.
To compute $S_a$, first unfold the Fourier coefficient at $\bk\in\Z^r$:
\begin{align*}
\widehat F_{\by}(\bk)
&=\sum_{\bn\in\Z^r}\int_{[0,1]^r}
 \exp\left(-\pi\frac{\|\bm B(\bu+\bn)-\by\|_2^2}{a^2}\right)
 e^{-2\pi i\bk\cdot\bu}\,\dd\bu\\
&=\sum_{\bn\in\Z^r}\int_{\bn+[0,1]^r}
 \exp\left(-\pi\frac{\|\bm B\bv-\by\|_2^2}{a^2}\right)
 e^{-2\pi i\bk\cdot(\bv-\bn)}\,\dd\bv\\
&=\int_{\R^r}\exp\left(-\pi\frac{\|\bm B\bv-\by\|_2^2}{a^2}\right)
 e^{-2\pi i\bk\cdot\bv}\,\dd\bv,
\end{align*}
since $\bk\cdot\bn\in\Z$ and the translated unit cubes partition
$\R^r$ up to their boundaries.
With $\bz=\bm B\bv-\by$, so $\bv=\bm B^{-1}(\bz+\by)$ and
$\dd\bv=|\det\bm B|^{-1}\dd\bz$, this becomes
\[
\widehat{F}_{\by}(\bk)=\frac{e^{-2\pi i(\bm{B}^{-T}\bk)\cdot\by}}{|\det \bm{B}|}
\int_{\R^r}e^{-\pi\norm{\bz}_2^2/a^2}e^{-2\pi i(\bm{B}^{-T}\bk)\cdot\bz}\,\dd\bz.
\]
The integral is computed coordinate by coordinate.  For real $\xi$,
\begin{align*}
\int_{\R}e^{-\pi t^2/a^2}e^{-2\pi i\xi t}\,\dd t
=e^{-\pi a^2\xi^2}
  \int_{\R+i a^2\xi}e^{-\pi z^2/a^2}\,\dd z
=e^{-\pi a^2\xi^2}
  \int_{\R}e^{-\pi z^2/a^2}\,\dd z
 =a e^{-\pi a^2\xi^2}.
\end{align*}
For the contour shift, apply Cauchy's theorem on the rectangle with vertical
sides at $\pm R$; their integrals tend to zero as $R\to\infty$ because the
height is fixed and the Gaussian factor is $\mathcal O(e^{-\pi R^2/a^2})$.
Consequently,
\begin{align*}
\widehat{F}_{\by}(\bk)
&=
\frac{e^{-2\pi i(\bm B^{-T}\bk)\cdot\by}}{|\det\bm B|}
\prod_{j=1}^r
\int_{\R}e^{-\pi z_j^2/a^2}
e^{-2\pi i(\bm B^{-T}\bk)_jz_j}\,\dd z_j\\
&=
\frac{e^{-2\pi i(\bm B^{-T}\bk)\cdot\by}}{|\det\bm B|}
\prod_{j=1}^r a e^{-\pi a^2(\bm B^{-T}\bk)_j^2}\\
&=
\frac{a^r}{|\det\bm B|}
\exp\!\left(-\pi a^2\norm{\bm B^{-T}\bk}_2^2\right)
e^{-2\pi i(\bm B^{-T}\bk)\cdot\by}.
\end{align*}
The Fourier coefficients have Gaussian decay on $G^*$, so the
Fourier series converges absolutely to $F_{\by}$.  Evaluating it at zero
gives
\begin{align}
S_a(\by)
=\sum_{\bk\in\Z^r}\widehat F_{\by}(\bk)\nonumber
&=\frac{a^r}{|\det \bm{B}|}
  \sum_{\bk\in\Z^r}
  \exp\left(-\pi a^2\norm{\bm{B}^{-T}\bk}_2^2\right)
  e^{-2\pi i(\bm{B}^{-T}\bk)\cdot\by}\nonumber\\
&=\frac{a^r}{|\det \bm{B}|}
  \sum_{\bw\in G^*}
   e^{-\pi a^2\norm{\bw}_2^2}
   e^{-2\pi i\bw\cdot\by}.
\label{eq:poisson-gaussian-lattice}
\end{align}

We next estimate the nonzero lattice contribution.  For $k\ge1$, let
\[
A_k:=\Big\{\bw\in G^*: k\lambda\le\norm{\bw}_2<(k+1)\lambda \Big\}.
\]
The balls $B_{\lambda/2,|\cdot|}(\bw)$, $\bw\in A_k$, are pairwise
disjoint and lie in $B_{(k+3/2)\lambda,|\cdot|}(\bzero)$.  Indeed, distinct
centers differ by a nonzero vector of $G^*$ and are therefore at least
$\lambda$ apart, and the triangle inequality gives the containment.
Comparing volumes therefore yields
\[
\#A_k\,\vol\left( B_{\lambda/2,|\cdot|}(\bzero)\right)
\le \vol \left(B_{(k+3/2)\lambda,|\cdot|}(\bzero)\right),
\implies \#A_k\le(2k+3)^r.
\]
Since $s_0^2\lambda^2=r$,
\begin{align*}
\sum_{\bzero\ne\bw\in G^*}e^{-\pi s_0^2\norm{\bw}_2^2}
&=\sum_{k=1}^\infty\sum_{\bw\in A_k}e^{-\pi s_0^2\norm{\bw}_2^2}\le \sum_{k=1}^\infty \#A_k e^{-\pi s_0^2k^2\lambda^2}\le\sum_{k=1}^\infty\bigl((2k+3)e^{-\pi k^2}\bigr)^r.
\end{align*}
For $k\ge1$, $(2k+3)e^{-\pi k^2}\le5k e^{-\pi k^2}\le5e^{-\pi}<1$.  Since $r\ge2$, each base is less than one, so raising it to the power $r$ only decreases it; hence
\[
\sum_{k=1}^\infty\big((2k+3)e^{-\pi k^2}\big)^r
\le \sum_{k=1}^\infty 5k e^{-\pi k^2}
\le 5e^{-\pi}+5\int_1^\infty xe^{-\pi x^2}\,\dd x=5e^{-\pi}\left(1+\frac1{2\pi}\right)<\frac1{3},
\]
because $x\mapsto xe^{-\pi x^2}$ is decreasing on $[1,\infty)$.  Thus, with
\[
R_a:=\sum_{\bzero\ne\bw\in G^*}e^{-\pi a^2\norm{\bw}_2^2},
\]
we have $R_{s_0}<1/3$.  Each summand decreases with $a$, so also
$R_{\sqrt2s_0}\le R_{s_0}<1/3$.
Using \eqref{eq:poisson-gaussian-lattice} first with $a=s_0$ and then with $a=\sqrt2s_0$, together with
\[
\left|\frac{|\det \bm{B}|}{a^r}S_{a}(\by)-1\right|
=
\left|
\sum_{\bzero\ne\bw\in G^*}
e^{-\pi a^2\norm{\bw}_2^2}
e^{-2\pi i\bw\cdot\by}
\right|
\le
\sum_{\bzero\ne\bw\in G^*}
e^{-\pi a^2\norm{\bw}_2^2}
=
R_a,
\]
yields
\[
S_{s_0}(\by)\ge\frac{s_0^r}{|\det \bm{B}|}(1-R_{s_0})>\frac{2s_0^r}{3|\det \bm{B}|},
\quad
S_{\sqrt2s_0}(\by)\le\frac{(\sqrt2s_0)^r}{|\det \bm{B}|}(1+R_{\sqrt2s_0})
<\frac{4\,2^{r/2}s_0^r}{3|\det \bm{B}|}.
\]
Suppose no $\bv\in G$ satisfies $\norm{\bv-\by}_2\le s_0\sqrt r$.
Then $\norm{\bv-\by}_2^2/s_0^2>r$ for every $\bv\in G$.
Exponentiating and summing over these lattice points gives
\[
S_{s_0}(\by)=\sum_{\bv\in G}\exp\left(-\pi\frac{\norm{\bv-\by}_2^2}{s_0^2}\right)
\le 
\sum_{\bv\in G}e^{-\pi r/2}\exp\left(-\pi\frac{\norm{\bv-\by}_2^2}{2s_0^2}\right)
=e^{-\pi r/2}S_{\sqrt2s_0}(\by).
\]
Combining the preceding
bounds and canceling $s_0^r/(3|\det\bm B|)$ gives
\[
2<4\cdot\,2^{r/2}e^{-\pi r/2}=4\cdot\big(\sqrt2 e^{-\pi/2}\big)^{r}\le 4\cdot\big(\sqrt2/4\big)^{r}\le \frac12.
\]
which is impossible.  Hence some $\bv\in G$ satisfies
$\norm{\bv-\by}_2\le s_0\sqrt r$.  Finally,
\[
\norm{\bv-\by}_\infty\le\norm{\bv-\by}_2\le s_0\sqrt r=\frac{r}{\lambda}=\frac{r^{3/2}}{\Lambda}.
\]
The covering estimate is now proved.  We will finish the proof by building a lattice with $r=N+1$ and $\Lambda=r^{3/2}/\delta$ in the next step.

\par\medskip\noindent
\textbf{Step 3: convert the lattice bound into a phase approximation.}
The first $N$ lattice coordinates measure $\bp-q\balpha+\bt$, and the last one forces $|q|\le Q$.  Set $Q=(10N^{3/2}/\delta)^N$, $r=N+1$, and
\[
G=\left\{\binom{\bp-q\balpha}{(\delta/Q)q}:\bp\in\Z^N,
q\in\Z\right\}=\bm{B}\Z^r,
\qquad
\bm{B}=\begin{pmatrix}\bm{I}_N&-\balpha\\ \bzero^\top&\delta/Q\end{pmatrix}.
\]
Then
\[
\bm{B}^{-T}=\begin{pmatrix}\bm{I}_N&\bzero\\ (Q/\delta)\balpha^\top&Q/\delta\end{pmatrix},
\]
which follows by inverting the block upper-triangular matrix $\bm{B}$ and then transposing,
so every vector in $G^*=\bm{B}^{-T}\Z^r$ has the form
\[
\bw(\bh,m)=\binom{\bh}{(Q/\delta)(\bh\cdot\balpha+m)},
\qquad \bh\in\Z^N,\quad m\in\Z.
\]
If $\bh=\bzero$ and $\bw(\bh,m)\ne\bzero$, then $m\ne0$.  Since $\delta\in(0,1)$,
$Q\ge10N^{3/2}\ge(N+1)^{3/2}=r^{3/2}$, and hence
$\norm{\bw(\bh,m)}_1=(Q/\delta)|m|\ge r^{3/2}/\delta$.  If $\bh\ne\bzero$, then, using $|\bh\cdot\balpha+m|\ge\norm{\bh\cdot\balpha}_{\Torus}$ and Step~1,
\[
\norm{\bw(\bh,m)}_1\ge \norm{\bh}_1+\frac Q\delta\norm{\bh\cdot\balpha}_{\Torus}
\ge \frac{\delta \norm{\bh}_1+Q5^{-N}\norm{\bh}_1^{-N}}{\delta}.
\]
It remains to lower bound the numerator.  Minimizing over all real $x=\norm{\bh}_1>0$
can only decrease it.  For
$
\psi(x)=\delta x+Q5^{-N}x^{-N}
$, the unique critical point is
$
x_*=\left(NQ5^{-N}\delta^{-1}\right)^{1/(N+1)}
$.
It is the minimum because $\psi$ is strictly convex.  At $x_*$, the identity
$Q5^{-N}=(2N^{3/2}/\delta)^N$ gives
\begin{align*}
\min_{x>0}\psi(x)
&=(N+1)N^{-N/(N+1)}\delta^{N/(N+1)}(Q5^{-N})^{1/(N+1)}\\
&=(N+1)N^{-N/(N+1)}(2N^{3/2})^{N/(N+1)}\\
&=(N+1)\left[(4N)^{N}\right]^{\frac{1}{2(N+1)}}\\
&\ge(N+1)\left[(N+1)^{N+1}\right]^{\frac{1}{2(N+1)}}\\
&=r^{3/2}.
\end{align*}
In the penultimate line we used $(4N)^N\ge (N+1)^{N+1}$. Thus every nonzero vector
of $G^*$ has $\ell^1$ norm at least $r^{3/2}/\delta$.  Step~2, with
$\Lambda=r^{3/2}/\delta$ and target $\by=(-\bt^\top,0)^\top$, now gives a
lattice point of $G$ within $\delta$ in $\ell^\infty$ norm.  By the definition
of $G$, there are $\bp\in\Z^N$ and $q\in\Z$ such that
\[
\norm{\bp-q\balpha+\bt}_\infty\le\delta,
\qquad
\left|(\delta/Q)q\right|\le\delta.
\]
The second inequality gives $|q|\le Q$, and the first is equivalent to
$\norm{q\balpha-\bp-\bt}_\infty\le\delta$.  Since
$Q=(10N^{3/2}/\delta)^N$, this is the claimed bound.
\end{proof}

Apply the orbit lemma to the table in Proposition~\ExternalNumber{3.1}{prop:decoder}.  The
integers $p_i$ and residuals $e_i$ below are the integer and residual parts of
that approximation.

\begin{proof}[Proof of Proposition~\ExternalNumber{3.1}{prop:decoder}]
Let $y_i=(c_i+1)/2\in[0,1]$ for $i=1,...,N$.  Apply Lemma~\ref{lem:orbit} to
$(y_1/2,\ldots,y_N/2)^\top$ with tolerance $\eps/4$.  Then there exist
$m,p_1,\ldots,p_N\in\Z$ and $e_1,\ldots,e_N\in\R$ such that
\begin{align*}
|m|\le\left(40N^{3/2}\eps^{-1}\right)^N,\quad
m\,2^{-i/(N+1)}=p_i+\frac{y_i}{2}+e_i,
\quad |e_i|\le\frac\eps4,
\quad (i=1,..,N).
\end{align*}
Define, as in the proposition, the two affine maps $\mathcal B_1(v)=2mv+4|m|+2$, $\mathcal B_2(w)=2w-1$, together with the $\DTA$-network $\mathcal R=\mathcal B_2\circ\DTA\circ \mathcal B_1$.

Since
\begin{equation}\label{eq:bound on B1}
\mathcal B_1(v)=2mv+4|m|+2\ge 2|m|+2\ge2,
\qquad v\in(-1, 1],
\end{equation}
the activation identities in
(\ExternalNumber{2.2}{eq:dta-identities})
give $\DTA(\mathcal B_1(v))=\tau(\mathcal B_1(v))$.
For each $i=1,\ldots,N$, we compute
\begin{align*}
\mathcal R\left(2^{-i/(N+1)}-1\right)
&=2\DTA\left(2m\left(2^{-i/(N+1)}-1\right)
    +4|m|+2\right)-1\\
&=2\tau\left(2m\,2^{-i/(N+1)}
    -2m+4|m|+2\right)-1\\
&=2\tau\Big(\underbrace{\big(2p_i-2m+4|m|+2\big)}_{\in\,2\Z}
    +y_i+2e_i\Big)-1\\
&=2\tau(y_i+2e_i)-1,
\end{align*}
where the last equality follows from the $2$-periodicity of
$\tau$. Since $\tau(y_i)=y_i=(c_i+1)/2$ and $\tau$ is
$1$-Lipschitz, we obtain
\begin{align*}
\left|\mathcal R\left(2^{-i/(N+1)}-1\right)-c_i\right|
&=2|\tau(y_i+2e_i)-\tau(y_i)|
\le4|e_i|\le\eps.
\end{align*}
Equation~\eqref{eq:bound on B1} also gives $|\mathcal R(v)|
=|2\tau(\mathcal B_1(v))-1|\le1$ for every $v\in(-1,1]$, proving the uniform
output bound.
Finally, the parameter radius of $\mathcal R$ is bounded by
\[
\max\{2|m|,4|m|+2,2,1\}
=4|m|+2
\le6\left(40N^{3/2}\eps^{-1}\right)^N,
\]
as stated.
\end{proof}

\section{Coupled block decoder: Proof of Proposition~\ExternalNumber{3.4}{prop:block-decoder}}
\label{app:block-decoder}

We approximate an integer sequence whose increments lie in
$\{-1,0,1\}$ within each row.  Each short block is stored through its
starting value and increments.  In the width-two construction, two
approximate evaluations have a common error that cancels when forming
the next decoder input.  The wider construction preserves this cancellation
while distributing the data among shorter integers.

\subsection{A phase-coding lemma}
The first lemma stores any finite list of numbers in $[\zeta,1-\zeta]$.
Dividing the resulting integer by a power of $G$ recovers one prescribed
number up to an integer and a small error.  Keeping the numbers away from
$0$ and $1$ keeps each perturbation within the same linear branch of the triangular wave.

\begin{lemma}[Guarded phase code]\label{lem:guarded-code}
Let $N\in\mathbb N$, $\zeta\in(0,1/2]$, and
$\theta_1,\ldots,\theta_N\in[\zeta,1-\zeta]$.  For every integer
$G\ge2/\zeta$, there are $m\in\{0,\ldots,G^N-1\}$, an
integer-valued function $\bm k:\{1,\ldots,N\}\to\mathbb N_0$, and an error
function $\bm e:\{1,\ldots,N\}\to\mathbb R$ such that
\[
mG^{-i}=\bm k[i]+\theta_i+\bm e[i],
\qquad
|\bm e[i]|\le\frac1{2G}\le\frac\zeta4,
\qquad i=1,...,N.
\]
\end{lemma}
\begin{proof}
Choose the base-$G$ digits successively. The quantity $r_i$ is the
contribution of the already chosen digits, and $q_i$ rounds the remaining
amount needed to match $\theta_i$. For $i=1,\ldots,N$, define
\[
r_i:=\sum_{j=1}^{i-1}q_jG^{j-i-1},
\qquad
q_i:=\left\lfloor G(\theta_i-r_i)+\frac12\right\rfloor,
\]
where the empty sum gives $r_1=0$.  We verify inductively that
$q_i\in\{0,\ldots,G-1\}$.  For $i=1$, since
$\theta_1\in[\zeta,1-\zeta]$ and $G\zeta\ge2$, we have
$0<G\theta_1+1/2<G$, and hence $q_1\in\{0,\ldots,G-1\}$.  If
$2\le i\le N$ and $q_1,\ldots,q_{i-1}\in\{0,\ldots,G-1\}$, then
\[
0\le r_i\le(G-1)\sum_{\ell=2}^{i}G^{-\ell}
=G^{-1}\bigl(1-G^{-(i-1)}\bigr)
<G^{-1}\le(2/\zeta)^{-1}.
\]
Thus $\zeta/2\le\theta_i-r_i\le1-\zeta$, and consequently
$0<G(\theta_i-r_i)+1/2<G$, which gives
$q_i\in\{0,\ldots,G-1\}$ and completes the induction.

Since $q_i=\lfloor G(\theta_i-r_i)+1/2\rfloor$, the definition of the
floor function gives
$q_i-1/2\le G(\theta_i-r_i)<q_i+1/2$, and hence
\[
\left|G^{-1}q_i+r_i-\theta_i\right|
=G^{-1}\left|q_i-G(\theta_i-r_i)\right|
\le(2G)^{-1}.
\]
Now let
\[
m:=\sum_{j=1}^{N}q_jG^{j-1},
\qquad
\bm k[i]:=\sum_{j=i+1}^{N}q_jG^{j-i-1},
\qquad
\bm e[i]:=G^{-1}q_i+r_i-\theta_i.
\]
Since $0\le q_j\le G-1$, we have $0\le m<G^N$, while $\bm k[i]\in\mathbb N_0$.
Separating the $i$th digit gives
\begin{align*}
mG^{-i}
&=\sum_{j=i+1}^{N}q_jG^{j-i-1}+G^{-1}q_i
  +\sum_{j=1}^{i-1}q_jG^{j-i-1}
 =\bm k[i]+G^{-1}q_i+r_i
 =\bm k[i]+\theta_i+\bm e[i].
\end{align*}
Lastly, $|\bm e[i]|\le1/(2G)\le\zeta/4$ because $G\ge2/\zeta$.
\end{proof}

\subsection{Proof of Proposition~\ExternalNumber{3.4}{prop:block-decoder}, part~(a)}\label{app:coupled-a-proof}
\begin{proof}[Proof of Proposition~\ExternalNumber{3.4}{prop:block-decoder}, part~(a)]
Let $H,Q,\Upsilon$ and $(a_t)_{t=0}^{Q-1}$ satisfy the hypotheses of the proposition.
For $n=\lceil\log_2 \Upsilon\rceil$, choose
\[
\Gamma=
\begin{cases}
1,&n=2,...,7,\\
2^{\lfloor\log_2(n/4)\rfloor},&n\ge8.
\end{cases}
\]
Since $4\le \Upsilon\le H$ and $H$ is a power of two, we have
$2^{n-1}<\Upsilon\le2^n\le H\le Q$. By construction, $\Gamma$ is a power
of two, and $4\Gamma\le n<8\Gamma$ whenever $\Gamma\ge2$.
We shall use $\Gamma^2\le2^\Gamma$, which follows from $2\ell\le2^\ell$
for $\ell\in\mathbb N_0$ upon writing $\Gamma=2^\ell$.

\begin{itemize}
\item \textit{Conditions on the block length.}
We verify that $\Gamma\mid H$, $1\le \Gamma\le \Upsilon/4$, and $\Gamma\le \Upsilon^{1/5}$,
as required in the proposition. These assertions follow from
$\Upsilon\ge4$ when $\Gamma=1$. For $\Gamma\ge2$, we have
$4\Gamma\le n\le2^{n-1}<\Upsilon$ and
$\Gamma^5\le2^{5\Gamma/2}\le2^{n-1}<\Upsilon$, using $\Gamma^2\le2^\Gamma$ and $4\Gamma\le n$.
Since $\Gamma\le \Upsilon/4\le H/4$ and both $\Gamma$ and $H$ are powers of two,
we also have $\Gamma\mid H$.

\item \textit{Bounds for the encoding and the parameter estimates.}
In Steps~1 and~5, and in the proof of part~(b), we will use
\begin{equation}\label{eq:block-length-bounds}
3^{\Gamma-1}\in[\Gamma, 2\Upsilon^{2/5}),
\qquad
4\Gamma^23^{\Gamma-1}\le H.
\end{equation}
The inequality $\Gamma\le3^{\Gamma-1}$ follows from Bernoulli's inequality,
and the two upper bounds are immediate when $\Gamma=1$.
For $\Gamma\ge2$, the inequality $\log_2 3<8/5$ and $4\Gamma\le n$ give
$3^{\Gamma-1}<2^{8\Gamma/5}\le2^{2n/5}<2\Upsilon^{2/5}$.
Also, $\Gamma^2\le2^\Gamma$ yields the stronger estimate
$16\Gamma^23^{\Gamma-1}\le2^{3\Gamma+2}\le2^{4\Gamma}\le2^n\le H$,
which will be used in Step~5.
\end{itemize}

\par\medskip\noindent
\textbf{Step 1: encoding the starting values and increments.}
Since $\Gamma\mid H$ and $H\mid Q$, the sequence $(a_t)_{t=0}^{Q-1}$
splits into $Q/\Gamma$ consecutive blocks of length $\Gamma$, each contained
in a row of length $H$. We encode each block by its starting value
$a_{b\Gamma}$ and its $\Gamma-1$ successive increments. Define
\[
\left\{
\begin{aligned}
&w(b):=\sum_{j=1}^{\Gamma-1}
\bigl(a_{b\Gamma+j}-a_{b\Gamma+j-1}+1\bigr)3^{j-1},
&& b\in\{0,\ldots,Q/\Gamma-1\},\\[2mm]
&s(w,r):=\sum_{j=1}^{r}
\left(
\left\lfloor\frac{w}{3^{j-1}}\right\rfloor
-3\left\lfloor\frac{w}{3^j}\right\rfloor-1
\right),
&&
\begin{aligned}
&w\in\{0,\ldots,3^{\Gamma-1}-1\},\quad
r\in\{0,\ldots,\Gamma-1\}.
\end{aligned}
\end{aligned}
\right.
\]
where an empty sum is understood to be zero.
For the reconstruction in Step~3, we need $w(b)\in\{0,\ldots,3^{\Gamma-1}-1\}$,
$s(w(b),r)=a_{b\Gamma+r}-a_{b\Gamma}$, and the bound $|s(w,r)|\le r$ on the
entire domain of $s$.

The adjacent-difference hypothesis applies within each block, so
$a_{b\Gamma+j}-a_{b\Gamma+j-1}+1\in\{0,1,2\}$ for $j=1,\ldots,\Gamma-1$.
In particular, $w(b)$ is an integer with
$0\le w(b)\le2\sum_{j=1}^{\Gamma-1}3^{j-1}=3^{\Gamma-1}-1$.
To prove the reconstruction identity, we recover each increment
from the corresponding ternary digit of $w(b)$. For $j=1,\ldots,\Gamma$,
dividing the defining sum for $w(b)$ by $3^{j-1}$ gives
\begin{align*}
&3^{1-j}w(b)=
\underbrace{
\sum_{\ell=1}^{j-1}
(a_{b\Gamma+\ell}-a_{b\Gamma+\ell-1}+1)3^{\ell-j}
}_{\in[0,1)}
+
\underbrace{
\sum_{\ell=j}^{\Gamma-1}
(a_{b\Gamma+\ell}-a_{b\Gamma+\ell-1}+1)3^{\ell-j}
}_{\in\Z},
\end{align*}
where the first sum is nonnegative and at most
$2\sum_{\ell=1}^{j-1}3^{\ell-j}=1-3^{1-j}<1$, and the second
is an integer because $\ell-j\ge0$. Thus
\[
\left\lfloor\frac{w(b)}{3^{j-1}}\right\rfloor
=\sum_{\ell=j}^{\Gamma-1}
\bigl(a_{b\Gamma+\ell}-a_{b\Gamma+\ell-1}+1\bigr)3^{\ell-j}.
\]
Applying this identity with $j$ and $j+1$ gives, for $j=1,\ldots,\Gamma-1$,
\begin{align*}
\left\lfloor\frac{w(b)}{3^{j-1}}\right\rfloor
-3\left\lfloor\frac{w(b)}{3^j}\right\rfloor-1
&=
\sum_{\ell=j}^{\Gamma-1}
(a_{b\Gamma+\ell}-a_{b\Gamma+\ell-1}+1)3^{\ell-j}
-\sum_{\ell=j+1}^{\Gamma-1}
(a_{b\Gamma+\ell}-a_{b\Gamma+\ell-1}+1)3^{\ell-j}-1\\
&=a_{b\Gamma+j}-a_{b\Gamma+j-1}+1-1
=a_{b\Gamma+j}-a_{b\Gamma+j-1}.
\end{align*}
Summing over $j=1,\ldots,r$ in the definition of $s(w(b),r)$ yields
\begin{equation}\label{eq:block-reconstruction}
s(w(b),r)=\sum_{j=1}^{r}\left(\left\lfloor\frac{w(b)}{3^{j-1}}\right\rfloor
-3\left\lfloor\frac{w(b)}{3^j}\right\rfloor-1\right)
=\sum_{j=1}^{r}
(a_{b\Gamma+j}-a_{b\Gamma+j-1})=a_{b\Gamma+r}-a_{b\Gamma}.
\end{equation}
For arbitrary $w,r$ in the domain of $s$ and $j=1,\ldots,r$, we have
\[
\left\lfloor\frac{w}{3^{j-1}}\right\rfloor
-3\left\lfloor\frac{w}{3^j}\right\rfloor
=\left\lfloor
3\Bigl(\frac{w}{3^j}-\left\lfloor\frac{w}{3^j}\right\rfloor\Bigr)
\right\rfloor
\in\{0,1,2\}.
\]
Every summand defining $s(w,r)$ therefore belongs to $\{-1,0,1\}$,
which proves $|s(w,r)|\le r\le \Gamma-1$.

For the starting value $a_{b\Gamma}$, the encoding requires a normalized
value separated from the endpoints of $[0,1]$. We define
\[
y(b):=\min\left\{
1-\frac1{64\Upsilon},
\max\left\{
\frac1{64\Upsilon},\frac12\Bigl(1+\frac{a_{b\Gamma}}\Upsilon\Bigr)
\right\}
\right\}.
\]
Since $a_{b\Gamma}\in[-\Upsilon,\Upsilon]$, the truncation changes
$\frac12(1+a_{b\Gamma}/\Upsilon)$ by at most $(64\Upsilon)^{-1}$. Hence
\[
y(b)\in\left[\frac1{64\Upsilon},1-\frac1{64\Upsilon}\right],
\qquad
\left|2y(b)-1-\frac{a_{b\Gamma}}\Upsilon\right|\le\frac1{32\Upsilon}.
\]

The guarded phase code of Lemma~\ref{lem:guarded-code} will store one
number for each pair $(w(b),y(b))$ and one for each possible partial
sum indexed by $(w,r)$. Consider the following ordered list of
$P:=Q/\Gamma+\Gamma3^{\Gamma-1}$ phases:
\begin{equation}\label{eq:block-phase-list}
\frac{2w(b)+y(b)}{8\,3^{\Gamma-1}}
\quad(b=0,\ldots,Q/\Gamma-1),
\qquad
\frac14\Bigl(1+\frac{s(w,r)}\Gamma\Bigr)
\quad(w=0,\ldots,3^{\Gamma-1}-1,\ r=0,\ldots,\Gamma-1),
\end{equation}
where all indices are integers, the first family is ordered
by $b$, and the second family is ordered lexicographically
by $(w,r)$. By the preceding bounds on $w(b)$, $y(b)$,
and $s(w,r)$, we have
\[
\frac{2w(b)+y(b)}{8\,3^{\Gamma-1}}
\in\left[\frac1{512\,3^{\Gamma-1}\Upsilon},\frac14\right],
\qquad
\frac14\Bigl(1+\frac{s(w,r)}\Gamma\Bigr)
\in\left[\frac1{4\Gamma},\frac12-\frac1{4\Gamma}\right].
\]
By \eqref{eq:block-length-bounds} and $\Upsilon\ge4$, all these phases belong
to $[\zeta,1-\zeta]$, with $\zeta=(512\,3^{\Gamma-1}\Upsilon)^{-1}$.

Choose the smallest positive integer $g$ such that $G:=2^g\ge2/\zeta$.
Lemma~\ref{lem:guarded-code} applied to the ordered list gives an integer
$0\le m<G^P$, an integer-valued function
$\bm k:\{1,\ldots,P\}\to\mathbb N_0$, and an error function
$\bm e:\{1,\ldots,P\}\to\R$ satisfying
\begin{equation}\label{eq:block-code-relations}
\left\{
\begin{aligned}
&mG^{-(b+1)}
=\bm k[b+1]+\frac{2w(b)+y(b)}{8\,3^{\Gamma-1}}+\bm e[b+1],\\
&mG^{-(Q/\Gamma+1+\Gamma w+r)}
=\bm k[Q/\Gamma+1+\Gamma w+r]
+\frac14\Bigl(1+\frac{s(w,r)}\Gamma\Bigr)
+\bm e[Q/\Gamma+1+\Gamma w+r],
\end{aligned}
\right.
\end{equation}
for the indicated ranges of $b,w,r$. By minimality of $g$ and the
error bound in Lemma~\ref{lem:guarded-code},
\begin{equation}\label{eq:block-code-error}
G\in[1024\,3^{\Gamma-1}\Upsilon,2048\,3^{\Gamma-1}\Upsilon),
\qquad
|\bm e[i]|\le\frac1{2G}\le\frac1{2048\,3^{\Gamma-1}\Upsilon},
\quad (i=1,..,P).
\end{equation}

\par\medskip\noindent
\textbf{Step 2: affine maps.}
We construct a $\DTA$-network $\mathcal D_2:\R^2\to\R$
with four hidden layers of width two and verify the two bounds
in part~(a):
\[
\max_{\substack{b,r\in\Z\\0\le b<Q/\Gamma,\ 0\le r<\Gamma}}
\left|\mathcal D_2\left(b,\frac{r+1/2}{\Gamma}\right)
-\frac{a_{b\Gamma+r}}\Upsilon\right|<\frac1{25\Upsilon},
\qquad
\sup_{\substack{b\ge-1\\\nu\in[0,1]}}
|\mathcal D_2(b,\nu)|\le1+\frac \Gamma \Upsilon.
\]
Define the affine maps
$\mathcal A_{2,1},\ldots,\mathcal A_{2,4}:\R^2\to\R^2$
and $\mathcal A_{2,5}:\R^2\to\R$ by
\begin{align}
&\mathcal A_{2,1}\binom{b}{\nu}
=
\begin{bmatrix}
-g&0\\[1.5mm]
0&1
\end{bmatrix}
\binom{b}{\nu}
+\binom{-g}{0},
\notag\\[2mm]
&\mathcal A_{2,2}\binom{u}{v}
=
\begin{bmatrix}
8\,3^{\Gamma-1}m&0\\[1.5mm]
2m&\dfrac1{2\,3^{\Gamma-1}}
\end{bmatrix}
\binom{u}{v}
+
\begin{bmatrix}
24\,3^{\Gamma-1}m+2\\[1.5mm]
6m+2
\end{bmatrix},
\notag\\[2mm]
&\mathcal A_{2,3}\binom{u}{v}
=
\begin{bmatrix}
\dfrac{g\Gamma}{2}&-2g\Gamma\,3^{\Gamma-1}\\[1.5mm]
1&0
\end{bmatrix}
\binom{u}{v}
+
\begin{bmatrix}
-g\Bigl(\dfrac Q\Gamma+\dfrac12\Bigr)\\[2mm]
0
\end{bmatrix},
\notag\\[2mm]
&\mathcal A_{2,4}\binom{u}{v}
=
\begin{bmatrix}
2m&0\\[1.5mm]
0&1
\end{bmatrix}
\binom{u}{v}
+\binom{6m+2}{0},
\qquad
\mathcal A_{2,5}\binom{u}{v}
=
\begin{bmatrix}
\dfrac{2\Gamma}{\Upsilon}&2
\end{bmatrix}
\binom{u}{v}
-\Bigl(1+\frac \Gamma \Upsilon\Bigr).
\label{eq:scalar-decoder-affine-maps}
\end{align}
The decoder is
\[
\mathcal D_2
=\mathcal A_{2,5}\circ\DTA\circ\mathcal A_{2,4}
\circ\DTA\circ\mathcal A_{2,3}
\circ\DTA\circ\mathcal A_{2,2}
\circ\DTA\circ\mathcal A_{2,1}
\]
whose architecture is shown in Figure~\ref{fig:decoder-network}.

\begin{figure}[htbp]
\centering
\includegraphics[width=\FigureDecoderWidth]{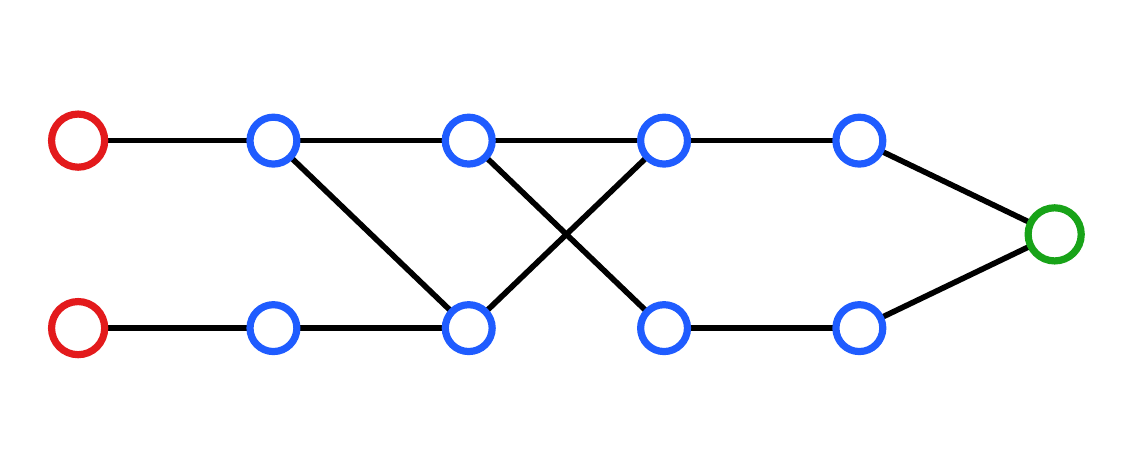}
\caption{The scalar decoder network $\mathcal D_2$.}
\label{fig:decoder-network}
\end{figure}

\par\medskip\noindent
\textbf{Step 3: approximation bound.}
Notice that
\eqref{eq:block-code-error} gives
$8\,3^{\Gamma-1}|\bm e[b+1]|\le(256\Upsilon)^{-1}$, which, together with
the bounds on $y(b)$ in Step~1, yields
\[
y(b)+8\,3^{\Gamma-1}\bm e[b+1]
\in\left[\frac3{256\Upsilon},1-\frac3{256\Upsilon}\right]\subset(0,1).
\]
Since $w(b)+\nu\in[0,3^{\Gamma-1}]$, we further have
\begin{align*}
(4\,3^{\Gamma-1})^{-1}(2w(b)+y(b)+2\nu)+2\bm e[b+1]
&=(4\,3^{\Gamma-1})^{-1}
\bigl(2(w(b)+\nu)+y(b)+8\,3^{\Gamma-1}\bm e[b+1]\bigr)\\
&\in\left(0,\frac12+\frac1{4\,3^{\Gamma-1}}\right)
\subset(0,1).
\end{align*}
Since $m\ge0$, both second-layer preactivations below are positive.
Using the activation identities in
(\ExternalNumber{2.2}{eq:dta-identities}) and the first
identity in \eqref{eq:block-code-relations}, we obtain
\begin{align}
&(\DTA\circ\mathcal A_{2,2}\circ\DTA\circ\mathcal A_{2,1})(b,\nu)\notag\\
&\quad=(\DTA\circ\mathcal A_{2,2})
\bigl([\DTA(-g(b+1)),\DTA(\nu)]^\top\bigr)
=(\DTA\circ\mathcal A_{2,2})
\bigl([G^{-(b+1)}-1,\nu]^\top\bigr)\notag\\
&\quad=\Bigl[
\DTA\bigl(8\,3^{\Gamma-1}mG^{-(b+1)}+16\,3^{\Gamma-1}m+2\bigr),\notag\\
&\qquad\quad
\DTA\bigl(2mG^{-(b+1)}+(2\,3^{\Gamma-1})^{-1}\nu+4m+2\bigr)
\Bigr]^\top\notag\\
&\quad=\Bigl[
\tau\Bigl(
\underbrace{
8\,3^{\Gamma-1}\bm k[b+1]+2w(b)+16\,3^{\Gamma-1}m+2
}_{\in2\mathbb N_0}
+\underbrace{
y(b)+8\,3^{\Gamma-1}\bm e[b+1]
}_{\in(0,1)}
\Bigr),\notag\\
&\qquad\quad
\tau\Bigl(
\underbrace{2\bm k[b+1]+4m+2}_{\in2\mathbb N_0}
+\underbrace{
(4\,3^{\Gamma-1})^{-1}(2w(b)+y(b)+2\nu)+2\bm e[b+1]
}_{\in(0,1)}
\Bigr)
\Bigr]^\top\notag\\
&\quad=\Bigl[
y(b)+8\,3^{\Gamma-1}\bm e[b+1],\,
(4\,3^{\Gamma-1})^{-1}(2w(b)+y(b)+2\nu)+2\bm e[b+1]
\Bigr]^\top.
\label{eq:scalar-first-read}
\end{align}
The two coordinates share the error $\bm e[b+1]$ with coefficients
that cancel in the following linear combination:
\begin{align}
&\bigl[-\tfrac12,\,2\,3^{\Gamma-1}\bigr]
(\DTA\circ\mathcal A_{2,2}\circ\DTA\circ\mathcal A_{2,1})(b,\nu)
\notag\\
&\quad=-\frac12\bigl(y(b)+8\,3^{\Gamma-1}\bm e[b+1]\bigr)
+2\,3^{\Gamma-1}\Bigl(
(4\,3^{\Gamma-1})^{-1}(2w(b)+y(b)+2\nu)+2\bm e[b+1]
\Bigr)
\notag\\
&\quad=-\frac12y(b)-4\,3^{\Gamma-1}\bm e[b+1]
+w(b)+\frac12y(b)+\nu+4\,3^{\Gamma-1}\bm e[b+1]
\notag\\
&\quad=w(b)+\nu.
\label{eq:exact-word-position}
\end{align}

For the fourth activation, use $|s(w(b),r)|\le \Gamma-1$ from Step~1
and $G\ge1024\,3^{\Gamma-1}\Upsilon\ge4096\Gamma$, by \eqref{eq:block-code-error}
and the bound $3^{\Gamma-1}\ge \Gamma$ in \eqref{eq:block-length-bounds}. Then
\[
\frac12\Bigl(1+\frac{s(w(b),r)}\Gamma\Bigr)+2\bm e[Q/\Gamma+1+\Gamma w(b)+r]
\in\left[\frac1{2\Gamma}-\frac1G,1-\frac1{2\Gamma}+\frac1G\right]
\subset(0,1).
\]
The first preactivation of $\mathcal A_{2,4}$ is positive since $m\ge0$,
and its second coordinate lies in $(0,1)$ by the estimate for
$y(b)+8\,3^{\Gamma-1}\bm e[b+1]$ at the beginning of this step.
Using \eqref{eq:exact-word-position}, the second identity in
\eqref{eq:block-code-relations} with $w=w(b)$, and the activation
identities in (\ExternalNumber{2.2}{eq:dta-identities}), we compute
\begin{align*}
&(\DTA\circ\mathcal A_{2,4}\circ\DTA\circ\mathcal A_{2,3}
\circ\DTA\circ\mathcal A_{2,2}\circ\DTA\circ\mathcal A_{2,1})(b,\nu)\\
&\quad=(\DTA\circ\mathcal A_{2,4})\Bigl(
\Bigl[
\DTA\Bigl(-g\Bigl(Q/\Gamma+\Gamma(w(b)+\nu)+\frac12\Bigr)\Bigr),\,
\DTA\bigl(y(b)+8\,3^{\Gamma-1}\bm e[b+1]\bigr)
\Bigr]^\top
\Bigr)\\
&\quad=(\DTA\circ\mathcal A_{2,4})\Bigl(
\bigl[G^{-(Q/\Gamma+1+\Gamma w(b)+r)}-1,\,
y(b)+8\,3^{\Gamma-1}\bm e[b+1]\bigr]^\top
\Bigr)\\
&\quad=\Bigl[
\DTA\bigl(2mG^{-(Q/\Gamma+1+\Gamma w(b)+r)}+4m+2\bigr),\,
\DTA\bigl(y(b)+8\,3^{\Gamma-1}\bm e[b+1]\bigr)
\Bigr]^\top\\
&\quad=\Bigl[
\tau\Bigl(
\underbrace{2\bm k[Q/\Gamma+1+\Gamma w(b)+r]+4m+2}_{\in2\mathbb N_0}\\
&\qquad\qquad
+\underbrace{
\tfrac12(1+\Gamma^{-1}s(w(b),r))+2\bm e[Q/\Gamma+1+\Gamma w(b)+r]
}_{\in(0,1)}
\Bigr),
y(b)+8\,3^{\Gamma-1}\bm e[b+1]\Bigr]^\top\\
&\quad=\Bigl[
\tfrac12\bigl(1+\Gamma^{-1}s(w(b),r)\bigr)+2\bm e[Q/\Gamma+1+\Gamma w(b)+r],\,
y(b)+8\,3^{\Gamma-1}\bm e[b+1]
\Bigr]^\top.
\end{align*}
The final map $\mathcal A_{2,5}$ adds the normalized starting value
and the normalized partial sum. By \eqref{eq:block-reconstruction},
\begin{align*}
&\mathcal D_2\left(b,\frac{r+1/2}{\Gamma}\right)\\
&\quad=\mathcal A_{2,5}\Bigl(
\Bigl[
\tfrac12\bigl(1+\Gamma^{-1}s(w(b),r)\bigr)+2\bm e[Q/\Gamma+1+\Gamma w(b)+r],\,
y(b)+8\,3^{\Gamma-1}\bm e[b+1]
\Bigr]^\top
\Bigr)\\
&\quad=\frac{2\Gamma}{\Upsilon}\Bigl(
\frac12\bigl(1+\Gamma^{-1}s(w(b),r)\bigr)+2\bm e[Q/\Gamma+1+\Gamma w(b)+r]
\Bigr)\\
&\qquad+2\bigl(y(b)+8\,3^{\Gamma-1}\bm e[b+1]\bigr)-1-\frac \Gamma \Upsilon\\
&\quad=\frac \Gamma \Upsilon+\frac{s(w(b),r)}\Upsilon+\frac{4\Gamma}{\Upsilon}\bm e[Q/\Gamma+1+\Gamma w(b)+r]
+2y(b)+16\,3^{\Gamma-1}\bm e[b+1]-1-\frac \Gamma \Upsilon\\
&\quad=2y(b)-1+\frac{a_{b\Gamma+r}-a_{b\Gamma}}\Upsilon
+16\,3^{\Gamma-1}\bm e[b+1]+\frac{4\Gamma}{\Upsilon}\bm e[Q/\Gamma+1+\Gamma w(b)+r].
\end{align*}
It follows from the approximation bound for $y(b)$ in Step~1,
\eqref{eq:block-code-error}, and $\Gamma/\Upsilon\le1/4$ that
\begin{align*}
&\left|\mathcal D_2\left(b,\frac{r+1/2}{\Gamma}\right)
-\frac{a_{b\Gamma+r}}\Upsilon\right|\\
&\quad=\left|
2y(b)-1-\frac{a_{b\Gamma}}\Upsilon
+16\,3^{\Gamma-1}\bm e[b+1]+\frac{4\Gamma}{\Upsilon}\bm e[Q/\Gamma+1+\Gamma w(b)+r]
\right|\\
&\quad\le
\left|2y(b)-1-\frac{a_{b\Gamma}}\Upsilon\right|
+16\,3^{\Gamma-1}|\bm e[b+1]|+\frac{4\Gamma}{\Upsilon}|\bm e[Q/\Gamma+1+\Gamma w(b)+r]|\\
&\quad\le\frac1{32\Upsilon}
+\frac{8\,3^{\Gamma-1}+2\Gamma/\Upsilon}{G}
\le\frac1{32\Upsilon}+\frac{17\,3^{\Gamma-1}}{2G}
\le\frac{81}{2048\Upsilon}<\frac1{25\Upsilon}.
\end{align*}

\par\medskip\noindent
\textbf{Step 4: uniform bound.}
It suffices to show that the output of the fourth hidden layer
belongs to $[0,1]^2$ for every $b\ge-1$ and $\nu\in[0,1]$.
Indeed, once this inclusion holds, the formula for $\mathcal A_{2,5}$
in \eqref{eq:scalar-decoder-affine-maps} gives
\begin{align*}
\sup_{\substack{b\ge-1\\\nu\in[0,1]}}
|\mathcal D_2(b,\nu)|
&\le\sup_{u,v\in[0,1]}|\mathcal A_{2,5}(u,v)|
=\sup_{u,v\in[0,1]}
\left|2v-1+\frac \Gamma \Upsilon(2u-1)\right|\\
&\le\sup_{u,v\in[0,1]}
\Bigl(|2v-1|+\frac \Gamma \Upsilon|2u-1|\Bigr)
\le1+\frac \Gamma \Upsilon.
\end{align*}
To verify the inclusion, use the activation identities in
(\ExternalNumber{2.2}{eq:dta-identities}) and the maps in
\eqref{eq:scalar-decoder-affine-maps}. For $b\ge-1$ and $\nu\in[0,1]$,
\[
(\DTA\circ\mathcal A_{2,1})(b,\nu)
=[\DTA(-g(b+1)),\DTA(\nu)]^\top
=[G^{-(b+1)}-1,\nu]^\top
\in(-1,0]\times[0,1].
\]
For $(u,v)\in(-1,0]\times[0,1]$, the two coordinates of
$\mathcal A_{2,2}(u,v)$ are at least $16\,3^{\Gamma-1}m+2$ and $4m+2$,
respectively. Since $m\ge0$ and $\DTA([0,\infty))\subset[0,1]$,
we have $(\DTA\circ\mathcal A_{2,2})((-1,0]\times[0,1])\subset[0,1]^2$.
Moreover, the second coordinate of $\mathcal A_{2,3}(u,v)$ is $u$;
using $\DTA(\R)\subset(-1,1]$ and $\DTA(u)=u$ for $u\in[0,1]$ gives
$(\DTA\circ\mathcal A_{2,3})([0,1]^2)\subset(-1,1]\times[0,1]$.
Finally, for $(u,v)\in(-1,1]\times[0,1]$,
\[
\mathcal A_{2,4}(u,v)=[2mu+6m+2,v]^\top
\in[4m+2,\infty)\times[0,1],
\]
so $(\DTA\circ\mathcal A_{2,4})((-1,1]\times[0,1])\subset[0,1]^2$.
Composing the four inclusions proves the required bound on the fourth
hidden layer.

\par\medskip\noindent
\textbf{Step 5: parameter bounds.}
The five affine maps in \eqref{eq:scalar-decoder-affine-maps}
contain at most $3,5,4,3$, and $3$ nonzero parameters,
respectively, giving a total of at most $18$.
By \eqref{eq:block-length-bounds} and $\Gamma/\Upsilon\le1/4$, their radii satisfy
$\parnorm{\mathcal A_{2,1}}=g$,
$\parnorm{\mathcal A_{2,2}}=24\,3^{\Gamma-1}m+2$,
$\parnorm{\mathcal A_{2,3}}=g(Q/\Gamma+1/2)$,
$\parnorm{\mathcal A_{2,4}}=6m+2$, and
$\parnorm{\mathcal A_{2,5}}=2$.
In particular,
$T_{\mathcal D_2}=\max\{\parnorm{\mathcal A_{2,2}},
\parnorm{\mathcal A_{2,3}}\}$.
Recalling that $m<G^P$, $Q/\Gamma+1/2\le P$, and
$gP\le2^{gP}=G^P$, we obtain
\begin{equation*}
T_{\mathcal D_2}
=\max\left\{24\,3^{\Gamma-1}m+2,g\Bigl(\frac Q\Gamma+\frac12\Bigr)\right\}
\le32\,3^{\Gamma-1}G^P
\le2^{gP+2\Gamma+3}.
\end{equation*}
Thus, if $gP+2\Gamma+3\le69Q/4$, then
\begin{equation}\label{eq:scalar-decoder-radius}
\log T_{\mathcal D_2}
\le(gP+2\Gamma+3)\log2
\le\frac{69}{4}Q\log2<11.96Q<12Q.
\end{equation}

It remains to verify that $gP+2\Gamma+3\le69Q/4$.
Recall from Step~1 that $P=Q/\Gamma+\Gamma3^{\Gamma-1}$ and
$g=\lceil10+(\Gamma-1)\log_2 3+\log_2 \Upsilon\rceil\le n+2\Gamma+8$.

\begin{itemize}
\item If $\Gamma=1$, then $2\le n\le7$, $g=n+10$, and $P=Q+1$.
Since $Q\ge2^n$, we have
\begin{align*}
gP+2\Gamma+3
&=(n+10)(Q+1)+5
\le Q\max_{n\in\{2,\ldots,7\}}
\bigl\{n+10+(n+15)2^{-n}\bigr\}
=\frac{1099}{64}Q.
\end{align*}

\item If $\Gamma\ge2$, then $n<8\Gamma$ gives
$g\le n+2\Gamma+8\le14\Gamma$. Using the bound
$16\Gamma^23^{\Gamma-1}\le H\le Q$ established following
\eqref{eq:block-length-bounds}, we obtain
\begin{align*}
gP+2\Gamma+3
&\le14\Gamma\Bigl(\frac Q\Gamma+\Gamma3^{\Gamma-1}+1\Bigr)
\le14\Gamma\Bigl(\frac Q\Gamma+2\Gamma3^{\Gamma-1}\Bigr)
\le14\Gamma\Bigl(\frac Q\Gamma+\frac{Q}{8\Gamma}\Bigr)
=\frac{63}{4}Q.
\end{align*}
\end{itemize}

Thus $gP+2\Gamma+3<69Q/4$ in both cases, proving the radius bound.
For part~(b), the same estimates also give $g\le17\Gamma$ and hence
$2\lceil g/2\rceil\le g+1\le18\Gamma<20\Gamma$. This completes the proof.
\end{proof}

\subsection{Two technical lemmas}\label{app:coupled-technical}
Parallel copies of part~(a) store $\mathcal O(N)$ integers at width
$\mathcal O(N)$. To obtain the $Q/N^2$ term in part~(b), we store
$\mathcal O(N^2)$ shorter integers as entries of a weight matrix.
The next two lemmas select a row and recover the required entry.
The coupled decoding then uses the same cancellation as
\eqref{eq:exact-word-position}.

\begin{lemma}[Cyclic triangular selector]\label{lem:cyclic-selector}
Let $W\ge3$ be an odd integer, set $s=(W-1)/2$, and let
$\bm Q\bm e_j=\bm e_{(j+1)\bmod W}$.
Define
\begin{equation}\label{eq:shared-selector-matrices}
\bm\Gamma
=W\bm I_W+\frac W2(\bm Q^s+\bm Q^{s+1})
-\frac{2W^2}{W^2+1}\bmone_W\bmone_W^\top.
\end{equation}
For every integer $p$, with $\bm e_p$ interpreted modulo $W$,
\begin{equation}\label{eq:shared-selector-exact}
\bm\Gamma\left(1-\tau\left(\frac{2(p-j)}W\right)\right)_{j=0}^{W-1}
=\bm e_p.
\end{equation}
For every $x\in\R$,
\begin{equation}\label{eq:shared-selector-l1}
\left\|\bm\Gamma
\left(1-\tau\left(\frac{2(x-j)}W\right)\right)_{j=0}^{W-1}\right\|_1<3.
\end{equation}
Moreover, $\|\bm\Gamma\|_{\max}\le W$ and
$\bm\Gamma\bmone_W=2W(W^2+1)^{-1}\bmone_W$.
\end{lemma}

\begin{proof}
We first invert the triangular-wave matrix at integer indices, then
calculate its output between two consecutive integers.

\par\medskip\noindent
\textbf{Step 1: integer indices.}
All indices are interpreted modulo $W$. Set
$\kappa_j=1-\tau(2j/W)$ and
$\bm K=(\kappa_{p-j})_{p,j=0}^{W-1}$.
The two branches of $\tau$ give $\kappa_j=1-2j/W$ for
$j=0,\ldots,s$ and $\kappa_j=2j/W-1$ for $j=s+1,\ldots,W-1$.
Their sum is
\[
\sum_{j=0}^{W-1}\kappa_j
=1+2\sum_{j=1}^{s}\left(1-\frac{2j}{W}\right)
=1+2s-\frac{2s(s+1)}W
=\frac{W^2+1}{2W}.
\]
The two shifts appearing in \eqref{eq:shared-selector-matrices} satisfy
\begin{align*}
\kappa_{j-s}+\kappa_{j-s-1}
&=\begin{cases}
2/W,&j=0,\\
(2j+1)/W+(2j-1)/W,&j=1,\ldots,s,\\
1-2(j-s)/W+1-2(j-s-1)/W,&j=s+1,\ldots,W-1
\end{cases}\\
&=2-2\kappa_j+\frac2W\one_{\{j=0\}}.
\end{align*}
Consequently,
\begin{align*}
\bm K\bm\Gamma
&=W\bm K+\frac W2\bm K(\bm Q^s+\bm Q^{s+1})
-\frac{2W^2}{W^2+1}(\bm K\bmone_W)\bmone_W^\top\\
&=W\bm K+\frac W2
\left(2\bmone_W\bmone_W^\top-2\bm K+\frac2W\bm I_W\right)
-\frac{2W^2}{W^2+1}\frac{W^2+1}{2W}\bmone_W\bmone_W^\top\\
&=W\bm K+W\bmone_W\bmone_W^\top-W\bm K+\bm I_W
-W\bmone_W\bmone_W^\top
=\bm I_W.
\end{align*}
Thus $\bm\Gamma\bm K=\bm I_W$ as well. By evenness of $\tau$,
the $p$th column of $\bm K$ is
$(1-\tau(2(p-j)/W))_{j=0}^{W-1}$, proving
\eqref{eq:shared-selector-exact}.

\par\medskip\noindent
\textbf{Step 2: real indices.}
Cyclic shifts reduce the calculation to $x=t\in[0,1)$, because
$\bm\Gamma$ commutes with $\bm Q$. Set $u=\min\{t,1-t\}$.
Every coordinate of $(1-\tau(2(t-j)/W))_j$ interpolates its values
at zero and one, except coordinate $s+1$, whose value is
$|2t-1|/W$ instead of $1/W$. The complete output is therefore
\begin{align*}
\bm\Gamma\left(1-\tau\left(\frac{2(t-j)}W\right)\right)_{j=0}^{W-1}
&=(1-t)\bm\Gamma\bm K\bm e_0
+t\bm\Gamma\bm K\bm e_1
+\frac{|2t-1|-1}{W}\bm\Gamma\bm e_{s+1}\\
&=(1-t)\bm e_0+t\bm e_1
-\frac{2u}{W}\left[
W\bm e_{s+1}+\frac W2(\bm e_0+\bm e_1)
-\frac{2W^2}{W^2+1}\bmone_W\right]\\
&=(1-t-u)\bm e_0+(t-u)\bm e_1-2u\bm e_{s+1}
+\frac{4uW}{W^2+1}\bmone_W.
\end{align*}
Only coordinate $s+1$ can be negative, and its value is
$-2u+4uW/(W^2+1)=-2u(W-1)^2/(W^2+1)$. Hence
\begin{align*}
\left\|\bm\Gamma
\left(1-\tau\left(\frac{2(t-j)}W\right)\right)_{j=0}^{W-1}\right\|_1
&=1-4u+\frac{4uW^2}{W^2+1}
+\frac{4u(W-1)^2}{W^2+1}\\
&=1+\frac{4uW(W-2)}{W^2+1}
\le1+\frac{2W(W-2)}{W^2+1}<3.
\end{align*}
Finally, the entries and row sums in
\eqref{eq:shared-selector-matrices} satisfy
\begin{align*}
&\|\bm\Gamma\|_{\max}
\le W\max\left\{\frac{2W}{W^2+1},\,
1-\frac{2W}{W^2+1},\,
\left|\frac12-\frac{2W}{W^2+1}\right|\right\}\le W,\\
&\bm\Gamma\bmone_W
=\left(W+\frac W2+\frac W2-\frac{2W^3}{W^2+1}\right)\bmone_W
=\frac{2W}{W^2+1}\bmone_W.
\end{align*}
\end{proof}

The next lemma combines row selection with reading an entry from an $N\times N$ array.

\begin{lemma}[Short-code array]\label{lem:short-code-array}
Let $N,M\in\mathbb N$ with $N\ge3$ odd, and let $g\ge2$ be an even
integer. Set $G=2^g$, and prescribe integers
$m_{pq}\in\{0,\ldots,G^M-1\}$ for $p,q=0,\ldots,N-1$.
There are affine maps $\mathcal R:\R\to\R^{4N-1}$ and
$\mathcal Z,\mathcal P:\R^{4N-1}\to\R^N$ satisfying:
\begin{enumerate}[label=(\roman*),leftmargin=2em,topsep=3pt,itemsep=3pt]
\item For $p,q=0,\ldots,N-1$ and $\ell=1,\ldots,M$,
let $i=qNM+N(\ell-1)+p+1$. Then
\begin{equation}\label{eq:short-code-exact}
[(\mathcal Z\circ\DTA\circ\mathcal R)(i)]_p=m_{pq}G^{-\ell},
\qquad (\mathcal P\circ\DTA\circ\mathcal R)(i)=\bm e_p.
\end{equation}
The standard basis is indexed from zero, and
$\|(\mathcal P\circ\DTA\circ\mathcal R)(x)\|_1<3$ for $x\ge0$.
\item If $\mathcal R(x)=\bm a x+\bm b$, then
$\|\bm a\|_\infty=g/N$ and
$\|\bm b\|_\infty\le\parnorm{\mathcal R}\le\max\{gNM,2N^2M\}$.
Also $\parnorm{\mathcal Z}\le2N^2MG^M$ and
$\parnorm{\mathcal P}\le N$.
\end{enumerate}
\end{lemma}

\begin{proof}
The hidden layer records the row, the column, and successive powers of
$G$. The two affine output maps then return the row indicator and the
rescaled entry.

\par\medskip\noindent
\textbf{Step 1: affine maps.}
Use the $N\times N$ matrix $\bm\Gamma$ from
Lemma~\ref{lem:cyclic-selector}, and set $m_{p,-1}=m_{p,N}=0$.
We regard $\R^{4N-1}$ as
$\R^{N+1}\times\R^{N-1}\times\R^{N-1}\times\R^N$.
In an input $(\bm x,\bm y,\bm z,\bm u)$, the coordinates of
$\bm x$ are numbered from zero and those of $\bm y,\bm z$ from one.
Define
\begingroup\fontsize{10.95}{13.2}\selectfont
\setlength{\arraycolsep}{3pt}
\begin{equation}\label{eq:short-code-address-map}
\begin{aligned}
&\mathcal R(x)=
\begin{bmatrix}
(-gx/N+gjM)_{j=0}^{N}\\[1mm]
(1+(x-jNM)/(2N^2M))_{j=1}^{N-1}\\[1mm]
(1+(x-1-jNM)/(2N^2M))_{j=1}^{N-1}\\[1mm]
(2N^2M+2(x-j-1)/N)_{j=0}^{N-1}
\end{bmatrix},
\qquad
\mathcal P(\bm x,\bm y,\bm z,\bm u)=\bm\Gamma(\bmone_N-\bm u).
\end{aligned}
\end{equation}
\begin{equation}\label{eq:short-code-affine-maps}
\begin{aligned}
&[\mathcal Z(\bm x,\bm y,\bm z,\bm u)]_p
=\frac{2^{g(p+1)/N}}G\biggl\{
\sum_{j=0}^{N}(m_{pj}-G^{-M}m_{p,j-1})
\left[x_j-\tau\!\left(\frac{g(p+1)}N\right)\right]\\
&\hspace{22mm}+\left[1+\tau\!\left(\frac{g(p+1)}N\right)\right]
\left[m_{p0}+\sum_{j=1}^{N-1}(m_{pj}-G^{-M}m_{p,j-1})
\left(\frac12+N^2M(z_j-y_j)\right)\right]\biggr\}.
\end{aligned}
\end{equation}
\endgroup
Here $p=0,\ldots,N-1$. For an input in column $q$, the hidden coordinates
will satisfy $1/2+N^2M(z_j-y_j)=\one_{\{j\le q\}}$.
The second line of \eqref{eq:short-code-affine-maps} then cancels the constant
terms in the first line, leaving a sum of powers of $G$.

\par\medskip\noindent
\textbf{Step 2: output.}
Fix $p,q=0,\ldots,N-1$ and $\ell=1,\ldots,M$, and let
$i=qNM+N(\ell-1)+p+1$. Since $qNM<i\le(q+1)NM$, the first
$N+1$ preactivations are negative for $j=0,\ldots,q$ and
nonnegative for $j=q+1,\ldots,N$. The next two groups lie in
$(1/2,3/2)$, and the last group is positive. The exponential branch,
$\DTA(1+t)=1-|t|$ for $t\in[-1,1]$, and the even integer $g$ give
\begingroup\fontsize{10.95}{13.2}\selectfont
\setlength{\arraycolsep}{3pt}
\begin{align}
(\DTA\circ\mathcal R)(i)
&=\begin{bmatrix}
\bigl(\DTA(gM(j-q)-g(\ell-1)-g(p+1)/N)\bigr)_{j=0}^{N}\\[1mm]
\bigl(\DTA(1+(i-jNM)/(2N^2M))\bigr)_{j=1}^{N-1}\\[1mm]
\bigl(\DTA(1+(i-1-jNM)/(2N^2M))\bigr)_{j=1}^{N-1}\\[1mm]
\bigl(\tau(2N^2M+2qM+2(\ell-1)+2(p-j)/N)\bigr)_{j=0}^{N-1}
\end{bmatrix}\notag\\[2mm]
&=\begin{bmatrix}
\left(\begin{cases}
G^{M(j-q)+1-\ell}2^{-g(p+1)/N}-1,&j=0,\ldots,q,\\
\tau(g(p+1)/N),&j=q+1,\ldots,N
\end{cases}\right)_{j=0}^{N}\\[2mm]
\bigl(1-|i-jNM|/(2N^2M)\bigr)_{j=1}^{N-1}\\[1mm]
\bigl(1-|i-1-jNM|/(2N^2M)\bigr)_{j=1}^{N-1}\\[1mm]
\bigl(\tau(2(p-j)/N)\bigr)_{j=0}^{N-1}
\end{bmatrix}.
\label{eq:short-code-geometric-tail}
\end{align}
\endgroup
First evaluate $\mathcal P$ from the last $N$ coordinates of
\eqref{eq:short-code-geometric-tail}. By
\eqref{eq:shared-selector-exact},
\[
(\mathcal P\circ\DTA\circ\mathcal R)(i)
=\bm\Gamma\left[\bmone_N-\bigl(\tau(2(p-j)/N)\bigr)_{j=0}^{N-1}\right]
=\bm\Gamma\bigl(1-\tau(2(p-j)/N)\bigr)_{j=0}^{N-1}
=\bm e_p.
\]
For $x\in[0,\infty)$, the last $N$ preactivations remain positive.
Period two and \eqref{eq:shared-selector-l1} therefore give
\begin{align*}
\|(\mathcal P\circ\DTA\circ\mathcal R)(x)\|_1
&=\left\|\bm\Gamma
\left(1-\DTA\left(2N^2M+\frac{2(x-1-j)}N\right)\right)_{j=0}^{N-1}\right\|_1\\
&=\left\|\bm\Gamma
\left(1-\tau\left(\frac{2(x-1-j)}N\right)\right)_{j=0}^{N-1}\right\|_1<3.
\end{align*}
We next evaluate $\mathcal Z$. For $j=1,\ldots,N-1$, the difference of the middle two groups is
\begin{align*}
&\frac12+N^2M\left[
\left(1-\frac{|i-1-jNM|}{2N^2M}\right)
-\left(1-\frac{|i-jNM|}{2N^2M}\right)\right]\\
&\qquad=\frac{1+|i-jNM|-|i-1-jNM|}{2}
=\begin{cases}
(1+1)/2=1,&j=1,\ldots,q,\\
(1-1)/2=0,&j=q+1,\ldots,N-1
\end{cases}
=\one_{\{j\le q\}}.
\end{align*}
Apply $\mathcal Z$ to the complete hidden vector in
\eqref{eq:short-code-geometric-tail}. Terms with $j>q$ in its first
sum vanish, and the remaining terms give
\begingroup\fontsize{10.5}{13}\selectfont
\thinmuskip=2mu\medmuskip=2mu\thickmuskip=3mu
\begin{align*}
&[(\mathcal Z\circ\DTA\circ\mathcal R)(i)]_p\\
&\quad=\frac{2^{g(p+1)/N}}G\cdot\Biggl\{
\sum_{j=0}^{q}\bigl(m_{pj}-G^{-M}m_{p,j-1}\bigr)
\cdot\biggl[G^{M(j-q)+1-\ell}2^{-g(p+1)/N}
-1-\tau\Bigl(\frac{g(p+1)}N\Bigr)\biggr]\\[-1mm]
&\hspace{20mm}
+\biggl[1+\tau\Bigl(\frac{g(p+1)}N\Bigr)\biggr]
\cdot\biggl[m_{p0}+\sum_{j=1}^{N-1}
\bigl(m_{pj}-G^{-M}m_{p,j-1}\bigr)
\cdot\one_{\{j\le q\}}\biggr]\Biggr\}\\[1mm]
&\quad=\frac{2^{g(p+1)/N}}G\cdot\Biggl\{
\sum_{j=0}^{q}\bigl(m_{pj}-G^{-M}m_{p,j-1}\bigr)
\cdot G^{M(j-q)+1-\ell}2^{-g(p+1)/N}\\[-1mm]
&\hspace{20mm}
+\biggl[1+\tau\Bigl(\frac{g(p+1)}N\Bigr)\biggr]
\cdot\biggl[m_{p0}+\sum_{j=1}^{q}
\bigl(m_{pj}-G^{-M}m_{p,j-1}\bigr)
-\sum_{j=0}^{q}
\bigl(m_{pj}-G^{-M}m_{p,j-1}\bigr)\biggr]\Biggr\}\\[1mm]
&\quad=\frac{2^{g(p+1)/N}}G
\sum_{j=0}^{q}\bigl(m_{pj}-G^{-M}m_{p,j-1}\bigr)
\cdot G^{M(j-q)+1-\ell}2^{-g(p+1)/N}\\
&\quad=G^{-\ell}\cdot\biggl[
\sum_{j=0}^{q}m_{pj}G^{M(j-q)}
-\sum_{j=0}^{q}m_{p,j-1}G^{M(j-q-1)}\biggr]\\
&\quad=G^{-\ell}\cdot\biggl[
\sum_{j=0}^{q-1}m_{pj}G^{M(j-q)}+m_{pq}
-m_{p,-1}G^{-M(q+1)}
-\sum_{j=0}^{q-1}m_{pj}G^{M(j-q)}\biggr]\\
&\quad=m_{pq}G^{-\ell}.
\end{align*}
\endgroup
This calculation also applies to $q=0$, with empty sums equal to zero.
This proves~(i).

\par\medskip\noindent
\textbf{Step 3: parameter bounds.}
For $j=0,\ldots,N$, we have
$|m_{pj}-G^{-M}m_{p,j-1}|<G^M$.
Since $2^{g(p+1)/N}/G\in(0,1]$ and $\tau(g(p+1)/N)\in[0,1]$,
the coefficients of $\bm x$ in $\mathcal Z$ have magnitude at most
$G^M$, and those of $\bm y,\bm z$ at most $2N^2MG^M$.
Its bias satisfies
\begingroup\fontsize{10.5}{13}\selectfont
\thinmuskip=2mu\medmuskip=1.5mu\thickmuskip=3mu
\begin{align*}
|[\mathcal Z(\bm0)]_p|
&=\frac{2^{g(p+1)/N}}G\left|
m_{p0}+\tau\!\left(\frac{g(p+1)}N\right)G^{-M}m_{p,N-1}
+\frac{1-\tau(g(p+1)/N)}2
\sum_{j=1}^{N-1}(m_{pj}-G^{-M}m_{p,j-1})\right|\\
&\le G^M+1+\frac{N-1}{2}G^M
\le\frac{N+3}{2}G^M\le2N^2MG^M.
\end{align*}
\endgroup
Thus $\parnorm{\mathcal Z}\le2N^2MG^M$.
For $\mathcal R(x)=\bm a x+\bm b$, formula
\eqref{eq:short-code-address-map} gives
\begin{align*}
&\|\bm a\|_\infty
=\max\left\{\frac gN,\frac1{2N^2M},\frac2N\right\}
=\frac gN\le gNM,\\
&\|\bm b\|_\infty
=\max\left\{gNM,\,1-\frac1{2N},\,
1-\frac1{2N}-\frac1{2N^2M},\,2N^2M-\frac2N\right\}
\le\max\{gNM,2N^2M\}.
\end{align*}
Consequently, $\parnorm{\mathcal R}\le\max\{gNM,2N^2M\}$.
Finally, the last assertion of Lemma~\ref{lem:cyclic-selector} gives
\[
\parnorm{\mathcal P}
=\max\{\|\bm\Gamma\|_{\max},\|\bm\Gamma\bmone_N\|_\infty\}
\le\max\left\{N,\frac{2N}{N^2+1}\right\}=N.
\]
\end{proof}

\subsection{Proof of Proposition~\ExternalNumber{3.4}{prop:block-decoder}, part~(b)}\label{app:coupled-b-proof}
\begin{proof}[Proof of Proposition~\ExternalNumber{3.4}{prop:block-decoder}, part~(b)]
We retain the starting values and increments from part~(a), and
store their phases in a square array of shorter integers. The first
two hidden layers recover, in the selected row, the two scalar
expressions in \eqref{eq:scalar-first-read}. Their common error then cancels by
\eqref{eq:exact-word-position}, and the last two hidden layers
recover the required partial sum.

\par\medskip\noindent
\Needspace{4\baselineskip}
\textbf{Step 1: setup.}
\begin{itemize}
\item \textit{Width and phases.}
We first construct $\mathcal D_{4N}$ under the assumptions
\begingroup
\normalsize
\setlength{\arraycolsep}{3pt}
\[
N\in\mathbb N,\qquad N\ge3\text{ odd},\qquad N^2\le Q.
\]
\endgroup
Step~5 will extend the construction to every width.
Retain $\Gamma$, $w(b)$, $s(w,r)$, and $y(b)$ from Step~1 of part~(a).
Here $\Gamma$ is the block length, $w(b)$ encodes the increments in block
$b$, $s(w,r)$ sums the first $r$ increments encoded by $w$, and $y(b)$
is the normalized starting value of block $b$.
Let $\theta_1,\ldots,\theta_P$ be the ordered phases in
\eqref{eq:block-phase-list}, and choose
\begingroup
\normalsize
\setlength{\arraycolsep}{3pt}
\[
P=\frac Q\Gamma+\Gamma3^{\Gamma-1},\qquad
M=\left\lceil\frac P{N^2}\right\rceil,\qquad
g=2\left\lceil\frac{10+(\Gamma-1)\log_2 3+\log_2\Upsilon}{2}\right\rceil,
\qquad G=2^g.
\]
\endgroup
In particular, the phases being stored are
\begingroup
\normalsize
\setlength{\arraycolsep}{3pt}
\[
\theta_{b+1}=\frac{2w(b)+y(b)}{8\,3^{\Gamma-1}},
\qquad
\theta_{Q/\Gamma+\Gamma w+r+1}
=\frac14\left(1+\frac{s(w,r)}\Gamma\right),
\]
\endgroup
with the integer ranges of $b,w,r$ in \eqref{eq:block-phase-list}.
Append copies of $1/4$ to the phase list until its length is $N^2M$.
Thus every phase still belongs to $[\zeta,1-\zeta]$, with
$\zeta=(512\,3^{\Gamma-1}\Upsilon)^{-1}$ as in part~(a).
The exponent $g$ is even, and the choice of $\Gamma$ in part~(a)
gives $\log_2\Upsilon<8\Gamma$.  Consequently,
\begingroup
\normalsize
\setlength{\arraycolsep}{3pt}
\[
G\in[1024\,3^{\Gamma-1}\Upsilon,4096\,3^{\Gamma-1}\Upsilon),
\qquad
g<12+2(\Gamma-1)+8\Gamma=10\Gamma+10\le20\Gamma.
\]
\endgroup

\item \textit{The shorter integers.}
The entry in row $p$ and column $q$ stores the $M$ phases with indices
$i=qNM+N(\ell-1)+p+1$, $\ell=1,\ldots,M$.
Apply Lemma~\ref{lem:guarded-code} separately to these lists.
It gives integers $m_{pq}\in\{0,\ldots,G^M-1\}$, together with
functions $\bm k:\{1,\ldots,N^2M\}\to\mathbb N_0$ and
$\bm e:\{1,\ldots,N^2M\}\to\R$, such that
\begingroup
\normalsize
\setlength{\arraycolsep}{3pt}
\begin{equation}\label{eq:matrix-code}
m_{pq}G^{-\ell}=\bm k[i]+\theta_i+\bm e[i],\qquad
|\bm e[i]|\le\frac1{2G},\qquad
i=qNM+N(\ell-1)+p+1.
\end{equation}
\endgroup
Here $p,q=0,\ldots,N-1$ and $\ell=1,\ldots,M$;
these triples give every index exactly once.  The functions $\bm k$
and $\bm e$ replace their scalar-code counterparts in part~(a).
The phases are those of part~(a), and
$|\bm e[i]|\le(2048\,3^{\Gamma-1}\Upsilon)^{-1}$.
Thus all three interval inclusions in Step~3 of part~(a) remain valid.
\end{itemize}

\par\medskip\noindent
\textbf{Step 2: affine maps.}
Apply Lemma~\ref{lem:short-code-array} to the integers $m_{pq}$.
Its map $\mathcal R$ takes a phase index as input; the maps
$\mathcal Z$ and $\mathcal P$ then read the corresponding integer
and select its row. We also use $\mathcal Z$ and $\mathcal P$ on
$\R^{4N}$, where they ignore the last input coordinate.

The decoder has the form
\begingroup
\normalsize
\setlength{\arraycolsep}{3pt}
\[
\mathcal D_{4N}
=\mathcal A_{4N,5}\circ\DTA\circ\mathcal A_{4N,4}
\circ\DTA\circ\mathcal A_{4N,3}
\circ\DTA\circ\mathcal A_{4N,2}
\circ\DTA\circ\mathcal A_{4N,1}.
\]
\endgroup
We define
$\mathcal A_{4N,1}:\R^2\to\R^{4N}$,
$\mathcal A_{4N,2},\mathcal A_{4N,3}:\R^{4N}\to\R^{4N}$,
$\mathcal A_{4N,4}:\R^{4N}\to\R^{2N+1}$, and
$\mathcal A_{4N,5}:\R^{2N+1}\to\R$ as follows.
In the second and fourth maps, $\bm z=(z_1,\ldots,z_{4N})^\top$.
In the third and fifth maps, each $\bm z_j$ is a column in $\R^N$,
and $t$ is a scalar:
\begingroup
\normalsize
\setlength{\arraycolsep}{3pt}
\predisplaypenalty=0
\begin{equation}\label{eq:matrix-affine-maps}
\left\{
\begin{aligned}
&\mathcal A_{4N,1}\binom b\nu
=\begin{bmatrix}\mathcal R(b+1)\\\nu\end{bmatrix},\\[2mm]
&\mathcal A_{4N,2}(\bm z)
=\begin{bmatrix}
8\,3^{\Gamma-1}\mathcal Z(\bm z)+2\bmone_N\\[1mm]
8\,3^{\Gamma-1}\mathcal Z(\bm z)+\mathcal P(\bm z)+2\bmone_N\\[1mm]
2\mathcal Z(\bm z)+\bigl(2+\tfrac{z_{4N}}{2\,3^{\Gamma-1}}\bigr)\bmone_N\\[1mm]
2\mathcal Z(\bm z)+\mathcal P(\bm z)+\bigl(2+\tfrac{z_{4N}}{2\,3^{\Gamma-1}}\bigr)\bmone_N
\end{bmatrix},\\[2mm]
&\mathcal A_{4N,3}
\begin{bmatrix}\bm z_1\\\bm z_2\\\bm z_3\\\bm z_4\end{bmatrix}
=\begin{bmatrix}
\mathcal R\!\left(\tfrac Q\Gamma+\Gamma3^{\Gamma-1}-\tfrac\Gamma4+\tfrac12
+\Gamma\bmone_N^\top\bigl[3^{\Gamma-1}(\bm z_3-\bm z_4)-\tfrac14(\bm z_1-\bm z_2)\bigr]\right)\\[2mm]
\tfrac12+\tfrac12\bmone_N^\top(\bm z_1-\bm z_2)
\end{bmatrix},\\[2mm]
&\mathcal A_{4N,4}(\bm z)
=\begin{bmatrix}
2\mathcal Z(\bm z)+2\bmone_N\\[1mm]
2\mathcal Z(\bm z)+\mathcal P(\bm z)+2\bmone_N\\[1mm]
1+z_{4N}
\end{bmatrix},\\[2mm]
&\mathcal A_{4N,5}
\begin{bmatrix}\bm z_1\\\bm z_2\\t\end{bmatrix}
=\begin{bmatrix}\tfrac\Gamma\Upsilon\bmone_N^\top&-\tfrac\Gamma\Upsilon\bmone_N^\top&-2\end{bmatrix}
\begin{bmatrix}\bm z_1\\\bm z_2\\t\end{bmatrix}+1.
\end{aligned}
\right.
\end{equation}
\endgroup
The hidden-width vector is $\langle 4N,4N,4N,2N+1\rangle$.
Appending $2N-1$ zero coordinates to the fourth layer and zero
columns to the last weight matrix makes every hidden width equal
to $4N$. These maps use the unpadded coordinates.

\par\medskip\noindent
\textbf{Step 3: approximation bound.}
Fix $b=0,\ldots,Q/\Gamma-1$ and $r=0,\ldots,\Gamma-1$, and let
$\nu=(r+1/2)/\Gamma$. We calculate the four hidden layers in order,
then apply the final affine map.
Since $\nu\in[0,1]$, the first layer is
\begingroup\normalsize
\[
(\DTA\circ\mathcal A_{4N,1})(b,\nu)
=\DTA\begin{bmatrix}\mathcal R(b+1)\\\nu\end{bmatrix}
=\begin{bmatrix}(\DTA\circ\mathcal R)(b+1)\\\DTA(\nu)\end{bmatrix}
=\begin{bmatrix}(\DTA\circ\mathcal R)(b+1)\\\nu\end{bmatrix}.
\]
\endgroup
There are unique $p,q=0,\ldots,N-1$ and $\ell=1,\ldots,M$ with
$b+1=qNM+N(\ell-1)+p+1$.
At this input, \eqref{eq:short-code-exact} says that $\mathcal P$ returns
$\bm e_p$, while coordinate $p$ of $\mathcal Z$ returns $m_{pq}G^{-\ell}$.
Using \eqref{eq:matrix-code}, these identities are
\begingroup\normalsize
\[
(\mathcal P\circ\DTA\circ\mathcal R)(b+1)=\bm e_p,
\qquad
[(\mathcal Z\circ\DTA\circ\mathcal R)(b+1)]_p
=m_{pq}G^{-\ell}
=\bm k[b+1]+\frac{2w(b)+y(b)}{8\,3^{\Gamma-1}}+\bm e[b+1].
\]
\endgroup
The second hidden layer has four length-$N$ blocks, denoted by
$\bm z_1,\ldots,\bm z_4$, with coordinates indexed by $j=0,\ldots,N-1$.
Their complete formula is
\begingroup\normalsize
\setlength{\arraycolsep}{3pt}
\begin{equation}\label{eq:matrix-paired-read}
\begin{aligned}
\begin{bmatrix}\bm z_1\\\bm z_2\\\bm z_3\\\bm z_4\end{bmatrix}
&:=(\DTA\circ\mathcal A_{4N,2}\circ\DTA\circ\mathcal A_{4N,1})(b,\nu)
=(\DTA\circ\mathcal A_{4N,2})
\begin{bmatrix}(\DTA\circ\mathcal R)(b+1)\\\nu\end{bmatrix}\\[1mm]
&=\DTA\begin{bmatrix}
8\,3^{\Gamma-1}(\mathcal Z\circ\DTA\circ\mathcal R)(b+1)+2\bmone_N\\[1mm]
8\,3^{\Gamma-1}(\mathcal Z\circ\DTA\circ\mathcal R)(b+1)+\bm e_p+2\bmone_N\\[1mm]
2(\mathcal Z\circ\DTA\circ\mathcal R)(b+1)
+\bigl(2+\tfrac\nu{2\,3^{\Gamma-1}}\bigr)\bmone_N\\[1mm]
2(\mathcal Z\circ\DTA\circ\mathcal R)(b+1)
+\bm e_p+\bigl(2+\tfrac\nu{2\,3^{\Gamma-1}}\bigr)\bmone_N
\end{bmatrix}.
\end{aligned}
\end{equation}
\endgroup
The next map subtracts $\bm z_2$ from $\bm z_1$ and $\bm z_4$ from
$\bm z_3$. Formula \eqref{eq:matrix-paired-read} gives
$z_{1,j}=z_{2,j}$ and $z_{3,j}=z_{4,j}$ whenever $j\ne p$.
We therefore need only the four coordinates in row $p$.
Their preactivations are positive because $m_{pq}G^{-\ell}\ge0$.
The two fractional parts below lie in $(0,1)$ by the interval
calculations preceding \eqref{eq:scalar-first-read}; those calculations
use only the phase formula and the error bound, both retained in
\eqref{eq:matrix-code}. Period two and $\tau(t)=t$ on $[0,1]$ now give
\begingroup\normalsize
\begin{align*}
z_{1,p}
&=\DTA\bigl(8\,3^{\Gamma-1}m_{pq}G^{-\ell}+2\bigr)
=\tau\!\left(8\,3^{\Gamma-1}
\left[\bm k[b+1]+\frac{2w(b)+y(b)}{8\,3^{\Gamma-1}}+\bm e[b+1]\right]+2\right)\\
&=\tau\!\left(8\,3^{\Gamma-1}\bm k[b+1]+2w(b)+2
+y(b)+8\,3^{\Gamma-1}\bm e[b+1]\right)
\\
&=\tau\bigl(y(b)+8\,3^{\Gamma-1}\bm e[b+1]\bigr)
=y(b)+8\,3^{\Gamma-1}\bm e[b+1],\\[2mm]
z_{3,p}
&=\DTA\!\left(2m_{pq}G^{-\ell}+2+\frac\nu{2\,3^{\Gamma-1}}\right)
=\tau\!\left(2\left[\bm k[b+1]+\frac{2w(b)+y(b)}{8\,3^{\Gamma-1}}+\bm e[b+1]\right]
+2+\frac\nu{2\,3^{\Gamma-1}}\right)\\
&=\tau\!\left(2\bm k[b+1]+2+
\frac{2w(b)+y(b)+2\nu}{4\,3^{\Gamma-1}}+2\bm e[b+1]\right)
\\
&=\tau\!\left(\frac{2w(b)+y(b)+2\nu}{4\,3^{\Gamma-1}}+2\bm e[b+1]\right)
=\frac{2w(b)+y(b)+2\nu}{4\,3^{\Gamma-1}}+2\bm e[b+1].
\end{align*}
\endgroup
The other two coordinates have preactivation larger by $1$.
The identity $\tau(t+1)=1-\tau(t)$ therefore gives
\begingroup\normalsize
\[
z_{2,p}=1-z_{1,p}=1-y(b)-8\,3^{\Gamma-1}\bm e[b+1],
\qquad
z_{4,p}=1-z_{3,p}
=1-\frac{2w(b)+y(b)+2\nu}{4\,3^{\Gamma-1}}-2\bm e[b+1].
\]
\endgroup
These four values determine the entire input used by
$\mathcal A_{4N,3}$. First, its argument of $\mathcal R$ simplifies as
follows; the first equality uses the equality of the paired coordinates
outside row $p$:
\begingroup\small
\begin{align*}
&\frac Q\Gamma+\Gamma3^{\Gamma-1}-\frac\Gamma4+\frac12
+\Gamma\sum_{j=0}^{N-1}
\left[3^{\Gamma-1}(z_{3,j}-z_{4,j})-\frac{z_{1,j}-z_{2,j}}4\right]\\
&\quad=\frac Q\Gamma+\Gamma3^{\Gamma-1}-\frac\Gamma4+\frac12
+\Gamma\left[3^{\Gamma-1}(2z_{3,p}-1)-\frac{2z_{1,p}-1}{4}\right]\\
&\quad=\frac Q\Gamma+\frac12+\Gamma\left[
3^{\Gamma-1}-\frac14+
3^{\Gamma-1}\left(\frac{2w(b)+y(b)+2\nu}{2\,3^{\Gamma-1}}+4\bm e[b+1]-1\right)
-\frac{2y(b)+16\,3^{\Gamma-1}\bm e[b+1]-1}{4}\right]\\
&\quad=\frac Q\Gamma+\frac12+\Gamma\left[
3^{\Gamma-1}-\frac14+w(b)+\frac{y(b)}2+\nu
+4\,3^{\Gamma-1}\bm e[b+1]-3^{\Gamma-1}
-\frac{y(b)}2-4\,3^{\Gamma-1}\bm e[b+1]+\frac14\right]\\
&\quad=\frac Q\Gamma+\Gamma\bigl(w(b)+\nu\bigr)+\frac12
=\frac Q\Gamma+\Gamma w(b)+r+1.
\end{align*}
\endgroup
This is precisely the cancellation in \eqref{eq:exact-word-position},
now applied after the row selection. The last coordinate of
$\mathcal A_{4N,3}$ is
\begingroup\normalsize
\[
\frac12+\frac12\sum_{j=0}^{N-1}(z_{1,j}-z_{2,j})
=\frac12+\frac12(2z_{1,p}-1)
=\frac12+z_{1,p}-\frac12
=y(b)+8\,3^{\Gamma-1}\bm e[b+1]\in(0,1).
\]
\endgroup
The two calculations determine the output of $\DTA\circ\mathcal A_{4N,3}$
and hence the full third hidden layer:
\begingroup\normalsize
\setlength{\arraycolsep}{3pt}
\begin{equation}\label{eq:matrix-second-address}
\begin{aligned}
&(\DTA\circ\mathcal A_{4N,3}\circ\DTA\circ\mathcal A_{4N,2}
\circ\DTA\circ\mathcal A_{4N,1})(b,\nu)
=(\DTA\circ\mathcal A_{4N,3})
\begin{bmatrix}\bm z_1\\\bm z_2\\\bm z_3\\\bm z_4\end{bmatrix}\\[1mm]
&\quad=\DTA\begin{bmatrix}
\mathcal R\bigl(Q/\Gamma+\Gamma w(b)+r+1\bigr)\\[1mm]
y(b)+8\,3^{\Gamma-1}\bm e[b+1]
\end{bmatrix}
=\begin{bmatrix}
(\DTA\circ\mathcal R)\bigl(Q/\Gamma+\Gamma w(b)+r+1\bigr)\\[1mm]
y(b)+8\,3^{\Gamma-1}\bm e[b+1]
\end{bmatrix}.
\end{aligned}
\end{equation}
\endgroup
The integer fed to $\mathcal R$ belongs to
$\{Q/\Gamma+1,\ldots,P\}$, since $w(b)\in\{0,\ldots,3^{\Gamma-1}-1\}$
and $r=0,\ldots,\Gamma-1$. Apply \eqref{eq:short-code-exact} at this
integer, and let $p'=0,\ldots,N-1$ denote its selected row.
The second family of phases in \eqref{eq:block-phase-list} and
\eqref{eq:matrix-code} give
\begingroup\normalsize
\begin{align*}
&(\mathcal P\circ\DTA\circ\mathcal R)(Q/\Gamma+\Gamma w(b)+r+1)=\bm e_{p'},\\
&[(\mathcal Z\circ\DTA\circ\mathcal R)(Q/\Gamma+\Gamma w(b)+r+1)]_{p'}\\
&\qquad=\bm k[Q/\Gamma+\Gamma w(b)+r+1]
+\frac14\left(1+\frac{s(w(b),r)}\Gamma\right)
+\bm e[Q/\Gamma+\Gamma w(b)+r+1].
\end{align*}
\endgroup
The fourth hidden layer takes \eqref{eq:matrix-second-address} as input.
Denote its two activated length-$N$ blocks by
$\bm z'_1,\bm z'_2$, and its activated last coordinate by $t$.
Because $y(b)+8\,3^{\Gamma-1}\bm e[b+1]\in(0,1)$,
$\DTA(1+u)=1-u$ applies to that coordinate and yields
\begingroup\normalsize
\setlength{\arraycolsep}{3pt}
\begin{align*}
\begin{bmatrix}\bm z'_1\\\bm z'_2\\t\end{bmatrix}
&:=(\DTA\circ\mathcal A_{4N,4}\circ\DTA\circ\mathcal A_{4N,3}
\circ\DTA\circ\mathcal A_{4N,2}\circ\DTA\circ\mathcal A_{4N,1})(b,\nu)\\[1mm]
&=(\DTA\circ\mathcal A_{4N,4})
\begin{bmatrix}
(\DTA\circ\mathcal R)(Q/\Gamma+\Gamma w(b)+r+1)\\[1mm]
y(b)+8\,3^{\Gamma-1}\bm e[b+1]
\end{bmatrix}\\[1mm]
&=\DTA\begin{bmatrix}
2(\mathcal Z\circ\DTA\circ\mathcal R)(Q/\Gamma+\Gamma w(b)+r+1)+2\bmone_N\\[1mm]
2(\mathcal Z\circ\DTA\circ\mathcal R)(Q/\Gamma+\Gamma w(b)+r+1)+\bm e_{p'}+2\bmone_N\\[1mm]
1+y(b)+8\,3^{\Gamma-1}\bm e[b+1]
\end{bmatrix}\\[1mm]
&=\begin{bmatrix}
\bigl(\DTA(2[(\mathcal Z\circ\DTA\circ\mathcal R)(Q/\Gamma+\Gamma w(b)+r+1)]_j+2)\bigr)_{j=0}^{N-1}\\[1mm]
\bigl(\DTA(2[(\mathcal Z\circ\DTA\circ\mathcal R)(Q/\Gamma+\Gamma w(b)+r+1)]_j+2+\one_{\{j=p'\}})\bigr)_{j=0}^{N-1}\\[1mm]
1-y(b)-8\,3^{\Gamma-1}\bm e[b+1]
\end{bmatrix}.
\end{align*}
\endgroup
Thus $z'_{1,j}=z'_{2,j}$ for $j\ne p'$. At $p'$, the selected integer
is nonnegative, and the fractional part lies in $(0,1)$ by the last
interval calculation in Step~3 of part~(a), immediately before the
fourth-layer evaluation there. Therefore
\begingroup\normalsize
\begin{align*}
z'_{1,p'}
&=\DTA\bigl(2[(\mathcal Z\circ\DTA\circ\mathcal R)(Q/\Gamma+\Gamma w(b)+r+1)]_{p'}+2\bigr)\\
&=\tau\!\left(2\bm k[Q/\Gamma+\Gamma w(b)+r+1]+2
+\frac12\left(1+\frac{s(w(b),r)}\Gamma\right)
+2\bm e[Q/\Gamma+\Gamma w(b)+r+1]\right)\\
&=\tau\!\left(\frac12\left(1+\frac{s(w(b),r)}\Gamma\right)
+2\bm e[Q/\Gamma+\Gamma w(b)+r+1]\right)\\
&=\frac12\left(1+\frac{s(w(b),r)}\Gamma\right)
+2\bm e[Q/\Gamma+\Gamma w(b)+r+1],
\end{align*}
\endgroup
and $z'_{2,p'}=1-z'_{1,p'}$ by $\tau(t+1)=1-\tau(t)$.
We can now apply the final affine map to the complete fourth-layer
output. Using the equalities outside row $p'$ and then
\eqref{eq:block-reconstruction}, which states
$s(w(b),r)=a_{b\Gamma+r}-a_{b\Gamma}$, gives
\begingroup\fontsize{10.5}{13}\selectfont
\thinmuskip=2mu\medmuskip=3mu\thickmuskip=4mu
\begin{align*}
&\mathcal D_{4N}\left(b,\frac{r+1/2}{\Gamma}\right)
=\mathcal A_{4N,5}\begin{bmatrix}\bm z'_1\\\bm z'_2\\t\end{bmatrix}
=\frac\Gamma\Upsilon\sum_{j=0}^{N-1}(z'_{1,j}-z'_{2,j})-2t+1\\
&\quad=\frac\Gamma\Upsilon(2z'_{1,p'}-1)
-2\bigl[1-y(b)-8\,3^{\Gamma-1}\bm e[b+1]\bigr]+1\\
&\quad=\frac\Gamma\Upsilon\left[
2\left(\frac12\left(1+\frac{s(w(b),r)}\Gamma\right)
+2\bm e[Q/\Gamma+\Gamma w(b)+r+1]\right)-1\right]
-2+2y(b)+16\,3^{\Gamma-1}\bm e[b+1]+1\\
&\quad=\frac{s(w(b),r)}\Upsilon
+\frac{4\Gamma}\Upsilon\bm e[Q/\Gamma+\Gamma w(b)+r+1]
-1+2y(b)+16\,3^{\Gamma-1}\bm e[b+1]\\
&\quad=2y(b)-1+\frac{a_{b\Gamma+r}-a_{b\Gamma}}\Upsilon
+16\,3^{\Gamma-1}\bm e[b+1]
+\frac{4\Gamma}\Upsilon\bm e[Q/\Gamma+\Gamma w(b)+r+1].
\end{align*}
\endgroup
This is the reconstruction formula obtained in Step~3 of part~(a).
Its starting-value error is at most $(32\Upsilon)^{-1}$, while
\eqref{eq:matrix-code} bounds both phase errors by $(2G)^{-1}$.
Together with $\Gamma/\Upsilon\le1/4$ and
$G\ge1024\,3^{\Gamma-1}\Upsilon$, these bounds give
\begingroup\normalsize
\begin{align*}
&\left|\mathcal D_{4N}\left(b,\frac{r+1/2}\Gamma\right)
-\frac{a_{b\Gamma+r}}\Upsilon\right|\\
&\quad=\left|2y(b)-1-\frac{a_{b\Gamma}}\Upsilon
+16\,3^{\Gamma-1}\bm e[b+1]
+\frac{4\Gamma}\Upsilon\bm e[Q/\Gamma+\Gamma w(b)+r+1]\right|\\
&\quad\le\left|2y(b)-1-\frac{a_{b\Gamma}}\Upsilon\right|
+16\,3^{\Gamma-1}|\bm e[b+1]|
+\frac{4\Gamma}\Upsilon|\bm e[Q/\Gamma+\Gamma w(b)+r+1]|\\
&\quad\le\frac1{32\Upsilon}
+16\,3^{\Gamma-1}\frac1{2G}
+\frac{4\Gamma}\Upsilon\frac1{2G}
=\frac1{32\Upsilon}+\frac{8\,3^{\Gamma-1}+2\Gamma/\Upsilon}{G}\\
&\quad\le\frac1{32\Upsilon}+\frac{17\,3^{\Gamma-1}}{2G}
\le\frac1{32\Upsilon}+\frac{17}{2048\Upsilon}
=\frac{81}{2048\Upsilon}<\frac1{25\Upsilon}.
\end{align*}
\endgroup
Taking the maximum over $b=0,\ldots,Q/\Gamma-1$ and
$r=0,\ldots,\Gamma-1$ proves the approximation bound.

\par\medskip\noindent
\textbf{Step 4: uniform bound.}
Let $b\ge-1$ and $\nu\in[0,1]$. We first determine where the third
hidden layer takes its values.
For the four output blocks $(\bm z_1,\bm z_2,\bm z_3,\bm z_4)$
of the second hidden layer at input $(b,\nu)$, the two pairs in
\eqref{eq:matrix-affine-maps} differ before activation by
$(\mathcal P\circ\DTA\circ\mathcal R)(b+1)$.
The $1$-Lipschitz property of $\DTA$ and $b+1\ge0$ give
\begingroup
\normalsize
\setlength{\arraycolsep}{3pt}
\begin{align*}
&\max\bigl\{|\bmone_N^\top(\bm z_1-\bm z_2)|,\,
|\bmone_N^\top(\bm z_3-\bm z_4)|\bigr\}
\le\max\left\{
\sum_{j=0}^{N-1}|z_{1,j}-z_{2,j}|,\,
\sum_{j=0}^{N-1}|z_{3,j}-z_{4,j}|\right\}\\
&\quad\le\sum_{j=0}^{N-1}
\bigl|[(\mathcal P\circ\DTA\circ\mathcal R)(b+1)]_j\bigr|
=\|(\mathcal P\circ\DTA\circ\mathcal R)(b+1)\|_1<3.
\end{align*}
\endgroup
Consequently, \eqref{eq:block-length-bounds} bounds the input to
$\mathcal R$ in the third affine map by
\begingroup
\normalsize
\setlength{\arraycolsep}{3pt}
\begin{align*}
&\frac Q\Gamma+\Gamma3^{\Gamma-1}-\frac\Gamma4+\frac12
+\Gamma\bmone_N^\top
\left(3^{\Gamma-1}(\bm z_3-\bm z_4)-\frac{\bm z_1-\bm z_2}{4}\right)\\
&\quad>\frac Q\Gamma+\Gamma3^{\Gamma-1}-\frac\Gamma4+\frac12
-3\Gamma3^{\Gamma-1}-\frac{3\Gamma}{4}
=\frac Q\Gamma-2\Gamma3^{\Gamma-1}-\Gamma+\frac12\\
&\quad\ge4\Gamma3^{\Gamma-1}-2\Gamma3^{\Gamma-1}-\Gamma+\frac12
=2\Gamma3^{\Gamma-1}-\Gamma+\frac12
\ge\Gamma+\frac12>0.
\end{align*}
\endgroup
Since $\DTA(\R)\subset[-1,1]$, the output of the third hidden layer
therefore belongs to
\begingroup
\normalsize
\setlength{\arraycolsep}{3pt}
\[
\left\{
\begin{bmatrix}(\DTA\circ\mathcal R)(x)\\t\end{bmatrix}
:x\ge0,\ t\in[-1,1]
\right\}.
\]
\endgroup
For every $x\ge0$ and $t\in[-1,1]$, the last two affine maps give,
using $\DTA(1+t)=1-|t|$,
\begingroup
\normalsize
\setlength{\arraycolsep}{3pt}
\begin{align*}
&\left|(\mathcal A_{4N,5}\circ\DTA\circ\mathcal A_{4N,4})
\begin{bmatrix}(\DTA\circ\mathcal R)(x)\\t\end{bmatrix}\right|\\
&\quad=\Biggl|\frac\Gamma\Upsilon\sum_{j=0}^{N-1}
\Bigl[
\DTA\bigl(2[(\mathcal Z\circ\DTA\circ\mathcal R)(x)]_j+2\bigr)\\
&\qquad\quad
-\DTA\bigl(2[(\mathcal Z\circ\DTA\circ\mathcal R)(x)]_j
+[(\mathcal P\circ\DTA\circ\mathcal R)(x)]_j+2\bigr)
\Bigr]+1-2\DTA(1+t)\Biggr|\\
&\quad\le\frac\Gamma\Upsilon\sum_{j=0}^{N-1}
\bigl|[(\mathcal P\circ\DTA\circ\mathcal R)(x)]_j\bigr|
+\bigl|1-2(1-|t|)\bigr|\\
&\quad=\frac\Gamma\Upsilon\|(\mathcal P\circ\DTA\circ\mathcal R)(x)\|_1
+|2|t|-1|
<\frac{3\Gamma}\Upsilon+1.
\end{align*}
\endgroup
Taking the supremum over $b\ge-1$ and $\nu\in[0,1]$ proves the
uniform bound in part~(b).

\par\medskip\noindent
\textbf{Step 5: parameter bounds.}
Write $\mathcal R(x)=\bm a x+\bm b$ as in
Lemma~\ref{lem:short-code-array}(ii).  The inequalities
$4\Gamma^23^{\Gamma-1}\le Q$ and $N^2\le Q$ give
$\Gamma N\le Q/2$.  Hence
\begingroup
\normalsize
\setlength{\arraycolsep}{3pt}
\begin{align*}
&N^2M\le P+N^2=\frac Q\Gamma+\Gamma3^{\Gamma-1}+N^2
\le\frac{5Q}{4\Gamma}+Q\le\frac{9Q}{4},\\
&\Gamma NM\le\frac{\Gamma P}{N}+\Gamma N\le\frac{5Q}{4N}+\frac Q2
\le\frac{5Q}{12}+\frac Q2=\frac{11Q}{12}<Q,\\
&\|\bm b\|_\infty\le\max\{gNM,2N^2M\}
\le\max\{20\Gamma NM,2N^2M\}\le\max\{20Q,9Q/2\}=20Q.
\end{align*}
\endgroup
\begin{itemize}
\item \textit{The first affine map.}
Suppose $Q=H^d$.  Let $\bm\omega=(1,H,\ldots,H^{d-1})^\top$,
and let $\bm z\in\R^{d+2}$ be the input to the last affine map
of the block-grid encoder.
The explicit encoder output in
(\ExternalNumber{3.9}{eq:block-encoder-A3-matrix}) and the shift $b+1$
in $\mathcal A_{4N,1}$ give
\begingroup
\normalsize
\setlength{\arraycolsep}{3pt}
\begin{align*}
&(\mathcal A_{4N,1}\circ\mathcal A_3^{\mathrm{bl}})(\bm z)\\
&\quad=\begin{bmatrix}
\mathcal R\!\left(\Gamma^{-1}\sum_{j=2}^dH^{j-1}z_j+z_{d+1}
+H\Gamma^{-1}(\bmone_d^\top\bm\omega)z_{d+2}
-1-\Gamma^{-1}(\bmone_d^\top\bm\omega-1)+1\right)\\[1mm]
\Gamma^{-1}z_1-z_{d+1}+1-(2\Gamma)^{-1}
\end{bmatrix}\\
&\quad=\begin{bmatrix}
\bm a\!\left(\Gamma^{-1}\sum_{j=2}^dH^{j-1}z_j+z_{d+1}
+H\Gamma^{-1}(\bmone_d^\top\bm\omega)z_{d+2}
-\Gamma^{-1}(\bmone_d^\top\bm\omega-1)\right)+\bm b\\[1mm]
\Gamma^{-1}z_1-z_{d+1}+1-(2\Gamma)^{-1}
\end{bmatrix}\\
&\quad=
\begin{bmatrix}
\Gamma^{-1}\bm a(\bm\omega-\bm e_1)^\top
&\bm a&H\Gamma^{-1}(\bmone_d^\top\bm\omega)\bm a\\[1mm]
\Gamma^{-1}\bm e_1^\top&-1&0
\end{bmatrix}\bm z
+\begin{bmatrix}
\bm b-\Gamma^{-1}(\bmone_d^\top\bm\omega-1)\bm a\\[1mm]
1-(2\Gamma)^{-1}
\end{bmatrix}.
\end{align*}
\endgroup
Since $H\bmone_d^\top\bm\omega=H(Q-1)/(H-1)<2Q$ and
$\bmone_d^\top\bm\omega-1<Q$, the coefficients satisfy
\begingroup
\normalsize
\setlength{\arraycolsep}{3pt}
\begin{align*}
\parnorm{\mathcal A_{4N,1}\circ\mathcal A_3^{\mathrm{bl}}}
&\le\max\left\{\frac{2gQ}{N\Gamma},
\frac{gQ}{N\Gamma}+\|\bm b\|_\infty,1\right\}
\le\max\left\{\frac{40Q}{N},\frac{20Q}{N}+20Q,1\right\}\\
&\le\max\left\{\frac{40Q}{3},\frac{80Q}{3},1\right\}
\le30Q=30H^d.
\end{align*}
\endgroup

\item \textit{The remaining affine maps.}
The formula in \eqref{eq:matrix-affine-maps} for
$\mathcal A_{4N,3}$ and $\mathcal R(x)=\bm a x+\bm b$ give its weight matrix and bias:
\begingroup
\normalsize
\setlength{\arraycolsep}{2pt}
\[
\mathcal A_{4N,3}(\bm z)
=\begin{bmatrix}
-\frac\Gamma4\bm a\bmone_N^\top&\frac\Gamma4\bm a\bmone_N^\top
&\Gamma3^{\Gamma-1}\bm a\bmone_N^\top&-\Gamma3^{\Gamma-1}\bm a\bmone_N^\top\\[1mm]
\frac12\bmone_N^\top&-\frac12\bmone_N^\top&\bm0_N^\top&\bm0_N^\top
\end{bmatrix}\bm z+\begin{bmatrix}
\left(\frac Q\Gamma+\Gamma3^{\Gamma-1}-\frac\Gamma4+\frac12\right)\bm a+\bm b\\[1mm]
\frac12
\end{bmatrix}.
\]
\endgroup
Using $\|\bm a\|_\infty=g/N$,
$P\le5Q/(4\Gamma)$, and $\Gamma\le Q/4$, we obtain
\begingroup
\normalsize
\setlength{\arraycolsep}{3pt}
\begin{align*}
\parnorm{\mathcal A_{4N,3}}
&\le\frac gN\left(\frac Q\Gamma+\Gamma3^{\Gamma-1}+\frac12\right)+20Q
\le\frac{20\Gamma}{N}\left(\frac{5Q}{4\Gamma}+\frac12\right)+20Q\\
&=\frac{25Q}{N}+\frac{10\Gamma}{N}+20Q
\le\frac{55Q}{2N}+20Q
\le\frac{55Q}{6}+20Q<30Q.
\end{align*}
\endgroup
The other maps in \eqref{eq:matrix-affine-maps} have radii
$\parnorm{\mathcal A_{4N,2}}\le8\,3^{\Gamma-1}\parnorm{\mathcal Z}
+\parnorm{\mathcal P}+2$,
$\parnorm{\mathcal A_{4N,4}}\le2\parnorm{\mathcal Z}
+\parnorm{\mathcal P}+2$, and
$\parnorm{\mathcal A_{4N,5}}=2$.
By Lemma~\ref{lem:short-code-array}(ii),
\begingroup
\normalsize
\setlength{\arraycolsep}{3pt}
\begin{align*}
\max_{j\in\{2,4,5\}}\parnorm{\mathcal A_{4N,j}}
&\le8\,3^{\Gamma-1}\parnorm{\mathcal Z}
+\parnorm{\mathcal P}+2
\le16\,3^{\Gamma-1}N^2MG^M+N+2\\
&\le36\,3^{\Gamma-1}QG^M+N+2
\le39\,3^{\Gamma-1}QG^M
\le2^6\,3^{\Gamma-1}QG^M.
\end{align*}
\endgroup
This also bounds $\mathcal A_{4N,3}$.
Now $M\le P/N^2+1$, $g\le20\Gamma$, and
$3^{\Gamma-1}<2\Upsilon^{2/5}$ from \eqref{eq:block-length-bounds} give
\begingroup
\normalsize
\setlength{\arraycolsep}{3pt}
\begin{equation}\label{eq:matrix-radius-array}
\begin{aligned}
\max_{2\le j\le5}\log\parnorm{\mathcal A_{4N,j}}
&\le M\log G+\log Q+(\Gamma-1)\log3+6\log2\\
&\le\frac{5Q}{4\Gamma N^2}\,20\Gamma\log2
+\log Q+\log\bigl(2^6\,3^{\Gamma-1}G\bigr)\\
&\le25\log2\,\frac Q{N^2}+\log Q
+\log\bigl(2^{18}(3^{\Gamma-1})^2\Upsilon\bigr)\\
&\le25\log2\,\frac Q{N^2}+\log Q
+\log\bigl(2^{20}\Upsilon^{9/5}\bigr)\\
&=25\log2\,\frac Q{N^2}+\log Q
+\frac95\log\Upsilon+20\log2.
\end{aligned}
\end{equation}
\endgroup

\item \textit{Every width.}
We now let $N\ge2$ denote the requested hidden width.
If $N\ge12$ and $Q\ge9$, choose the largest odd integer
$n\le\min\{N/4,\sqrt Q\}$ and use the construction above
with $n$ in place of $N$.  It has width $4n\le N$; add zero
coordinates to obtain width exactly $N$, and denote the resulting
affine maps by $\mathcal A_{N,1},\ldots,\mathcal A_{N,5}$.
Since $n\ge3$,
\begingroup
\normalsize
\setlength{\arraycolsep}{3pt}
\begin{equation*}
\begin{aligned}
n\le\min\{N/4,\sqrt Q\}<n+2\le\frac{5n}{3}
\quad&\Longrightarrow\quad
\frac Q{n^2}\le\frac{25}{9}
\max\left\{\frac{16Q}{N^2},1\right\}
\le\frac{400Q}{9N^2}+\frac{25}{9}.
\end{aligned}
\end{equation*}
\endgroup
The first affine-map bound remains $30H^d$.
The last inequality and \eqref{eq:matrix-radius-array}
give, with $C_{\mathrm{dec}}=1500$,
\begingroup
\normalsize
\setlength{\arraycolsep}{3pt}
\begin{equation*}
\begin{aligned}
\max_{2\le j\le5}\log\parnorm{\mathcal A_{N,j}}
&\le25\log2\left(\frac{400Q}{9N^2}+\frac{25}{9}\right)
+\log Q+\frac95\log\Upsilon+20\log2\\
&=\frac{10000\log2}{9}\frac Q{N^2}
+\log Q+\frac95\log\Upsilon+\frac{805\log2}{9}\\
&\le C_{\mathrm{dec}}\frac Q{N^2}+\log Q
+\frac95\log\Upsilon+C_{\mathrm{dec}}.
\end{aligned}
\end{equation*}
\endgroup
If $2\le N\le11$ or $Q<9$, instead add zero coordinates to the
width-two decoder of part~(a).  Its approximation and uniform bounds
imply those in part~(b).  Writing $g$ for the exponent in part~(a), the same encoder composition gives
$\parnorm{\mathcal A_{N,1}\circ\mathcal A_3^{\mathrm{bl}}}
\le2gQ/\Gamma\le40H^d$.
Finally, \eqref{eq:scalar-decoder-radius} gives
\begingroup
\normalsize
\setlength{\arraycolsep}{3pt}
\predisplaypenalty=0
\[
\max_{2\le j\le5}\log\parnorm{\mathcal A_{N,j}}
\le12Q
\le\begin{cases}
1452Q/N^2,&2\le N\le11,\\
96,&N\ge12,\ Q<9.
\end{cases}
\]
\endgroup
Both parameter bounds therefore hold for every $N\ge2$ with the same
constant $C_{\mathrm{dec}}=1500$.
\end{itemize}
\end{proof}

\section{Approximation with low precision: proof of Theorem~\ExternalNumber{3.1}{thm:unit-radius}}
\label{app:unit-radius}
\begingroup
\FinitePrecisionLayout{}
We compose two networks. The first returns binary grid indices using
only the parameters $-1$, $0$, and $1$. The second evaluates an integer
sequence with adjacent differences at most one at those indices, using
$B$-bit parameters in $[-1,1]$. We then round the target values on the
grid, compose the two networks, and estimate the approximation error.

For $s\in\mathbb N_0$ and $j=0,\ldots,2^s-1$, write
\begin{equation}\label{eq:unit-binary-digits}
\mathtt{bin}_s(j)=(b_1,\ldots,b_s)^\top,
\quad j=\sum_{i=1}^s2^{s-i}b_i=b_s+2b_{s-1}+...+2^{s-1}b_1,
\qquad b_i\in\{0,1\}.
\end{equation}
These are the binary digits of $j$, ordered from the most significant
to the least significant; $\mathtt{bin}_0(0)$ is the empty vector.

\subsection{Binary grid encoding}
The following proposition concerns the recovery of binary grid indices
using only the parameters $-1$, $0$, and $1$. One may reasonably expect a similar result for ReLU even under this parameter restriction. Here we establish it for $\DTA$. The key observation is the identity $\tau(a\tau(t))=\tau(at)$
for every $a\in\mathbb N$, which allows us to exploit the power
of depth to generate high-frequency oscillations with substantially
reduced width while preserving the parameters in $\{-1,0,1\}$.

\begin{proposition}[Binary grid encoder]\label{prop:unit-encoder}
Let $s,q\in\mathbb N$. There are a $\DTA$ network
$\mathcal E:[0,1]\to[0,1]^s$ with width at most
$16(s+q+1)^2 2^{s/3}$, hidden depth $5+q$, and parameters in
$\{-1,0,1\}$, and a measurable set $\Omega\subset[0,1]$ with
$|\Omega|\le2^{-qs/3}/2$, such that
\begin{equation}\label{eq:unit-encoder-exact}
\mathcal E(x)=\mathtt{bin}_s(\lfloor2^sx\rfloor),
\qquad x\in[0,1]\setminus\Omega.
\end{equation}
\end{proposition}

\begin{proof}
We first identify the binary digits using a difference of triangular waves.
We then construct this function in two stages, with all weights and
biases in $\{-1,0,1\}$.

\par\medskip\noindent
\textbf{Step 1: the goal.}
Choose $m=2^{\lceil s/3\rceil+\lceil\log_2(s+1)\rceil}$ and
$\eta=(2m^q)^{-1}$. The first factor in $m$ supplies copies for frequency multiplication;
the second controls the total length of the exceptional set. We shall realize the following map
\[
\mathcal E(x)
=\left(1-m^q[\tau(2^rx+\eta)-\tau(2^rx)+\eta]\right)_{r=1}^{s},
\qquad x\in[0,1].
\]
Since $\tau$ is $1$-Lipschitz, the expression in brackets belongs to
$[0,2\eta]=[0,m^{-q}]$, so $\mathcal E(x)\in[0,1]^s$.
To verify the required binary digits, define
\[
\Omega=\bigcup_{r=1}^s\bigcup_{j=1}^{2^r}
\bigl((j-\eta)2^{-r},j2^{-r}\bigr].
\]
For $x\in[0,1]\setminus\Omega$ and $r=1,\ldots,s$, the points
$2^rx$ and $2^rx+\eta$ lie in the same interval $[j,j+1]$, where
$j=\lfloor2^rx\rfloor$ and the slope of $\tau$ is $(-1)^j$.
Thus each coordinate is a binary digit, and their weighted sum is the grid index:
\begin{align*}
&\begin{aligned}[t]
[\mathcal E(x)]_r
&=1-(2\eta)^{-1}[\tau(2^rx+\eta)-\tau(2^rx)+\eta]
=1-(2\eta)^{-1}[(-1)^j\eta+\eta]\\*
&=\frac{1-(-1)^j}{2}
=j-2\lfloor j/2\rfloor
=\lfloor2^rx\rfloor-2\lfloor2^{r-1}x\rfloor\in\{0,1\},
\end{aligned}\\[1mm]
&\begin{aligned}[t]
\sum_{r=1}^s2^{s-r}[\mathcal E(x)]_r
&=\sum_{r=1}^s2^{s-r}\lfloor2^rx\rfloor
-\sum_{r=1}^s2^{s-r+1}\lfloor2^{r-1}x\rfloor\\*
&=\sum_{r=1}^s2^{s-r}\lfloor2^rx\rfloor
-\sum_{r=0}^{s-1}2^{s-r}\lfloor2^rx\rfloor
=\lfloor2^sx\rfloor-2^s\lfloor x\rfloor
=\lfloor2^sx\rfloor.
\end{aligned}
\end{align*}
Here $x<1$ because $1\in\Omega$. These identities prove
\eqref{eq:unit-encoder-exact} once the network is constructed. Moreover,
\[
|\Omega|
\le\sum_{r=1}^s\sum_{j=1}^{2^r}\eta2^{-r}
=s\eta=\frac{s}{2m^q}
=\frac{s}{2^{1+q\lceil s/3\rceil+q\lceil\log_2(s+1)\rceil}}
\le\frac{s}{2(s+1)^q2^{qs/3}}
\le\frac{2^{-qs/3}}{2}.
\]

\par\medskip\noindent
\textbf{Step 2: triangular waves.}
The first five hidden layers will have output
\[
\left(\bigl(\tau(2^rx)\bmone_m,\tau(2^rx+\eta)\bmone_m\bigr)_{r=1}^{s},
\eta\bmone_m\right).
\]
The maps
\mbox{$\mathcal A_0:\R\to\R^{m+1+q\log_2m+2\lfloor s/3\rfloor}$} and
\mbox{$\mathcal A_1:\R^{m+1+q\log_2m+2\lfloor s/3\rfloor}\to\R^{m+s+1}$}
are given, for $\bm x\in\R^m$ and
$\bm y\in\R^{1+q\log_2m+2\lfloor s/3\rfloor}$, by
\begingroup
\setlength{\arraycolsep}{3pt}
\begin{equation*}
\left\{\begin{aligned}
&\mathcal A_0(x)=
\begin{bmatrix}x\bmone_m\\\bmone_{1+q\log_2m+2\lfloor s/3\rfloor}\end{bmatrix},\\[2mm]
&\mathcal A_1\begin{bmatrix}\bm x\\\bm y\end{bmatrix}
=\begin{bmatrix}
\bm x\\
\left(-\sum_{j=1}^{1+q\log_2m+2\lfloor r/3\rfloor}y_j\right)_{r=1}^{s}\\
-\sum_{j=1}^{1+q\log_2m}y_j
\end{bmatrix}.
\end{aligned}\right.
\end{equation*}
\endgroup
For $\bm x\in\R^m$ and $\bm y\in\R^{s+1}$, define
$\mathcal A_2:\R^{m+s+1}\to\R^{(2s+1)m}$ by
\begingroup
\setlength{\arraycolsep}{3pt}
\begin{equation*}
\mathcal A_2\begin{bmatrix}\bm x\\\bm y\end{bmatrix}
=\begin{bmatrix}
\left[\begin{matrix}
\left(\sum_{j=1}^{2^{r-2\lfloor r/3\rfloor}}x_j\right)\bmone_m\\
\left(\sum_{j=1}^{2^{r-2\lfloor r/3\rfloor}}x_j+y_r+1\right)\bmone_m
\end{matrix}\right]_{r=1}^{s}\\
(y_{s+1}+1)\bmone_m
\end{bmatrix}.
\end{equation*}
\endgroup
The last two maps
$\mathcal A_3,\mathcal A_4:\R^{(2s+1)m}\to\R^{(2s+1)m}$ are,
for $i=3,4$ and $\bm x_r,\bm y_r,\bm z\in\R^m$,
\begingroup
\setlength{\arraycolsep}{3pt}
\begin{equation*}
\mathcal A_i
\begin{bmatrix}\left[\begin{matrix}\bm x_r\\\bm y_r\end{matrix}\right]_{r=1}^{s}\\\bm z\end{bmatrix}
=\begin{bmatrix}
\left[\begin{matrix}
\left(\sum_{j=1}^{2^{\lfloor r/3\rfloor}}x_{r,j}\right)\bmone_m\\
\left(\sum_{j=1}^{2^{\lfloor r/3\rfloor}}y_{r,j}\right)\bmone_m
\end{matrix}\right]_{r=1}^{s}\\\bm z
\end{bmatrix}.
\end{equation*}
\endgroup
All vector families are stacked in increasing index.
The sums use distinct available coordinates, since
$2^{r-2\lfloor r/3\rfloor}\le2^{\lceil s/3\rceil+1}\le m$
and $2^{\lfloor r/3\rfloor}\le m$.
For $x\in[0,1]$, the first two hidden layers create the shifts,
and the third combines them with copies of $x$. Their successive outputs are
\begingroup\fontsize{10.95}{13.2}\selectfont
\setlength{\arraycolsep}{3pt}
\begin{align*}
x
&\xmapsto{\DTA\circ\mathcal A_0}
\begin{bmatrix}x\bmone_m\\\bmone_{1+q\log_2m+2\lfloor s/3\rfloor}\end{bmatrix}
\xmapsto{\DTA\circ\mathcal A_1}
\DTA\begin{bmatrix}
x\bmone_m\\
\left(-\sum_{j=1}^{1+q\log_2m+2\lfloor r/3\rfloor}1\right)_{r=1}^{s}\\
-\sum_{j=1}^{1+q\log_2m}1
\end{bmatrix}\\
&=\begin{bmatrix}
x\bmone_m\\
\left(2^{-1-q\log_2m-2\lfloor r/3\rfloor}-1\right)_{r=1}^{s}\\
2^{-1-q\log_2m}-1
\end{bmatrix}
=\begin{bmatrix}
x\bmone_m\\
\left(4^{-\lfloor r/3\rfloor}\eta-1\right)_{r=1}^{s}\\
\eta-1
\end{bmatrix}\\
&\quad\xmapsto{\DTA\circ\mathcal A_2}
\DTA\begin{bmatrix}
\left[\begin{matrix}
\left(\sum_{j=1}^{2^{r-2\lfloor r/3\rfloor}}x\right)\bmone_m\\
\left(\sum_{j=1}^{2^{r-2\lfloor r/3\rfloor}}x+4^{-\lfloor r/3\rfloor}\eta-1+1\right)\bmone_m
\end{matrix}\right]_{r=1}^{s}\\
(\eta-1+1)\bmone_m
\end{bmatrix}\\
&\qquad=\DTA\begin{bmatrix}
\left[\begin{matrix}
4^{-\lfloor r/3\rfloor}2^rx\bmone_m\\
4^{-\lfloor r/3\rfloor}(2^rx+\eta)\bmone_m
\end{matrix}\right]_{r=1}^{s}\\\eta\bmone_m
\end{bmatrix}
=\begin{bmatrix}
\left[\begin{matrix}
\tau(4^{-\lfloor r/3\rfloor}2^rx)\bmone_m\\
\tau(4^{-\lfloor r/3\rfloor}(2^rx+\eta))\bmone_m
\end{matrix}\right]_{r=1}^{s}\\\eta\bmone_m
\end{bmatrix}.
\end{align*}
\endgroup
For every positive integer $a$, the identity
$\tau(a\tau(t))=\tau(at)$ follows from evenness and period two:
if $\tau(t)=\pm t+2j$ for an integer $j$, then
$\tau(a\tau(t))=\tau(\pm at+2aj)=\tau(at)$.
The last two layers therefore multiply each frequency twice by
$2^{\lfloor r/3\rfloor}$:
\begingroup\fontsize{10.95}{13.2}\selectfont
\setlength{\arraycolsep}{3pt}
\begin{align}
&\begin{bmatrix}
\left[\begin{matrix}
\tau(4^{-\lfloor r/3\rfloor}2^rx)\bmone_m\\
\tau(4^{-\lfloor r/3\rfloor}(2^rx+\eta))\bmone_m
\end{matrix}\right]_{r=1}^{s}\\\eta\bmone_m
\end{bmatrix}
\xmapsto{\DTA\circ\mathcal A_3}
\begin{bmatrix}
\left[\begin{matrix}
\tau\left(\sum_{j=1}^{2^{\lfloor r/3\rfloor}}
\tau(4^{-\lfloor r/3\rfloor}2^rx)\right)\bmone_m\\
\tau\left(\sum_{j=1}^{2^{\lfloor r/3\rfloor}}
\tau(4^{-\lfloor r/3\rfloor}(2^rx+\eta))\right)\bmone_m
\end{matrix}\right]_{r=1}^{s}\\\eta\bmone_m
\end{bmatrix}\notag\\
&\quad=\begin{bmatrix}
\left[\begin{matrix}
\tau(2^{-\lfloor r/3\rfloor}2^rx)\bmone_m\\
\tau(2^{-\lfloor r/3\rfloor}(2^rx+\eta))\bmone_m
\end{matrix}\right]_{r=1}^{s}\\\eta\bmone_m
\end{bmatrix}
\xmapsto{\DTA\circ\mathcal A_4}
\begin{bmatrix}
\left[\begin{matrix}
\tau\left(\sum_{j=1}^{2^{\lfloor r/3\rfloor}}
\tau(2^{-\lfloor r/3\rfloor}2^rx)\right)\bmone_m\\
\tau\left(\sum_{j=1}^{2^{\lfloor r/3\rfloor}}
\tau(2^{-\lfloor r/3\rfloor}(2^rx+\eta))\right)\bmone_m
\end{matrix}\right]_{r=1}^{s}\\\eta\bmone_m
\end{bmatrix}\notag\\
&\qquad=\begin{bmatrix}
\left[\begin{matrix}
\tau(2^rx)\bmone_m\\
\tau(2^rx+\eta)\bmone_m
\end{matrix}\right]_{r=1}^{s}\\\eta\bmone_m
\end{bmatrix}.
\label{eq:unit-encoder-frequency-arrow}
\end{align}
\endgroup

\par\medskip\noindent
\textbf{Step 3: binary digits.}
Starting with \eqref{eq:unit-encoder-frequency-arrow}, we construct
$q$ further hidden layers with output
\[
\left(m^q[\tau(2^rx+\eta)-\tau(2^rx)+\eta]\bmone_m\right)_{r=1}^{s}.
\]
The final affine map will subtract one copy of each value from one.
Define $\mathcal A_5:\R^{(2s+1)m}\to\R^{sm}$,
$\mathcal A_{i+4}:\R^{sm}\to\R^{sm}$ for $i=2,\ldots,q$, and
$\mathcal A_{q+5}:\R^{sm}\to\R^s$ by
\begingroup
\setlength{\arraycolsep}{3pt}
\begin{equation*}
\left\{\begin{aligned}
&\mathcal A_5
\begin{bmatrix}\left[\begin{matrix}\bm x_r\\\bm y_r\end{matrix}\right]_{r=1}^{s}\\\bm z\end{bmatrix}
=\left(\left[\sum_{j=1}^m(y_{r,j}-x_{r,j}+z_j)\right]\bmone_m\right)_{r=1}^{s},\\[2mm]
&\mathcal A_{i+4}((\bm x_r)_{r=1}^{s})
=\left(\left[\sum_{j=1}^m x_{r,j}\right]\bmone_m\right)_{r=1}^{s},
\qquad i=2,\ldots,q,\\[1mm]
&\mathcal A_{q+5}((\bm x_r)_{r=1}^{s})
=(1-x_{r,1})_{r=1}^{s}.
\end{aligned}\right.
\end{equation*}
\endgroup
In the first line, $\bm x_r,\bm y_r,\bm z\in\R^m$; in the last
two lines, $\bm x_r\in\R^m$. The middle family is empty when $q=1$.
The bracket in Step~1 belongs to $[0,m^{-q}]$. After the $i$th
multiplication, every activation argument therefore belongs to
$[0,m^{i-q}]\subset[0,1]$, for $i=1,\ldots,q$, so the activation
preserves it. Starting with \eqref{eq:unit-encoder-frequency-arrow},
we obtain the first multiplication:
\begingroup\fontsize{10.95}{13.2}\selectfont
\setlength{\arraycolsep}{3pt}
\begin{align*}
&\begin{bmatrix}
\left[\begin{matrix}\tau(2^rx)\bmone_m\\\tau(2^rx+\eta)\bmone_m\end{matrix}\right]_{r=1}^{s}\\\eta\bmone_m
\end{bmatrix}
\xmapsto{\DTA\circ\mathcal A_5}
\left(\DTA\left(\sum_{j=1}^m[\tau(2^rx+\eta)-\tau(2^rx)+\eta]\right)\bmone_m\right)_{r=1}^{s}\\
&\quad=\left(\DTA\left(m[\tau(2^rx+\eta)-\tau(2^rx)+\eta]\right)\bmone_m\right)_{r=1}^{s}
=\left(m[\tau(2^rx+\eta)-\tau(2^rx)+\eta]\bmone_m\right)_{r=1}^{s}.
\end{align*}
\endgroup
For $i=2,\ldots,q$, the next multiplication has the same form:
\begingroup\fontsize{10.95}{13.2}\selectfont
\setlength{\arraycolsep}{3pt}
\begin{align*}
&\left(m^{i-1}[\tau(2^rx+\eta)-\tau(2^rx)+\eta]\bmone_m\right)_{r=1}^{s}\\
&\quad\xmapsto{\DTA\circ\mathcal A_{i+4}}
\left(\DTA\left(\sum_{j=1}^m m^{i-1}[\tau(2^rx+\eta)-\tau(2^rx)+\eta]\right)\bmone_m\right)_{r=1}^{s}\\
&\quad=\left(\DTA\left(m^i[\tau(2^rx+\eta)-\tau(2^rx)+\eta]\right)\bmone_m\right)_{r=1}^{s}
=\left(m^i[\tau(2^rx+\eta)-\tau(2^rx)+\eta]\bmone_m\right)_{r=1}^{s}.
\end{align*}
\endgroup
Define
$\mathcal E=\mathcal A_{q+5}\circ\DTA\circ\mathcal A_{q+4}
\circ\cdots\circ\DTA\circ\mathcal A_0$. Its output is exactly
\begin{align*}
\mathcal E(x)
&=(\mathcal A_{q+5}\circ\DTA\circ\mathcal A_{q+4}
\circ\cdots\circ\DTA\circ\mathcal A_0)(x)\\
&=\mathcal A_{q+5}
\left(\left(m^q[\tau(2^rx+\eta)-\tau(2^rx)+\eta]\bmone_m\right)_{r=1}^{s}\right)\\
&=\left(1-m^q[\tau(2^rx+\eta)-\tau(2^rx)+\eta]\right)_{r=1}^{s}.
\end{align*}
This is the map we wanted in Step 1.

\Needspace{5\baselineskip}
\par\medskip\noindent
\textbf{Step 4: network size and parameters.}
\begin{itemize}
\item Every weight and bias is $-1$, $0$, or $1$.
The first layer supplies distinct constant coordinates; their sums
in $\mathcal A_1$ produce the required negative integers. All later
multipliers are sums of distinct copies.
\item The first two hidden widths are
$m+1+q\log_2m+2\lfloor s/3\rfloor$ and $m+s+1$.
The next three are $(2s+1)m$, and the remaining $q$ are $sm$.
Since $m\le4(s+1)2^{s/3}$ and $\log_2m\le2s$,
\begin{align*}
\operatorname{width}(\mathcal E)
&\le(2s+1)m+1+q\log_2m+2\lfloor s/3\rfloor\\
&\le4(2s+1)(s+1)2^{s/3}+1+2sq+2s/3\\
&\le\bigl[12(s+1)^2+2(s+1)(q+1)\bigr]2^{s/3}
\le16(s+q+1)^2 2^{s/3}.
\end{align*}
There are five hidden layers through $\mathcal A_4$ and $q$ further
hidden layers through $\mathcal A_{q+4}$, giving hidden depth $5+q$.
\end{itemize}
\end{proof}

\subsection{Low-precision point fitting}
The next three lemmas turn binary indices into coordinate vectors,
compute partial sums of binary digits, and recover row values and their
changes. The first is equivalent to the exact interpolation problem on the vertices of the
unit cube (points in $\{0,1\}^s$), with the output values forming the standard basis of
$\R^{2^s}$ and all parameters restricted to $\{-1,0,1\}$.
While interpolation is a standard tool in ReLU approximation theory
(see, e.g., \citealp[Lemma~3.3]{shijun:Characterized:by:Numer:Neurons}
for univariate interpolation), even this severe parameter restriction
poses no difficulty for ReLU in the present setting.
We establish below the corresponding result for $\DTA$, which is intrinsically harder than the case of ReLU.

\begin{lemma}[Binary index selection]\label{lem:unit-index-encoder}
For $s\in\mathbb N_0$, there is a $\DTA$ network
$\mathcal S:\R^s\to\R^{2^s}$ with width $3\cdot2^s$, hidden
depth three, and parameters in $\{-1,0,1\}$ such that
\begin{equation}\label{eq:unit-index-output}
\mathcal S\bigl(\mathtt{bin}_s(j)\bigr)=\bm e_j,
\qquad j=0,\ldots,2^s-1,
\end{equation}
where $\bm e_j$ is the $j$th coordinate vector, numbered from zero.
\end{lemma}

\begin{proof}
We compare the input digits with each binary index and return the
coordinate vector of the matching index. All comparisons use only the
coefficients $-1$, $0$, and $1$.

\par\medskip\noindent
\textbf{Step 1: affine maps.}
Write $\mathtt{bin}_s(i)=(b_{i,1},\ldots,b_{i,s})^\top$.
The first two affine maps will form two copies of
$1-\sum_{a=1}^s|b_{i,a}-b_{j,a}|$ and one copy of
$2-\sum_{a=1}^s|b_{i,a}-b_{j,a}|$ at the input
$\mathtt{bin}_s(j)$. Define
\begin{equation*}
\mathcal A_0:\R^s\longrightarrow\R^{2s+1},\qquad
\mathcal A_0(\bm x)=\begin{bmatrix}\bm x\\\bmone_{s+1}\end{bmatrix}.
\end{equation*}
For $\bm y\in\R^{2s+1}$, the map
$\mathcal A_1:\R^{2s+1}\to\R^{3\cdot2^s}$ has the following
coordinates for $i=0,\ldots,2^s-1$:
\begin{equation*}
\left\{\begin{aligned}
&(\mathcal A_1(\bm y))_{3i+1}
 =(\mathcal A_1(\bm y))_{3i+2}
 =1+\sum_{a=1}^s\bigl[(2b_{i,a}-1)y_a-b_{i,a}y_{s+a}\bigr],\\
&(\mathcal A_1(\bm y))_{3i+3}
 =1+y_{2s+1}
 +\sum_{a=1}^s\bigl[(2b_{i,a}-1)y_a-b_{i,a}y_{s+a}\bigr].
\end{aligned}\right.
\end{equation*}
The next hidden layer will return one at $i=j$ and zero elsewhere.
Define $\mathcal A_2:\R^{3\cdot2^s}\to\R^{2^s}$ and
$\mathcal A_3:\R^{2^s}\to\R^{2^s}$ by
\begin{equation*}
\left\{\begin{aligned}
&\mathcal A_2(\bm z)
=\bigl(1+z_{3i+1}+z_{3i+2}-z_{3i+3}\bigr)_{i=0}^{2^s-1},\\
&\mathcal A_3(\bm u)=\bm u.
\end{aligned}\right.
\end{equation*}
The complete network is
\[
\mathcal S=\mathcal A_3\circ\DTA\circ\mathcal A_2
\circ\DTA\circ\mathcal A_1\circ\DTA\circ\mathcal A_0.
\]

\par\medskip\noindent
\textbf{Step 2: evaluation.}
Fix $j=0,\ldots,2^s-1$. Each bracket below stacks the three
coordinates numbered $3i+1,3i+2,3i+3$. Since all input digits
are zero or one, the first two affine maps give
\begingroup\fontsize{10.95}{13.2}\selectfont
\setlength{\arraycolsep}{3pt}
\begin{align*}
&(\mathcal A_1\circ\DTA\circ\mathcal A_0)
 \bigl(\mathtt{bin}_s(j)\bigr)
=\mathcal A_1\left(\DTA
 \begin{bmatrix}\mathtt{bin}_s(j)\\\bmone_{s+1}\end{bmatrix}\right)
=\mathcal A_1
 \begin{bmatrix}\mathtt{bin}_s(j)\\\bmone_{s+1}\end{bmatrix}\\
&\quad=\left[
 \begin{matrix}
 \left(1+\sum_{a=1}^s[(2b_{i,a}-1)b_{j,a}-b_{i,a}]\right)\bmone_2\\
 1+1+\sum_{a=1}^s[(2b_{i,a}-1)b_{j,a}-b_{i,a}]
 \end{matrix}
 \right]_{i=0}^{2^s-1}\\
&\quad=\left[
 \begin{matrix}
 \left(1-\sum_{a=1}^s[b_{i,a}+b_{j,a}-2b_{i,a}b_{j,a}]\right)\bmone_2\\
 2-\sum_{a=1}^s[b_{i,a}+b_{j,a}-2b_{i,a}b_{j,a}]
 \end{matrix}
 \right]_{i=0}^{2^s-1}
=\left[
 \begin{matrix}
 \left(1-\sum_{a=1}^s|b_{i,a}-b_{j,a}|\right)\bmone_2\\
 2-\sum_{a=1}^s|b_{i,a}-b_{j,a}|
 \end{matrix}
 \right]_{i=0}^{2^s-1}.
\end{align*}
\endgroup
For $i=0,\ldots,2^s-1$, the last two maps act on this
computed output as follows:
\begingroup\normalsize
\setlength{\arraycolsep}{3pt}
\begin{align*}
[\mathcal S(\mathtt{bin}_s(j))]_i
&=\left[
 (\mathcal A_3\circ\DTA\circ\mathcal A_2)
 \left(
 \left[\begin{matrix}
 \DTA\!\left(1-\sum_{a=1}^s|b_{k,a}-b_{j,a}|\right)\bmone_2\\
 \DTA\!\left(2-\sum_{a=1}^s|b_{k,a}-b_{j,a}|\right)
 \end{matrix}\right]_{k=0}^{2^s-1}
 \right)\right]_i\\
&=\DTA\!\left(
 1+2\DTA\!\left(1-\sum_{a=1}^s|b_{i,a}-b_{j,a}|\right)
 -\DTA\!\left(2-\sum_{a=1}^s|b_{i,a}-b_{j,a}|\right)
 \right).
\end{align*}
\endgroup
For $i=j$, this is $\DTA(1+2\DTA(1)-\DTA(2))=\DTA(3)=1$.
For $i\ne j$, the digit differences have positive integer sum,
so both inner arguments are integers at most one. Using
$\DTA(t)=2^t-1$ at such integers, the same output coordinate satisfies
\begingroup\normalsize
\begin{align*}
[\mathcal S(\mathtt{bin}_s(j))]_i
&=\DTA\!\left(
 1+2\left(2^{1-\sum_{a=1}^s|b_{i,a}-b_{j,a}|}-1\right)
 -\left(2^{2-\sum_{a=1}^s|b_{i,a}-b_{j,a}|}-1\right)\right)\\
&=\DTA\!\left(
 1+2^{2-\sum_{a=1}^s|b_{i,a}-b_{j,a}|}-2
 -2^{2-\sum_{a=1}^s|b_{i,a}-b_{j,a}|}+1\right)\\
&=\DTA(0)=0.
\end{align*}
\endgroup
This proves \eqref{eq:unit-index-output}.
The hidden-width vector is $\langle 2s+1,3\cdot2^s,2^s\rangle$.
Every weight and bias is zero or $\pm1$, since
$b_{i,a}\in\{0,1\}$ and $2b_{i,a}-1\in\{-1,1\}$.
For $s=0$, the successive hidden outputs are
$(1)$, $(1,1,0)$, and $(1)$, so the same maps also handle the
empty binary string.
\end{proof}

A scalar input can store the digits of a short segment. The next network
returns their sum up to the selected position. This network mechanism indeed exploits the full power of $\DTA$ and reveals
a fundamental distinction between $\DTA$ and classical activations
such as ReLU under finite-precision parameter constraints.

\begin{lemma}[Binary prefix sums]\label{lem:unit-prefix-reader}
Let $n,m\in\mathbb N$ with $n,m\ge2$. There is a $\DTA$ network
$\mathcal T:\R\times\R^n\to\R$ such that for all $(b_0,\ldots,b_{n-2})\in\{0,1\}^{n-1}$
\begin{equation}\label{eq:unit-prefix-output}
\mathcal T\left(2\sum_{t=0}^{n-2}b_t4^{-(t+1)},\bm e_r\right)
=2^{-m}\sum_{t=0}^{r-1}b_t,
\qquad r=0,\ldots,n-1.
\end{equation}
Its hidden-width vector is $\langle 4^{n-2}+n+m,3n+m,4n-4\rangle$,
and its parameters belong to $\{-1,-1/2,0,1\}$.
The first affine map only copies input coordinates and supplies ones;
the last has coefficients in $\{-1,0,1\}$ and zero bias.
\end{lemma}

\begin{proof}
We first recover each tail of the encoded digit sequence. A third
hidden layer then selects the digits before position $r$, and the
output map adds their contributions.

\par\medskip\noindent
\textbf{Step 1: the tails.}
For the input in \eqref{eq:unit-prefix-output}, the first two hidden
layers will return two copies of
$2\sum_{\ell=t}^{n-2}b_\ell4^{t-\ell-1}$ for each $t=0,\ldots,n-2$,
followed by a zero row, while retaining $\bm e_r$ and supplying $m$
ones. Define
$\mathcal A_0:\R\times\R^n\to\R^{4^{n-2}}\times\R^n\times\R^m$
and
$\mathcal A_1:\R^{4^{n-2}}\times\R^n\times\R^m
\to\R^{n\times2}\times\R^n\times\R^m$ by
\begin{equation*}
\left\{\begin{aligned}
&\mathcal A_0(x,\bm y)=(x\bmone_{4^{n-2}},\bm y,\bmone_m),\\[1mm]
&\mathcal A_1(\bm x,\bm y,\bm z)
=\left(
\begin{bmatrix}
\left(\sum_{a=1}^{4^t}x_a\right)_{t=0}^{n-2}\\0
\end{bmatrix}\bmone_2^\top,\bm y,\bm z
\right).
\end{aligned}\right.
\end{equation*}
The two columns of the first output are identical. Matrix entries
are read row by row as network coordinates.
For the input in \eqref{eq:unit-prefix-output}, the first hidden layer
copies a number in
$[0,2/3]$ and coordinates equal to zero or one.
For $t=0,\ldots,n-2$, the sum
$2\sum_{\ell=0}^{t-1}b_\ell4^{t-\ell-1}$ is an even integer,
whereas $2\sum_{\ell=t}^{n-2}b_\ell4^{t-\ell-1}\in[0,2/3]$.
The first two hidden layers therefore give
\begingroup\fontsize{10.95}{13.2}\selectfont
\setlength{\arraycolsep}{3pt}
\begin{align}
&(\DTA\circ\mathcal A_1\circ\DTA\circ\mathcal A_0)
\left(2\sum_{\ell=0}^{n-2}b_\ell4^{-(\ell+1)},\bm e_r\right)\notag\\
&\quad=(\DTA\circ\mathcal A_1)
\left(2\sum_{\ell=0}^{n-2}b_\ell4^{-(\ell+1)}\bmone_{4^{n-2}},
\bm e_r,\bmone_m\right)\notag\\
&\quad=\DTA\left(
\begin{bmatrix}
\left(\sum_{a=1}^{4^t}2\sum_{\ell=0}^{n-2}b_\ell4^{-(\ell+1)}\right)_{t=0}^{n-2}\\0
\end{bmatrix}\bmone_2^\top,\bm e_r,\bmone_m
\right)\notag\\
&\quad=\left(
\begin{bmatrix}
\left(\tau\left(2\sum_{\ell=0}^{n-2}b_\ell4^{t-\ell-1}\right)\right)_{t=0}^{n-2}\\0
\end{bmatrix}\bmone_2^\top,\bm e_r,\bmone_m
\right)\notag\\
&\quad=\left(
\begin{bmatrix}
\left(\tau\left(2\sum_{\ell=0}^{t-1}b_\ell4^{t-\ell-1}
 +2\sum_{\ell=t}^{n-2}b_\ell4^{t-\ell-1}\right)\right)_{t=0}^{n-2}\\0
\end{bmatrix}\bmone_2^\top,\bm e_r,\bmone_m
\right)\notag\\
&\quad=\left(
\begin{bmatrix}
\left(2\sum_{\ell=t}^{n-2}b_\ell4^{t-\ell-1}\right)_{t=0}^{n-2}\\0
\end{bmatrix}\bmone_2^\top,\bm e_r,\bmone_m
\right).
\label{eq:unit-prefix-second-output}
\end{align}
\endgroup

\par\medskip\noindent
\textbf{Step 2: the sum.}
The last hidden layer will form four exponential terms whose
alternating sum is $2^{-m}b_t\one_{\{t<r\}}$. Define
$\mathcal A_2:\R^{n\times2}\times\R^n\times\R^m
\to\R^{(n-1)\times4}$ and
$\mathcal A_3:\R^{(n-1)\times4}\to\R$ as follows. The rows of
$\bm x\in\R^{n\times2}$ and the coordinates of $\bm y\in\R^n$
are numbered from zero, and those of $\bm z\in\R^m$ from one:
\begin{equation*}
\left\{\begin{aligned}
&\bm e_t^\top\mathcal A_2(\bm x,\bm y,\bm z)
=\left[\begin{matrix}
x_{t,1}+x_{t,2}-2^{-1}x_{t+1,1}
 +\sum_{a=t+1}^{n-1}y_a-\sum_{a=1}^mz_a\\
x_{t,1}+x_{t,2}-2^{-1}x_{t+1,1}-\sum_{a=1}^mz_a\\
\sum_{a=t+1}^{n-1}y_a-\sum_{a=1}^mz_a\\
-\sum_{a=1}^mz_a
\end{matrix}\right]^{\!\top},\\[2mm]
&\mathcal A_3(\bm x)
=\sum_{t=0}^{n-2}(x_{t,1}-x_{t,2}-x_{t,3}+x_{t,4}).
\end{aligned}\right.
\end{equation*}
Here $t=0,\ldots,n-2$, and the argument of $\mathcal A_3$ is an
$(n-1)\times4$ matrix. The complete network is
\begin{equation}\label{eq:unit-prefix-composition}
\mathcal T=\mathcal A_3\circ\DTA\circ\mathcal A_2
\circ\DTA\circ\mathcal A_1\circ\DTA\circ\mathcal A_0.
\end{equation}

At the output in \eqref{eq:unit-prefix-second-output}, the last
hidden layer has the following $t$th row, for $t=0,\ldots,n-2$:
\begingroup\fontsize{10.95}{13.2}\selectfont
\setlength{\arraycolsep}{3pt}
\begin{align}
&\bm e_t^\top(\DTA\circ\mathcal A_2)
\left(
\begin{bmatrix}
\left(2\sum_{\ell=a}^{n-2}b_\ell4^{a-\ell-1}\right)_{a=0}^{n-2}\\0
\end{bmatrix}\bmone_2^\top,\bm e_r,\bmone_m
\right)\notag\\
&\quad=\DTA\left[\begin{matrix}
4\sum_{\ell=t}^{n-2}b_\ell4^{t-\ell-1}
 -\sum_{\ell=t+1}^{n-2}b_\ell4^{t-\ell}
 +\sum_{a=t+1}^{n-1}(\bm e_r)_a-\sum_{a=1}^m1\\
4\sum_{\ell=t}^{n-2}b_\ell4^{t-\ell-1}
 -\sum_{\ell=t+1}^{n-2}b_\ell4^{t-\ell}-\sum_{a=1}^m1\\
\sum_{a=t+1}^{n-1}(\bm e_r)_a-\sum_{a=1}^m1\\
-\sum_{a=1}^m1
\end{matrix}\right]^{\!\top}\notag\\
&\quad=\DTA\left[\begin{matrix}
b_t+\sum_{\ell=t+1}^{n-2}b_\ell4^{t-\ell}
 -\sum_{\ell=t+1}^{n-2}b_\ell4^{t-\ell}+\one_{\{t<r\}}-m\\
b_t+\sum_{\ell=t+1}^{n-2}b_\ell4^{t-\ell}
 -\sum_{\ell=t+1}^{n-2}b_\ell4^{t-\ell}-m\\
\one_{\{t<r\}}-m\\-m
\end{matrix}\right]^{\!\top}\\
&\quad=\DTA
\begin{bmatrix}b_t+\one_{\{t<r\}}-m\\b_t-m\\\one_{\{t<r\}}-m\\-m\end{bmatrix}^{\!\top}\notag\\
&\quad=\left[
2^{b_t+\one_{\{t<r\}}-m}-1,\,
2^{b_t-m}-1,\,
2^{\one_{\{t<r\}}-m}-1,\,
2^{-m}-1
\right].
\label{eq:unit-prefix-third-output}
\end{align}
\endgroup
The activation uses its exponential formula because $m\ge2$.
For $t=n-2$, the two sums with lower limit $t+1$ are empty, and
the zero last row in \eqref{eq:unit-prefix-second-output} gives the
same calculation. By \eqref{eq:unit-prefix-third-output}, the final
affine map in \eqref{eq:unit-prefix-composition} evaluates the complete network:
\begingroup\normalsize
\begin{align*}
\mathcal T\left(2\sum_{t=0}^{n-2}b_t4^{-(t+1)},\bm e_r\right)
&=\mathcal A_3\circ(\DTA\circ\mathcal A_2
\circ\DTA\circ\mathcal A_1\circ\DTA\circ\mathcal A_0)
\left(2\sum_{t=0}^{n-2}b_t4^{-(t+1)},\bm e_r\right)\\
&=\mathcal A_3\left(
\begin{bmatrix}
2^{b_t+\one_{\{t<r\}}-m}-1&2^{b_t-m}-1&
2^{\one_{\{t<r\}}-m}-1&2^{-m}-1
\end{bmatrix}_{t=0}^{n-2}\right)\\
&=\sum_{t=0}^{n-2}\left[
(2^{b_t+\one_{\{t<r\}}-m}-1)-(2^{b_t-m}-1)
 -(2^{\one_{\{t<r\}}-m}-1)+(2^{-m}-1)\right]\\
&=2^{-m}\sum_{t=0}^{n-2}
\left(2^{b_t+\one_{\{t<r\}}}-2^{b_t}-2^{\one_{\{t<r\}}}+1\right)\\
&=2^{-m}\sum_{t=0}^{n-2}(2^{b_t}-1)(2^{\one_{\{t<r\}}}-1)\\
&=2^{-m}\sum_{t=0}^{n-2}b_t\one_{\{t<r\}}
=2^{-m}\sum_{t=0}^{r-1}b_t.
\end{align*}
\endgroup

\par\medskip\noindent
\textbf{Step 3: parameters.}
The outputs of $\mathcal A_0$, $\mathcal A_1$, and $\mathcal A_2$
have respectively $4^{n-2}+n+m$, $2n+n+m=3n+m$, and
$4(n-1)=4n-4$ coordinates. Every weight and bias is zero or $\pm1$,
except the coefficient $-1/2$ in $\mathcal A_2$.
The formulas for $\mathcal A_0$ and $\mathcal A_3$ give the stated
first and last affine maps.
\end{proof}

The next lemma further exploits the distinctive features of $\DTA$
by decomposing a selected value into the first value of its row and
the change along that row under low-precision parameter constraints.
Both networks have five hidden layers, matching those for changes
within the selected segment in Proposition~\ref{prop:unit-decoder}.

\begin{lemma}[Row values and changes]\label{lem:unit-row-changes}
Let $N,m\in\mathbb N$ with $N,m\ge2$, and let $k\in\{0,\ldots,m+1\}$.
Suppose that $y[0],\ldots,y[N^2-1]\in[-2^m,2^m]\cap\mathbb Z$ and
$|y[j+1]-y[j]|\le2^k$ for $j=0,\ldots,N^2-2$.
There are $\DTA$ networks with the following properties.
\begin{enumerate}[label=\textup{(\roman*)}]
\item $\mathcal R:\R^N\to\R$ has five hidden widths
$\langle N,N,N,2m+3,4m+6\rangle$, parameters in $\{-1,0,1\}$, and
\[
\mathcal R(\bm e_i)=2^{-m}y[Ni],\qquad i=0,\ldots,N-1.
\]
\item $\mathcal P:\R^N\times\R^N\to\R$ has five hidden widths
$\langle 2N,2N+m+4,3N-1,2N,8N\rangle$, parameters in
$2^{-k}\mathbb Z\cap[-1,1]$, and
\[
\mathcal P(\bm e_i,\bm e_j)
=2^{-m}\bigl(y[Ni+j]-y[Ni]\bigr),\qquad i,j=0,\ldots,N-1.
\]
\end{enumerate}
Both last affine maps have coefficients in $\{-1,0,1\}$;
their biases are $-1$ and $0$, respectively.
\end{lemma}

\begin{proof}
We represent the first value of each row by its binary digits and
the remaining values by their successive changes. The following
two constructions recover these quantities separately.
\par\medskip\noindent
\textbf{(i) Row values: construction of $\mathcal R$.}
We construct $\mathcal R$ so that $\mathcal R(\bm e_i)=2^{-m}y[Ni]$.
Since $(2^{-m}y[Ni]+1)/2$ belongs to
$2^{-m-1}\mathbb Z\cap[0,1]$, choose
$\bm\beta=(\beta_{\ell i})\in\{0,1\}^{(m+2)\times N}$ such that
\begin{equation}\label{eq:unit-decoder-row-start}
\frac{2^{-m}y[Ni]+1}{2}
=\sum_{\ell=0}^{m+1}2^{-\ell}\beta_{\ell i},
\qquad i=0,\ldots,N-1.
\end{equation}
The affine maps have domains and codomains
\[
\R^N\xrightarrow{\mathcal A_0}\R^N
\xrightarrow{\mathcal A_1}\R^N
\xrightarrow{\mathcal A_2}\R^N
\xrightarrow{\mathcal A_3}\R^{2m+3}
\xrightarrow{\mathcal A_4}\R^{4m+6}
\xrightarrow{\mathcal A_5}\R,
\]
and are defined by
\begin{equation}\label{eq:unit-row-maps}
\left\{\begin{aligned}
&\mathcal A_0(\bm x)=\mathcal A_1(\bm x)=\mathcal A_2(\bm x)=\bm x,\\[1mm]
&\mathcal A_3(\bm x)=(\bm\beta\bm x,\bmone_{m+1}),\\[1mm]
&\mathcal A_4(\bm x,\bm y)
=\left(
 x_0\bmone_2,
 \left(x_\ell-\sum_{a=1}^{\ell}y_a\right)_{\ell=1}^{m+1}\bmone_2^\top,
 \left(-\sum_{a=1}^{\ell}y_a\right)_{\ell=1}^{m+1}\bmone_2^\top
\right),\\[1mm]
&\mathcal A_5(\bm x,\bm y,\bm z)
=\bmone_2^\top\bm x+\bmone_{m+1}^\top(\bm y-\bm z)\bmone_2-1.
\end{aligned}\right.
\end{equation}
Here the inputs of $\mathcal A_4$ are
$\bm x=(x_0,\ldots,x_{m+1})$ and $\bm y=(y_1,\ldots,y_{m+1})$.
Its outputs are a vector in $\R^2$ and two $(m+1)\times2$ matrices,
which are the three inputs of $\mathcal A_5$.
Matrix entries are read row by row as network coordinates.
Define
\begin{equation*}
\mathcal R=\mathcal A_5\circ\DTA\circ\mathcal A_4
\circ\cdots\circ\DTA\circ\mathcal A_0.
\end{equation*}
For $i=0,\ldots,N-1$, the first three identity maps and activations
preserve $\bm e_i$, so the first four hidden layers satisfy
\begin{align*}
&(\DTA\circ\mathcal A_3\circ\DTA\circ\mathcal A_2
 \circ\DTA\circ\mathcal A_1\circ\DTA\circ\mathcal A_0)(\bm e_i)\\
&\quad=(\DTA\circ\mathcal A_3)(\bm e_i)
 =\DTA(\bm\beta\bm e_i,\bmone_{m+1})
 =\DTA\left((\beta_{\ell i})_{\ell=0}^{m+1},\bmone_{m+1}\right)
 =\left((\beta_{\ell i})_{\ell=0}^{m+1},\bmone_{m+1}\right).
\end{align*}
The last hidden layer evaluates the binary coefficients in
\eqref{eq:unit-decoder-row-start}. Since
$\beta_{\ell i}-\ell\le0$ for $\ell=1,\ldots,m+1$, the complete
network therefore gives
\begingroup\fontsize{10.95}{13.2}\selectfont
\begin{align*}
\mathcal R(\bm e_i)
&=(\mathcal A_5\circ\DTA\circ\mathcal A_4)
 \left((\beta_{\ell i})_{\ell=0}^{m+1},\bmone_{m+1}\right)\\
&=(\mathcal A_5\circ\DTA)\left(
 \beta_{0i}\bmone_2,
 \left(\beta_{\ell i}-\sum_{a=1}^{\ell}1\right)_{\ell=1}^{m+1}\bmone_2^\top,
 \left(-\sum_{a=1}^{\ell}1\right)_{\ell=1}^{m+1}\bmone_2^\top
 \right)\\
&=\mathcal A_5\left(
 \beta_{0i}\bmone_2,
 \left(2^{\beta_{\ell i}-\ell}-1\right)_{\ell=1}^{m+1}\bmone_2^\top,
 \left(2^{-\ell}-1\right)_{\ell=1}^{m+1}\bmone_2^\top
 \right)\\
&=2\beta_{0i}+2\sum_{\ell=1}^{m+1}
 \left[(2^{\beta_{\ell i}-\ell}-1)-(2^{-\ell}-1)\right]-1\\
&=2\beta_{0i}+2\sum_{\ell=1}^{m+1}2^{-\ell}(2^{\beta_{\ell i}}-1)-1
 =2\sum_{\ell=0}^{m+1}2^{-\ell}\beta_{\ell i}-1\\
&=2\frac{2^{-m}y[Ni]+1}{2}-1=2^{-m}y[Ni].
\end{align*}
\endgroup
All coefficients and biases in \eqref{eq:unit-row-maps} belong to
$\{-1,0,1\}$. Their output dimensions give the five widths in
part~(i), and $\mathcal A_5$ has bias $-1$.

\par\medskip\noindent
\textbf{(ii) Row changes: construction of $\mathcal P$.}
We next construct $\mathcal P$ with output
$2^{-m}(y[Ni+j]-y[Ni])$. Its coefficient matrix records the
successive changes within each row:
\begin{equation}\label{eq:unit-decoder-word-totals}
\bm\Delta=(\Delta_{it})\in\R^{N\times(N-1)},\qquad
\Delta_{it}=2^{-k}\bigl(y[Ni+t+1]-y[Ni+t]\bigr)
\in2^{-k}\mathbb Z\cap[-1,1],
\end{equation}
where $i=0,\ldots,N-1$ and $t=0,\ldots,N-2$.
We first add the changes before column $j$ in every row, then select
row $i$. At $(\bm e_i,\bm e_j)$, the fourth hidden layer will return
\[
\left(\left(\frac12+2^{-m-3}\bigl(y[Na+j]-y[Na]\bigr)\right)_{a=0}^{N-1},\bm e_i\right);
\]
the last hidden layer and output map will then select row $i$ and
recover $2^{-m}(y[Ni+j]-y[Ni])$.
The affine maps have domains and codomains
\[
\R^{2N}\xrightarrow{\mathcal B_0}\R^{2N}
\xrightarrow{\mathcal B_1}\R^{2N+m+4}
\xrightarrow{\mathcal B_2}\R^{3N-1}
\xrightarrow{\mathcal B_3}\R^{2N}
\xrightarrow{\mathcal B_4}\R^{8N}
\xrightarrow{\mathcal B_5}\R.
\]
For $\bm x,\bm y\in\R^N$, $\bm z\in\R^{m+3}$, and $u\in\R$,
define the first three maps by
\begin{equation*}
\left\{\begin{aligned}
&\mathcal B_0(\bm x,\bm y)=(\bm x,\bm y),\\[1mm]
&\mathcal B_1(\bm x,\bm y)=(\bm x,\bm y,\bmone_{m+3},-1),\\[1mm]
&\mathcal B_2(\bm x,\bm y,\bm z,u)
=\begin{bmatrix}
 \left(\sum_{a=t+1}^{N-1}y_a-\sum_{a=1}^{m+3-k}z_a\right)_{t=0}^{N-2}\\
 -\left(\sum_{a=1}^{m+3-k}z_a\right)\bmone_{N-1}\\
 \bm x\\u+1
\end{bmatrix}.
\end{aligned}\right.
\end{equation*}
The coordinates of $\bm z$ are numbered from one, and
$m+3-k\in\{2,\ldots,m+3\}$. The four output groups of
$\mathcal B_2$ have dimensions $N-1$, $N-1$, $N$, and one.
For these four groups, and then for two vectors in $\R^N$ and
two $N\times4$ matrices, respectively, define
\begin{equation*}
\left\{\begin{aligned}
&\mathcal B_3(\bm x,\bm y,\bm z,u)
=\bigl(u\bmone_N+\bm\Delta(\bm x-\bm y),\bm z\bigr),\\[1mm]
&\mathcal B_4(\bm x,\bm y)
=(\bm x\bmone_4^\top,(\bm x+\bm y)\bmone_4^\top),\\[1mm]
&\mathcal B_5(\bm x,\bm y)
=\bmone_N^\top(\bm x-\bm y)\bmone_4.
\end{aligned}\right.
\end{equation*}
Thus $\mathcal B_3$ outputs two vectors in $\R^N$, and
$\mathcal B_4$ outputs the two $N\times4$ matrices used by
$\mathcal B_5$. Define
\begin{equation*}
\mathcal P=\mathcal B_5\circ\DTA\circ\mathcal B_4
\circ\cdots\circ\DTA\circ\mathcal B_0.
\end{equation*}
The first three hidden layers retain the row index and represent
$\one_{\{t<j\}}$ by a difference of two exponential outputs.
For $i,j=0,\ldots,N-1$, their composition gives
\begingroup\fontsize{10.95}{13.2}\selectfont
\setlength{\arraycolsep}{4pt}
\begin{align*}
&(\DTA\circ\mathcal B_2\circ\DTA\circ\mathcal B_1
 \circ\DTA\circ\mathcal B_0)(\bm e_i,\bm e_j)\\
&\quad=(\DTA\circ\mathcal B_2)
 \bigl(\DTA(\bm e_i,\bm e_j,\bmone_{m+3},-1)\bigr)
 =(\DTA\circ\mathcal B_2)(\bm e_i,\bm e_j,\bmone_{m+3},-1/2)\\
&\quad=\DTA\begin{bmatrix}
 \left(\sum_{a=t+1}^{N-1}(\bm e_j)_a-\sum_{a=1}^{m+3-k}1\right)_{t=0}^{N-2}\\
 -\left(\sum_{a=1}^{m+3-k}1\right)\bmone_{N-1}\\
 \bm e_i\\-1/2+1
\end{bmatrix}
=\begin{bmatrix}
 \left(2^{k-m-3+\one_{\{t<j\}}}-1\right)_{t=0}^{N-2}\\
 (2^{k-m-3}-1)\bmone_{N-1}\\
 \bm e_i\\1/2
\end{bmatrix}.
\end{align*}
\endgroup
The first two groups use the exponential formula since
$k-m-3+\one_{\{t<j\}}\le-1$. Multiplying their difference by
$\bm\Delta$ adds exactly the changes before $j$. Thus, for
$a=0,\ldots,N-1$, the next affine map gives
\begin{align*}
&\left[
 (\mathcal B_3\circ\DTA\circ\mathcal B_2\circ\DTA\circ\mathcal B_1
 \circ\DTA\circ\mathcal B_0)(\bm e_i,\bm e_j)
 \right]_a\\
&\quad=\frac12+\sum_{t=0}^{N-2}2^{-k}
 \bigl(y[Na+t+1]-y[Na+t]\bigr)
 \left[(2^{k-m-3+\one_{\{t<j\}}}-1)-(2^{k-m-3}-1)\right]\\
&\quad=\frac12+2^{-m-3}\sum_{t=0}^{N-2}
 \bigl(y[Na+t+1]-y[Na+t]\bigr)(2^{\one_{\{t<j\}}}-1)\\
&\quad=\frac12+2^{-m-3}\sum_{t=0}^{j-1}
 \bigl(y[Na+t+1]-y[Na+t]\bigr)\\
&\quad=\frac12+2^{-m-3}
 \left(\sum_{t=1}^{j-1}y[Na+t]+y[Na+j]
       -y[Na]-\sum_{t=1}^{j-1}y[Na+t]\right)\\
&\quad=\frac12+2^{-m-3}\bigl(y[Na+j]-y[Na]\bigr)
 \in[1/4,3/4].
\end{align*}
For $j=0$, the sum is empty and the value is $1/2$. The other output group is
$\bm e_i$, so the fourth activation preserves all coordinates.
The last hidden layer selects row $i$ by
$\DTA(t+1)=1-t$ for $t\in[1/4,3/4]$. Stacking the two
$N\times4$ inputs of $\mathcal B_5$, we obtain
\begingroup\fontsize{10.95}{13.2}\selectfont
\medmuskip=2mu\thinmuskip=2mu\relax
\begin{align*}
&\mathcal P(\bm e_i,\bm e_j)\\
&\quad=(\mathcal B_5\circ\DTA\circ\mathcal B_4)
 \left(\left(\frac12+2^{-m-3}(y[Na+j]-y[Na])\right)_{a=0}^{N-1},
 \bm e_i\right)\\[1mm]
&\quad=(\mathcal B_5\circ\DTA)\left(
 \begin{bmatrix}
 \left(\frac12+2^{-m-3}(y[Na+j]-y[Na])\right)_a\\[2mm]
 \left(\frac12+2^{-m-3}(y[Na+j]-y[Na])+\one_{\{a=i\}}\right)_a
 \end{bmatrix}\bmone_4^\top
 \right)\\[2mm]
&\quad=\mathcal B_5\left(
 \begin{bmatrix}
 \left(\frac12+2^{-m-3}(y[Na+j]-y[Na])\right)_a\\[2mm]
 \left(\frac12+(1-2\one_{\{a=i\}})2^{-m-3}(y[Na+j]-y[Na])\right)_a
 \end{bmatrix}\bmone_4^\top
 \right)\\[1.5mm]
&\quad=4\sum_{a=0}^{N-1}\left[
 \frac12+2^{-m-3}(y[Na+j]-y[Na])
 -\frac12-(1-2\one_{\{a=i\}})2^{-m-3}(y[Na+j]-y[Na])
 \right]\\
&\quad=4\cdot2\cdot2^{-m-3}\sum_{a=0}^{N-1}
 \one_{\{a=i\}}(y[Na+j]-y[Na])\\[1.5mm]
&\quad=2^{-m}(y[Ni+j]-y[Ni]).
\end{align*}
\endgroup
The only coefficients outside $\{-1,0,1\}$ are the entries of
$\bm\Delta$, which belong to $2^{-k}\mathbb Z\cap[-1,1]$ by
\eqref{eq:unit-decoder-word-totals}. All integer offsets are sums
of distinct coordinates equal to one. The output dimensions of
$\mathcal B_0,\ldots,\mathcal B_4$ give the five widths in part~(ii),
and $\mathcal B_5$ has coefficients in $\{-1,0,1\}$ and zero bias.
\end{proof}

Combining the preceding three lemmas, we construct the low-precision
decoder that we will ultimately use for fitting integer sequences
at their binary indices.

\begin{proposition}[Low-precision point fitting]\label{prop:unit-decoder}
Let $N,B,s,m\in\mathbb N$ satisfy $N\ge256$, $B\ge2$,
$2\le m\le\lfloor N/16\rfloor$, and
$8192\cdot2^s\le N^2\min\{B,\log(eN)\}$.
Every sequence $y[0],\ldots,y[2^s-1]\in[-2^m,2^m]\cap\mathbb Z$
with $|y[j]-y[j-1]|\le1$ for $j=1,\ldots,2^s-1$ admits
$\mathcal D\in\NN_\DTA(s,1;N,8,\mathbb Q_{B-2}^0\cap[-1,1])$
satisfying
\begin{equation*}
\mathcal D\bigl(\mathtt{bin}_s(j)\bigr)=2^{-m}y[j],
\qquad j=0,\ldots,2^s-1.
\end{equation*}
\end{proposition}

\begin{proof}
Arrange the sequence in rows and divide each row into short segments.
Lemma~\ref{lem:unit-row-changes} reads the first value of the row and
the change up to the first value of the selected segment. We then use
Lemma~\ref{lem:unit-prefix-reader} to add the changes within that segment.

\Needspace{3\baselineskip}
\par\medskip\noindent
\textbf{Step 1: setup.}
\begin{itemize}
\item \textit{Rows and segments.}
We use $2^h$ rows, each containing $2^h$ consecutive segments of length $n$,
where $h=\lfloor\log_2(N/16)\rfloor$ and
$n=2^{\lfloor\log_2\min\{B/2,1+h/2\}\rfloor}$.
These choices and the hypothesis give $h\ge4$,
$2(n-1)\le\min\{B-2,h\}$, and $2^{2h}n\ge2^s$ (we skip the straightforward numercial checking here).
Extend the sequence to length $2^{2h}n$ by repeating
its last value. Segment $j$ in row $i$ is
$(y[2^hni+nj+r])_{r=0}^{n-1}$, for $0\le i,j<2^h$.
Every index $t$ is therefore uniquely written as $t=2^hni+nj+r$,
where $i,j$ specify the row and segment, and $r=0,...,n-1$ specifies
the position within the segment.

The segment starts $y[0],y[n],\ldots,y[n(2^{2h}-1)]$ have successive
differences at most $\min\{n,2^{m+1}\}=2^k$, where
$k=\min\{\log_2n,m+1\}$. Lemma~\ref{lem:unit-row-changes}, applied
with $2^h$ in place of $N$, therefore gives five-layer networks
$\mathcal R:\R^{2^h}\to\R$ and
$\mathcal P:\R^{2^h}\times\R^{2^h}\to\R$ satisfying
\begin{equation}\label{eq:unit-decoder-row-changes}
\left\{\begin{aligned}
&\mathcal R(\bm e_i)=2^{-m}y[2^hni],\\[1mm]
&\mathcal P(\bm e_i,\bm e_j)
=2^{-m}\bigl(y[2^hni+nj]-y[2^hni]\bigr).
\end{aligned}\right.
\end{equation}
Their parameters lie in $\mathbb Q_{B-2}^0\cap[-1,1]$, since
$k\le\log_2n\le2(n-1)\le B-2$.

\item \textit{Binary index selection.}
Write $\ell=\log_2n$, so $s\le2h+\ell$. For $\bm z\in\R^s$, split
$(\bm 0_{2h+\ell-s},\bm z)$ into blocks
$\bm z^{(1)},\bm z^{(2)},\bm z^{(3)}$ of lengths $h,h,\ell$.
Let $\mathcal S_d$ denote the network of
Lemma~\ref{lem:unit-index-encoder} with input dimension $d$, and
define $\mathcal S:\R^s\to\R^{2^{h+1}+n}$ by
$\mathcal S(\bm z)=(\mathcal S_h(\bm z^{(1)}),\allowbreak
\mathcal S_h(\bm z^{(2)}),\allowbreak\mathcal S_\ell(\bm z^{(3)}))$.
Its first hidden affine map sends $\bm z$ to
$(\bm z^{(1)},\bmone_{h+1},\bm z^{(2)},\bmone_{h+1},\bm z^{(3)},\bmone_{\ell+1})$.
Write $\mathcal A_1^{(d)},\mathcal A_2^{(d)}$ for the second and third
hidden affine maps of $\mathcal S_d$. The corresponding maps of $\mathcal S$ are
$(\bm u,\bm v,\bm w)\mapsto
(\mathcal A_a^{(h)}(\bm u),\allowbreak
\mathcal A_a^{(h)}(\bm v),\allowbreak\mathcal A_a^{(\ell)}(\bm w))$,
for $a=1,2$. These maps have block diagonal weight matrices and stacked
biases. Hence $\mathcal S$ has three hidden layers, width at most
$6\cdot2^h+3n$, parameters in $\{-1,0,1\}$, and identity output map.

For $t=2^hni+nj+r<2^s$, the three terms occupy disjoint binary positions,
so concatenation of their digits and \eqref{eq:unit-index-output} give
\begin{align*}
&(\bm 0_{2h+\ell-s},\mathtt{bin}_s(t))
=\mathtt{bin}_{2h+\ell}(2^{h+\ell}i+2^\ell j+r)
=(\mathtt{bin}_h(i),\mathtt{bin}_h(j),\mathtt{bin}_\ell(r)),\\[1mm]
&\quad\implies\mathcal S(\mathtt{bin}_s(t))
=\bigl(\mathcal S_h(\mathtt{bin}_h(i)),
\mathcal S_h(\mathtt{bin}_h(j)),
\mathcal S_\ell(\mathtt{bin}_\ell(r))\bigr)
=(\bm e_i,\bm e_j,\bm e_r).
\end{align*}
When $n=1$, the third input block is empty and $\mathcal S_0$ returns
$\bm e_0=1$.
\end{itemize}

\par\medskip\noindent
\textbf{Step 2: reduction.}
Suppose that two five-layer networks
$\mathcal T_+,\mathcal T_-:\R^{2^h}\times\R^{2^h}\times\R^n\to\R$
read the changes within the selected segment through the identity
\begin{equation}\label{eq:unit-code-goal}
\mathcal T_+(\bm e_i,\bm e_j,\bm e_r)
-\mathcal T_-(\bm e_i,\bm e_j,\bm e_r)
=2^{-m}\sum_{t=0}^{r-1}
\bigl(y[2^hni+nj+t+1]-y[2^hni+nj+t]\bigr).
\end{equation}
For $n\ge2$, assume both hidden-width vectors are
$\langle 2^{h+1}+n,2^{h+1}+n,4^{n-2}+n+m,3n+m,4n-4\rangle$.
Assume also that their parameters belong to $\mathbb Q_{B-2}^0\cap[-1,1]$ and their
last affine maps have coefficients in $\{-1,0,1\}$ and zero bias.
For $n=1$, set $\mathcal T_+=\mathcal T_-=0$ and omit their coordinates.
Define $\mathcal G:\R^{2^{h+1}+n}\to\R$ and $\mathcal D:\R^s\to\R$ by
\begin{equation*}
\left\{\begin{aligned}
&\mathcal G(\bm x,\bm y,\bm z)
=\mathcal R(\bm x)+\mathcal P(\bm x,\bm y)
+\mathcal T_+(\bm x,\bm y,\bm z)-\mathcal T_-(\bm x,\bm y,\bm z),\\[1mm]
&\mathcal D=\mathcal G\circ\mathcal S.
\end{aligned}\right.
\end{equation*}
These four networks have the same hidden depth. Stack their first
affine maps, using zero coefficients for unused inputs, and place their
remaining hidden maps in diagonal blocks. Every hidden parameter is
therefore a parameter of one of the four networks or zero. Combining
their last affine maps with the indicated signs leaves all output
coefficients in $\{-1,0,1\}$ and the bias equal to $-1$, by
Lemma~\ref{lem:unit-row-changes}. Hence this parallel realization of
$\mathcal G$ has five hidden layers and the required parameter alphabet.
For $r\ge1$, equations~\eqref{eq:unit-decoder-row-changes} and
\eqref{eq:unit-code-goal} give the complete output:
\begingroup\medmuskip=1.8mu\thinmuskip=2mu\relax
\begin{align*}
&\mathcal G(\bm e_i,\bm e_j,\bm e_r)\\
&\quad=\mathcal R(\bm e_i)+\mathcal P(\bm e_i,\bm e_j)
+\mathcal T_+(\bm e_i,\bm e_j,\bm e_r)-\mathcal T_-(\bm e_i,\bm e_j,\bm e_r)\\
&\quad=2^{-m}\left[y[2^hni]+y[2^hni+nj]-y[2^hni]
+\sum_{t=0}^{r-1}\bigl(y[2^hni+nj+t+1]-y[2^hni+nj+t]\bigr)\right]\\
&\quad=2^{-m}\left[y[2^hni+nj]
+\sum_{t=1}^{r}y[2^hni+nj+t]-\sum_{t=0}^{r-1}y[2^hni+nj+t]\right]\\
&\quad=2^{-m}\left[y[2^hni+nj]+y[2^hni+nj+r]
+\sum_{t=1}^{r-1}\bigl(y[2^hni+nj+t]-y[2^hni+nj+t]\bigr)-y[2^hni+nj]\right]\\
&\quad=2^{-m}y[2^hni+nj+r].
\end{align*}
\endgroup
For $r=0$, the sum in \eqref{eq:unit-code-goal} is empty and the first
two terms give the same result. Whenever $2^hni+nj+r<2^s$, we obtain
\begin{align*}
\mathcal D\bigl(\mathtt{bin}_s(2^hni+nj+r)\bigr)
=(\mathcal G\circ\mathcal S)\bigl(\mathtt{bin}_s(2^hni+nj+r)\bigr)
=\mathcal G(\bm e_i,\bm e_j,\bm e_r)
=2^{-m}y[2^hni+nj+r].
\end{align*}
The identity output map of $\mathcal S$ preserves every parameter of
the first map of $\mathcal G$ at the composition. Thus the desired
evaluation and hidden depth $3+5=8$ follow once we construct
$\mathcal T_\pm$ and check the total width.

\par\medskip\noindent
\textbf{Step 3: short segments.}
Suppose $n\ge2$. Store the positive and negative increments in
$\bm\Theta^+,\bm\Theta^-\in\R^{2^h\times2^h}$. For $i,j=0,\ldots,2^h-1$, define
\begin{equation*}
\Theta^\pm_{ij}
=\sum_{t=0}^{n-2}\max\{\pm(y[2^hni+nj+t+1]-y[2^hni+nj+t]),0\}\,4^{-(t+1)}.
\end{equation*}
Every digit in these sums is zero or one, so
$\Theta^\pm_{ij}\in[0,1/3]\cap4^{-(n-1)}\mathbb Z$.
Two hidden layers will send $(\bm e_i,\bm e_j,\bm e_r)$ to
$(2\Theta^\pm_{ij},\bm e_r)$; the three-layer network
$\mathcal T$ of Lemma~\ref{lem:unit-prefix-reader} will then return
the sum of the first $r$ stored digits, multiplied by $2^{-m}$.
Define $\mathcal C_0^\pm,\mathcal C_1:\R^{2^{h+1}+n}\to\R^{2^{h+1}+n}$
and $\mathcal C_2:\R^{2^{h+1}+n}\to\R^{n+1}$ by
\begin{equation*}
\left\{\begin{aligned}
&\mathcal C_0^\pm(\bm x,\bm y,\bm z)=(\bm\Theta^\pm\bm y,\bm x,\bm z),\\[1mm]
&\mathcal C_1(\bm x,\bm y,\bm z)=(\bm x+\bm y,\bm x,\bm z),\\[1mm]
&\mathcal C_2(\bm x,\bm y,\bm z)
=\bigl(1+\bmone_{2^h}^\top(\bm y-\bm x),\bm z\bigr),
\end{aligned}\right.
\end{equation*}
where $\bm x,\bm y\in\R^{2^h}$ and $\bm z\in\R^n$. The networks are
\begin{equation*}
\mathcal T_\pm
=\mathcal T\circ\mathcal C_2\circ\DTA\circ\mathcal C_1
\circ\DTA\circ\mathcal C_0^\pm.
\end{equation*}
The bounds $\Theta^\pm_{aj}\in[0,1/3]$ give the successive outputs
\begingroup\fontsize{10.95}{13.2}\selectfont
\begin{align*}
(\bm e_i,\bm e_j,\bm e_r)
&\xmapsto{\DTA\circ\mathcal C_0^\pm}
\DTA(\bm\Theta^\pm\bm e_j,\bm e_i,\bm e_r)
=\bigl((\Theta^\pm_{aj})_{a=0}^{2^h-1},\bm e_i,\bm e_r\bigr)\\
&\quad\xmapsto{\DTA\circ\mathcal C_1}
\DTA\bigl((\Theta^\pm_{aj}+\one_{\{a=i\}})_a,(\Theta^\pm_{aj})_a,\bm e_r\bigr)
=\bigl((\DTA(\Theta^\pm_{aj}+\one_{\{a=i\}}))_a,(\Theta^\pm_{aj})_a,\bm e_r\bigr)\\
&\qquad\xmapsto{\mathcal C_2}
\left(1+\sum_{a=0}^{2^h-1}
\bigl[\Theta^\pm_{aj}-\DTA(\Theta^\pm_{aj}+\one_{\{a=i\}})\bigr],\bm e_r\right)\\
&\quad\qquad=\left(1+\sum_{a\ne i}(\Theta^\pm_{aj}-\Theta^\pm_{aj})
+\Theta^\pm_{ij}-(1-\Theta^\pm_{ij}),\bm e_r\right)
=(2\Theta^\pm_{ij},\bm e_r).
\end{align*}
\endgroup
Now apply \eqref{eq:unit-prefix-output} to the stored digits. This gives
\begin{align*}
\mathcal T_\pm(\bm e_i,\bm e_j,\bm e_r)
&=\mathcal T(2\Theta^\pm_{ij},\bm e_r)\\
&=\mathcal T\left(
2\sum_{t=0}^{n-2}\max\{\pm(y[2^hni+nj+t+1]-y[2^hni+nj+t]),0\}\,4^{-(t+1)},\bm e_r\right)\\
&=2^{-m}\sum_{t=0}^{r-1}
\max\{\pm(y[2^hni+nj+t+1]-y[2^hni+nj+t]),0\}.
\end{align*}
Since $\max\{a,0\}-\max\{-a,0\}=a$ for every real $a$, subtracting
these two outputs yields \eqref{eq:unit-code-goal}.

\Needspace{3\baselineskip}
\par\medskip\noindent
\textbf{Step 4: network size and parameters.}
\begin{itemize}
\item Each entry of $\bm\Theta^\pm$ is a multiple of $4^{-(n-1)}$
in $[0,1/3]$, and $2(n-1)\le B-2$. Hence all coefficients of
$\mathcal C_0^\pm$ lie in $\mathbb Q_{B-2}^0\cap[-1,1]$.
The other two maps have coefficients in $\{-1,0,1\}$ and biases in
$\{0,1\}$. The first affine map of $\mathcal T$ copies its input
coordinates and supplies ones. Composing it with $\mathcal C_2$
therefore retains these coefficients and biases. The only remaining
coefficient outside $\{-1,0,1\}$ is $-1/2$ from
Lemma~\ref{lem:unit-prefix-reader}; here $n\ge2$ implies $B\ge4$.
The last affine maps of $\mathcal T_\pm$ are those of $\mathcal T$,
so they have coefficients in $\{-1,0,1\}$ and zero bias.
For $n=1$, the two networks are absent and $k=0$, so the entire
decoder has parameters in $\{-1,0,1\}$, including at $B=2$.
\item Each column below lists the five hidden widths of one branch.
The first two come from Lemma~\ref{lem:unit-row-changes}; the third
comes from the two selection layers followed by
Lemma~\ref{lem:unit-prefix-reader}. Adding them gives the widths of
$\mathcal G$:
\begingroup\fontsize{10.95}{13.2}\selectfont
\setlength{\arraycolsep}{3pt}
\begin{equation*}
\begin{bmatrix}2^h\\2^h\\2^h\\2m+3\\4m+6\end{bmatrix}
+\begin{bmatrix}2^{h+1}\\2^{h+1}+m+4\\3\cdot2^h-1\\2^{h+1}\\8\cdot2^h\end{bmatrix}
+2\begin{bmatrix}2^{h+1}+n\\2^{h+1}+n\\4^{n-2}+n+m\\3n+m\\4n-4\end{bmatrix}
=\begin{bmatrix}
7\cdot2^h+2n\\7\cdot2^h+2n+m+4\\
4\cdot2^h+2\cdot4^{n-2}+2n+2m-1\\
2^{h+1}+6n+4m+3\\8\cdot2^h+8n+4m-2
\end{bmatrix}.
\end{equation*}
\endgroup
For $n=1$, only the first two columns are needed. From $h\ge4$ and
$2(n-1)\le h$, we have $n\le1+h/2\le2^{h-2}$ and
$4^{n-2}\le2^{h-2}$. In both cases,
\begin{align*}
\max\{\operatorname{width}(\mathcal G),\operatorname{width}(\mathcal S)\}
&\le8\cdot2^h+8n+4m+2\cdot4^{n-2}+4\\
&\le\frac{8N}{16}+\frac{8N}{64}+\frac{4N}{16}
+\frac{2N}{64}+4
=\frac{29N}{32}+4\le N.
\end{align*}
The parallel realization and composition
in Step 2 now give the asserted width, depth, and parameter alphabet.
\end{itemize}
\end{proof}

\subsection{Proof of Theorem~\ExternalNumber{3.1}{thm:unit-radius}}
\label{app:unit-theorem-proof}
\begin{proof}[Proof of Theorem~\ExternalNumber{3.1}{thm:unit-radius}]
Fix $f\in\mathcal H^\beta([0,1]^d)$ and integers $N,B\ge1$ with $B\ge2$.
We encode the grid indices, fit the rounded grid values, and estimate the
error on the grid cells and the encoder's exceptional set. Let
$q=\lceil3p\beta\rceil$ and $N_0=\lceil2^{128+2d/\beta+2q}\rceil$.
We first consider $N\ge N_0$ and write $t=\log_2N$.

\Needspace{3\baselineskip}
\par\medskip\noindent
\textbf{Step 1: setup.}
\begin{itemize}
\item Choose the grid exponent
\begin{equation}\label{eq:unit-main-scale}
s=\left\lfloor\frac1d\log_2
\frac{N^2\min\{B,\log(eN)\}}{32768}\right\rfloor.
\end{equation}
Round the grid values at scale $2^{-m}$, where $m=\lfloor\beta s\rfloor$.
There are $2^{ds}$ grid cells. The denominator leaves four sequence
positions per cell while keeping $8192\cdot2^{ds+2}$ within the decoder's
budget $N^2\min\{B,\log(eN)\}$. The same choice gives
$2^{ds}>N^2\min\{B,\log(eN)\}/2^{d+15}$.
Our lower bound on $N$ ensures $s\le3t/d$ and
$2\le m\le3t\le\lfloor N/16\rfloor$.
\item Apply Proposition~\ref{prop:unit-encoder} to the $d$ coordinates
in parallel. This gives $\mathcal E:[0,1]^d\to[0,1]^{ds}$ with parameters
in $\{-1,0,1\}$, hidden depth $5+q$, and width at most
$16d(s+q+1)^2 2^{s/3}$. Outside the union $\Omega$ of the scalar
exceptional sets, extended along the other coordinates, its output is
\begin{equation}\label{eq:unit-parallel-encoder}
\mathcal E(\bm x)=
\begin{bmatrix}
\mathtt{bin}_s(\lfloor2^sx_1\rfloor)\\[-1mm]
\vdots\\[-1mm]
\mathtt{bin}_s(\lfloor2^sx_d\rfloor)
\end{bmatrix},
\qquad \bm x\in[0,1]^d\setminus\Omega.
\end{equation}
The measure bound in that proposition and $q\ge3p\beta$ give
$\mathbf{m}(\Omega)\le(d/2)2^{-qs/3}\le(d/2)2^{-p\beta s}$.
\end{itemize}

\par\medskip\noindent
\textbf{Step 2: sequence.}
The decoder requires an integer sequence with adjacent differences
at most one. Rounded values already have this property along each grid
row; between rows, we insert enough values to keep the same bound.
The resulting sequence is $y[0],\ldots,y[2^{ds+2}-1]\in[-2^m,2^m]$. For $\bm\ell\in\{0,\ldots,2^s-1\}^d$, its row number is
$i=\sum_{j=1}^{d-1}2^{s(d-1-j)}\ell_j$; when $d=1$, this sum is zero.
Define
\begin{equation}\label{eq:unit-slow-order}
y[2^{s+2}i+\ell_d]
=\left\lfloor2^m f(2^{-s}\bm\ell)+\frac12\right\rfloor.
\end{equation}
Each row has $2^s$ grid values followed by $3\cdot2^s$ positions.
Between consecutive rows, fill these positions by
\begingroup\normalsize
\begin{equation}\label{eq:unit-slow-interpolation}
y[2^{s+2}i+2^s-1+r]
=\left\lfloor
\frac{(3\cdot2^s+1-r)y[2^{s+2}i+2^s-1]
+r\,y[2^{s+2}(i+1)]}{3\cdot2^s+1}+\frac12
\right\rfloor,
\end{equation}
\endgroup
for $i=0,\ldots,2^{s(d-1)}-2$ and $r=1,\ldots,3\cdot2^s$.
After the last row, repeat its last value. All entries are integers
in $[-2^m,2^m]$, since both endpoints of each interpolation lie in
that integer interval. Within a row, the H\"older condition gives
$2^m|f(2^{-s}(\bm\ell+\bm e_d))-f(2^{-s}\bm\ell)|
\le2^{m-\beta s}\le1$.
Across a bridge, successive unrounded values differ by
$|y[2^{s+2}(i+1)]-y[2^{s+2}i+2^s-1]|/(3\cdot2^s+1)
\le2^{m+1}/(3\cdot2^s+1)<1$.
Rounding numbers at distance at most one preserves this bound, and
\eqref{eq:unit-slow-interpolation} gives the existing endpoints also at
$r=0$ and $r=3\cdot2^s+1$. Thus the joins satisfy the same bound, and
$|y[j]-y[j-1]|\le1$ for $j=1,\ldots,2^{ds+2}-1$.
At every grid point, \eqref{eq:unit-slow-order} gives
$|2^{-m}y[2^{s+2}i+\ell_d]-f(2^{-s}\bm\ell)|\le2^{-m-1}$.

\par\medskip\noindent
\textbf{Step 3: affine map.}
The sequence has length $2^{ds+2}$, and Step 1 ensures
$8192\cdot2^{ds+2}\le N^2\min\{B,\log(eN)\}$ and
$2\le m\le\lfloor N/16\rfloor$. Proposition~\ref{prop:unit-decoder}
therefore gives
$\mathcal D\in\NN_\DTA(ds+2,1;N,8,\mathbb Q_{B-2}^0\cap[-1,1])$
with
\begin{equation}\label{eq:unit-main-decoder-error}
\mathcal D(\mathtt{bin}_{ds+2}(j))=2^{-m}y[j],
\qquad j=0,\ldots,2^{ds+2}-1.
\end{equation}
To send each grid index to its position in this sequence, insert two
zero digits before its final coordinate group. For
$\bm z_1,\ldots,\bm z_d\in\R^s$, define
\begin{equation*}
\mathcal A:\R^{ds}\longrightarrow\R^{ds+2},
\qquad
\mathcal A
\begin{bmatrix}\bm z_1\\[-1mm]\vdots\\[-1mm]\bm z_d\end{bmatrix}
=\begin{bmatrix}\bm z_1\\[-1mm]\vdots\\[-1mm]
\bm z_{d-1}\\0\\0\\\bm z_d\end{bmatrix}.
\end{equation*}
For $d=1$, this means $\mathcal A(\bm z_1)=(0,0,\bm z_1^\top)^\top$.
Let $\phi=\mathcal D\circ\mathcal A\circ\mathcal E$ and
$g=\clip\circ\phi$.
For $\bm x\in[0,1]^d\setminus\Omega$, write
$\bm\ell=(\lfloor2^sx_1\rfloor,\ldots,\lfloor2^sx_d\rfloor)$ and use
its row number $i$ from Step 2. Equations~\eqref{eq:unit-binary-digits}
and \eqref{eq:unit-parallel-encoder} give the complete input to the
decoder:
\begingroup\fontsize{10.95}{13.2}\selectfont
\begin{align*}
(\mathcal A\circ\mathcal E)(\bm x)
&=\mathcal A
\begin{bmatrix}\mathtt{bin}_s(\ell_1)\\[-1mm]\vdots\\[-1mm]
\mathtt{bin}_s(\ell_d)\end{bmatrix}
=\begin{bmatrix}
\mathtt{bin}_s(\ell_1)\\[-1mm]\vdots\\[-1mm]
\mathtt{bin}_s(\ell_{d-1})\\0\\0\\\mathtt{bin}_s(\ell_d)
\end{bmatrix}\\
&=\mathtt{bin}_{ds+2}\!\left(
\sum_{j=1}^{d-1}\sum_{r=1}^s
2^{(d-j)s+2+s-r}[\mathtt{bin}_s(\ell_j)]_r
+\sum_{r=1}^s2^{s-r}[\mathtt{bin}_s(\ell_d)]_r
\right)\\
&=\mathtt{bin}_{ds+2}\!\left(
2^{s+2}\sum_{j=1}^{d-1}2^{s(d-1-j)}
\sum_{r=1}^s2^{s-r}[\mathtt{bin}_s(\ell_j)]_r
+\sum_{r=1}^s2^{s-r}[\mathtt{bin}_s(\ell_d)]_r
\right)\\
&=\mathtt{bin}_{ds+2}\!\left(
2^{s+2}\sum_{j=1}^{d-1}2^{s(d-1-j)}\ell_j+\ell_d\right)
=\mathtt{bin}_{ds+2}(2^{s+2}i+\ell_d).
\end{align*}
\endgroup
This is exactly the index of the rounded value assigned to the grid
cell in Step 2. Since $2^{-m}y[j]\in[-1,1]$, clipping preserves the
decoder output, and \eqref{eq:unit-main-decoder-error} gives
\begin{align*}
g(\bm x)
&=(\clip\circ\mathcal D\circ\mathcal A\circ\mathcal E)(\bm x)
=(\clip\circ\mathcal D)
\bigl(\mathtt{bin}_{ds+2}(2^{s+2}i+\ell_d)\bigr)\\
&=\clip\bigl(2^{-m}y[2^{s+2}i+\ell_d]\bigr)
=2^{-m}y[2^{s+2}i+\ell_d]
=2^{-m}\left\lfloor2^m f(2^{-s}\bm\ell)+\frac12\right\rfloor.
\end{align*}

\Needspace{3\baselineskip}
\par\medskip\noindent
\textbf{Step 4: network size and parameters.}
\begin{itemize}
\item Each output coordinate of the encoder is one minus a hidden
coordinate. The first decoder map copies input digits and supplies
constant coordinates. Combining these maps with $\mathcal A$ gives
coordinates of the form one minus a hidden coordinate, zero, or one.
Every coefficient and bias at this junction therefore belongs to
$\{-1,0,1\}$. The other affine maps retain their parameters from
Propositions~\ref{prop:unit-encoder} and \ref{prop:unit-decoder}.
Thus every parameter of $\phi$ belongs to
$\mathbb Q_{B-2}^0\cap[-1,1]$, and its hidden depth is at most
$(5+q)+8=13+q=13+\lceil3p\beta\rceil$.
\item The decoder has width at most $N$. For the parallel encoder,
$d,q\le t/2$ and $s\le3t/d$ give $s+q+1\le4t$.
Using \eqref{eq:unit-main-scale} and
$\min\{B,\log(eN)\}\le2t$, we obtain
\begin{align*}
\operatorname{width}(\mathcal E)
&\le16d(s+q+1)^2 2^{s/3}
\le16\cdot\frac t2\cdot(4t)^2
\left[\frac{N^2\min\{B,\log(eN)\}}{32768}\right]^{1/(3d)}\\
&\le128t^3(2t)^{1/3}N^{2/3}
\le192t^{10/3}N^{2/3}
\le N,
\end{align*}
where the last inequality holds for $t\ge128$.
Hence every hidden layer of $\phi$ has width at most $N$.
\end{itemize}
We have proved
\begin{equation*}
g\in\mathcal F_\DTA
\bigl(N,13+\lceil3p\beta\rceil,\mathbb Q_{B-2}^0\cap[-1,1]\bigr).
\end{equation*}

\par\medskip\noindent
\textbf{Step 5: approximation.}
For each $\bm x\notin\Omega$, use its grid index $\bm\ell$ and row
number $i$ from Step 3. The rounding error from Step 2 is at most
$2^{-m-1}\le2^{-\beta s}$, and the H\"older condition bounds the error within its
grid cell. On $\Omega$, both $f$ and $g$ take values in $[-1,1]$.
Thus, for every $p\in(0,\infty)$,
\begingroup\medmuskip=2mu\thinmuskip=2mu\relax
\begin{align*}
\|f-g\|_{L^p([0,1]^d)}
&=\left[
\int_{[0,1]^d\setminus\Omega}
\left|f(\bm x)-2^{-m}y[2^{s+2}i+\ell_d]\right|^p\,d\bm x
+\int_\Omega|f-g|^p\,d\bm x
\right]^{1/p}\\
&\le\left[
\int_{[0,1]^d\setminus\Omega}
\left(|f(\bm x)-f(2^{-s}\bm\ell)|+2^{-m-1}\right)^p\,d\bm x
+2^p \mathbf{m}(\Omega)
\right]^{1/p}\\
&\le\left[
\left(2^{-\beta s}+2^{-m-1}\right)^p
+2^{p-1}d\,2^{-qs/3}
\right]^{1/p}
\le\left[2^p2^{-p\beta s}+2^{p-1}d\,2^{-p\beta s}\right]^{1/p}\\
&=2(1+d/2)^{1/p}\left(2^{ds}\right)^{-\beta/d}
<2(1+d/2)^{1/p}
\left[\frac{N^2\min\{B,\log(eN)\}}{2^{d+15}}\right]^{-\beta/d}\\
&=2(1+d/2)^{1/p}2^{\beta(d+15)/d}
\bigl[N^2\min\{B,\log(eN)\}\bigr]^{-\beta/d}.
\end{align*}
\endgroup
For $N<N_0$, the zero network belongs to the same class and has
error at most one. Consequently,
\begin{equation*}
C_{\mathrm{ur}}=
\max\left\{
2(1+d/2)^{1/p}2^{\beta(d+15)/d},\,
[N_0^2\log(eN_0)]^{\beta/d}
\right\}
\end{equation*}
covers every $N\in\mathbb N$ and $B\ge2$. Since the construction
applies to every $f\in\mathcal H^\beta$, we conclude that
\[
\sup_{f\in\mathcal H^\beta}
\inf_{g\in\mathcal F_\DTA
(N,13+\lceil3p\beta\rceil,\mathbb Q_{B-2}^0\cap[-1,1])}
\|f-g\|_{L^p}
\le C_{\mathrm{ur}}
\bigl[N^2\min\{B,\log(eN)\}\bigr]^{-\beta/d}.
\qedhere
\]
\end{proof}

\endgroup

\section{Proof of Theorem~\ExternalNumber{2.6}{thm:transformer}}\label{app:transformer}

The input block reuses the block-grid encoder. One attention layer combines
its token outputs, and the output block applies coupled decoder part~(a)
to each coordinate.

\begin{proof}[Proof of Theorem~\ExternalNumber{2.6}{thm:transformer}]
The attention layer will give every token the same grid index and position
within its block.  The output network then applies the five affine maps
of Proposition~\ExternalNumber{3.4}{prop:block-decoder}(a), with a different table offset
for each output coordinate.

\par\medskip\noindent
\textbf{Step 1: the table.}
Let $D=dn$, and choose
\begin{equation*}
\Upsilon=\left\lceil\frac{77(2D)^{1/p}}{50\varepsilon}\right\rceil,
\qquad
H=2^{\left\lceil\log_2((\Upsilon+1)^{1/\beta})\right\rceil},
\qquad M=H^D,
\qquad Q=DM.
\end{equation*}
Thus $\Upsilon\ge4$, $H\ge8$, and $H$ is a power of two with
$\Upsilon<H$.  The choice of $H$ gives $H^{-\beta}\le(\Upsilon+1)^{-1}$
and $H<2(\Upsilon+1)^{1/\beta}$; the choice of $\Upsilon$ gives
$\varepsilon^{-1}\le\Upsilon$ and
$\Upsilon+1\le[1+77(2D)^{1/p}/50]\varepsilon^{-1}$.
For $\bm\ell\in\{0,\ldots,H-1\}^{d\times n}$, define its grid centre and
integer index by
\[
\bm X_{\bm\ell}=H^{-1}(\bm\ell+\bmone_{d\times n}/2),
\qquad
 t(\bm\ell)=\sum_{s=1}^n\sum_{r=1}^d
 H^{(s-1)d+r-1}\ell_{rs}.
\]
The $M$ grid values for output $(r,s)$ occupy positions
$(q(r,s)-1)M,\ldots,q(r,s)M-1$, where $q(r,s)=(s-1)d+r$.  Define
\begin{equation}\label{eq:transformer-quantized-table}
a_{t(\bm\ell)+(q(r,s)-1)M}
=\left\lfloor\Upsilon f_{rs}(\bm X_{\bm\ell})+\frac12\right\rfloor,
\qquad r=1,\ldots,d,\quad s=1,\ldots,n.
\end{equation}
These $Q$ integers lie in $[-\Upsilon,\Upsilon]$.  Since $H\mid M$, an
index that is not congruent to $H-1$ modulo $H$ advances only the first
grid coordinate $\ell_{11}$.  At each such index, rounding and H\"older
continuity give
\[
|a_{j+1}-a_j|
\le \frac12+\Upsilon H^{-\beta}+\frac12
\le 1+\frac{\Upsilon}{\Upsilon+1}<2,
\qquad j=0,\ldots,Q-2,\quad j\not\equiv H-1\pmod H.
\]
The difference is an integer, so it is at most one in absolute value.
Every boundary between two output tables has index $qM-1\equiv H-1\pmod H$
and is excluded from this condition.  Proposition~\ExternalNumber{3.4}{prop:block-decoder}(a)
therefore applies.  Let $\Gamma$, $\mathcal D_2$, and
$\mathcal A_{2,1},\ldots,\mathcal A_{2,5}$ be the divisor, decoder, and
five affine maps obtained there.  In particular,
$1\le\Gamma\le\Upsilon/4$, $\Gamma\mid H$, and
$\log T_{\mathcal D_2}\le12Q$.
Choose $\delta=\varepsilon^p/[2D^2H(2+\Gamma/\Upsilon)^p]$.
Then $\delta\in(0,1/H)$.  Let $\Omega$ be the set of matrices for which
at least one coordinate belongs to
$\bigcup_{j=1}^H(j/H-\delta,j/H]$; its Lebesgue measure satisfies
$|\Omega|\le DH\delta$.

\par\medskip\noindent
\textbf{Step 2: the input maps.}
We use the two hidden layers of the block-grid encoder in
Proposition~\ExternalNumber{3.3}{prop:block-grid-encoder}, and add one hidden
layer to retain the position supplied by the first token.
Let $\mathcal E=(\mathcal E_1,\mathcal E_2)$ and
$\mathcal A_1^{\mathrm{bl}},\mathcal A_2^{\mathrm{bl}},
\mathcal A_3^{\mathrm{bl}}$ be exactly the network and affine maps
constructed in that proof. Write $\bm X=(\bm x_1,\ldots,\bm x_n)$,
$\bm\omega=(1,H,\ldots,H^{d-1})^\top$, and
$\iota_s=\one_{\{s=1\}}$. Adding the two output formulas in that proof gives
\begin{equation*}
\Gamma\left(\mathcal E_1(\bm x)+\mathcal E_2(\bm x)
 -\frac1{2\Gamma}\right)
=\sum_{r=1}^dH^{r-1}S_{H,\delta}(x_r)
\in[0,H\bmone_d^\top\bm\omega],
\qquad \bm x\in[0,1]^d.
\end{equation*}
We will obtain the input-block output
\begin{equation}\label{eq:transformer-block-local-output}
\bm Z=
\begin{bmatrix}
0&d\log H&\cdots&d(n-1)\log H\\
C_1&C_2&\cdots&C_n\\
\nu&0&\cdots&0\\
1&1&\cdots&1\\
0&0&\cdots&0\\
0&0&\cdots&0
\end{bmatrix},
\end{equation}
where
$C_s=\mathcal E_1(\bm x_s)+(1-\iota_s)
[\mathcal E_2(\bm x_s)-(2\Gamma)^{-1}]$
and $\nu=\mathcal E_2(\bm x_1)$.
Outside $\Omega$, token one supplies its block index and position; every other token
supplies its full grid index divided by $\Gamma$.
The encoder bounds give $C_s\ge0$ and
$\nu\in[(2\Gamma)^{-1},1-(2\Gamma)^{-1}]$ on the whole cube.

Define $\mathcal A_1^{\mathrm{in}}:\R^{d\times n}\to\R^{(2d+3)\times n}$
and $\mathcal A_2^{\mathrm{in}}:\R^{(2d+3)\times n}\to\R^{(d+2)\times n}$
by applying $\mathcal A_j^{\mathrm{bl}}$ to each column for $j=1,2$:
$[\mathcal A_j^{\mathrm{in}}(\bm Y)]_{\cdot,s}
=\mathcal A_j^{\mathrm{bl}}(\bm Y_{\cdot,s})$.
The remaining affine maps
$\mathcal A_3^{\mathrm{in}}:\R^{(d+2)\times n}\to\R^{6\times n}$
and $\mathcal A_4^{\mathrm{in}}:\R^{6\times n}\to\R^{6\times n}$
are defined columnwise by
\begingroup
\TransformerMatrixLayout{}
\begin{equation*}
\left\{\begin{aligned}
&[\mathcal A_3^{\mathrm{in}}(\bm Y)]_{\cdot,s}
=\begin{bmatrix}
\dfrac{\Gamma}{H\bmone_d^\top\bm\omega}&
\dfrac{\Gamma}{H\bmone_d^\top\bm\omega}\\[2mm]
0&1\\0&1\\\multicolumn{2}{c}{\bzero_{3\times2}}
\end{bmatrix}
\mathcal A_3^{\mathrm{bl}}(\bm Y_{\cdot,s})
+\begin{bmatrix}
-(2H\bmone_d^\top\bm\omega)^{-1}\\
-(2\Gamma)^{-1}\\1-\iota_s-(2\Gamma)^{-1}\\
\iota_s\\(s-1)/n\\0
\end{bmatrix},\\[3mm]
&[\mathcal A_4^{\mathrm{in}}(\bm Y)]_{\cdot,s}
=\begin{bmatrix}
0&0&0&0&D\log H&0\\[1mm]
(H\bmone_d^\top\bm\omega)/\Gamma&-1/2&-1/2&-1/2&0&0\\[1mm]
0&1/2&1/2&(\Gamma+1)/(2\Gamma)&0&0\\
\multicolumn{6}{c}{\bzero_{3\times6}}
\end{bmatrix}\bm Y_{\cdot,s}
+\begin{bmatrix}0\\1/2\\-1/2\\1\\0\\0\end{bmatrix}.
\end{aligned}\right.
\end{equation*}
\endgroup
The occurrence of $\mathcal A_3^{\mathrm{bl}}$ in
$\mathcal A_3^{\mathrm{in}}$ is an affine composition, so the input block
\[
\bm{\mathcal F}_{\mathrm{FF}}^{\mathrm{in}}
=\mathcal A_4^{\mathrm{in}}\circ\DTA\circ\mathcal A_3^{\mathrm{in}}
 \circ\DTA\circ\mathcal A_2^{\mathrm{in}}
 \circ\DTA\circ\mathcal A_1^{\mathrm{in}}
\]
has hidden widths $\langle2d+3,d+2,6\rangle$.
The weights are shared across tokens, and only the biases depend on $s$.

We now check the third hidden layer and the output map.
The first two layers followed by $\mathcal A_3^{\mathrm{bl}}$ give
$\mathcal E(\bm x_s)$, exactly as proved in
Proposition~\ExternalNumber{3.3}{prop:block-grid-encoder}.
The first two entries below lie in $[0,1]$ and $[0,1-1/\Gamma]$,
respectively. Hence the third hidden output is
\begingroup
\fontsize{10.95}{13.2}\selectfont
\begin{align*}
[\bm F_3^{\mathrm{in}}]_{\cdot,s}
&:=\bigl[(\DTA\circ\mathcal A_3^{\mathrm{in}}
 \circ\DTA\circ\mathcal A_2^{\mathrm{in}}
 \circ\DTA\circ\mathcal A_1^{\mathrm{in}})(\bm X)\bigr]_{\cdot,s}\\
&=\DTA\begin{bmatrix}
\dfrac{\Gamma[\mathcal E_1(\bm x_s)+\mathcal E_2(\bm x_s)-(2\Gamma)^{-1}]}
 {H\bmone_d^\top\bm\omega}\\[2mm]
\mathcal E_2(\bm x_s)-(2\Gamma)^{-1}\\
\mathcal E_2(\bm x_s)-(2\Gamma)^{-1}+1-\iota_s\\
\iota_s\\(s-1)/n\\0
\end{bmatrix}\\[1mm]
&=\begin{bmatrix}
\dfrac{\Gamma[\mathcal E_1(\bm x_s)+\mathcal E_2(\bm x_s)-(2\Gamma)^{-1}]}
 {H\bmone_d^\top\bm\omega}\\[2mm]
\mathcal E_2(\bm x_s)-(2\Gamma)^{-1}\\
(2\iota_s-1)[\mathcal E_2(\bm x_s)-(2\Gamma)^{-1}]+1-\iota_s\\
\iota_s\\(s-1)/n\\0
\end{bmatrix}.
\end{align*}
\endgroup
The last equality uses $\DTA(t)=t$ and $\DTA(1+t)=1-t$ for $t\in[0,1]$.
The final affine map then gives
\begingroup
\fontsize{10.95}{13.2}\selectfont
\begin{align*}
&\begin{aligned}[t]
[\mathcal A_4^{\mathrm{in}}(\bm F_3^{\mathrm{in}})]_{2,s}
&=\mathcal E_1(\bm x_s)+\mathcal E_2(\bm x_s)-\frac1{2\Gamma}
 -\frac12\left\{2\iota_s
 \left[\mathcal E_2(\bm x_s)-\frac1{2\Gamma}\right]+1-\iota_s\right\}
 -\frac{\iota_s}2+\frac12\\*
&=\mathcal E_1(\bm x_s)+\mathcal E_2(\bm x_s)-\frac1{2\Gamma}
 -\iota_s\left[\mathcal E_2(\bm x_s)-\frac1{2\Gamma}\right]
 -\frac{1-\iota_s}2-\frac{\iota_s}2+\frac12\\*
&=\mathcal E_1(\bm x_s)+(1-\iota_s)
 \left[\mathcal E_2(\bm x_s)-\frac1{2\Gamma}\right]=C_s,
\end{aligned}\\[2mm]
&\begin{aligned}[t]
[\mathcal A_4^{\mathrm{in}}(\bm F_3^{\mathrm{in}})]_{3,s}
&=\frac12\left\{2\iota_s
 \left[\mathcal E_2(\bm x_s)-\frac1{2\Gamma}\right]+1-\iota_s\right\}
 +\frac{\Gamma+1}{2\Gamma}\iota_s-\frac12\\*
&=\iota_s\left[\mathcal E_2(\bm x_s)-\frac1{2\Gamma}\right]
 +\frac{1-\iota_s}2+\frac{\iota_s}2+\frac{\iota_s}{2\Gamma}-\frac12
 =\iota_s\mathcal E_2(\bm x_s)=\iota_s\nu.
\end{aligned}
\end{align*}
\endgroup
The first row is $D\log H\,(s-1)/n=d(s-1)\log H$, and the remaining
three rows are $1,0,0$. Thus
$\bm{\mathcal F}_{\mathrm{FF}}^{\mathrm{in}}(\bm X)=\bm Z$,
as required in \eqref{eq:transformer-block-local-output}.

\par\medskip\noindent
\textbf{Step 3: attention.}
The attention layer must sum the token contributions with weights
$1,H^d,\ldots,H^{d(n-1)}$, and send both this sum and $\nu$ to every token.
Use one head of size two with
\[
\bm W_K=\begin{bmatrix}1&0&0&0&0&0\\0&0&0&0&0&0\end{bmatrix},
\qquad
\bm W_Q=\begin{bmatrix}0&0&0&1&0&0\\0&0&0&0&0&0\end{bmatrix},
\]
\[
\bm W_V=\begin{bmatrix}0&1&0&0&0&0\\0&0&1&0&0&0\end{bmatrix},
\qquad
\bm W_O=\left(\sum_{t=1}^nH^{d(t-1)}\right)
\begin{bmatrix}\bzero_{4\times2}\\\bm I_2\end{bmatrix}.
\]
We first check that these keys and queries assign weights proportional
to $1,H^d,\ldots,H^{d(n-1)}$ in every column. By
\eqref{eq:transformer-block-local-output},
\begin{align*}
\bm\sigma_{\mathrm S}\bigl((\bm W_K\bm Z)^\top(\bm W_Q\bm Z)\bigr)
&=\bm\sigma_{\mathrm S}\left(
\begin{bmatrix}
0&0\\d\log H&0\\\vdots&\vdots\\d(n-1)\log H&0
\end{bmatrix}
\begin{bmatrix}\bmone_n^\top\\\bzero_n^\top\end{bmatrix}\right)
=\bm\sigma_{\mathrm S}\left(
\begin{bmatrix}0\\d\log H\\\vdots\\d(n-1)\log H\end{bmatrix}
\bmone_n^\top\right)\\[1mm]
&=\frac{(1,e^{d\log H},\ldots,e^{d(n-1)\log H})^\top}
 {\sum_{t=1}^ne^{d(t-1)\log H}}\bmone_n^\top
=\frac{(1,H^d,\ldots,H^{d(n-1)})^\top}
 {\sum_{t=1}^nH^{d(t-1)}}\bmone_n^\top.
\end{align*}
Thus the output matrix cancels the softmax denominator, and the residual
connection keeps the first four rows of $\bm Z$:
\begin{align*}
\bm Y:=\bm{\mathcal F}_{\mathrm{SA}}(\bm Z)
&=\bm Z+\bm W_O\bm W_V\bm Z\,
 \bm\sigma_{\mathrm S}\bigl((\bm W_K\bm Z)^\top(\bm W_Q\bm Z)\bigr)\\
&=\bm Z+
 \left(\sum_{t=1}^nH^{d(t-1)}\right)
 \begin{bmatrix}\bzero_{4\times2}\\\bm I_2\end{bmatrix}
 \begin{bmatrix}C_1&C_2&\cdots&C_n\\\nu&0&\cdots&0\end{bmatrix}
 \frac{(1,H^d,\ldots,H^{d(n-1)})^\top}
 {\sum_{t=1}^nH^{d(t-1)}}\bmone_n^\top\\
&=\bm Z+
 \begin{bmatrix}\bzero_4\\\sum_{s=1}^nH^{d(s-1)}C_s\\\nu\end{bmatrix}
 \bmone_n^\top\\[1mm]
&=\bm Z+\begin{bmatrix}\bzero_4\\b_H(\bm X)\\\nu\end{bmatrix}\bmone_n^\top,
\end{align*}
where $b_H(\bm X)=\sum_{s=1}^nH^{d(s-1)}C_s\ge0$.
For $\bm X\notin\Omega$, set $\ell_{rs}=\lfloor Hx_{rs}\rfloor$.
The exact outputs in Proposition~\ExternalNumber{3.3}{prop:block-grid-encoder}
and $\Gamma\mid H$ identify the two rows received by every token:
\begin{align*}
&\begin{aligned}[t]
b_H(\bm X)
&=\left\lfloor\frac{\ell_{11}}\Gamma\right\rfloor
 +\Gamma^{-1}\sum_{r=2}^dH^{r-1}\ell_{r1}
 +\Gamma^{-1}\sum_{s=2}^n\sum_{r=1}^dH^{(s-1)d+r-1}\ell_{rs}\\*
&=\left\lfloor\frac{\ell_{11}}\Gamma\right\rfloor
 +\frac{t(\bm\ell)-\ell_{11}}\Gamma
 =\left\lfloor\frac{t(\bm\ell)}\Gamma\right\rfloor,
\end{aligned}\\[2pt]
&\nu
=\frac{\ell_{11}-\Gamma\lfloor\ell_{11}/\Gamma\rfloor+1/2}\Gamma
 =\frac{t(\bm\ell)-\Gamma\lfloor t(\bm\ell)/\Gamma\rfloor+1/2}\Gamma.
\end{align*}
The two final equalities use that
$[t(\bm\ell)-\ell_{11}]/\Gamma$ is an integer.  In particular,
$0\le b_H(\bm X)<M/\Gamma$, and $\Gamma\nu-1/2\in\{0,\ldots,\Gamma-1\}$.

\par\medskip\noindent
\textbf{Step 4: the output maps.}
We now arrange $d$ copies of coupled decoder part~(a), so that coordinate
$(r,s)$ reads the table of $f_{rs}$. Let $g$ and $m$ be the integers in its
five affine maps \eqref{eq:scalar-decoder-affine-maps}.  Coordinates $r$ and $d+r$ carry the two coordinates of copy $r$.
The first affine map has domain $\R^{6\times n}$ and codomain
$\R^{2d\times n}$, the next three act on $\R^{2d\times n}$, and the last
maps $\R^{2d\times n}$ to $\R^{d\times n}$.
Define $\mathcal A_j^{\mathrm{out}}(\bm Y)
=\bm W_j^{\mathrm{out}}\bm Y+\bm B_j^{\mathrm{out}}$ by
\begingroup
\TransformerMatrixLayout{}
\[
\bm W_1^{\mathrm{out}}=
\begin{bmatrix}\bzero_{d\times4}&-g\bmone_d&\bzero_d\\
\bzero_{d\times5}&\bmone_d\end{bmatrix},
\qquad
\bm B_1^{\mathrm{out}}=
\begin{bmatrix}
\bigl(-g[1+(q(r,s)-1)M/\Gamma]\bigr)_{r,s}\\
\bzero_{d\times n}
\end{bmatrix},
\]
\[
\bm W_2^{\mathrm{out}}=
\begin{bmatrix}8\,3^{\Gamma-1}m\bm I_d&\bzero\\
2m\bm I_d&(2\,3^{\Gamma-1})^{-1}\bm I_d\end{bmatrix},
\qquad
\bm B_2^{\mathrm{out}}=
\begin{bmatrix}(24\,3^{\Gamma-1}m+2)\bmone_d\bmone_n^\top\\
(6m+2)\bmone_d\bmone_n^\top\end{bmatrix},
\]
\[
\bm W_3^{\mathrm{out}}=
\begin{bmatrix}(g\Gamma/2)\bm I_d&-2g\Gamma\,3^{\Gamma-1}\bm I_d\\
\bm I_d&\bzero\end{bmatrix},
\qquad
\bm B_3^{\mathrm{out}}=
\begin{bmatrix}-g(Q/\Gamma+1/2)\bmone_d\bmone_n^\top\\
\bzero_{d\times n}\end{bmatrix},
\]
\begin{align*}
&\bm W_4^{\mathrm{out}}=
 \begin{bmatrix}2m\bm I_d&\bzero\\\bzero&\bm I_d\end{bmatrix},
\qquad
 \bm B_4^{\mathrm{out}}=
 \begin{bmatrix}(6m+2)\bmone_d\bmone_n^\top\\\bzero_{d\times n}\end{bmatrix},\\[3pt]
&\bm W_5^{\mathrm{out}}=
 \begin{bmatrix}(2\Gamma/\Upsilon)\bm I_d&2\bm I_d\end{bmatrix},
\qquad
 \bm B_5^{\mathrm{out}}=-(1+\Gamma/\Upsilon)\bmone_d\bmone_n^\top.
\end{align*}
\endgroup
Let $\bm G_0=\bm Y$ and
$\bm G_j=(\DTA\circ\mathcal A_j^{\mathrm{out}})(\bm G_{j-1})$
for $j=1,\ldots,4$.  We check the composition against the scalar decoder,
starting with its first map:
\begin{align*}
\binom{[\bm G_1]_{r,s}}{[\bm G_1]_{d+r,s}}
&=\DTA\binom{-g[\bm Y]_{5,s}-g[1+(q(r,s)-1)M/\Gamma]}
 {[\bm Y]_{6,s}}\\
&=\DTA\binom{-g[b_H(\bm X)+(q(r,s)-1)M/\Gamma+1]}{\nu}\\
&=(\DTA\circ\mathcal A_{2,1})
 \left(b_H(\bm X)+\frac{(q(r,s)-1)M}{\Gamma},\nu\right).
\end{align*}
For $j=2,3,4$, the two rows $r,d+r$ of the affine maps above give exactly
$\mathcal A_{2,j}$, with no other output coordinate entering either row.
Consequently, for each of these three layers,
\begin{align*}
\binom{[\bm G_j]_{r,s}}{[\bm G_j]_{d+r,s}}
&=\binom{[(\DTA\circ\mathcal A_j^{\mathrm{out}})
 (\bm G_{j-1})]_{r,s}}
 {[(\DTA\circ\mathcal A_j^{\mathrm{out}})
 (\bm G_{j-1})]_{d+r,s}}\\
&=(\DTA\circ\mathcal A_{2,j})
 \binom{[\bm G_{j-1}]_{r,s}}{[\bm G_{j-1}]_{d+r,s}}\\
&=(\DTA\circ\mathcal A_{2,j}\circ\cdots
 \circ\DTA\circ\mathcal A_{2,1})
 \left(b_H(\bm X)+\frac{(q(r,s)-1)M}{\Gamma},\nu\right).
\end{align*}
Thus the output $\bm{\mathcal T}_\varepsilon(\bm X)
=\mathcal A_5^{\mathrm{out}}(\bm G_4)$ satisfies
\begin{align}
[\bm{\mathcal T}_\varepsilon(\bm X)]_{r,s}
&=\frac{2\Gamma}{\Upsilon}[\bm G_4]_{r,s}
 +2[\bm G_4]_{d+r,s}-\left(1+\frac\Gamma\Upsilon\right)\notag\\
&=\mathcal A_{2,5}
 \binom{[\bm G_4]_{r,s}}{[\bm G_4]_{d+r,s}}\notag\\
&=(\mathcal A_{2,5}\circ\DTA\circ\mathcal A_{2,4}
 \circ\cdots\circ\DTA\circ\mathcal A_{2,1})
 \left(b_H(\bm X)+\frac{(q(r,s)-1)M}{\Gamma},\nu\right)\notag\\
&=\mathcal D_2\left(b_H(\bm X)+\frac{(q(r,s)-1)M}{\Gamma},\nu\right).
\label{eq:transformer-block-output}
\end{align}
Each output coordinate is therefore evaluated by the decoder of
Proposition~\ExternalNumber{3.4}{prop:block-decoder}(a).  The output network has hidden
widths $\langle2d,2d,2d,2d\rangle$ and output dimension $d$.
The table offset enters only the first bias, so all weight matrices
remain shared across tokens, as required.

\par\medskip\noindent
\textbf{Step 5: approximation.}
We apply the uniform and approximation bounds of part~(a) to the two
coordinates obtained above. Since $b_H(\bm X)\ge0$ and $\nu\in[0,1]$ on
the whole cube, the uniform
bound in Proposition~\ExternalNumber{3.4}{prop:block-decoder}(a) applies directly to
\eqref{eq:transformer-block-output}:
\[
\sup_{\bm X\in[0,1]^{d\times n}}
\|\bm{\mathcal T}_\varepsilon(\bm X)\|_{\max}
\le \sup_{\substack{b\ge-1\\\nu\in[0,1]}}|\mathcal D_2(b,\nu)|
\le1+\frac\Gamma\Upsilon\le\frac54.
\]
For $\bm X\notin\Omega$, the input to the same decoder has an integer
block index in $\{0,\ldots,Q/\Gamma-1\}$, and its table index is exactly
\begin{align*}
\Gamma\left(b_H(\bm X)+\frac{(q(r,s)-1)M}{\Gamma}\right)
 +\left(\Gamma\nu-\frac12\right)
&=\Gamma\left\lfloor\frac{t(\bm\ell)}\Gamma\right\rfloor
 +(q(r,s)-1)M
 +t(\bm\ell)-\Gamma\left\lfloor\frac{t(\bm\ell)}\Gamma\right\rfloor\\
&=t(\bm\ell)+(q(r,s)-1)M.
\end{align*}
Using the approximation estimate in part~(a), the rounding in
\eqref{eq:transformer-quantized-table}, and H\"older continuity, we obtain
\begin{align*}
|[\bm{\mathcal T}_\varepsilon(\bm X)]_{r,s}-f_{rs}(\bm X)|
&\le\left|\mathcal D_2\left(b_H(\bm X)
 +\frac{(q(r,s)-1)M}{\Gamma},\nu\right)
 -\frac{a_{t(\bm\ell)+(q(r,s)-1)M}}\Upsilon\right|\\
&\quad+\left|\frac{a_{t(\bm\ell)+(q(r,s)-1)M}}\Upsilon
 -f_{rs}(\bm X_{\bm\ell})\right|
 +|f_{rs}(\bm X_{\bm\ell})-f_{rs}(\bm X)|\\
&<\frac1{25\Upsilon}+\frac1{2\Upsilon}+H^{-\beta}
\le\frac1{25\Upsilon}+\frac1{2\Upsilon}+\frac1{\Upsilon+1}
<\frac{77}{50\Upsilon}\le\frac\varepsilon{(2D)^{1/p}}.
\end{align*}
On $\Omega$, each coordinate error is at most $2+\Gamma/\Upsilon$.
Therefore, for every $p\in(0,\infty)$,
\begin{align*}
d_p(\bm f,\bm{\mathcal T}_\varepsilon)^p
&=\sum_{r=1}^d\sum_{s=1}^n
 \left(\int_{[0,1]^{d\times n}\setminus\Omega}
 |f_{rs}(\bm X)-[\bm{\mathcal T}_\varepsilon(\bm X)]_{r,s}|^p\,\dd\bm X
 +\int_{\Omega}
 |f_{rs}(\bm X)-[\bm{\mathcal T}_\varepsilon(\bm X)]_{r,s}|^p\,\dd\bm X\right)\\
&\le D\left(\frac\varepsilon{(2D)^{1/p}}\right)^p
 +D\left(2+\frac\Gamma\Upsilon\right)^p|\Omega|
\le\frac{\varepsilon^p}{2}
 +D^2\left(2+\frac\Gamma\Upsilon\right)^pH\delta\\
&=\frac{\varepsilon^p}{2}
 +D^2\left(2+\frac\Gamma\Upsilon\right)^pH
 \frac{\varepsilon^p}{2D^2H(2+\Gamma/\Upsilon)^p}
=\varepsilon^p.
\end{align*}

\par\medskip\noindent
\textbf{Step 6: parameter bounds.}
The output network contains only entries of the scalar decoder maps,
except for the first bias.  That bias also stays within their original
radius: since $M/\Gamma\ge1$, \eqref{eq:scalar-decoder-affine-maps} gives
\begin{align*}
&\max_{r,s}|[\bm B_1^{\mathrm{out}}]_{r,s}|
=g\left(1+\frac{(D-1)M}{\Gamma}\right)
\le\frac{gDM}{\Gamma}
=\frac{gQ}{\Gamma}
<g\left(\frac Q\Gamma+\frac12\right)
\le\parnorm{\mathcal A_{2,3}}\le T_{\mathcal D_2},\\[2pt]
&T_{\mathrm{out}}
\le\max\left\{T_{\mathcal D_2},
 \max_{r,s}|[\bm B_1^{\mathrm{out}}]_{r,s}|\right\}
=T_{\mathcal D_2}\le e^{12Q}.
\end{align*}
For the input block, the first two maps are those of the block-grid
encoder. To bound the third map after the affine composition, use
(\ExternalNumber{3.9}{eq:block-encoder-A3-matrix}) in its definition:
\begingroup
\fontsize{10.95}{13.2}\selectfont
\[
[\mathcal A_3^{\mathrm{in}}(\bm Y)]_{\cdot,s}
=\begin{bmatrix}
(H\bmone_d^\top\bm\omega)^{-1}
 \sum_{r=1}^dH^{r-1}Y_{r,s}+Y_{d+2,s}-H^{-1}\\
\Gamma^{-1}Y_{1,s}-Y_{d+1,s}+1-\Gamma^{-1}\\
\Gamma^{-1}Y_{1,s}-Y_{d+1,s}+2-\Gamma^{-1}-\iota_s\\
\iota_s\\(s-1)/n\\0
\end{bmatrix}.
\]
\endgroup
Thus $\parnorm{\mathcal A_3^{\mathrm{in}}}\le2$. Since
$H\bmone_d^\top\bm\omega=H(H^d-1)/(H-1)\le2H^d$,
the four input maps therefore satisfy
$T_{\mathrm{in}}\le\max\{2H^d,\Gamma/(H\delta),D\log H\}$.
The attention matrices have radius at most
$\sum_{t=1}^nH^{d(t-1)}\le nH^{D-d}$.
Since $\Gamma\le\Upsilon/4\le H$, $2+\Gamma/\Upsilon\le9/4\le H$,
and $\varepsilon^{-1}\le H$, these three bounds combine as
\begin{align*}
T_{\mathrm{Tr},\varepsilon}
&\le\max\left\{2H^d,\frac\Gamma{H\delta},D\log H,
 nH^{D-d},e^{12Q}\right\}\\
&=\max\left\{2H^d,
 2D^2\Gamma(2+\Gamma/\Upsilon)^p\varepsilon^{-p},D\log H,
 nH^{D-d},e^{12Q}\right\}\\
&\le\max\left\{2H^d,2D^2H^{1+2p},D\log H,
 nH^{D-d},e^{12Q}\right\}
\le\max\left\{2D^2H^{D+1+2p},e^{12Q}\right\}.
\end{align*}
Finally, $\log(2D^2)\le2D$, $\log H\le H^D$, and $Q=DH^D$ give
\begin{align*}
\log T_{\mathrm{Tr},\varepsilon}
&\le\max\{\log(2D^2)+(D+1+2p)\log H,12Q\}\\
&\le\max\{2D+(D+1+2p)H^D,12Q\}
\le\max\{(4+2p)Q,12Q\}\le(12+2p)Q\\
&=(12+2p)DH^D
<(12+2p)D2^D(\Upsilon+1)^{D/\beta}\\
&\le\underbrace{(12+2p)D2^D
 \left(1+\frac{77(2D)^{1/p}}{50}\right)^{D/\beta}}_{=:C_{\mathrm{Tr}}}
 \varepsilon^{-D/\beta}
=C_{\mathrm{Tr}}\varepsilon^{-dn/\beta}.
\end{align*}
\end{proof}

\section{Proofs of Propositions~\ExternalNumber{4.2}{prop:holder-class-cover} and~\ExternalNumber{4.3}{prop:oracle}}\label{app:gen}

We first propagate parameter perturbations through the network to obtain
empirical covers.  We then bound the quadratic fluctuation from the random
design and the noise multiplier, and combine them in the least-squares
oracle inequality.

\begingroup
\StatisticsProofLayout{}
\begin{proof}[Proof of Proposition~\ExternalNumber{4.2}{prop:holder-class-cover}]
Fix a hidden-width sequence $\langle n_1,\ldots,n_\ell\rangle$ with
$1\le\ell\le L$ and $1\le n_j\le N$, and set $n_0=d$ and
$n_{\ell+1}=1$.  Write
\[
N_{\bm n}=\max_{1\le j\le\ell}n_j,
\quad
P_{\bm n}=\sum_{j=1}^{\ell+1}n_j(n_{j-1}+1),
\quad
K_{\varphi,d}=(d+1)(1+|\varphi(0)|+H_\varphi).
\]
We first prove a covering estimate for a general parameterized function
family.  We then apply it to a fixed network architecture and finally
take the union over all admissible hidden-width sequences.

\par\medskip\noindent
\emph{A parameter cover for a general function family.}
Let $\mathcal X$ be nonempty,
$\varnothing\ne\Theta\subset[-T,T]^P$, and
$h_{\bm\theta}:\mathcal X\to\mathbb R$ for $\bm\theta\in\Theta$.
Suppose
\[
\sup_{x\in\mathcal X}
|h_{\bm\theta}(x)-h_{\bm\theta'}(x)|
\le G\norm{\bm\theta-\bm\theta'}_\infty^q
\]
whenever $\norm{\bm\theta-\bm\theta'}_\infty\le1$, where
$G\ge1$ and $q\in(0,1]$. For the family
$\mathcal H=\{h_{\bm\theta}:\bm\theta\in\Theta\}$, fix
$x_1,\ldots,x_M\in\mathcal X$ and choose the radius of the parameter cover as
$r=(\eps/G)^{1/q}\le1$.
With this choice of $r$, an $r$-cover of the parameter set gives an
$\eps$-cover of the corresponding functions.

Choose a maximal $r$-separated sequence
$\bm\theta_1,\ldots,\bm\theta_J$ in $\Theta$.  The cubes
$\bm\theta_j+(-r/2,r/2)^P$ are pairwise disjoint and contained in
$[-T-r/2,T+r/2]^P$.  Hence, for every partial selection of size $m$,
\[
m r^P
=
\left|
\bigcup_{j=1}^m
\bigl(\bm\theta_j+(-r/2,r/2)^P\bigr)
\right|
\le
(2T+r)^P,
\qquad
m\le
\left(1+2T(G/\eps)^{1/q}\right)^P.
\]
Maximality gives
$\norm{\bm\theta-\bm\theta_j}_\infty\le r$ for some $j$ and every
$\bm\theta\in\Theta$.  Therefore
\[
\max_{1\le i\le M}
|h_{\bm\theta}(x_i)-h_{\bm\theta_j}(x_i)|
\le
G\norm{\bm\theta-\bm\theta_j}_\infty^q
\le\eps.
\]
The centers belong to $\mathcal H$, hence the cover is proper.  Taking
the supremum over $x_1,\ldots,x_M$ gives
\begin{equation}\label{eq:general-parameterized-cover}
\log\mathcal N(\eps,\mathcal H,M)
\le
P\log\left(1+2T(G/\eps)^{1/q}\right).
\end{equation}

\par\medskip\noindent
\emph{Specialization to the H\"older network class.}
We now verify the parameter H\"older estimate needed above for the fixed
architecture $\langle d,n_1,\ldots,n_\ell,1\rangle$.  Take
$\Theta=[-T,T]^{P_{\bm n}}$.  For
$\bm\theta=(\bm W_j,\bm b_j)_{j=1}^{\ell+1}\in\Theta$, let
$\mathcal A_j^{\bm\theta}(\bm z)=\bm W_j\bm z+\bm b_j$ and define
\[
\phi_{\bm\theta}
=
\mathcal A_{\ell+1}^{\bm\theta}\circ\varphi\circ \mathcal A_\ell^{\bm\theta}
\circ\cdots\circ\varphi\circ \mathcal A_1^{\bm\theta},
\qquad
h_{\bm\theta}
=
\left.(\clip\circ\phi_{\bm\theta})\right|_{[0,1]^d},
\]
and set
\(
\mathcal H
=
\left\{
h_{\bm\theta}:
\bm\theta\in[-T,T]^{P_{\bm n}}
\right\}.
\)
Take $\bm\theta,\bm\theta'\in\Theta$ with
$\norm{\bm\theta-\bm\theta'}_\infty\le1$.  For
$\bm x\in[0,1]^d$, write
\[
\bm h_0=\bm h_0'=\bm x,\qquad
\bm z_j=\bm W_j\bm h_{j-1}+\bm b_j,\qquad
\bm h_j=\varphi(\bm z_j),
\quad j=1,\ldots,\ell,
\]
with primed quantities defined analogously.  All vector norms below are
max norms and all matrix norms are $\ell^\infty\!\to\ell^\infty$ norms.
We have
$\norm{\bm W_j}_{\infty\to\infty}\le n_{j-1}T$ and
$\norm{\bm W_j-\bm W_j'}_{\infty\to\infty}
\le n_{j-1}\norm{\bm\theta-\bm\theta'}_\infty$.

We first control the size of the hidden states.  Whenever
$\norm{\bm u}_\infty\le A$ and $A\ge1$,
\begin{equation}\label{eq:holder-growth-base}
\begin{aligned}
1+\norm{\varphi(\bm u)}_\infty
&\le
1+\norm{\varphi(\bm u)-\varphi(\bzero)}_\infty
+\norm{\varphi(\bzero)}_\infty
\le
1+H_\varphi\norm{\bm u}_\infty^\alpha+|\varphi(0)|
\le
K_{\varphi,d}A^\alpha.
\end{aligned}
\end{equation}
We control parameter perturbations by proving the following size and
difference estimates together, for $j=1,\ldots,\ell$:
\begin{align*}
1+\norm{\bm h_j}_\infty,\ 1+\norm{\bm h_j'}_\infty
&\le
K_{\varphi,d}
[K_{\varphi,d}(N_{\bm n}+1)]^{\alpha+\cdots+\alpha^{j-1}}
(K_{\varphi,d}T)^{\alpha+\cdots+\alpha^j},\\
\norm{\bm h_j-\bm h_j'}_\infty
&\le
K_{\varphi,d}
[K_{\varphi,d}(N_{\bm n}+1)]^{\alpha+\cdots+\alpha^{j-1}}
(K_{\varphi,d}T)^{\alpha+\cdots+\alpha^{j-1}}
\norm{\bm\theta-\bm\theta'}_\infty^{\alpha^j},
\end{align*}
where empty sums are zero.

For the induction base $j=1$,
$\norm{\bm z_1}_\infty\le(d+1)T\le K_{\varphi,d}T$, and the same holds
for $\bm z_1'$.  Thus \eqref{eq:holder-growth-base} gives the first
estimate.  Also
\[
\norm{\bm z_1-\bm z_1'}_\infty
\le(d+1)\norm{\bm\theta-\bm\theta'}_\infty,
\]
so
\[
\norm{\bm h_1-\bm h_1'}_\infty
\le
H_\varphi(d+1)^\alpha
\norm{\bm\theta-\bm\theta'}_\infty^\alpha
\le
K_{\varphi,d}
\norm{\bm\theta-\bm\theta'}_\infty^\alpha.
\]

Suppose the estimates hold at level $j-1$, where $2\le j\le\ell$.
Using the size estimate at level $j-1$,
\begin{align*}
\norm{\bm z_j}_\infty
&\le
N_{\bm n}T\norm{\bm h_{j-1}}_\infty+T
\le
(N_{\bm n}+1)T
\bigl(1+\norm{\bm h_{j-1}}_\infty\bigr)\\
&\le
[K_{\varphi,d}(N_{\bm n}+1)]^{1+\alpha+\cdots+\alpha^{j-2}}
(K_{\varphi,d}T)^{1+\alpha+\cdots+\alpha^{j-1}}.
\end{align*}
Applying \eqref{eq:holder-growth-base} gives the required size estimate,
and the primed case is identical.  For the difference estimate, write
\[
\bm z_j-\bm z_j'
=
\bm W_j(\bm h_{j-1}-\bm h_{j-1}')
+
(\bm W_j-\bm W_j')\bm h_{j-1}'
+
(\bm b_j-\bm b_j'),
\]
hence
\begin{align*}
\norm{\bm z_j-\bm z_j'}_\infty
&\le
N_{\bm n}T\norm{\bm h_{j-1}-\bm h_{j-1}'}_\infty
+(N_{\bm n}+1)\norm{\bm\theta-\bm\theta'}_\infty
\bigl(1+\norm{\bm h_{j-1}'}_\infty\bigr)\\
&\le
2[K_{\varphi,d}(N_{\bm n}+1)]^{1+\alpha+\cdots+\alpha^{j-2}}
(K_{\varphi,d}T)^{1+\alpha+\cdots+\alpha^{j-2}}
\norm{\bm\theta-\bm\theta'}_\infty^{\alpha^{j-1}}.
\end{align*}
Here $T\ge1$,
$\norm{\bm\theta-\bm\theta'}_\infty
\le
\norm{\bm\theta-\bm\theta'}_\infty^{\alpha^{j-1}}$,
and
$\alpha+\cdots+\alpha^{j-1}
\le1+\alpha+\cdots+\alpha^{j-2}$.  Applying the
$\alpha$-H\"older continuity of $\varphi$ therefore gives
\[
\norm{\bm h_j-\bm h_j'}_\infty
\le
K_{\varphi,d}
[K_{\varphi,d}(N_{\bm n}+1)]^{\alpha+\cdots+\alpha^{j-1}}
(K_{\varphi,d}T)^{\alpha+\cdots+\alpha^{j-1}}
\norm{\bm\theta-\bm\theta'}_\infty^{\alpha^j},
\]
because $2^\alpha H_\varphi\le K_{\varphi,d}$.  This completes the
induction.

It remains to control the affine output layer.  Using the estimates at
level $\ell$,
\begin{align}
|\phi_{\bm\theta}(\bm x)-\phi_{\bm\theta'}(\bm x)|
&\le
n_\ell T\norm{\bm h_\ell-\bm h_\ell'}_\infty
+\norm{\bm\theta-\bm\theta'}_\infty
\bigl(n_\ell\norm{\bm h_\ell'}_\infty+1\bigr)\notag\\
&\le
2[K_{\varphi,d}(N_{\bm n}+1)]^{1+\alpha+\cdots+\alpha^{\ell-1}}
(K_{\varphi,d}T)^{1+\alpha+\cdots+\alpha^{\ell-1}}
\norm{\bm\theta-\bm\theta'}_\infty^{\alpha^\ell}\notag\\
&\le
\left[2K_{\varphi,d}^2(N_{\bm n}+1)T\right]^\ell
\norm{\bm\theta-\bm\theta'}_\infty^{\alpha^\ell}.
\label{eq:holder-parameter-continuity}
\end{align}
Since $\clip$ is $1$-Lipschitz, the same bound holds for
$h_{\bm\theta}-h_{\bm\theta'}$.  Thus the condition in
\eqref{eq:general-parameterized-cover} holds with
$G=[2K_{\varphi,d}^2(N_{\bm n}+1)T]^\ell$,
$q=\alpha^\ell$, and $P=P_{\bm n}$.  Hence
\begin{equation*}
\log\mathcal N(\eps,\mathcal H,M)
\le
P_{\bm n}\log\left[
1+2T
\left\{2K_{\varphi,d}^2(N_{\bm n}+1)T\right\}^{\ell/\alpha^\ell}
\eps^{-1/\alpha^\ell}
\right].
\end{equation*}

\par\medskip\noindent
\emph{Union over the allowed hidden-width sequences.}
We now combine the fixed-architecture covers.  For every such sequence,
$P_{\bm n}\le N(d+1)+(\ell-1)N(N+1)+N+1\le2(dN+LN^2)$,
and there are at most
$\sum_{\ell=1}^L N^\ell\le(N+1)^L$ possible sequences.
The union of the corresponding proper covers is again proper.  Since
$N_{\bm n}\le N$ and $\ell/\alpha^\ell\le L/\alpha^L$,
\begin{align*}
&\log\,\mathcal N\!\left(
\eps,\mathcal F_\varphi(N,L,T),M
\right)\\
&\qquad\le
\log\,\Biggl\{
\sum_{\ell=1}^L
\sum_{n_1,\ldots,n_\ell=1}^N
\Bigl(
1+2T
\bigl(2K_{\varphi,d}^2(N_{\bm n}+1)T\bigr)^{\ell/\alpha^\ell}
\eps^{-1/\alpha^\ell}
\Bigr)^{P_{\bm n}}
\Biggr\}\\
&\qquad\le
\log\,\Biggl\{
(N+1)^L
\Bigl(
1+2T
\bigl(2K_{\varphi,d}^2(N+1)T\bigr)^{L/\alpha^L}
\eps^{-1/\alpha^L}
\Bigr)^{2(dN+LN^2)}
\Biggr\}\\[2pt]
&\qquad=
L\log\,(N+1)
+2(dN+LN^2)
\log\,\Bigl(
1+2T
\bigl(2K_{\varphi,d}^2(N+1)T\bigr)^{L/\alpha^L}
\eps^{-1/\alpha^L}
\Bigr)\\[2pt]
&\qquad\le
L\log\,(N+1)
+2(dN+LN^2)
\log\,\Bigl(
4T
\bigl(2K_{\varphi,d}^2(N+1)T\bigr)^{L/\alpha^L}
\eps^{-1/\alpha^L}
\Bigr)\\[2pt]
&\qquad\le
L\log\,(N+1)
+2\alpha^{-L}(dN+LN^2)
\log\,\left(
\frac{4(2K_{\varphi,d}^2)^L(N+1)^LT^{L+1}}{\eps}
\right)\\
&\qquad\le
\alpha^{-L}(dN+LN^2)
\Biggl\{
5\log 2
+2\bigl[1+\log_2\,(2K_{\varphi,d}^2)\bigr]
\log\,\left(
\frac{(N+1)^LT^{L+1}}{\eps}
\right)
\Biggr\}\\
&\qquad\le
2\bigl[1+\log_2\,(2K_{\varphi,d}^2)\bigr]
\alpha^{-L}(dN+LN^2)
\log\,\left(
\frac{e(N+1)^LT^{L+1}}{\eps}
\right)\\
&\qquad=
\underbrace{\left[
4+
4\log_2\,\left(
(d+1)(1+|\varphi(0)|+H_\varphi)
\right)
\right]}_{C_{\varphi,d}}
\alpha^{-L}(dN+LN^2)
\log\,\left(
\frac{e(N+1)^LT^{L+1}}{\eps}
\right).
\end{align*}
This proves the result.
\end{proof}
\endgroup

For the remainder of this appendix, use the assumptions of
Proposition~\ExternalNumber{4.3}{prop:oracle}.  For $g\in\mathcal G$, write
\begin{equation}\label{eq:risk-processes}
\begin{aligned}
&
r(g)=\int_{[0,1]^d}(g-f)^2\,\mathrm d\mu,
\quad
r_M(g)=\frac1M\sum_{i=1}^M
\bigl(g(\bm X_i)-f(\bm X_i)\bigr)^2,\\&\nu_M(g)=\frac1M\sum_{i=1}^M
\varepsilon_i\bigl(g(\bm X_i)-f(\bm X_i)\bigr).
\end{aligned}
\end{equation}
The first lemma controls the quadratic process generated by the random
design. It bounds the relative deviation of the bounded squared loss
through a proper empirical cover; compare
Gy\"orfi et al.~\citeyearpar{GyorfiEtAl2002}. Related oracle inequalities
based on covering numbers appear in
Schmidt-Hieber~\citeyearpar{SchmidtHieber2020,SchmidtHieberVu2024}.

\begin{lemma}[Quadratic deviation]\label{lem:quadratic}
Under the boundedness assumptions of Proposition~\ExternalNumber{4.3}{prop:oracle},
\[
\mathbb E_{\bm X}\left[
\sup_{g\in\mathcal G}\{r(g)-2r_M(g)\}
\right]_+
\le\frac{46}{M}
\left\{
\log\left[2\mathcal N\!\left(\frac1{38M},\mathcal G,2M\right)\right]+1
\right\}.
\]
\end{lemma}

\begin{proof}
Recall $r$ and $r_M$ from \eqref{eq:risk-processes}.  The claim is immediate
if the covering number on the right is infinite, so assume it is finite.
Fix a countable uniformly dense subclass $\mathcal G_0=\{g_1,g_2,\ldots\}$.
Continuity in the supremum norm lets us take all suprema over
$\mathcal G_0$, which also makes them measurable.  For $y>0$, let
\[
E_y=\left\{\sup_{g\in\mathcal G}\bigl(r(g)-2r_M(g)\bigr)>y\right\}.
\]
On $E_y$, choose the first $g_j$ satisfying $r(g_j)-2r_M(g_j)>y$;
this choice is measurable with respect to the original sample.
For the squared loss $\ell_g(\bx)=\abs{g(\bx)-f(\bx)}^2$, the common
range $[-1,1]$ of $f$ and the functions in $\mathcal G$ gives
\[
0\le\ell_g\le4,
\qquad
\abs{\ell_g(\bx)-\ell_h(\bx)}
\le4\abs{g(\bx)-h(\bx)}.
\]

Fix $b\in(0,1/2)$ and set $a_0=1/2-b$.  If $g$ witnesses $E_y$, then
$r_M(g)<(r(g)-y)/2$ and $r(g)>y$.  Let
$\widetilde{\bm X}_1,\ldots,\widetilde{\bm X}_M$ be an independent copy of the design sample, and set
\[
\widetilde r_M(g)=\frac1M\sum_{i=1}^M
\ell_g(\widetilde{\bm X}_i).
\]
Condition on the original design sample.  For a fixed possible witness, the
variables $\ell_g(\widetilde{\bm X}_i)$ have mean $r(g)$, lie in
$[0,4]$, and have variance at most $4r(g)$.  Bernstein's lower-tail inequality
gives
\begin{align*}
\mathbb P_{\widetilde{\bm X}}\!\left\{
\widetilde r_M(g)
<r(g)-a_0(r(g)+y)
\right\}&\le
\exp\!\left\{
-\frac{Ma_0^2(r(g)+y)^2}
{2\operatorname{Var}(\ell_g(\widetilde{\bm X}_1))+(8/3)a_0(r(g)+y)}
\right\}\\
&\le
\exp\!\left\{-
\frac{4a_0^2}{8+(16/3)a_0}\,yM
\right\}.
\end{align*}
The last line uses $(r(g)+y)^2\ge4r(g)y$ and $r(g)+y\le2r(g)$.

Let
\[
q_y=\exp\!\left\{-
\frac{4a_0^2}{8+(16/3)a_0}\,yM
\right\}.
\]
If $q_y\le1/2$, the chosen witness satisfies the ghost-sample lower
bound with conditional probability at least $1/2$.  Integrating over the
original sample gives 
\[
\mathbb P_{\bm X}(E_y)\le2\mathbb P_{\bm X,\widetilde{\bm X}}(F_y),
\]
where $F_y$ is the measurable event that some $g\in\mathcal G_0$ satisfies
\[
r_M(g)\le\frac{r(g)-y}{2},
\qquad
\widetilde r_M(g)\ge r(g)-a_0(r(g)+y).
\]

Fix the $2M$ design locations and take a proper empirical $\delta$-net $\mathcal H$ of $\mathcal G$ on those locations. Choose the net once for the unordered collection of locations and keep it fixed under every subsequent pairwise swap; the empirical max norm is invariant under these swaps. The bound below is
pointwise on each finite swap orbit, so no measurable selection of the net
is needed. If $g$ witnesses $F_y$, choose $h\in\mathcal H$ with $\abs{g-h}\le\delta$ at all sample and ghost-sample points.  The loss changes by at most $4\delta$ at each point.  We obtain
\[
r_M(h)\le\frac{r(g)-y}{2}+4\delta,
\qquad
\widetilde r_M(h)\ge r(g)-a_0(r(g)+y)-4\delta.
\]
Set $D_h=\widetilde r_M(h)-r_M(h)$.  The preceding inequalities imply
\begin{align*}
D_h
&\ge b(r(g)+y)-8\delta,\qquad
\widetilde r_M(h)+r_M(h)
\le D_h+r(g)-y+8\delta.
\end{align*}
Let $\vartheta=b/(1+b)$ and choose
$\delta=by/[8(1+b)]$.  The two preceding inequalities then give
\[
D_h\ge b(r(g)+y)-8\delta
=b\bigl(r(g)+8\delta\bigr)
\]
and
$y+\widetilde r_M(h)+r_M(h)\le D_h+r(g)+8\delta$.  Since
$1-\vartheta=1/(1+b)$, the first inequality is equivalent to
$(1-\vartheta)D_h\ge\vartheta(r(g)+8\delta)$.  Combining the two
relations yields
\[
D_h\ge\vartheta\bigl(y+\widetilde r_M(h)+r_M(h)\bigr).
\]

Fix $h$ and exchange $\bm X_i$ with $\widetilde{\bm X}_i$ independently
with probability $1/2$ in each pair. These exchanges preserve the joint law because the two
observations are independent and have the same design distribution.
For fixed paired loss values $u_i,v_i\in[0,4]$, write
$\mathbb E_{\mathrm{sw}}$ for expectation over the exchanges and let
$s_i=u_i+v_i$ and $d_i=u_i-v_i$. The $i$th summand in
$(1-\vartheta)\widetilde r_M(h)-(1+\vartheta)r_M(h)$ equals
$d_i-\vartheta s_i$ or $-d_i-\vartheta s_i$, each with probability $1/2$.  With $\lambda=\vartheta/2$,
\begin{align*}
&\mathbb E_{\mathrm{sw}}\exp\!\left\{
\lambda\bigl[(1-\vartheta)\ell_h(\widetilde{\bm X}_i)
 -(1+\vartheta)\ell_h(\bm X_i)\bigr]
\right\}
=e^{-\lambda\vartheta s_i}\cosh(\lambda d_i)\\
&\qquad\le
\exp\!\left(-\lambda\vartheta s_i+\frac{\lambda^2d_i^2}{2}\right)
\le1,
\end{align*}
where $d_i^2\le4s_i$ was used in the last line.  Markov's inequality and a union bound over the net give
\[
\mathbb P_{\bm X,\widetilde{\bm X}}(F_y)
\le\mathcal N\!\left(\frac{by}{8(1+b)},\mathcal G,2M\right)
\exp\!\left[-\frac12\left(\frac{b}{1+b}\right)^2yM\right].
\]
Choose $b=3/11$.  Then
\[
\frac{b}{8(1+b)}=\frac3{112}\ge\frac1{38},
\qquad
\frac12\left(\frac{b}{1+b}\right)^2=\frac9{392}>\frac1{45},
\]
and $a_0=1/2-b=5/22$ gives
$
\frac{4a_0^2}{8+(16/3)a_0}>\frac1{45}
$.
If $q_y\le1/2$, the preceding symmetrization and union bound therefore imply
\[
\mathbb P_{\bm X}(E_y)
\le2\mathcal N\!\left(\frac y{38},\mathcal G,2M\right)
\exp\!\left(-\frac{yM}{45}\right).
\]
If $q_y>1/2$, then $e^{-yM/45}>1/2$ by the last numerical inequality.  Since every covering number of a nonempty class is at least one, the right-hand side of the same estimate is larger than one, so the bound remains valid by the trivial bound $\mathbb P_{\bm X}(E_y)\le1$.  Thus the tail estimate holds for every $y>0$.

Let
\[
A=2\mathcal N\!\left(\frac1{38M},\mathcal G,2M\right),
\qquad c=\frac M{45}.
\]
By monotonicity of covering numbers, for $y\ge M^{-1}$ the preceding
probability is at most
$
\min\{1,Ae^{-cy}\}
$.
Hence
\begin{align*}
\mathbb E_{\bm X}\left[
\sup_{g\in\mathcal G}\{r(g)-2r_M(g)\}
\right]_+
&\le\frac1M+\int_0^\infty\min\{1,Ae^{-cy}\}\,\mathrm dy\le\frac1M+c^{-1}(\log A+1)\\
&\le
\frac{46}{M}
\left\{
\log\left[2\mathcal N\!\left(\frac1{38M},\mathcal G,2M\right)\right]+1
\right\}.
\end{align*}
This proves the lemma.
\end{proof}

The preceding lemma handles the random-design fluctuation.  We next treat the
sub-Gaussian multiplier.  Its linear-minus-quadratic form is closely related to
the offset processes used for square-loss localization by
\citet{LiangEtAl2015}; see also the bounded-multiplier exponential-moment
estimate of \citet[Proposition~7]{KanadeEtAl2024}. We prove below the
finite-cover form needed for our sub-Gaussian noise assumption.

\begin{lemma}[Sub-Gaussian multiplier]\label{lem:multiplier}
Under the assumptions of Proposition~\ExternalNumber{4.3}{prop:oracle},
\[
\mathbb E\left[
\sup_{g\in\mathcal G}
\left\{2|\nu_M(g)|-\frac13r_M(g)\right\}
\right]_+
\le\frac{145(1+\sigma^2)}{24M}
\left\{
\log\left[2\mathcal N\!\left(\frac1{48M},\mathcal G,2M\right)\right]+1
\right\}.
\]
\end{lemma}
\begin{proof}
We condition on the design, establish an exponential tail for each
function, extend it to the class using a proper empirical cover, and
integrate the tail.  Recall $r_M$ and $\nu_M$ from
\eqref{eq:risk-processes}.  If the covering number on the right is
infinite, the claim is immediate.  Fix a realization
\(\bm X_1=\bm x_1,\ldots,\bm X_M=\bm x_M\).  Write
$\mathbb E_{\varepsilon}$ and $\mathbb P_{\varepsilon}$ for expectation and
probability over the independent noise variables.  By separability, the suprema
below are measurable. For a fixed \(g\in\mathcal G\), write its residuals as
\(a_i=g(\bm x_i)-f(\bm x_i)\) and choose
\(\lambda=[6(1+\sigma^2)]^{-1}\).  Then
\(\nu_M(g)=M^{-1}\sum_{i=1}^M a_i\varepsilon_i\) and
\(r_M(g)=M^{-1}\sum_{i=1}^M a_i^2\).  Hence, for either
\(s\in\{-1,1\}\), independence and the sub-Gaussian bound give
\begin{align*}
&\mathbb E_{\varepsilon}
 \exp\!\left\{\lambda M
 \left(2s\nu_M(g)-\frac13r_M(g)\right)\right\}                         \\
&=
 \prod_{i=1}^M e^{-\lambda a_i^2/3}
 \mathbb E_{\varepsilon}e^{2\lambda s a_i\varepsilon_i}
 \le
 \prod_{i=1}^M
 \exp\!\left\{\left(2\sigma^2\lambda^2-\frac{\lambda}{3}\right)a_i^2\right\} \\
&=
 \exp\!\left\{
 \left[
 \frac{\sigma^2}{18(1+\sigma^2)^2}
 -
 \frac{1+\sigma^2}{18(1+\sigma^2)^2}
 \right]
 \sum_{i=1}^M a_i^2
 \right\}
 =
 \exp\!\left\{
 -\frac1{18(1+\sigma^2)^2}\sum_{i=1}^M a_i^2
 \right\}
 \le1.
\end{align*}
Therefore, for every \(y>0\), Markov's inequality yields
\[
\mathbb P_{\varepsilon}\!\left\{
2s\nu_M(g)-\frac13r_M(g)\ge y
\right\}
\le
e^{-\lambda My}
\mathbb E_{\varepsilon}
\exp\!\left\{\lambda M
\left(2s\nu_M(g)-\frac13r_M(g)\right)\right\}
\le e^{-\lambda My}.
\]
Since
\(2|\nu_M(g)|=\max_{s\in\{-1,1\}}2s\nu_M(g)\), a union bound over the two
signs gives
\[
\mathbb P_{\varepsilon}\!\left\{
2|\nu_M(g)|-\frac13r_M(g)\ge y
\right\}
\le2e^{-\lambda My}.
\]

Set \(\delta=(48M)^{-1}\).  Apply the definition of the uniform proper
empirical covering number to the repeated sequence
\[
\bm x_1,\ldots,\bm x_M,\bm x_1,\ldots,\bm x_M.
\]
Since repetition does not change the empirical supremum metric, there exists
a proper net \(\mathcal H\subseteq\mathcal G\) such that, for every
\(g\in\mathcal G\), some \(h\in\mathcal H\) satisfies
\[
\max_{1\le i\le M}|g(\bm x_i)-h(\bm x_i)|\le\delta,
\qquad
\#\mathcal H
\le
\mathcal N(\delta,\mathcal G,2M).
\]
The preceding tail estimate and a union bound over \(\mathcal H\) imply
\[
\mathbb P_{\varepsilon}\!\left\{
\max_{h\in\mathcal H}
\left(2|\nu_M(h)|-\frac13r_M(h)\right)\ge y
\right\}
\le
\min\!\left\{1,\,2\#\mathcal H\,e^{-\lambda My}\right\}.
\]
Using
\(\mathbb E_{\varepsilon}[Z_+]=\int_0^\infty\mathbb P_{\varepsilon}\{Z\ge y\}\,\mathrm dy\), and
splitting the integral at
\(\log(2\#\mathcal H)/(\lambda M)\), we obtain
\begin{align*}
\mathbb E_{\varepsilon}\left[
\max_{h\in\mathcal H}
\left\{2|\nu_M(h)|-\frac13r_M(h)\right\}
\right]_+
&\le
\int_0^{\frac{\log(2\#\mathcal H)}{\lambda M}}1\,\mathrm dy
+
\int_{\frac{\log(2\#\mathcal H)}{\lambda M}}^\infty
2\#\mathcal H\,e^{-\lambda My}\,\mathrm dy\\
&=
\frac{\log(2\#\mathcal H)}{\lambda M}
+
\frac{2\#\mathcal H}{\lambda M}
e^{-\log(2\#\mathcal H)}\\
&=
\frac{6(1+\sigma^2)}{M}
\left\{\log(2\#\mathcal H)+1\right\}.
\end{align*}

For any \(g\in\mathcal G\), choose \(h\in\mathcal H\) as above.  Since
\(f,g,h\in[-1,1]\),
\begin{align*}
&|r_M(g)-r_M(h)|
\le
\frac1M\sum_{i=1}^M
|g(\bm x_i)-h(\bm x_i)|
\,|g(\bm x_i)+h(\bm x_i)-2f(\bm x_i)|
\le4\delta,\\
&|\nu_M(g)-\nu_M(h)|
\le
\frac1M\sum_{i=1}^M
|\varepsilon_i|
\,|g(\bm x_i)-h(\bm x_i)|
\le
\delta\frac1M\sum_{i=1}^M|\varepsilon_i|.
\end{align*}
Consequently,
\begin{align*}
2|\nu_M(g)|-\frac13r_M(g)
&\le
2|\nu_M(h)|-\frac13r_M(h)
+2|\nu_M(g)-\nu_M(h)|
+\frac13|r_M(g)-r_M(h)|                                      \\
&\le
2|\nu_M(h)|-\frac13r_M(h)
+2\delta\frac1M\sum_{i=1}^M|\varepsilon_i|
+\frac43\delta.
\end{align*}
Taking the supremum over \(g\), then the positive part and expectation over the noise, and using
\(\mathbb E_{\varepsilon}|\varepsilon_i|
\le(\mathbb E_{\varepsilon}\varepsilon_i^2)^{1/2}\le\sigma\), gives
\begin{align*}
\mathbb E_{\varepsilon}\left[
\sup_{g\in\mathcal G}
\left\{2|\nu_M(g)|-\frac13r_M(g)\right\}
\right]_+
&\le
\frac{6(1+\sigma^2)}{M}
\left\{\log(2\#\mathcal H)+1\right\}
+\left(\frac43+2\sigma\right)\delta\\
&\le
\frac{6(1+\sigma^2)}{M}
\left\{
\log\left[
2\mathcal N\!\left(\frac1{48M},\mathcal G,2M\right)
\right]+1
\right\}
+\frac{1+\sigma^2}{24M},
\end{align*}
where the second inequality uses
$(\frac43+2\sigma)\delta\le2(1+\sigma^2)\delta
=(1+\sigma^2)/(24M)$, since
$2(1+\sigma^2)-(4/3+2\sigma)
=2(\sigma-\tfrac12)^2+\tfrac16\ge0$.

Finally,
\(\log[2\mathcal N((48M)^{-1},\mathcal G,2M)]+1\ge1\), and therefore
\begin{align*}
&\mathbb E_{\varepsilon}\left[
\sup_{g\in\mathcal G}
\left\{2|\nu_M(g)|-\frac13r_M(g)\right\}
\right]_+
\le
\frac{145(1+\sigma^2)}{24M}
\left\{
\log\left[
2\mathcal N\!\left(\frac1{48M},\mathcal G,2M\right)
\right]+1
\right\}.
\end{align*}
The bound is independent of the fixed design points, so integrating it with
respect to \((\bm X_1,\ldots,\bm X_M)\) completes the proof.
\end{proof}

\par\medskip\noindent
\emph{Existence of measurable approximate minimizers.}
For fixed $M,N,T$, the class $\mathcal F_\DTA(N,19,T)$ is a finite union
of compact images of parameter cubes in $C([0,1]^d)$, and hence is separable.
Choose a deterministic dense sequence $(g_j)_{j\ge1}$ and denote its
infimum empirical risk by $I=\inf_{j\ge1}\widehat{\mathcal R}_M(g_j)$.
Continuity of empirical risk makes $I$ equal to the infimum over the full
class.  The first index $J$ satisfying
$\widehat{\mathcal R}_M(g_J)\le I+M^{-1}$ is finite and measurable:
$\{J=j\}$ is the intersection of
$\{\widehat{\mathcal R}_M(g_j)\le I+M^{-1}\}$ and the finitely many events
$\{\widehat{\mathcal R}_M(g_k)>I+M^{-1}\}$, $k<j$.
Thus $\widehat f_M=g_J$ is a measurable approximate minimizer.

\begin{proof}[Proof of Proposition~\ExternalNumber{4.3}{prop:oracle}]
Fix $h\in\mathcal G$ and recall \eqref{eq:risk-processes}.  Expanding
$Y_i=f(\bm X_i)+\varepsilon_i$ gives
\begin{equation*}
\widehat{\mathcal R}_M(g)
=r_M(g)-2\nu_M(g)+\frac1M\sum_{i=1}^M\varepsilon_i^2
\qquad(g\in\mathcal G).
\end{equation*}
The common noise term cancels in the empirical-risk comparison.
Approximate minimality and the triangle inequality therefore yield
\begin{align*}
r_M(\widehat g)
&\le r_M(h)+2\nu_M(\widehat g)-2\nu_M(h)+\xi
\le r_M(h)+2|\nu_M(\widehat g)|+2|\nu_M(h)|+\xi\\
&\le\frac43r_M(h)+\frac13r_M(\widehat g)
+2\left(\sup_{u\in\mathcal G}
 \left\{2|\nu_M(u)|-\frac13r_M(u)\right\}\right)_++\xi.
\end{align*}
The last step bounds the two terms
$2|\nu_M(g)|-r_M(g)/3$, for $g=h,\widehat g$, by the same positive
supremum.  Subtract $r_M(\widehat g)/3$ and multiply by three.
Comparing empirical and population losses then gives
\begin{align*}
r(\widehat g)
&\le2r_M(\widehat g)
+\left(\sup_{u\in\mathcal G}\{r(u)-2r_M(u)\}\right)_+\\
&\le4r_M(h)+3\xi
+6\left(\sup_{u\in\mathcal G}
 \left\{2|\nu_M(u)|-\frac13r_M(u)\right\}\right)_+
+\left(\sup_{u\in\mathcal G}\{r(u)-2r_M(u)\}\right)_+.
\end{align*}
Since $h$ is deterministic,
$\mathbb E_{\bm X}r_M(h)
=M^{-1}\sum_{i=1}^M\mathbb E_{\bm X}|h(\bm X_i)-f(\bm X_i)|^2=r(h)$.
Lemmas~\ref{lem:multiplier} and~\ref{lem:quadratic}, and monotonicity of the
covering number, yield
\begin{align*}
\mathbb E r(\widehat g)
&\le4r(h)+3\xi
+\frac{145(1+\sigma^2)}{4M}
 \left\{\log\left[2\mathcal N\left(\frac1{48M},\mathcal G,2M\right)\right]+1\right\}\\
&\quad+\frac{46}{M}
 \left\{\log\left[2\mathcal N\left(\frac1{38M},\mathcal G,2M\right)\right]+1\right\}\\
&\le4r(h)+3\xi
+\frac{46+\frac{145}{4}(1+\sigma^2)}{M}
 \left\{\log\left[2\mathcal N\left(\frac1{48M},\mathcal G,2M\right)\right]+1\right\}\\
&\le4r(h)+3\xi+\frac{83+37\sigma^2}{M}
 \left\{\log\left[2\mathcal N\left(\frac1{48M},\mathcal G,2M\right)\right]+1\right\}.
\end{align*}
Here $46+145(1+\sigma^2)/4=(329+145\sigma^2)/4\le83+37\sigma^2$.
Taking the infimum over $h\in\mathcal G$ proves the proposition.
\end{proof}

\end{document}

%% file: settings.tex
\newif\ifanonymous
\anonymousfalse

\ifdefined\pdftexversion
\else
  \PackageError{JASA}{Use pdfLaTeX to reproduce the reference PDF}{Select pdfLaTeX in Overleaf Settings / Compiler.}
\fi

\usepackage{type1cm}
\usepackage[OT1]{fontenc}
\usepackage{amsmath,amssymb,amsthm,bm,mathtools}
\usepackage[margin=1in,a4paper]{geometry}
\usepackage[mathcal]{eucal}
\DeclareMathSizes{10.5}{10.5}{7.5}{6}
\usepackage{graphicx,xcolor,booktabs,array,enumitem}
\usepackage{tabularray}
\usepackage[authoryear,round]{natbib}
\bibpunct{(}{)}{;}{a}{}{,}
\usepackage{url}
\usepackage{microtype}
\usepackage{etoolbox}
\usepackage{tikz}
\usetikzlibrary{arrows.meta,calc,decorations.pathreplacing,positioning,shapes.geometric}
\definecolor{mypurple}{HTML}{743096}
\usepackage[colorlinks=true,citecolor=blue,hypertexnames=false]{hyperref}
\usepackage{bookmark}
\usepackage{xr-hyper}
\usepackage{caption}
\usepackage{placeins}
\usepackage{needspace}
\graphicspath{{figures/}}
\numberwithin{equation}{section}
\allowdisplaybreaks[4]
\theoremstyle{plain}
\newtheorem{theorem}{Theorem}[section]
\newtheorem{proposition}{Proposition}[section]
\newtheorem{lemma}{Lemma}[section]
\newtheorem{corollary}{Corollary}[section]
\theoremstyle{definition}
\newtheorem{remark}{Remark}[section]
\theoremstyle{plain}
\hypersetup{pdfcreator={LaTeX},pdfauthor={}}
\definecolor{nodeblue}{RGB}{48,92,255}
\definecolor{edgepink}{RGB}{255,102,166}
\definecolor{actgreen}{RGB}{39,145,74}
\definecolor{goodblue}{RGB}{205,232,247}
\definecolor{stripred}{RGB}{245,142,151}
\definecolor{gridred}{RGB}{211,72,81}
\definecolor{softgray}{RGB}{95,100,108}
\definecolor{templategray}{RGB}{183,188,197}
\definecolor{goodgreen}{RGB}{210,240,210}
\definecolor{goodgreenstrong}{RGB}{112,195,128}

\newcommand{\R}{\mathbb R}
\newcommand{\Z}{\mathbb Z}
\newcommand{\Q}{\mathbb Q}
\newcommand{\Torus}{\mathbb T}
\newcommand{\NN}{\mathcal N}
\newcommand{\dd}{\,\mathrm d}
\newcommand{\eps}{\varepsilon}
\newcommand{\norm}[1]{\lVert#1\rVert}
\newcommand{\parnorm}[1]{\lVert #1\rVert_{\mathrm{par}}}
\newcommand{\abs}[1]{\lvert#1\rvert}
\newcommand{\vol}{\operatorname{vol}}
\newcommand{\bx}{\bm x}
\newcommand{\by}{\bm y}
\newcommand{\bu}{\bm u}
\newcommand{\bv}{\bm v}
\newcommand{\bk}{\bm k}
\newcommand{\bw}{\bm w}
\newcommand{\bz}{\bm z}
\newcommand{\bh}{\bm h}
\newcommand{\bp}{\bm p}
\newcommand{\bt}{\bm t}
\newcommand{\bn}{\bm n}

\newcommand{\balpha}{\bm\alpha}
\newcommand{\bzero}{\bm 0}

\newcommand{\one}{\bm 1}
\newcommand{\bmone}{\bm 1}
\usepackage{dsfont}
\def\one{{\mathds{1}}}

\providecommand{\DTA}{\ensuremath{\mathtt{DTA}}}
\DeclareMathOperator{\clip}{clip}

\tikzset{
  >={Latex[length=2.25mm,width=1.45mm]},
  netinput/.style={circle,draw=softgray,line width=.68pt,fill=white,
    minimum size=7.4mm,inner sep=.8pt,font=\scriptsize,align=center},
  prebox/.style={rounded corners=2.2pt,draw=softgray,line width=.62pt,
    fill=softgray!4,minimum height=8.5mm,minimum width=22mm,
    inner xsep=3pt,inner ysep=2pt,font=\scriptsize,align=center},
  preboxwide/.style={rounded corners=2.2pt,draw=softgray,line width=.62pt,
    fill=softgray!4,minimum height=9.5mm,minimum width=29mm,
    inner xsep=3pt,inner ysep=2pt,font=\scriptsize,align=center},
  actbox/.style={rounded corners=4.2pt,draw=nodeblue,line width=.78pt,
    fill=nodeblue!7,minimum height=9mm,minimum width=23mm,
    inner xsep=3pt,inner ysep=2pt,font=\scriptsize,align=center},
  actboxwide/.style={rounded corners=4.2pt,draw=nodeblue,line width=.78pt,
    fill=nodeblue!7,minimum height=9.5mm,minimum width=30mm,
    inner xsep=3pt,inner ysep=2pt,font=\scriptsize,align=center},
  actcompact/.style={rounded corners=3.7pt,draw=nodeblue,line width=.72pt,
    fill=nodeblue!7,minimum height=7.8mm,minimum width=27mm,
    inner xsep=2.5pt,inner ysep=1.7pt,font=\scriptsize,align=center},
  acttiny/.style={rounded corners=3.4pt,draw=nodeblue,line width=.70pt,
    fill=nodeblue!7,minimum height=7.4mm,minimum width=24mm,
    inner xsep=2.1pt,inner ysep=1.4pt,font=\scriptsize,align=center},
  acttail/.style={rounded corners=3.8pt,draw=nodeblue,line width=.74pt,
    fill=nodeblue!7,minimum height=9mm,minimum width=24mm,
    inner xsep=2.5pt,inner ysep=1.8pt,font=\scriptsize,align=center},
  netout/.style={ellipse,draw=softgray,line width=.68pt,fill=softgray!3,
    minimum height=8.5mm,minimum width=22mm,inner xsep=4pt,
    font=\scriptsize,align=center},
  denseconn/.style={draw=templategray,line width=.42pt,opacity=.42,line cap=round},
  activeconn/.style={draw=edgepink,line width=.92pt,line cap=round,
    shorten <=1.3pt,shorten >=1.8pt,-{Latex[length=2.05mm,width=1.35mm]}},
  affinearrow/.style={draw=edgepink,line width=.94pt,line cap=round,
    shorten <=1.3pt,shorten >=1.8pt,-{Latex[length=2.05mm,width=1.35mm]}},
  actarrow/.style={draw=actgreen,line width=.94pt,line cap=round,
    shorten <=1.2pt,shorten >=1.8pt,-{Latex[length=2.05mm,width=1.35mm]}},
  maptextclean/.style={font=\footnotesize,text=softgray,fill=white,
    rounded corners=.6pt,inner xsep=3.4pt,inner ysep=1.8pt},
  sigmatext/.style={font=\footnotesize,text=actgreen!80!black,fill=white,
    rounded corners=.6pt,inner xsep=2.8pt,inner ysep=1.5pt},
  layername/.style={font=\scriptsize,text=softgray,align=center},
  module/.style={trapezium,trapezium left angle=74,trapezium right angle=106,
    draw=nodeblue,line width=.72pt,fill=goodblue,minimum height=13mm,
    minimum width=35mm,inner xsep=4pt,font=\scriptsize,align=center},
  decodermodule/.style={regular polygon,regular polygon sides=6,
    draw=nodeblue,line width=.72pt,fill=goodblue,minimum height=20mm,
    minimum width=34mm,inner sep=2pt,font=\scriptsize,align=center}
}

\newcommand{\PaperDisplaySpacing}{}
\appto{\appendix}{%
  \clubpenalty=10000
  \widowpenalty=10000
  \displaywidowpenalty=10000
  \PaperDisplaySpacing
}

\newcommand{\PaperTitle}{Optimal Tradeoffs Between Network Size and Parameter Magnitude in Neural Approximation and Minimax Regression}

\newcommand{\PaperKeywords}{neural network, approximation theory, parameter magnitude, bit complexity, minimax nonparametric regression, Transformer}

\newcommand{\AuthorOne}{Baicheng Li}
\newcommand{\AuthorTwo}{Zuowei Shen}
\newcommand{\AuthorThree}{Haizhao Yang}
\newcommand{\AuthorFour}{Shijun Zhang}
\newcommand{\PaperAuthors}{\AuthorOne, \AuthorTwo, \AuthorThree, and \AuthorFour}

\newcommand{\InstitutionOne}{Department of Mathematics, University of Maryland, College Park, MD 20742-4015, USA}
\newcommand{\InstitutionTwo}{Department of Mathematics, National University of Singapore, Singapore 119076}
\newcommand{\InstitutionThree}{Department of Applied Mathematics, The Hong Kong Polytechnic University, Hong Kong SAR, China}

\hypersetup{
  pdftitle={\PaperTitle},
  pdfkeywords={\PaperKeywords}
}

\ifanonymous
  \hypersetup{pdfauthor={}}
\else
  \hypersetup{pdfauthor={\PaperAuthors}}
\fi

\makeatletter

\newcommand{\MakeMainTitle}{%
  \title{\PaperTitle}%
  \date{}%
  \ifanonymous
    \author{}%
  \else
    \author{%
      \AuthorOne
      \thanks{%
        \InstitutionOne.
        \texttt{baichl@umd.edu};
        \texttt{hzyang@umd.edu}%
      }%
      \quad
      \AuthorTwo
      \thanks{%
        \InstitutionTwo.
        \texttt{matzuows@nus.edu.sg}%
      }%
      \quad
      \AuthorThree
      \footnotemark[1]%
      \quad
      \AuthorFour
      \thanks{%
        \InstitutionThree.
        \texttt{shijun.zhang@polyu.edu.hk}%
      }%
    }%
  \fi

  \ifanonymous
    \begingroup
      \def\@maketitle{%
        \newpage
        \null
        \vskip2em
        \begin{center}
          \let\footnote\thanks
          {\LARGE\@title\par}
        \end{center}
        \par
        \vskip1.5em
      }%
      \maketitle
    \endgroup
  \else
    \maketitle
  \fi
}

\makeatother

\newcommand{\DiagramLabelFont}{\normalsize}
\newcommand{\DiagramPanelFont}{\footnotesize}

\newcommand{\FigureActivationWidth}{.58\linewidth}
\newcommand{\FigureStaircaseWidth}{.58\linewidth}
\newcommand{\FigureEncoderWidth}{.48\linewidth}
\newcommand{\FigureNetworkWidth}{.60\linewidth}
\newcommand{\FigureOptimalNetworkWidth}{.72\linewidth}
\newcommand{\FigureDecoderWidth}{.52\linewidth}

\newcommand{\EncoderMatrixLayout}{%
\setlength{\arraycolsep}{3.5pt}
\renewcommand{\arraystretch}{1.16}%
}

\newcommand{\FinitePrecisionLayout}{%
\PaperDisplaySpacing

}

\newcommand{\TransformerMatrixLayout}{%
\setlength{\arraycolsep}{2pt}
\renewcommand{\arraystretch}{0.98}%
}

\newcommand{\StatisticsProofLayout}{%
\PaperDisplaySpacing

}

\newcommand{\GridLegendGap}{%
\vspace{2mm}%
}

\DeclareRobustCommand{\ExternalNumber}[2]{\ref{#2}}
\newcommand{\IntroBrace}[2]{%
  \underbrace{\vphantom{\displaystyle\sup_{f\in\mathcal F}\inf_{g\in\mathcal G}\norm{g-f}_{L^2(\mu)}^2+\frac{\log[e\mathcal N((48M)^{-1},\mathcal G,2M)]}{M}}#1}_{\text{#2}}}

\usepackage{amsmath}
\usepackage{tabularray}
\usepackage{adjustbox}

\definecolor{darkblue}{rgb}{0,0.22,0.66}

\hypersetup{citecolor=darkblue,linkcolor=darkblue,urlcolor=darkblue}

\usepackage{refcount}

%% file: figures/phase_pair.tex
% Include this wrapper directly inside a full-width figure; do not resize it.
\begin{minipage}[t]{.495\linewidth}
\centering
\input{figures/phase.tex}\par
\vspace{1pt}
{\fontsize{10}{12}\selectfont (a) Width--radius\par}
\end{minipage}\hfill
\begin{minipage}[t]{.495\linewidth}
\centering
\input{figures/phase_bits.tex}\par
\vspace{1pt}
{\fontsize{10}{12}\selectfont (b) Width--bits\par}
\end{minipage}

%% file: figures/phase_bits.tex
% Width--bits companion to phase.tex; identical palette and axis geometry.
% This is a continuous schematic envelope for integer B >= 2.
% Affine display coordinates: N=x+3.5, B=2+y/k, with a fixed k>0.
% The centreline satisfies N^2 B=C; only asymptotic ticks are labelled.
\begin{tikzpicture}[x=.635cm,y=.78cm,
  font=\fontsize{9}{11}\selectfont,
  line join=round,line cap=butt,>=Stealth]
\definecolor{phaseblue}{RGB}{196,220,242}
\definecolor{phasered}{RGB}{250,196,198}
\definecolor{phasegreen}{RGB}{140,214,88}
\definecolor{phasegray}{RGB}{242,242,242}
\definecolor{phasegrayedge}{RGB}{210,210,210}
\definecolor{phasegraystrip}{RGB}{184,184,184}
\definecolor{phasegold}{RGB}{252,230,185}
\definecolor{phasegoldedge}{RGB}{223,191,159}
\definecolor{phasegreenink}{RGB}{0,110,18}
\definecolor{phasebluetext}{RGB}{0,0,115}
\definecolor{phaseredtext}{RGB}{140,0,0}
\definecolor{phasegoldtext}{RGB}{153,102,51}
\def\BitsUpper{(0.68500,5.23000) (0.69585,5.19919) (0.70675,5.16847) (0.71770,5.13784) (0.72871,5.10731) (0.73977,5.07686) (0.75089,5.04651) (0.76207,5.01624) (0.77330,4.98607) (0.78458,4.95599) (0.79592,4.92600) (0.80732,4.89610) (0.81878,4.86629) (0.83029,4.83657) (0.84186,4.80694) (0.85349,4.77741) (0.86518,4.74796) (0.87692,4.71861) (0.88873,4.68935) (0.90060,4.66018) (0.91252,4.63109) (0.92451,4.60211) (0.93656,4.57321) (0.94867,4.54440) (0.96084,4.51568) (0.97308,4.48706) (0.98538,4.45852) (0.99774,4.43008) (1.01017,4.40173) (1.02266,4.37347) (1.03521,4.34529) (1.04783,4.31722) (1.06052,4.28923) (1.07327,4.26133) (1.08609,4.23352) (1.09897,4.20581) (1.11193,4.17818) (1.12495,4.15065) (1.13804,4.12321) (1.15120,4.09586) (1.16443,4.06860) (1.17773,4.04143) (1.19109,4.01435) (1.20453,3.98736) (1.21805,3.96047) (1.23163,3.93366) (1.24529,3.90695) (1.25901,3.88032) (1.27282,3.85379) (1.28669,3.82735) (1.30065,3.80100) (1.31467,3.77474) (1.32877,3.74857) (1.34295,3.72250) (1.35721,3.69651) (1.37154,3.67062) (1.38595,3.64481) (1.40044,3.61910) (1.41501,3.59348) (1.42966,3.56795) (1.44439,3.54251) (1.45919,3.51716) (1.47408,3.49190) (1.48906,3.46673) (1.50411,3.44166) (1.51925,3.41667) (1.53447,3.39178) (1.54977,3.36698) (1.56516,3.34226) (1.58064,3.31764) (1.59620,3.29311) (1.61185,3.26868) (1.62759,3.24433) (1.64341,3.22007) (1.65932,3.19591) (1.67533,3.17183) (1.69142,3.14785) (1.70760,3.12396) (1.72388,3.10015) (1.74024,3.07644) (1.75670,3.05282) (1.77325,3.02930) (1.78990,3.00586) (1.80664,2.98251) (1.82347,2.95926) (1.84040,2.93609) (1.85743,2.91302) (1.87456,2.89004) (1.89178,2.86715) (1.90910,2.84435) (1.92652,2.82164) (1.94405,2.79902) (1.96167,2.77649) (1.97939,2.75405) (1.99722,2.73171) (2.01515,2.70946) (2.03318,2.68729) (2.05132,2.66522) (2.06956,2.64324) (2.08791,2.62135) (2.10637,2.59955) (2.12493,2.57784) (2.14360,2.55622) (2.16238,2.53470) (2.18127,2.51326) (2.20027,2.49192) (2.21939,2.47067) (2.23861,2.44950) (2.25795,2.42843) (2.27740,2.40745) (2.29697,2.38656) (2.31665,2.36577) (2.33644,2.34506) (2.35636,2.32444) (2.37639,2.30392) (2.39654,2.28349) (2.41681,2.26314) (2.43720,2.24289) (2.45771,2.22273) (2.47834,2.20266) (2.49910,2.18268) (2.51998,2.16279) (2.54098,2.14300) (2.56211,2.12329) (2.58336,2.10368) (2.60475,2.08415) (2.62625,2.06472) (2.64789,2.04538) (2.66966,2.02613) (2.69156,2.00697) (2.71358,1.98790) (2.73574,1.96892) (2.75803,1.95004) (2.78046,1.93124) (2.80302,1.91254) (2.82572,1.89392) (2.84855,1.87540) (2.87152,1.85697) (2.89462,1.83863) (2.91787,1.82038) (2.94125,1.80222) (2.96478,1.78416) (2.98844,1.76618) (3.01225,1.74829) (3.03620,1.73050) (3.06030,1.71280) (3.08454,1.69518) (3.10892,1.67766) (3.13345,1.66023) (3.15813,1.64289) (3.18296,1.62565) (3.20793,1.60849) (3.23305,1.59142) (3.25833,1.57445) (3.28375,1.55756) (3.30933,1.54077) (3.33506,1.52407) (3.36094,1.50746) (3.38698,1.49094) (3.41317,1.47451) (3.43952,1.45817) (3.46603,1.44192) (3.49269,1.42577) (3.51951,1.40970) (3.54649,1.39373) (3.57362,1.37785) (3.60092,1.36205) (3.62838,1.34635) (3.65600,1.33074) (3.68378,1.31523) (3.71173,1.29980) (3.73983,1.28446) (3.76810,1.26922) (3.79654,1.25406) (3.82514,1.23900) (3.85390,1.22402) (3.88284,1.20914) (3.91193,1.19435) (3.94120,1.17965) (3.97063,1.16504) (4.00023,1.15053) (4.03000,1.13610) (4.05993,1.12176) (4.09004,1.10752) (4.12031,1.09337) (4.15076,1.07930) (4.18137,1.06533) (4.21215,1.05145) (4.24311,1.03766) (4.27423,1.02396) (4.30552,1.01036) (4.33699,0.99684) (4.36862,0.98342) (4.40042,0.97008) (4.43240,0.95684) (4.46454,0.94369) (4.49686,0.93062) (4.52934,0.91765) (4.56199,0.90477) (4.59482,0.89199) (4.62781,0.87929) (4.66097,0.86668) (4.69429,0.85417) (4.72779,0.84174) (4.76145,0.82941) (4.79528,0.81717) (4.82927,0.80502) (4.86343,0.79296) (4.89775,0.78099) (4.93223,0.76911) (4.96688,0.75732) (5.00168,0.74563) (5.03665,0.73402) (5.07177,0.72251) (5.10705,0.71109) (5.14249,0.69975) (5.17808,0.68851) (5.21382,0.67736) (5.24971,0.66630) (5.28575,0.65534) (5.32194,0.64446) (5.35827,0.63367) (5.39475,0.62298) (5.43137,0.61238) (5.46812,0.60186) (5.50501,0.59144) (5.54204,0.58111) (5.57919,0.57087) (5.61648,0.56072) (5.65388,0.55067) (5.69141,0.54070) (5.72906,0.53082) (5.76683,0.52104) (5.80470,0.51135) (5.84269,0.50174) (5.88078,0.49223) (5.91897,0.48281) (5.95726,0.47348) (5.99564,0.46424) (6.03411,0.45510) (6.07266,0.44604) (6.11129,0.43707) (6.15000,0.42820) (6.18878,0.41942) (6.22762,0.41072) (6.26652,0.40212) (6.30547,0.39361) (6.34448,0.38519) (6.38352,0.37686) (6.42260,0.36863) (6.46171,0.36048) (6.50085,0.35243) (6.54000,0.34446) (6.57916,0.33659) (6.61833,0.32881) (6.65749,0.32112) (6.69664,0.31352) (6.73577,0.30601) (6.77488,0.29859) (6.81395,0.29126) (6.85298,0.28402) (6.89195,0.27688) (6.93087,0.26983) (6.96972,0.26286) (7.00850,0.25599) (7.04718,0.24921) (7.08578,0.24252) (7.12426,0.23592) (7.16264,0.22941) (7.20088,0.22300) (7.23900,0.21667) (7.27696,0.21044) (7.31478,0.20429) (7.35242,0.19824) (7.38989,0.19228) (7.42717,0.18641) (7.46425,0.18063) (7.50112,0.17494) (7.53776,0.16934) (7.57418,0.16383) (7.61034,0.15842) (7.64625,0.15309) (7.68188,0.14786) (7.71724,0.14272) (7.75229,0.13767) (7.78704,0.13271) (7.82147,0.12784) (7.85556,0.12306) (7.88931,0.11837) (7.92269,0.11377) (7.95571,0.10927) (7.98833,0.10485) (8.02056,0.10053) (8.05237,0.09630) (8.08376,0.09216) (8.11471,0.08811) (8.14520,0.08415) (8.17523,0.08028) (8.20478,0.07650) (8.23383,0.07282) (8.26237,0.06922) (8.29040,0.06572) (8.31789,0.06230) (8.34484,0.05898) (8.37122,0.05575) (8.39703,0.05261) (8.42226,0.04956) (8.44689,0.04660) (8.47091,0.04373) (8.49430,0.04096) (8.51706,0.03827) (8.53917,0.03568) (8.56062,0.03318) (8.58140,0.03076) (8.60150,0.02844) (8.62091,0.02621) (8.63961,0.02407) (8.65760,0.02203) (8.67486,0.02007) (8.69139,0.01820) (8.70718,0.01643) (8.72221,0.01475) (8.73649,0.01315) (8.74999,0.01165) (8.76272,0.01024) (8.77466,0.00892) (8.78581,0.00769) (8.79616,0.00655) (8.80570,0.00551) (8.81444,0.00455) (8.82236,0.00369) (8.82945,0.00291) (8.83573,0.00223) (8.84117,0.00164) (8.84579,0.00114) (8.84956,0.00073) (8.85251,0.00041) (8.85461,0.00018) (8.85587,0.00005) (8.85629,0.00000)}
\def\BitsLower{(0.58518,5.23000) (0.59603,5.19919) (0.60693,5.16847) (0.61788,5.13784) (0.62889,5.10731) (0.63995,5.07686) (0.65107,5.04651) (0.66225,5.01624) (0.67348,4.98607) (0.68476,4.95599) (0.69610,4.92600) (0.70750,4.89610) (0.71896,4.86629) (0.73047,4.83657) (0.74204,4.80694) (0.75367,4.77741) (0.76536,4.74796) (0.77710,4.71861) (0.78891,4.68935) (0.80078,4.66018) (0.81270,4.63109) (0.82469,4.60211) (0.83674,4.57321) (0.84885,4.54440) (0.86102,4.51568) (0.87326,4.48706) (0.88556,4.45852) (0.89792,4.43008) (0.91035,4.40173) (0.92284,4.37347) (0.93539,4.34529) (0.94801,4.31722) (0.96070,4.28923) (0.97345,4.26133) (0.98627,4.23352) (0.99915,4.20581) (1.01211,4.17818) (1.02513,4.15065) (1.03822,4.12321) (1.05138,4.09586) (1.06461,4.06860) (1.07791,4.04143) (1.09127,4.01435) (1.10471,3.98736) (1.11823,3.96047) (1.13181,3.93366) (1.14547,3.90695) (1.15919,3.88032) (1.17300,3.85379) (1.18687,3.82735) (1.20083,3.80100) (1.21485,3.77474) (1.22895,3.74857) (1.24313,3.72250) (1.25739,3.69651) (1.27172,3.67062) (1.28613,3.64481) (1.30062,3.61910) (1.31519,3.59348) (1.32984,3.56795) (1.34457,3.54251) (1.35937,3.51716) (1.37426,3.49190) (1.38924,3.46673) (1.40429,3.44166) (1.41943,3.41667) (1.43465,3.39178) (1.44995,3.36698) (1.46534,3.34226) (1.48082,3.31764) (1.49638,3.29311) (1.51203,3.26868) (1.52777,3.24433) (1.54359,3.22007) (1.55950,3.19591) (1.57551,3.17183) (1.59160,3.14785) (1.60778,3.12396) (1.62406,3.10015) (1.64042,3.07644) (1.65688,3.05282) (1.67343,3.02930) (1.69008,3.00586) (1.70682,2.98251) (1.72365,2.95926) (1.74058,2.93609) (1.75761,2.91302) (1.77474,2.89004) (1.79196,2.86715) (1.80928,2.84435) (1.82670,2.82164) (1.84423,2.79902) (1.86185,2.77649) (1.87957,2.75405) (1.89740,2.73171) (1.91533,2.70946) (1.93336,2.68729) (1.95150,2.66522) (1.96974,2.64324) (1.98809,2.62135) (2.00655,2.59955) (2.02511,2.57784) (2.04378,2.55622) (2.06256,2.53470) (2.08145,2.51326) (2.10045,2.49192) (2.11957,2.47067) (2.13879,2.44950) (2.15813,2.42843) (2.17758,2.40745) (2.19715,2.38656) (2.21683,2.36577) (2.23662,2.34506) (2.25654,2.32444) (2.27657,2.30392) (2.29672,2.28349) (2.31699,2.26314) (2.33738,2.24289) (2.35789,2.22273) (2.37852,2.20266) (2.39928,2.18268) (2.42016,2.16279) (2.44116,2.14300) (2.46229,2.12329) (2.48354,2.10368) (2.50493,2.08415) (2.52643,2.06472) (2.54807,2.04538) (2.56984,2.02613) (2.59174,2.00697) (2.61376,1.98790) (2.63592,1.96892) (2.65821,1.95004) (2.68064,1.93124) (2.70320,1.91254) (2.72590,1.89392) (2.74873,1.87540) (2.77170,1.85697) (2.79480,1.83863) (2.81805,1.82038) (2.84143,1.80222) (2.86496,1.78416) (2.88862,1.76618) (2.91243,1.74829) (2.93638,1.73050) (2.96048,1.71280) (2.98472,1.69518) (3.00910,1.67766) (3.03363,1.66023) (3.05831,1.64289) (3.08314,1.62565) (3.10811,1.60849) (3.13323,1.59142) (3.15851,1.57445) (3.18393,1.55756) (3.20951,1.54077) (3.23524,1.52407) (3.26112,1.50746) (3.28716,1.49094) (3.31335,1.47451) (3.33970,1.45817) (3.36621,1.44192) (3.39287,1.42577) (3.41969,1.40970) (3.44667,1.39373) (3.47380,1.37785) (3.50110,1.36205) (3.52856,1.34635) (3.55618,1.33074) (3.58396,1.31523) (3.61191,1.29980) (3.64001,1.28446) (3.66828,1.26922) (3.69672,1.25406) (3.72532,1.23900) (3.75408,1.22402) (3.78302,1.20914) (3.81211,1.19435) (3.84138,1.17965) (3.87081,1.16504) (3.90041,1.15053) (3.93018,1.13610) (3.96011,1.12176) (3.99022,1.10752) (4.02049,1.09337) (4.05094,1.07930) (4.08155,1.06533) (4.11233,1.05145) (4.14329,1.03766) (4.17441,1.02396) (4.20570,1.01036) (4.23717,0.99684) (4.26880,0.98342) (4.30060,0.97008) (4.33258,0.95684) (4.36472,0.94369) (4.39704,0.93062) (4.42952,0.91765) (4.46217,0.90477) (4.49500,0.89199) (4.52799,0.87929) (4.56115,0.86668) (4.59447,0.85417) (4.62797,0.84174) (4.66163,0.82941) (4.69546,0.81717) (4.72945,0.80502) (4.76361,0.79296) (4.79793,0.78099) (4.83241,0.76911) (4.86706,0.75732) (4.90186,0.74563) (4.93683,0.73402) (4.97195,0.72251) (5.00723,0.71109) (5.04267,0.69975) (5.07826,0.68851) (5.11400,0.67736) (5.14989,0.66630) (5.18593,0.65534) (5.22212,0.64446) (5.25845,0.63367) (5.29493,0.62298) (5.33155,0.61238) (5.36830,0.60186) (5.40519,0.59144) (5.44222,0.58111) (5.47937,0.57087) (5.51666,0.56072) (5.55406,0.55067) (5.59159,0.54070) (5.62924,0.53082) (5.66701,0.52104) (5.70488,0.51135) (5.74287,0.50174) (5.78096,0.49223) (5.81915,0.48281) (5.85744,0.47348) (5.89582,0.46424) (5.93429,0.45510) (5.97284,0.44604) (6.01147,0.43707) (6.05018,0.42820) (6.08896,0.41942) (6.12780,0.41072) (6.16670,0.40212) (6.20565,0.39361) (6.24466,0.38519) (6.28370,0.37686) (6.32278,0.36863) (6.36189,0.36048) (6.40103,0.35243) (6.44018,0.34446) (6.47934,0.33659) (6.51851,0.32881) (6.55767,0.32112) (6.59682,0.31352) (6.63595,0.30601) (6.67506,0.29859) (6.71413,0.29126) (6.75316,0.28402) (6.79213,0.27688) (6.83105,0.26983) (6.86990,0.26286) (6.90868,0.25599) (6.94736,0.24921) (6.98596,0.24252) (7.02444,0.23592) (7.06282,0.22941) (7.10106,0.22300) (7.13918,0.21667) (7.17714,0.21044) (7.21496,0.20429) (7.25260,0.19824) (7.29007,0.19228) (7.32735,0.18641) (7.36443,0.18063) (7.40130,0.17494) (7.43794,0.16934) (7.47436,0.16383) (7.51052,0.15842) (7.54643,0.15309) (7.58206,0.14786) (7.61742,0.14272) (7.65247,0.13767) (7.68722,0.13271) (7.72165,0.12784) (7.75574,0.12306) (7.78949,0.11837) (7.82287,0.11377) (7.85589,0.10927) (7.88851,0.10485) (7.92074,0.10053) (7.95255,0.09630) (7.98394,0.09216) (8.01489,0.08811) (8.04538,0.08415) (8.07541,0.08028) (8.10496,0.07650) (8.13401,0.07282) (8.16255,0.06922) (8.19058,0.06572) (8.21807,0.06230) (8.24502,0.05898) (8.27140,0.05575) (8.29721,0.05261) (8.32244,0.04956) (8.34707,0.04660) (8.37109,0.04373) (8.39448,0.04096) (8.41724,0.03827) (8.43935,0.03568) (8.46080,0.03318) (8.48158,0.03076) (8.50168,0.02844) (8.52109,0.02621) (8.53979,0.02407) (8.55778,0.02203) (8.57504,0.02007) (8.59157,0.01820) (8.60736,0.01643) (8.62239,0.01475) (8.63667,0.01315) (8.65017,0.01165) (8.66290,0.01024) (8.67484,0.00892) (8.68599,0.00769) (8.69634,0.00655) (8.70588,0.00551) (8.71462,0.00455) (8.72254,0.00369) (8.72963,0.00291) (8.73591,0.00223) (8.74135,0.00164) (8.74597,0.00114) (8.74974,0.00073) (8.75269,0.00041) (8.75479,0.00018) (8.75605,0.00005) (8.75647,0.00000)}
\def\BitsUpperEq{(0.70746,5.23000) (0.71831,5.19919) (0.72921,5.16847) (0.74016,5.13784) (0.75117,5.10731) (0.76223,5.07686) (0.77335,5.04651) (0.78453,5.01624) (0.79576,4.98607) (0.80704,4.95599) (0.81838,4.92600) (0.82978,4.89610) (0.84124,4.86629) (0.85275,4.83657) (0.86432,4.80694) (0.87595,4.77741) (0.88764,4.74796) (0.89938,4.71861) (0.91119,4.68935) (0.92306,4.66018) (0.93498,4.63109) (0.94697,4.60211) (0.95902,4.57321) (0.97113,4.54440) (0.98330,4.51568) (0.99554,4.48706) (1.00784,4.45852) (1.02020,4.43008) (1.03263,4.40173) (1.04512,4.37347) (1.05767,4.34529) (1.07029,4.31722) (1.08298,4.28923) (1.09573,4.26133) (1.10855,4.23352) (1.12143,4.20581) (1.13439,4.17818) (1.14741,4.15065) (1.16050,4.12321) (1.17366,4.09586) (1.18689,4.06860) (1.20019,4.04143) (1.21355,4.01435) (1.22699,3.98736) (1.24051,3.96047) (1.25409,3.93366) (1.26775,3.90695) (1.28147,3.88032) (1.29528,3.85379) (1.30915,3.82735) (1.32311,3.80100) (1.33713,3.77474) (1.35123,3.74857) (1.36541,3.72250) (1.37967,3.69651) (1.39400,3.67062) (1.40841,3.64481) (1.42290,3.61910) (1.43747,3.59348) (1.45212,3.56795) (1.46685,3.54251) (1.48165,3.51716) (1.49654,3.49190) (1.51152,3.46673) (1.52657,3.44166) (1.54171,3.41667) (1.55693,3.39178) (1.57223,3.36698) (1.58762,3.34226) (1.60310,3.31764) (1.61866,3.29311) (1.63431,3.26868) (1.65005,3.24433) (1.66587,3.22007) (1.68178,3.19591) (1.69779,3.17183) (1.71388,3.14785) (1.73006,3.12396) (1.74634,3.10015) (1.76270,3.07644) (1.77916,3.05282) (1.79571,3.02930) (1.81236,3.00586) (1.82910,2.98251) (1.84593,2.95926) (1.86286,2.93609) (1.87989,2.91302) (1.89702,2.89004) (1.91424,2.86715) (1.93156,2.84435) (1.94898,2.82164) (1.96651,2.79902) (1.98413,2.77649) (2.00185,2.75405) (2.01968,2.73171) (2.03761,2.70946) (2.05564,2.68729) (2.07378,2.66522) (2.09202,2.64324) (2.11037,2.62135) (2.12883,2.59955) (2.14739,2.57784) (2.16606,2.55622) (2.18484,2.53470) (2.20373,2.51326) (2.22273,2.49192) (2.24185,2.47067) (2.26107,2.44950) (2.28041,2.42843) (2.29986,2.40745) (2.31943,2.38656) (2.33911,2.36577) (2.35890,2.34506) (2.37882,2.32444) (2.39885,2.30392) (2.41900,2.28349) (2.43927,2.26314) (2.45966,2.24289) (2.48017,2.22273) (2.50080,2.20266) (2.52156,2.18268) (2.54244,2.16279) (2.56344,2.14300) (2.58457,2.12329) (2.60582,2.10368) (2.62721,2.08415) (2.64871,2.06472) (2.67035,2.04538) (2.69212,2.02613) (2.71402,2.00697) (2.73604,1.98790) (2.75820,1.96892) (2.78049,1.95004) (2.80292,1.93124) (2.82548,1.91254) (2.84818,1.89392) (2.87101,1.87540) (2.89398,1.85697) (2.91708,1.83863) (2.94033,1.82038) (2.96371,1.80222) (2.98724,1.78416) (3.01090,1.76618) (3.03471,1.74829) (3.05866,1.73050) (3.08276,1.71280) (3.10700,1.69518) (3.13138,1.67766) (3.15591,1.66023) (3.18059,1.64289) (3.20542,1.62565) (3.23039,1.60849) (3.25551,1.59142) (3.28079,1.57445) (3.30621,1.55756) (3.33179,1.54077) (3.35752,1.52407) (3.38340,1.50746) (3.40944,1.49094) (3.43563,1.47451) (3.46198,1.45817) (3.48849,1.44192) (3.51515,1.42577) (3.54197,1.40970) (3.56895,1.39373) (3.59608,1.37785) (3.62338,1.36205) (3.65084,1.34635) (3.67846,1.33074) (3.70624,1.31523) (3.73419,1.29980) (3.76229,1.28446) (3.79056,1.26922) (3.81900,1.25406) (3.84760,1.23900) (3.87636,1.22402) (3.90530,1.20914) (3.93439,1.19435) (3.96366,1.17965) (3.99309,1.16504) (4.02269,1.15053) (4.05246,1.13610) (4.08239,1.12176) (4.11250,1.10752) (4.14277,1.09337) (4.17322,1.07930) (4.20383,1.06533) (4.23461,1.05145) (4.26557,1.03766) (4.29669,1.02396) (4.32798,1.01036) (4.35945,0.99684) (4.39108,0.98342) (4.42288,0.97008) (4.45486,0.95684) (4.48700,0.94369) (4.51932,0.93062) (4.55180,0.91765) (4.58445,0.90477) (4.61728,0.89199) (4.65027,0.87929) (4.68343,0.86668) (4.71675,0.85417) (4.75025,0.84174) (4.78391,0.82941) (4.81774,0.81717) (4.85173,0.80502) (4.88589,0.79296) (4.92021,0.78099) (4.95469,0.76911) (4.98934,0.75732) (5.02414,0.74563) (5.05911,0.73402) (5.09423,0.72251) (5.12951,0.71109) (5.16495,0.69975) (5.20054,0.68851) (5.23628,0.67736) (5.27217,0.66630) (5.30821,0.65534) (5.34440,0.64446) (5.38073,0.63367) (5.41721,0.62298) (5.45383,0.61238) (5.49058,0.60186) (5.52747,0.59144) (5.56450,0.58111) (5.60165,0.57087) (5.63894,0.56072) (5.67634,0.55067) (5.71387,0.54070) (5.75152,0.53082) (5.78929,0.52104) (5.82716,0.51135) (5.86515,0.50174) (5.90324,0.49223) (5.94143,0.48281) (5.97972,0.47348) (6.01810,0.46424) (6.05657,0.45510) (6.09512,0.44604) (6.13375,0.43707) (6.17246,0.42820) (6.21124,0.41942) (6.25008,0.41072) (6.28898,0.40212) (6.32793,0.39361) (6.36694,0.38519) (6.40598,0.37686) (6.44506,0.36863) (6.48417,0.36048) (6.52331,0.35243) (6.56246,0.34446) (6.60162,0.33659) (6.64079,0.32881) (6.67995,0.32112) (6.71910,0.31352) (6.75823,0.30601) (6.79734,0.29859) (6.83641,0.29126) (6.87544,0.28402) (6.91441,0.27688) (6.95333,0.26983) (6.99218,0.26286) (7.03096,0.25599) (7.06964,0.24921) (7.10824,0.24252) (7.14672,0.23592) (7.18510,0.22941) (7.22334,0.22300) (7.26146,0.21667) (7.29942,0.21044) (7.33724,0.20429) (7.37488,0.19824) (7.41235,0.19228) (7.44963,0.18641) (7.48671,0.18063) (7.52358,0.17494) (7.56022,0.16934) (7.59664,0.16383) (7.63280,0.15842) (7.66871,0.15309) (7.70434,0.14786) (7.73970,0.14272) (7.77475,0.13767) (7.80950,0.13271) (7.84393,0.12784) (7.87802,0.12306) (7.91177,0.11837) (7.94515,0.11377) (7.97817,0.10927) (8.01079,0.10485) (8.04302,0.10053) (8.07483,0.09630) (8.10622,0.09216) (8.13717,0.08811) (8.16766,0.08415) (8.19769,0.08028) (8.22724,0.07650) (8.25629,0.07282) (8.28483,0.06922) (8.31286,0.06572) (8.34035,0.06230) (8.36730,0.05898) (8.39368,0.05575) (8.41949,0.05261) (8.44472,0.04956) (8.46935,0.04660) (8.49337,0.04373) (8.51676,0.04096) (8.53952,0.03827) (8.56163,0.03568) (8.58308,0.03318) (8.60386,0.03076) (8.62396,0.02844) (8.64337,0.02621) (8.66207,0.02407) (8.68006,0.02203) (8.69732,0.02007) (8.71385,0.01820) (8.72964,0.01643) (8.74467,0.01475) (8.75895,0.01315) (8.77245,0.01165) (8.78518,0.01024) (8.79712,0.00892) (8.80827,0.00769) (8.81862,0.00655) (8.82816,0.00551) (8.83690,0.00455) (8.84482,0.00369) (8.85191,0.00291) (8.85819,0.00223) (8.86363,0.00164) (8.86825,0.00114) (8.87202,0.00073) (8.87497,0.00041) (8.87707,0.00018) (8.87833,0.00005) (8.87875,0.00000)}
\def\BitsLowerEq{(0.56272,5.23000) (0.57357,5.19919) (0.58447,5.16847) (0.59542,5.13784) (0.60643,5.10731) (0.61749,5.07686) (0.62861,5.04651) (0.63979,5.01624) (0.65102,4.98607) (0.66230,4.95599) (0.67364,4.92600) (0.68504,4.89610) (0.69650,4.86629) (0.70801,4.83657) (0.71958,4.80694) (0.73121,4.77741) (0.74290,4.74796) (0.75464,4.71861) (0.76645,4.68935) (0.77832,4.66018) (0.79024,4.63109) (0.80223,4.60211) (0.81428,4.57321) (0.82639,4.54440) (0.83856,4.51568) (0.85080,4.48706) (0.86310,4.45852) (0.87546,4.43008) (0.88789,4.40173) (0.90038,4.37347) (0.91293,4.34529) (0.92555,4.31722) (0.93824,4.28923) (0.95099,4.26133) (0.96381,4.23352) (0.97669,4.20581) (0.98965,4.17818) (1.00267,4.15065) (1.01576,4.12321) (1.02892,4.09586) (1.04215,4.06860) (1.05545,4.04143) (1.06881,4.01435) (1.08225,3.98736) (1.09577,3.96047) (1.10935,3.93366) (1.12301,3.90695) (1.13673,3.88032) (1.15054,3.85379) (1.16441,3.82735) (1.17837,3.80100) (1.19239,3.77474) (1.20649,3.74857) (1.22067,3.72250) (1.23493,3.69651) (1.24926,3.67062) (1.26367,3.64481) (1.27816,3.61910) (1.29273,3.59348) (1.30738,3.56795) (1.32211,3.54251) (1.33691,3.51716) (1.35180,3.49190) (1.36678,3.46673) (1.38183,3.44166) (1.39697,3.41667) (1.41219,3.39178) (1.42749,3.36698) (1.44288,3.34226) (1.45836,3.31764) (1.47392,3.29311) (1.48957,3.26868) (1.50531,3.24433) (1.52113,3.22007) (1.53704,3.19591) (1.55305,3.17183) (1.56914,3.14785) (1.58532,3.12396) (1.60160,3.10015) (1.61796,3.07644) (1.63442,3.05282) (1.65097,3.02930) (1.66762,3.00586) (1.68436,2.98251) (1.70119,2.95926) (1.71812,2.93609) (1.73515,2.91302) (1.75228,2.89004) (1.76950,2.86715) (1.78682,2.84435) (1.80424,2.82164) (1.82177,2.79902) (1.83939,2.77649) (1.85711,2.75405) (1.87494,2.73171) (1.89287,2.70946) (1.91090,2.68729) (1.92904,2.66522) (1.94728,2.64324) (1.96563,2.62135) (1.98409,2.59955) (2.00265,2.57784) (2.02132,2.55622) (2.04010,2.53470) (2.05899,2.51326) (2.07799,2.49192) (2.09711,2.47067) (2.11633,2.44950) (2.13567,2.42843) (2.15512,2.40745) (2.17469,2.38656) (2.19437,2.36577) (2.21416,2.34506) (2.23408,2.32444) (2.25411,2.30392) (2.27426,2.28349) (2.29453,2.26314) (2.31492,2.24289) (2.33543,2.22273) (2.35606,2.20266) (2.37682,2.18268) (2.39770,2.16279) (2.41870,2.14300) (2.43983,2.12329) (2.46108,2.10368) (2.48247,2.08415) (2.50397,2.06472) (2.52561,2.04538) (2.54738,2.02613) (2.56928,2.00697) (2.59130,1.98790) (2.61346,1.96892) (2.63575,1.95004) (2.65818,1.93124) (2.68074,1.91254) (2.70344,1.89392) (2.72627,1.87540) (2.74924,1.85697) (2.77234,1.83863) (2.79559,1.82038) (2.81897,1.80222) (2.84250,1.78416) (2.86616,1.76618) (2.88997,1.74829) (2.91392,1.73050) (2.93802,1.71280) (2.96226,1.69518) (2.98664,1.67766) (3.01117,1.66023) (3.03585,1.64289) (3.06068,1.62565) (3.08565,1.60849) (3.11077,1.59142) (3.13605,1.57445) (3.16147,1.55756) (3.18705,1.54077) (3.21278,1.52407) (3.23866,1.50746) (3.26470,1.49094) (3.29089,1.47451) (3.31724,1.45817) (3.34375,1.44192) (3.37041,1.42577) (3.39723,1.40970) (3.42421,1.39373) (3.45134,1.37785) (3.47864,1.36205) (3.50610,1.34635) (3.53372,1.33074) (3.56150,1.31523) (3.58945,1.29980) (3.61755,1.28446) (3.64582,1.26922) (3.67426,1.25406) (3.70286,1.23900) (3.73162,1.22402) (3.76056,1.20914) (3.78965,1.19435) (3.81892,1.17965) (3.84835,1.16504) (3.87795,1.15053) (3.90772,1.13610) (3.93765,1.12176) (3.96776,1.10752) (3.99803,1.09337) (4.02848,1.07930) (4.05909,1.06533) (4.08987,1.05145) (4.12083,1.03766) (4.15195,1.02396) (4.18324,1.01036) (4.21471,0.99684) (4.24634,0.98342) (4.27814,0.97008) (4.31012,0.95684) (4.34226,0.94369) (4.37458,0.93062) (4.40706,0.91765) (4.43971,0.90477) (4.47254,0.89199) (4.50553,0.87929) (4.53869,0.86668) (4.57201,0.85417) (4.60551,0.84174) (4.63917,0.82941) (4.67300,0.81717) (4.70699,0.80502) (4.74115,0.79296) (4.77547,0.78099) (4.80995,0.76911) (4.84460,0.75732) (4.87940,0.74563) (4.91437,0.73402) (4.94949,0.72251) (4.98477,0.71109) (5.02021,0.69975) (5.05580,0.68851) (5.09154,0.67736) (5.12743,0.66630) (5.16347,0.65534) (5.19966,0.64446) (5.23599,0.63367) (5.27247,0.62298) (5.30909,0.61238) (5.34584,0.60186) (5.38273,0.59144) (5.41976,0.58111) (5.45691,0.57087) (5.49420,0.56072) (5.53160,0.55067) (5.56913,0.54070) (5.60678,0.53082) (5.64455,0.52104) (5.68242,0.51135) (5.72041,0.50174) (5.75850,0.49223) (5.79669,0.48281) (5.83498,0.47348) (5.87336,0.46424) (5.91183,0.45510) (5.95038,0.44604) (5.98901,0.43707) (6.02772,0.42820) (6.06650,0.41942) (6.10534,0.41072) (6.14424,0.40212) (6.18319,0.39361) (6.22220,0.38519) (6.26124,0.37686) (6.30032,0.36863) (6.33943,0.36048) (6.37857,0.35243) (6.41772,0.34446) (6.45688,0.33659) (6.49605,0.32881) (6.53521,0.32112) (6.57436,0.31352) (6.61349,0.30601) (6.65260,0.29859) (6.69167,0.29126) (6.73070,0.28402) (6.76967,0.27688) (6.80859,0.26983) (6.84744,0.26286) (6.88622,0.25599) (6.92490,0.24921) (6.96350,0.24252) (7.00198,0.23592) (7.04036,0.22941) (7.07860,0.22300) (7.11672,0.21667) (7.15468,0.21044) (7.19250,0.20429) (7.23014,0.19824) (7.26761,0.19228) (7.30489,0.18641) (7.34197,0.18063) (7.37884,0.17494) (7.41548,0.16934) (7.45190,0.16383) (7.48806,0.15842) (7.52397,0.15309) (7.55960,0.14786) (7.59496,0.14272) (7.63001,0.13767) (7.66476,0.13271) (7.69919,0.12784) (7.73328,0.12306) (7.76703,0.11837) (7.80041,0.11377) (7.83343,0.10927) (7.86605,0.10485) (7.89828,0.10053) (7.93009,0.09630) (7.96148,0.09216) (7.99243,0.08811) (8.02292,0.08415) (8.05295,0.08028) (8.08250,0.07650) (8.11155,0.07282) (8.14009,0.06922) (8.16812,0.06572) (8.19561,0.06230) (8.22256,0.05898) (8.24894,0.05575) (8.27475,0.05261) (8.29998,0.04956) (8.32461,0.04660) (8.34863,0.04373) (8.37202,0.04096) (8.39478,0.03827) (8.41689,0.03568) (8.43834,0.03318) (8.45912,0.03076) (8.47922,0.02844) (8.49863,0.02621) (8.51733,0.02407) (8.53532,0.02203) (8.55258,0.02007) (8.56911,0.01820) (8.58490,0.01643) (8.59993,0.01475) (8.61421,0.01315) (8.62771,0.01165) (8.64044,0.01024) (8.65238,0.00892) (8.66353,0.00769) (8.67388,0.00655) (8.68342,0.00551) (8.69216,0.00455) (8.70008,0.00369) (8.70717,0.00291) (8.71345,0.00223) (8.71889,0.00164) (8.72351,0.00114) (8.72728,0.00073) (8.73023,0.00041) (8.73233,0.00018) (8.73359,0.00005) (8.73401,0.00000)}
\def\BitsBandEq{(0.70746,5.23000) (0.71831,5.19919) (0.72921,5.16847) (0.74016,5.13784) (0.75117,5.10731) (0.76223,5.07686) (0.77335,5.04651) (0.78453,5.01624) (0.79576,4.98607) (0.80704,4.95599) (0.81838,4.92600) (0.82978,4.89610) (0.84124,4.86629) (0.85275,4.83657) (0.86432,4.80694) (0.87595,4.77741) (0.88764,4.74796) (0.89938,4.71861) (0.91119,4.68935) (0.92306,4.66018) (0.93498,4.63109) (0.94697,4.60211) (0.95902,4.57321) (0.97113,4.54440) (0.98330,4.51568) (0.99554,4.48706) (1.00784,4.45852) (1.02020,4.43008) (1.03263,4.40173) (1.04512,4.37347) (1.05767,4.34529) (1.07029,4.31722) (1.08298,4.28923) (1.09573,4.26133) (1.10855,4.23352) (1.12143,4.20581) (1.13439,4.17818) (1.14741,4.15065) (1.16050,4.12321) (1.17366,4.09586) (1.18689,4.06860) (1.20019,4.04143) (1.21355,4.01435) (1.22699,3.98736) (1.24051,3.96047) (1.25409,3.93366) (1.26775,3.90695) (1.28147,3.88032) (1.29528,3.85379) (1.30915,3.82735) (1.32311,3.80100) (1.33713,3.77474) (1.35123,3.74857) (1.36541,3.72250) (1.37967,3.69651) (1.39400,3.67062) (1.40841,3.64481) (1.42290,3.61910) (1.43747,3.59348) (1.45212,3.56795) (1.46685,3.54251) (1.48165,3.51716) (1.49654,3.49190) (1.51152,3.46673) (1.52657,3.44166) (1.54171,3.41667) (1.55693,3.39178) (1.57223,3.36698) (1.58762,3.34226) (1.60310,3.31764) (1.61866,3.29311) (1.63431,3.26868) (1.65005,3.24433) (1.66587,3.22007) (1.68178,3.19591) (1.69779,3.17183) (1.71388,3.14785) (1.73006,3.12396) (1.74634,3.10015) (1.76270,3.07644) (1.77916,3.05282) (1.79571,3.02930) (1.81236,3.00586) (1.82910,2.98251) (1.84593,2.95926) (1.86286,2.93609) (1.87989,2.91302) (1.89702,2.89004) (1.91424,2.86715) (1.93156,2.84435) (1.94898,2.82164) (1.96651,2.79902) (1.98413,2.77649) (2.00185,2.75405) (2.01968,2.73171) (2.03761,2.70946) (2.05564,2.68729) (2.07378,2.66522) (2.09202,2.64324) (2.11037,2.62135) (2.12883,2.59955) (2.14739,2.57784) (2.16606,2.55622) (2.18484,2.53470) (2.20373,2.51326) (2.22273,2.49192) (2.24185,2.47067) (2.26107,2.44950) (2.28041,2.42843) (2.29986,2.40745) (2.31943,2.38656) (2.33911,2.36577) (2.35890,2.34506) (2.37882,2.32444) (2.39885,2.30392) (2.41900,2.28349) (2.43927,2.26314) (2.45966,2.24289) (2.48017,2.22273) (2.50080,2.20266) (2.52156,2.18268) (2.54244,2.16279) (2.56344,2.14300) (2.58457,2.12329) (2.60582,2.10368) (2.62721,2.08415) (2.64871,2.06472) (2.67035,2.04538) (2.69212,2.02613) (2.71402,2.00697) (2.73604,1.98790) (2.75820,1.96892) (2.78049,1.95004) (2.80292,1.93124) (2.82548,1.91254) (2.84818,1.89392) (2.87101,1.87540) (2.89398,1.85697) (2.91708,1.83863) (2.94033,1.82038) (2.96371,1.80222) (2.98724,1.78416) (3.01090,1.76618) (3.03471,1.74829) (3.05866,1.73050) (3.08276,1.71280) (3.10700,1.69518) (3.13138,1.67766) (3.15591,1.66023) (3.18059,1.64289) (3.20542,1.62565) (3.23039,1.60849) (3.25551,1.59142) (3.28079,1.57445) (3.30621,1.55756) (3.33179,1.54077) (3.35752,1.52407) (3.38340,1.50746) (3.40944,1.49094) (3.43563,1.47451) (3.46198,1.45817) (3.48849,1.44192) (3.51515,1.42577) (3.54197,1.40970) (3.56895,1.39373) (3.59608,1.37785) (3.62338,1.36205) (3.65084,1.34635) (3.67846,1.33074) (3.70624,1.31523) (3.73419,1.29980) (3.76229,1.28446) (3.79056,1.26922) (3.81900,1.25406) (3.84760,1.23900) (3.87636,1.22402) (3.90530,1.20914) (3.93439,1.19435) (3.96366,1.17965) (3.99309,1.16504) (4.02269,1.15053) (4.05246,1.13610) (4.08239,1.12176) (4.11250,1.10752) (4.14277,1.09337) (4.17322,1.07930) (4.20383,1.06533) (4.23461,1.05145) (4.26557,1.03766) (4.29669,1.02396) (4.32798,1.01036) (4.35945,0.99684) (4.39108,0.98342) (4.42288,0.97008) (4.45486,0.95684) (4.48700,0.94369) (4.51932,0.93062) (4.55180,0.91765) (4.58445,0.90477) (4.61728,0.89199) (4.65027,0.87929) (4.68343,0.86668) (4.71675,0.85417) (4.75025,0.84174) (4.78391,0.82941) (4.81774,0.81717) (4.85173,0.80502) (4.88589,0.79296) (4.92021,0.78099) (4.95469,0.76911) (4.98934,0.75732) (5.02414,0.74563) (5.05911,0.73402) (5.09423,0.72251) (5.12951,0.71109) (5.16495,0.69975) (5.20054,0.68851) (5.23628,0.67736) (5.27217,0.66630) (5.30821,0.65534) (5.34440,0.64446) (5.38073,0.63367) (5.41721,0.62298) (5.45383,0.61238) (5.49058,0.60186) (5.52747,0.59144) (5.56450,0.58111) (5.60165,0.57087) (5.63894,0.56072) (5.67634,0.55067) (5.71387,0.54070) (5.75152,0.53082) (5.78929,0.52104) (5.82716,0.51135) (5.86515,0.50174) (5.90324,0.49223) (5.94143,0.48281) (5.97972,0.47348) (6.01810,0.46424) (6.05657,0.45510) (6.09512,0.44604) (6.13375,0.43707) (6.17246,0.42820) (6.21124,0.41942) (6.25008,0.41072) (6.28898,0.40212) (6.32793,0.39361) (6.36694,0.38519) (6.40598,0.37686) (6.44506,0.36863) (6.48417,0.36048) (6.52331,0.35243) (6.56246,0.34446) (6.60162,0.33659) (6.64079,0.32881) (6.67995,0.32112) (6.71910,0.31352) (6.75823,0.30601) (6.79734,0.29859) (6.83641,0.29126) (6.87544,0.28402) (6.91441,0.27688) (6.95333,0.26983) (6.99218,0.26286) (7.03096,0.25599) (7.06964,0.24921) (7.10824,0.24252) (7.14672,0.23592) (7.18510,0.22941) (7.22334,0.22300) (7.26146,0.21667) (7.29942,0.21044) (7.33724,0.20429) (7.37488,0.19824) (7.41235,0.19228) (7.44963,0.18641) (7.48671,0.18063) (7.52358,0.17494) (7.56022,0.16934) (7.59664,0.16383) (7.63280,0.15842) (7.66871,0.15309) (7.70434,0.14786) (7.73970,0.14272) (7.77475,0.13767) (7.80950,0.13271) (7.84393,0.12784) (7.87802,0.12306) (7.91177,0.11837) (7.94515,0.11377) (7.97817,0.10927) (8.01079,0.10485) (8.04302,0.10053) (8.07483,0.09630) (8.10622,0.09216) (8.13717,0.08811) (8.16766,0.08415) (8.19769,0.08028) (8.22724,0.07650) (8.25629,0.07282) (8.28483,0.06922) (8.31286,0.06572) (8.34035,0.06230) (8.36730,0.05898) (8.39368,0.05575) (8.41949,0.05261) (8.44472,0.04956) (8.46935,0.04660) (8.49337,0.04373) (8.51676,0.04096) (8.53952,0.03827) (8.56163,0.03568) (8.58308,0.03318) (8.60386,0.03076) (8.62396,0.02844) (8.64337,0.02621) (8.66207,0.02407) (8.68006,0.02203) (8.69732,0.02007) (8.71385,0.01820) (8.72964,0.01643) (8.74467,0.01475) (8.75895,0.01315) (8.77245,0.01165) (8.78518,0.01024) (8.79712,0.00892) (8.80827,0.00769) (8.81862,0.00655) (8.82816,0.00551) (8.83690,0.00455) (8.84482,0.00369) (8.85191,0.00291) (8.85819,0.00223) (8.86363,0.00164) (8.86825,0.00114) (8.87202,0.00073) (8.87497,0.00041) (8.87707,0.00018) (8.87833,0.00005) (8.87875,0.00000) (8.73401,0.00000) (8.73359,0.00005) (8.73233,0.00018) (8.73023,0.00041) (8.72728,0.00073) (8.72351,0.00114) (8.71889,0.00164) (8.71345,0.00223) (8.70717,0.00291) (8.70008,0.00369) (8.69216,0.00455) (8.68342,0.00551) (8.67388,0.00655) (8.66353,0.00769) (8.65238,0.00892) (8.64044,0.01024) (8.62771,0.01165) (8.61421,0.01315) (8.59993,0.01475) (8.58490,0.01643) (8.56911,0.01820) (8.55258,0.02007) (8.53532,0.02203) (8.51733,0.02407) (8.49863,0.02621) (8.47922,0.02844) (8.45912,0.03076) (8.43834,0.03318) (8.41689,0.03568) (8.39478,0.03827) (8.37202,0.04096) (8.34863,0.04373) (8.32461,0.04660) (8.29998,0.04956) (8.27475,0.05261) (8.24894,0.05575) (8.22256,0.05898) (8.19561,0.06230) (8.16812,0.06572) (8.14009,0.06922) (8.11155,0.07282) (8.08250,0.07650) (8.05295,0.08028) (8.02292,0.08415) (7.99243,0.08811) (7.96148,0.09216) (7.93009,0.09630) (7.89828,0.10053) (7.86605,0.10485) (7.83343,0.10927) (7.80041,0.11377) (7.76703,0.11837) (7.73328,0.12306) (7.69919,0.12784) (7.66476,0.13271) (7.63001,0.13767) (7.59496,0.14272) (7.55960,0.14786) (7.52397,0.15309) (7.48806,0.15842) (7.45190,0.16383) (7.41548,0.16934) (7.37884,0.17494) (7.34197,0.18063) (7.30489,0.18641) (7.26761,0.19228) (7.23014,0.19824) (7.19250,0.20429) (7.15468,0.21044) (7.11672,0.21667) (7.07860,0.22300) (7.04036,0.22941) (7.00198,0.23592) (6.96350,0.24252) (6.92490,0.24921) (6.88622,0.25599) (6.84744,0.26286) (6.80859,0.26983) (6.76967,0.27688) (6.73070,0.28402) (6.69167,0.29126) (6.65260,0.29859) (6.61349,0.30601) (6.57436,0.31352) (6.53521,0.32112) (6.49605,0.32881) (6.45688,0.33659) (6.41772,0.34446) (6.37857,0.35243) (6.33943,0.36048) (6.30032,0.36863) (6.26124,0.37686) (6.22220,0.38519) (6.18319,0.39361) (6.14424,0.40212) (6.10534,0.41072) (6.06650,0.41942) (6.02772,0.42820) (5.98901,0.43707) (5.95038,0.44604) (5.91183,0.45510) (5.87336,0.46424) (5.83498,0.47348) (5.79669,0.48281) (5.75850,0.49223) (5.72041,0.50174) (5.68242,0.51135) (5.64455,0.52104) (5.60678,0.53082) (5.56913,0.54070) (5.53160,0.55067) (5.49420,0.56072) (5.45691,0.57087) (5.41976,0.58111) (5.38273,0.59144) (5.34584,0.60186) (5.30909,0.61238) (5.27247,0.62298) (5.23599,0.63367) (5.19966,0.64446) (5.16347,0.65534) (5.12743,0.66630) (5.09154,0.67736) (5.05580,0.68851) (5.02021,0.69975) (4.98477,0.71109) (4.94949,0.72251) (4.91437,0.73402) (4.87940,0.74563) (4.84460,0.75732) (4.80995,0.76911) (4.77547,0.78099) (4.74115,0.79296) (4.70699,0.80502) (4.67300,0.81717) (4.63917,0.82941) (4.60551,0.84174) (4.57201,0.85417) (4.53869,0.86668) (4.50553,0.87929) (4.47254,0.89199) (4.43971,0.90477) (4.40706,0.91765) (4.37458,0.93062) (4.34226,0.94369) (4.31012,0.95684) (4.27814,0.97008) (4.24634,0.98342) (4.21471,0.99684) (4.18324,1.01036) (4.15195,1.02396) (4.12083,1.03766) (4.08987,1.05145) (4.05909,1.06533) (4.02848,1.07930) (3.99803,1.09337) (3.96776,1.10752) (3.93765,1.12176) (3.90772,1.13610) (3.87795,1.15053) (3.84835,1.16504) (3.81892,1.17965) (3.78965,1.19435) (3.76056,1.20914) (3.73162,1.22402) (3.70286,1.23900) (3.67426,1.25406) (3.64582,1.26922) (3.61755,1.28446) (3.58945,1.29980) (3.56150,1.31523) (3.53372,1.33074) (3.50610,1.34635) (3.47864,1.36205) (3.45134,1.37785) (3.42421,1.39373) (3.39723,1.40970) (3.37041,1.42577) (3.34375,1.44192) (3.31724,1.45817) (3.29089,1.47451) (3.26470,1.49094) (3.23866,1.50746) (3.21278,1.52407) (3.18705,1.54077) (3.16147,1.55756) (3.13605,1.57445) (3.11077,1.59142) (3.08565,1.60849) (3.06068,1.62565) (3.03585,1.64289) (3.01117,1.66023) (2.98664,1.67766) (2.96226,1.69518) (2.93802,1.71280) (2.91392,1.73050) (2.88997,1.74829) (2.86616,1.76618) (2.84250,1.78416) (2.81897,1.80222) (2.79559,1.82038) (2.77234,1.83863) (2.74924,1.85697) (2.72627,1.87540) (2.70344,1.89392) (2.68074,1.91254) (2.65818,1.93124) (2.63575,1.95004) (2.61346,1.96892) (2.59130,1.98790) (2.56928,2.00697) (2.54738,2.02613) (2.52561,2.04538) (2.50397,2.06472) (2.48247,2.08415) (2.46108,2.10368) (2.43983,2.12329) (2.41870,2.14300) (2.39770,2.16279) (2.37682,2.18268) (2.35606,2.20266) (2.33543,2.22273) (2.31492,2.24289) (2.29453,2.26314) (2.27426,2.28349) (2.25411,2.30392) (2.23408,2.32444) (2.21416,2.34506) (2.19437,2.36577) (2.17469,2.38656) (2.15512,2.40745) (2.13567,2.42843) (2.11633,2.44950) (2.09711,2.47067) (2.07799,2.49192) (2.05899,2.51326) (2.04010,2.53470) (2.02132,2.55622) (2.00265,2.57784) (1.98409,2.59955) (1.96563,2.62135) (1.94728,2.64324) (1.92904,2.66522) (1.91090,2.68729) (1.89287,2.70946) (1.87494,2.73171) (1.85711,2.75405) (1.83939,2.77649) (1.82177,2.79902) (1.80424,2.82164) (1.78682,2.84435) (1.76950,2.86715) (1.75228,2.89004) (1.73515,2.91302) (1.71812,2.93609) (1.70119,2.95926) (1.68436,2.98251) (1.66762,3.00586) (1.65097,3.02930) (1.63442,3.05282) (1.61796,3.07644) (1.60160,3.10015) (1.58532,3.12396) (1.56914,3.14785) (1.55305,3.17183) (1.53704,3.19591) (1.52113,3.22007) (1.50531,3.24433) (1.48957,3.26868) (1.47392,3.29311) (1.45836,3.31764) (1.44288,3.34226) (1.42749,3.36698) (1.41219,3.39178) (1.39697,3.41667) (1.38183,3.44166) (1.36678,3.46673) (1.35180,3.49190) (1.33691,3.51716) (1.32211,3.54251) (1.30738,3.56795) (1.29273,3.59348) (1.27816,3.61910) (1.26367,3.64481) (1.24926,3.67062) (1.23493,3.69651) (1.22067,3.72250) (1.20649,3.74857) (1.19239,3.77474) (1.17837,3.80100) (1.16441,3.82735) (1.15054,3.85379) (1.13673,3.88032) (1.12301,3.90695) (1.10935,3.93366) (1.09577,3.96047) (1.08225,3.98736) (1.06881,4.01435) (1.05545,4.04143) (1.04215,4.06860) (1.02892,4.09586) (1.01576,4.12321) (1.00267,4.15065) (0.98965,4.17818) (0.97669,4.20581) (0.96381,4.23352) (0.95099,4.26133) (0.93824,4.28923) (0.92555,4.31722) (0.91293,4.34529) (0.90038,4.37347) (0.88789,4.40173) (0.87546,4.43008) (0.86310,4.45852) (0.85080,4.48706) (0.83856,4.51568) (0.82639,4.54440) (0.81428,4.57321) (0.80223,4.60211) (0.79024,4.63109) (0.77832,4.66018) (0.76645,4.68935) (0.75464,4.71861) (0.74290,4.74796) (0.73121,4.77741) (0.71958,4.80694) (0.70801,4.83657) (0.69650,4.86629) (0.68504,4.89610) (0.67364,4.92600) (0.66230,4.95599) (0.65102,4.98607) (0.63979,5.01624) (0.62861,5.04651) (0.61749,5.07686) (0.60643,5.10731) (0.59542,5.13784) (0.58447,5.16847) (0.57357,5.19919) (0.56272,5.23000)}
\begin{scope}
\clip (0,0) rectangle (9.8,5.23);
\fill[phaseblue] (0,0) rectangle (9.8,5.23);
\fill[phasered] (0,0) -- (0,5.23)
  -- plot coordinates {\BitsLower} -- cycle;
\fill[phasegold] (8.85629,0) rectangle (9.8,5.23);
\fill[phasegoldedge] (8.75647,0) rectangle (8.85629,5.23);
\fill[phasegray] (0,0) rectangle (0.68055,5.23);
\fill[phasegrayedge] (0.58981,0) rectangle (0.68055,5.23);
\begin{scope}
\clip (0.68055,0) rectangle (8.85629,5.23);
\fill[phasegreen] plot coordinates {\BitsBandEq} -- cycle;
\draw[phasegreenink,dash pattern=on 4pt off 2pt,line width=.85pt] plot coordinates {\BitsUpperEq};
\draw[phasegreenink,dash pattern=on 4pt off 2pt,line width=.85pt] plot coordinates {\BitsLowerEq};
\end{scope}
\begin{scope}
\clip (0,0) rectangle (0.68055,5.23);
\fill[phasegraystrip] plot coordinates {\BitsBandEq} -- cycle;
\draw[black!45,dash pattern=on 4pt off 2pt,line width=.55pt] plot coordinates {\BitsUpperEq};
\draw[black!45,dash pattern=on 4pt off 2pt,line width=.55pt] plot coordinates {\BitsLowerEq};
\end{scope}
\end{scope}
\draw[-{Stealth[length=1.7mm,width=1.1mm]},line width=.55pt] (0,0)--(9.94,0);
\draw[-{Stealth[length=1.7mm,width=1.1mm]},line width=.55pt] (0,0)--(0,5.40);
\node[anchor=south east,inner sep=1pt] at (0,5.48) {$B$};
\node[anchor=west,inner sep=1pt] at (10.00,0) {$N$};
\node[anchor=north east,inner sep=1pt] at (-.12,-.12) {$2$};
\node[anchor=north,inner sep=1pt] at (0.68055,-.14) {$\Theta(d)$};
\node[anchor=north,inner sep=1pt] at (8.80647,-.16)
  {$\displaystyle\Theta\left(M^{\frac{d}{4\beta+2d}}\right)$};
\node[rotate=90,text=black!60] at (0.34028,2.61500) {not admissible};
\node[rotate=90,text=phasegoldtext] at (9.32815,2.61500) {width too large};
\node[text=phasebluetext] at (5.65,3.84) {precision too high};
\node[anchor=west,text=phaseredtext,inner sep=0pt] at (.81,.19) {precision too low};
\node[fill=white,draw=phasegreenink,line width=.45pt,text=phasegreenink,inner xsep=4pt,inner ysep=4pt]
  at (5.70,2.05) {$N^2B\asymp M^{d/(2\beta+d)}$};
\pgfresetboundingbox
\path[use as bounding box] (-1.10,-.85) rectangle (11.20,5.95);
\end{tikzpicture}

%% file: main.bbl
\begin{thebibliography}{53}
\providecommand{\natexlab}[1]{#1}
\providecommand{\url}[1]{\texttt{#1}}
\expandafter\ifx\csname urlstyle\endcsname\relax
  \providecommand{\doi}[1]{doi: #1}\else
  \providecommand{\doi}{doi: \begingroup \urlstyle{rm}\Url}\fi

\bibitem[Anthony and Bartlett(1999)]{AnthonyBartlett1999}
Anthony, M., and Bartlett, P.~L. (1999), \emph{Neural Network Learning:
  Theoretical Foundations}, Cambridge: Cambridge University Press, DOI:
  \url{https://doi.org/10.1017/CBO9780511624216}.

\bibitem[Arnol'd(1957)]{Arnold1957}
Arnol'd, V.~I. (1957), ``On Functions of Three Variables,'' \emph{Doklady
  Akademii Nauk SSSR}, 114(4), 679--681, available at
  \url{https://www.mathnet.ru/eng/dan22002}.

\bibitem[Barron(1993)]{Barron1993}
Barron, A.~R. (1993), ``Universal Approximation Bounds for Superpositions of a
  Sigmoidal Function,'' \emph{IEEE Transactions on Information Theory}, 39(3),
  930--945, DOI: \url{https://doi.org/10.1109/18.256500}.

\bibitem[Bartlett et~al.(2019)Bartlett, Harvey, Liaw, and
  Mehrabian]{BartlettEtAl2019}
Bartlett, P.~L., Harvey, N., Liaw, C., and Mehrabian, A. (2019), ``Nearly-Tight
  {VC}-Dimension and Pseudodimension Bounds for Piecewise Linear Neural
  Networks,'' \emph{Journal of Machine Learning Research}, 20(63), 1--17,
  available at \url{https://jmlr.org/papers/v20/17-612.html}.

\bibitem[Bauer and Kohler(2019)]{BauerKohler2019}
Bauer, B., and Kohler, M. (2019), ``On Deep Learning as a Remedy for the Curse
  of Dimensionality in Nonparametric Regression,'' \emph{The Annals of
  Statistics}, 47(4), 2261--2285, DOI:
  \url{https://doi.org/10.1214/18-AOS1747}.

\bibitem[Beknazaryan(2022)]{Beknazaryan2022}
Beknazaryan, A. (2022), ``Neural Networks With Superexpressive Activations and
  Integer Weights,'' in \emph{Intelligent Computing: Proceedings of the 2022
  Computing Conference, Volume 2}, ed. K.~Arai, \emph{Lecture Notes in Networks
  and Systems} (Vol. 507), Cham: Springer, pp. 445--451, DOI:
  \url{https://doi.org/10.1007/978-3-031-10464-0_30}.

\bibitem[Beknazaryan(2026)]{Beknazaryan2026Ternary}
\rule[.5ex]{2em}{.4pt} (2026), ``Function Approximation and Nonparametric
  Regression with Binary and Ternary {ReLU} Networks,'' \emph{arXiv preprint
  arXiv:2206.13280}, Version 4, available at
  \url{https://arxiv.org/abs/2206.13280v4}.

\bibitem[Cybenko(1989)]{Cybenko1989}
Cybenko, G. (1989), ``Approximation by Superpositions of a Sigmoidal
  Function,'' \emph{Mathematics of Control, Signals, and Systems}, 2(4),
  303--314, DOI: \url{https://doi.org/10.1007/BF02551274}.

\bibitem[Fan et~al.(2026)Fan, Li, Wang, and Wang]{FanLiWangWang2026}
Fan, F.-L., Li, Z.-Y., Wang, C.-Y., and Wang, J.-J. (2026), ``On Explicit
  Super-Expressive Approximation for Neural Networks,'' arXiv:2607.06781,
  available at \url{https://arxiv.org/abs/2607.06781}.

\bibitem[Fan and Gu(2024)]{FanGu2024}
Fan, J., and Gu, Y. (2024), ``Factor Augmented Sparse Throughput Deep {ReLU}
  Neural Networks for High Dimensional Regression,'' \emph{Journal of the
  American Statistical Association}, 119(548), 2680--2694, DOI:
  \url{https://doi.org/10.1080/01621459.2023.2271605}.

\bibitem[Fukshansky and Moshchevitin(2018)]{FukshanskyMoshchevitin2018}
Fukshansky, L., and Moshchevitin, N. (2018), ``On an Effective Variation of
  {Kronecker}'s Approximation Theorem Avoiding Algebraic Sets,''
  \emph{Proceedings of the American Mathematical Society}, 146(10), 4151--4163,
  DOI: \url{https://doi.org/10.1090/proc/14110}.

\bibitem[Gonek and Montgomery(2016)]{GonekMontgomery2016}
Gonek, S.~M., and Montgomery, H.~L. (2016), ``{Kronecker}'s Approximation
  Theorem,'' \emph{Indagationes Mathematicae}, 27(2), 506--523, DOI:
  \url{https://doi.org/10.1016/j.indag.2016.02.002}.

\bibitem[Guliyev and Ismailov(2018)]{GuliyevIsmailov2018}
Guliyev, N.~J., and Ismailov, V.~E. (2018), ``Approximation Capability of Two
  Hidden Layer Feedforward Neural Networks With Fixed Weights,''
  \emph{Neurocomputing}, 316, 262--269, DOI:
  \url{https://doi.org/10.1016/j.neucom.2018.07.075}.

\bibitem[Gy{\"o}rfi et~al.(2002)Gy{\"o}rfi, Kohler, Krzy{\.z}ak, and
  Walk]{GyorfiEtAl2002}
Gy{\"o}rfi, L., Kohler, M., Krzy{\.z}ak, A., and Walk, H. (2002), \emph{A
  Distribution-Free Theory of Nonparametric Regression}, \emph{Springer Series
  in Statistics}, New York: Springer, DOI:
  \url{https://doi.org/10.1007/b97848}.

\bibitem[Hornik(1991)]{Hornik1991}
Hornik, K. (1991), ``Approximation Capabilities of Multilayer Feedforward
  Networks,'' \emph{Neural Networks}, 4(2), 251--257, DOI:
  \url{https://doi.org/10.1016/0893-6080(91)90009-T}.

\bibitem[Jiao et~al.(2023)Jiao, Lai, Lu, Wang, Yang, and
  Yang]{doi:10.1137/21M144431X}
Jiao, Y., Lai, Y., Lu, X., Wang, F., Yang, J.~Z., and Yang, Y. (2023), ``Deep
  Neural Networks with {ReLU-Sine-Exponential} Activations Break Curse of
  Dimensionality in Approximation on {H\"older} Class,'' \emph{SIAM Journal on
  Mathematical Analysis}, 55(4), 3635--3649, DOI:
  \url{https://doi.org/10.1137/21M144431X}.

\bibitem[Jiao et~al.(2026)Jiao, Lai, Wang, and Yan]{JiaoLaiWangYan2026}
Jiao, Y., Lai, Y., Wang, Y., and Yan, B. (2026), ``Transformers Can Overcome
  the Curse of Dimensionality: A Theoretical Study From an Approximation
  Perspective,'' \emph{Journal of Machine Learning Research}, 27(50), 1--34,
  available at \url{https://jmlr.org/papers/v27/25-1214.html}.

\bibitem[Kanade et~al.(2024)Kanade, Rebeschini, and
  Va{\v{s}}kevi{\v{c}}ius]{KanadeEtAl2024}
Kanade, V., Rebeschini, P., and Va{\v{s}}kevi{\v{c}}ius, T. (2024),
  ``Exponential Tail Local {Rademacher} Complexity Risk Bounds Without the
  {Bernstein} Condition,'' \emph{Journal of Machine Learning Research},
  25(388), 1--43, available at \url{https://jmlr.org/papers/v25/23-0063.html}.

\bibitem[Kerkyacharian and Picard(2003)]{KerkyacharianPicard2003}
Kerkyacharian, G., and Picard, D. (2003), ``Replicant Compression Coding in
  {Besov} Spaces,'' \emph{ESAIM: Probability and Statistics}, 7, 239--250, DOI:
  \url{https://doi.org/10.1051/ps:2003011}.

\bibitem[Kohler and Langer(2021)]{KohlerLanger2021}
Kohler, M., and Langer, S. (2021), ``On the Rate of Convergence of Fully
  Connected Deep Neural Network Regression Estimates,'' \emph{The Annals of
  Statistics}, 49(4), 2231--2249, DOI:
  \url{https://doi.org/10.1214/20-AOS2034}.

\bibitem[Kolmogorov(1957)]{Kolmogorov1957}
Kolmogorov, A.~N. (1957), ``On the Representation of Continuous Functions of
  Many Variables by Superposition of Continuous Functions of One Variable and
  Addition,'' \emph{Doklady Akademii Nauk SSSR}, 114(5), 953--956, available at
  \url{https://www.mathnet.ru/eng/dan22050}.

\bibitem[Kolmogorov and Tikhomirov(1961)]{KolmogorovTikhomirov1961}
Kolmogorov, A.~N., and Tikhomirov, V.~M. (1961), ``{$\varepsilon$}-Entropy and
  {$\varepsilon$}-Capacity of Sets in Functional Spaces,'' \emph{American
  Mathematical Society Translations, Series 2}, 17, 277--364, DOI:
  \url{https://doi.org/10.1090/trans2/017/10}.

\bibitem[Leshno et~al.(1993)Leshno, Lin, Pinkus, and Schocken]{LeshnoEtAl1993}
Leshno, M., Lin, V.~Y., Pinkus, A., and Schocken, S. (1993), ``Multilayer
  Feedforward Networks With a Nonpolynomial Activation Function Can Approximate
  Any Function,'' \emph{Neural Networks}, 6(6), 861--867, DOI:
  \url{https://doi.org/10.1016/S0893-6080(05)80131-5}.

\bibitem[Li et~al.(2026)Li, Yang, and Zhang]{LiYangZhang2026Sobolev}
Li, B., Yang, H., and Zhang, S. (2026), ``{Sobolev} Approximation by Fixed-Size
  Neural Networks With Arbitrary Accuracy,'' arXiv:2606.16975, available at
  \url{https://arxiv.org/abs/2606.16975}.

\bibitem[Liang et~al.(2015)Liang, Rakhlin, and Sridharan]{LiangEtAl2015}
Liang, T., Rakhlin, A., and Sridharan, K. (2015), ``Learning With Square Loss:
  Localization Through Offset {Rademacher} Complexity,'' in \emph{Proceedings
  of the 28th Conference on Learning Theory}, \emph{Proceedings of Machine
  Learning Research} (Vol. 40), PMLR, pp. 1260--1285, available at
  \url{https://proceedings.mlr.press/v40/Liang15.html}.

\bibitem[Liu et~al.(2022)Liu, Boukai, and Shang]{LiuBoukaiShang2022}
Liu, R., Boukai, B., and Shang, Z. (2022), ``Optimal Nonparametric Inference
  via Deep Neural Network,'' \emph{Journal of Mathematical Analysis and
  Applications}, 505(2), 125561, DOI:
  \url{https://doi.org/10.1016/j.jmaa.2021.125561}.

\bibitem[Liu et~al.(2026)Liu, Wang, Wu, and Zhang]{LiuWangWuZhang2026}
Liu, Y., Wang, Z., Wu, L., and Zhang, S. (2026), ``Smoothness Adaptivity in
  Constant-Depth Neural Networks: Optimal Rates via Smooth Activations,''
  arXiv:2602.19691v2, available at \url{https://arxiv.org/abs/2602.19691v2}.

\bibitem[Lu et~al.(2021)Lu, Shen, Yang, and Zhang]{LuShenYangZhang2021}
Lu, J., Shen, Z., Yang, H., and Zhang, S. (2021), ``Deep Network Approximation
  for Smooth Functions,'' \emph{SIAM Journal on Mathematical Analysis}, 53(5),
  5465--5506, DOI: \url{https://doi.org/10.1137/20M134695X}.

\bibitem[Maiorov and Pinkus(1999)]{MaiorovPinkus1999}
Maiorov, V., and Pinkus, A. (1999), ``Lower Bounds for Approximation by {MLP}
  Neural Networks,'' \emph{Neurocomputing}, 25(1--3), 81--91, DOI:
  \url{https://doi.org/10.1016/S0925-2312(98)00111-8}.

\bibitem[Maiti et~al.(2024)Maiti, Michelle, and Yang]{MaitiMichelleYang2024}
Maiti, A., Michelle, M., and Yang, H. (2024), ``Optimal Neural Network
  Approximation for High-Dimensional Continuous Functions,''
  arXiv:2409.02363v4, revised 2025, available at
  \url{https://arxiv.org/abs/2409.02363v4}.

\bibitem[Nakahara et~al.(2026)Nakahara, Kobayashi, and
  Matsushima]{NakaharaKobayashiMatsushima2026Ternary}
Nakahara, Y., Kobayashi, M., and Matsushima, T. (2026), ``A Lower Bound for the
  Number of Linear Regions of Ternary {ReLU} Regression Neural Networks,''
  \emph{Transactions on Machine Learning Research}, available at
  \url{https://openreview.net/forum?id=Yg7tt1hWiF}.

\bibitem[Neukirch(1999)]{Neukirch1999}
Neukirch, J. (1999), \emph{Algebraic Number Theory}, \emph{Grundlehren der
  mathematischen Wissenschaften} (Vol. 322), Berlin: Springer, DOI:
  \url{https://doi.org/10.1007/978-3-662-03983-0}.

\bibitem[Ou and B{\"o}lcskei(in press)]{OuBolcskei2024}
Ou, W., and B{\"o}lcskei, H. (in press), ``Covering Numbers for Deep {ReLU}
  Networks With Applications to Function Approximation and Nonparametric
  Regression,'' \emph{Foundations of Computational Mathematics},
  arXiv:2410.06378, available at \url{https://arxiv.org/abs/2410.06378}.

\bibitem[Ou et~al.(2024)Ou, Schenkel, and B{\"o}lcskei]{OuSchenkelBolcskei2024}
Ou, W., Schenkel, P., and B{\"o}lcskei, H. (2024), ``Three Quantization Regimes
  for {ReLU} Networks,'' \emph{arXiv preprint arXiv:2405.01952}, available at
  \url{https://arxiv.org/abs/2405.01952}.

\bibitem[Petersen and Voigtlaender(2018)]{PetersenVoigtlaender2018}
Petersen, P., and Voigtlaender, F. (2018), ``Optimal Approximation of Piecewise
  Smooth Functions Using Deep {ReLU} Neural Networks,'' \emph{Neural Networks},
  108, 296--330, DOI: \url{https://doi.org/10.1016/j.neunet.2018.08.019}.

\bibitem[Schmidt-Hieber(2020)]{SchmidtHieber2020}
Schmidt-Hieber, J. (2020), ``Nonparametric Regression Using Deep Neural
  Networks With {ReLU} Activation Function,'' \emph{The Annals of Statistics},
  48(4), 1875--1897, DOI: \url{https://doi.org/10.1214/19-AOS1875}.

\bibitem[Schmidt-Hieber and Vu(2024)]{SchmidtHieberVu2024}
Schmidt-Hieber, J., and Vu, D. (2024), ``Correction to `{N}onparametric
  Regression Using Deep Neural Networks With {ReLU} Activation Function',''
  \emph{The Annals of Statistics}, 52(1), 413--414, DOI:
  \url{https://doi.org/10.1214/24-AOS2351}.

\bibitem[Shen et~al.(2020)Shen, Yang, and
  Zhang]{shijun:Characterized:by:Numer:Neurons}
Shen, Z., Yang, H., and Zhang, S. (2020), ``Deep Network Approximation
  Characterized by Number of Neurons,'' \emph{Communications in Computational
  Physics}, 28(5), 1768--1811, DOI:
  \url{https://doi.org/10.4208/cicp.OA-2020-0149}.

\bibitem[Shen et~al.(2021)Shen, Yang, and Zhang{{}}]{ShenYangZhang2021FLES}
Shen, Z., Yang, H., and Zhang{{}}, S. (2021), ``Neural Network Approximation:
  Three Hidden Layers Are Enough,'' \emph{Neural Networks}, 141, 160--173, DOI:
  \url{https://doi.org/10.1016/j.neunet.2021.04.011}.

\bibitem[Shen et~al.(2022{\natexlab{a}})Shen, Yang, and
  Zhang]{ZhangShenYang2022}
Shen, Z., Yang, H., and Zhang, S. (2022{\natexlab{a}}), ``Deep Network
  Approximation: Achieving Arbitrary Accuracy With Fixed Number of Neurons,''
  \emph{Journal of Machine Learning Research}, 23(276), 1--60, available at
  \url{https://jmlr.org/papers/v23/21-1404.html}.

\bibitem[Shen et~al.(2022{\natexlab{b}})Shen, Yang, and
  Zhang{}]{ShenYangZhang2022ReLU}
Shen, Z., Yang, H., and Zhang{}, S. (2022{\natexlab{b}}), ``Optimal
  Approximation Rate of {ReLU} Networks in Terms of Width and Depth,''
  \emph{Journal de Math{\'e}matiques Pures et Appliqu{\'e}es}, 157, 101--135,
  DOI: \url{https://doi.org/10.1016/j.matpur.2021.07.009}.

\bibitem[Siegel(2023)]{Siegel2023}
Siegel, J.~W. (2023), ``Optimal Approximation Rates for Deep {ReLU} Neural
  Networks on {Sobolev} and {Besov} Spaces,'' \emph{Journal of Machine Learning
  Research}, 24(357), 1--52, available at
  \url{https://jmlr.org/papers/v24/23-0025.html}.

\bibitem[Stone(1982)]{Stone1982}
Stone, C.~J. (1982), ``Optimal Global Rates of Convergence for Nonparametric
  Regression,'' \emph{The Annals of Statistics}, 10(4), 1040--1053, DOI:
  \url{https://doi.org/10.1214/aos/1176345969}.

\bibitem[Suzuki(2019)]{Suzuki2019}
Suzuki, T. (2019), ``Adaptivity of Deep {ReLU} Network for Learning in {Besov}
  and Mixed Smooth {Besov} Spaces: Optimal Rate and Curse of Dimensionality,''
  in \emph{International Conference on Learning Representations}, available at
  \url{https://openreview.net/forum?id=H1ebTsActm}.

\bibitem[Tsybakov(2009)]{Tsybakov2009}
Tsybakov, A.~B. (2009), \emph{Introduction to Nonparametric Estimation},
  \emph{Springer Series in Statistics}, New York: Springer, DOI:
  \url{https://doi.org/10.1007/b13794}.

\bibitem[Wang et~al.(2025)Wang, Zhang, Zeng, Xie, Guo, Zeng, and
  Fan]{WangEtAl2025PEUAF}
Wang, Q., Zhang, S., Zeng, D., Xie, Z., Guo, H., Zeng, T., and Fan, F.-L.
  (2025), ``Don't Fear Peculiar Activation Functions: {EUAF} and Beyond,''
  \emph{Neural Networks}, 186, 107258, DOI:
  \url{https://doi.org/10.1016/j.neunet.2025.107258}.

\bibitem[Yang et~al.(2023)Yang, Wu, Yang, and Xiang]{YangWuYangXiang2023}
Yang, Y., Wu, Y., Yang, H., and Xiang, Y. (2023), ``Nearly Optimal
  Approximation Rates for Deep Super {ReLU} Networks on {Sobolev} Spaces,''
  arXiv:2310.10766v5, revised 2025, available at
  \url{https://arxiv.org/abs/2310.10766v5}.

\bibitem[Yang(2025)]{YunfeiYang2025Sobolev}
Yang, Y. (2025), ``On the Optimal Approximation of {Sobolev} and {Besov}
  Functions Using Deep {ReLU} Neural Networks,'' \emph{Applied and
  Computational Harmonic Analysis}, 79, 101797, DOI:
  \url{https://doi.org/10.1016/j.acha.2025.101797}.

\bibitem[Yarotsky{{}}(2017)]{Yarotsky2017}
Yarotsky{{}}, D. (2017), ``Error Bounds for Approximations With Deep {ReLU}
  Networks,'' \emph{Neural Networks}, 94, 103--114, DOI:
  \url{https://doi.org/10.1016/j.neunet.2017.07.002}.

\bibitem[Yarotsky{}(2018)]{Yarotsky2018}
Yarotsky{}, D. (2018), ``Optimal Approximation of Continuous Functions by Very
  Deep {ReLU} Networks,'' in \emph{Proceedings of the 31st Conference on
  Learning Theory}, \emph{Proceedings of Machine Learning Research} (Vol. 75),
  PMLR, pp. 639--649, available at
  \url{https://proceedings.mlr.press/v75/yarotsky18a.html}.

\bibitem[Yarotsky(2021)]{Yarotsky2021}
Yarotsky, D. (2021), ``Elementary Superexpressive Activations,'' in
  \emph{Proceedings of the 38th International Conference on Machine Learning},
  \emph{Proceedings of Machine Learning Research} (Vol. 139), PMLR, pp.
  11932--11940, available at
  \url{https://proceedings.mlr.press/v139/yarotsky21a.html}.

\bibitem[Yarotsky and Zhevnerchuk(2020)]{YarotskyZhevnerchuk2020}
Yarotsky, D., and Zhevnerchuk, A. (2020), ``The Phase Diagram of Approximation
  Rates for Deep Neural Networks,'' in \emph{Advances in Neural Information
  Processing Systems} (Vol. 33), Curran Associates, Inc., pp. 13005--13015,
  available at
  \url{https://proceedings.neurips.cc/paper_files/paper/2020/hash/979a3f14bae523dc5101c52120c535e9-Abstract.html}.

\bibitem[Zhang et~al.(2024)Zhang, Lu, and Zhao]{ZhangLuZhao2024}
Zhang, S., Lu, J., and Zhao, H. (2024), ``Deep Network Approximation: Beyond
  {ReLU} to Diverse Activation Functions,'' \emph{Journal of Machine Learning
  Research}, 25(35), 1--39, available at
  \url{https://jmlr.org/papers/v25/23-0912.html}.

\end{thebibliography}
